\documentclass{article}[11pt,onecolumn,letter]
\usepackage{fixltx2e}
\usepackage{float}
\usepackage{amsmath,amssymb,amsfonts,amsthm}
\usepackage{fullpage}
\usepackage{mymath}
\usepackage{ifthen,color}
\usepackage{graphics,graphicx,subfigure,psfrag}
\usepackage{float}
\usepackage{wrapfig}
\usepackage{rotating}
\usepackage[margin=10pt,font=small,labelfont=bf,format=plain,width=.8\textwidth]{caption}  
\usepackage{algorithm,algorithmic}
\renewcommand{\algorithmiccomment}[1]{// #1}

\usepackage{authblk,url}

\newcommand{\BlackBox}{\rule{1.5ex}{1.5ex}}  \renewenvironment{proof}{\par\noindent{\bf Proof\ }}{\hfill\BlackBox\\[2mm]}

\newtheorem{theorem}{Theorem}

\graphicspath{{./}}

\newcommand{\citep}[1]{\cite{#1}}
\newcommand{\citet}[1]{\cite{#1}}

\usepackage{tikz}
\usetikzlibrary{arrows,shapes}
\usetikzlibrary{shapes,arrows}
\tikzstyle{invisible} = [circle,draw=none,fill=none,minimum size=0mm,inner sep=0mm]
\tikzstyle{vertex}    = [circle,draw=black,fill=white,font=\small,minimum size=9mm,inner sep=0mm]
\tikzstyle{vertexobs} = [circle,draw=black,fill=gray,font=\small,minimum size=9mm,inner sep=0mm]
\tikzstyle{edge}      = [   draw,thick,line width=1.25pt]
\tikzstyle{dedge}     = [->,draw,thick,line width=1.25pt]
\tikzstyle{aggregate} = [circle,draw=black,fill=black,minimum size=10mm,scale=0.25]

\newcommand{\sq}[2] {
  \definecolor{thiscol}{gray}{.#2}
  \ifthenelse{#2<50}
  {\color{white}}
  {\color{black}}
  \fcolorbox{black}{thiscol}{\makebox[1.5em]{#1}}
}

\begin{document}

\title{Statistical and Computational Tradeoffs in \\Stochastic Composite Likelihood}
\author{Joshua V Dillon\textsuperscript{*}}
\author{Guy Lebanon}
\affil{
    School of Computational Science \& Engineering \\
    College of Computing \\
    Georgia Institute of Technology \\ Atlanta, Georgia
}
\date{\today}
\maketitle
\thispagestyle{empty}
\let\oldthefootnote\thefootnote
\renewcommand{\thefootnote}{\fnsymbol{footnote}}
\footnotetext[1]{To whom correspondence should be addressed. Email: \url{jvdillon@gatech.edu}}
\let\thefootnote\oldthefootnote

\maketitle

\maketitle
\begin{abstract}
Maximum likelihood estimators are often of limited practical use due to the intensive computation they require. We propose a family of alternative estimators that maximize a stochastic variation of the composite likelihood function. Each of the estimators resolve the computation-accuracy tradeoff differently, and taken together they span a continuous spectrum of computation-accuracy tradeoff resolutions. We prove the consistency of the estimators, provide formulas for their asymptotic variance, statistical robustness, and computational complexity. We discuss experimental results in the context of Boltzmann machines and conditional random fields. The theoretical and experimental studies demonstrate the effectiveness of the estimators when the computational resources are insufficient. They also demonstrate that in some cases reduced computational complexity is associated with robustness thereby increasing statistical accuracy.
\end{abstract}

\section{Introduction} \label{sec:intro}
Maximum likelihood estimation is by far the most popular point estimation
technique in machine learning and statistics. Assuming that the data consists
of $n,$ $m$-dimensional vectors
\begin{align} \label{eq:data}
   D=(X^{(1)},\ldots,X^{(n)}), \quad X^{(i)}\in\R^m,   
\end{align}
and is sampled iid from a parametric distribution $p_{\theta_0}$ with $\theta_0 \in \Theta\subset \R^r$, a maximum likelihood estimator  (mle) $\hat\theta_n^{\text{ml}}$ is a maximizer of the loglikelihood function
\begin{align} \label{eq:l}
   \ell_n(\theta\,; D) &= \sum_{i=1}^n \log p_{\theta}(X^{(i)})
\\ \hat\theta_n^{\text{ml}} &= \argmax_{\theta\in\Theta} \ell_n(\theta\,;D).
\end{align}

The use of the mle is motivated by its consistency\footnote{The consistency $\hat\theta_n^{\text{ml}}\to \theta_0$ with probability 1 is sometimes called strong consistency in order to differentiate it from the weaker notion of consistency in probability $P(|\hat\theta_n^{\text{ml}}- \theta_0|<\epsilon)\to 0$.}, i.e.
$\hat\theta_n^{\text{ml}}\to \theta_0$ as $n\to\infty$ with probability 1 \citep{Ferguson1996}.  The consistency property ensures that as the number $n$ of samples grows, the estimator will converge to the true parameter $\theta_0$ governing the data generation process.

An even stronger motivation for the use of the mle is that it has an
asymptotically normal distribution with mean vector $\theta_0$ and variance matrix $(nI(\theta_0))^{-1}$. More formally, we have the following convergence in distribution as $n\to\infty$ \citep{Ferguson1996}
\begin{align} \label{eq:mleEff}
   \sqrt{n}\,(\hat\theta_n^{\text{ml}}-\theta_0)\tood N(0,I^{-1}(\theta_0)),
\end{align}
where $I(\theta)$ is the $r\times r$ Fisher information matrix
\begin{align}
   I(\theta)&=\E_{p_{\theta}} \{\nabla \log p_{\theta}(X) (\nabla \log p_{\theta}(X))^{\top}\}
\end{align}
with $\nabla f$ representing the $r\times 1$ gradient vector of $f(\theta)$ with respect to $\theta$.  The convergence \eqref{eq:mleEff} is especially striking since according to the Cramer-Rao lower bound, the asymptotic variance $(n I(\theta_0))^{-1}$ of the mle is the smallest possible variance for any estimator. Since it achieves the lowest possible asymptotic variance, the mle (and other estimators which share this property) is said to be asymptotically efficient.

The consistency and asymptotic efficiency of the mle motivate its use in many circumstances. Unfortunately, in some situations the maximization or even evaluation of the loglikelihood \eqref{eq:l} and its derivatives is impossible due to computational considerations. For instance this is the situation in many high dimensional exponential family distributions, including Markov random fields whose graphical structure contains cycles. This has lead to the proposal of alternative estimators under the premise that a loss of asymptotic efficiency is acceptable--in return for reduced computational complexity.

In contrast to asymptotic efficiency, we view consistency as a less negotiable property and prefer to avoid inconsistent estimators if at all possible. This common viewpoint in statistics is somewhat at odds with recent advances in the machine learning literature promoting non-consistent estimators, for example using variational techniques \citep{Jordan1999}. Nevertheless, we feel that there is a consensus regarding the benefits of having consistent estimators over non-consistent ones. 

In this paper, we propose a family of estimators, for use in situations where the computation of the mle is intractable. In contrast to many previously proposed approximate estimators, our estimators are statistically consistent and admit a precise quantification of both computational complexity and statistical accuracy through their asymptotic variance. Due to the continuous parameterization of the estimator family, we obtain an effective framework for achieving a predefined problem-specific balance between computational tractability and statistical accuracy. We also demonstrate that in some cases reduced computational complexity may in  fact act as a regularizer, increasing robustness and therefore accomplishing both reduced computation and increased accuracy.  This ``win-win'' situation conflicts with the conventional wisdom stating that moving from the mle to pseudo-likelihood and other related estimators result in a computational win but a statistical loss. Nevertheless we show that this occurs in some practical situations.
 
For the sake of concreteness, we focus on the case of estimating the parameters associated with Markov random fields. In this case, we provide a detailed discussion of the accuracy--complexity tradeoff. We include experiments on both simulated and real world data for several models including the Boltzmann machine, conditional random fields, and the Boltzmann linear chain model.

\section{Related Work} \label{sec:related}
There is a large body of work dedicated to tractable learning techniques. Two popular categories are Markov chain Monte Carlo (MCMC) and variational methods. MCMC is a general purpose technique for approximating expectations and can be used to approximate the normalization term and other intractable portions of the loglikelihood and its gradient \citep{Casella2004}.  Variational methods are
techniques for conducting inference and learning based on tractable bounds \citep{Jordan1999}. 

Despite the substantial work on MCMC and variational methods, there are little practical results concerning the convergence and approximation rate of the resulting parameter estimators. Variational techniques are sometimes inconsistent and it is hard to analyze their asymptotic statistical behavior. In the case of MCMC, a number of asymptotic results exist \citep{Casella2004}, but since MCMC plays a role inside each gradient descent or EM iteration it is hard to analyze the asymptotic behavior of the resulting parameter estimates.  An advantage of our framework is that we are able to directly characterize the asymptotic behavior of the estimator and relate it to the amount of computational savings.

Our work draws on the composite likelihood method for parameter estimation proposed by \citet{Lindsay1988} which in turn generalized the pseudo likelihood of \citet{Besag1974}. A selection of more recent studies on pseudo and composite likelihood are \citep{Arnold1991, Liang2003, Varin2005b, Sutton2007, Hjort2008}.  Most of the
recent studies in this area examine the behavior of the pseudo or composite likelihood in a particular modeling situation. We believe that the present paper is the first to systematically examine statistical and computational tradeoffs in a general quantitative framework.  Possible exceptions are \citep{Zhu2002} which is an experimental study on texture generation, \citep{Xing2003a} which is focused on inference rather than parameter estimation, and \citep{Liang2008} which compares discriminative and generative risks.

\section{Stochastic Composite Likelihood} \label{sec:scl}
In many cases, the absence of a closed form expression for the normalization term prevents the computation of the loglikelihood \eqref{eq:l} and its derivatives thereby severely limiting the use of the mle. A popular example is Markov random fields, wherein the computation of the normalization term is often intractable (see Section~\ref{sec:comp} for more details). In this paper we propose alternative estimators based on the maximization of a stochastic
variation of the composite likelihood. 

We denote multiple samples using superscripts and individual dimensions using subscripts. Thus $X^{(r)}_j$ refers to the $j$-dimension of the $r$ sample. Following standard convention we refer to random variables (RV) using uppercase letters and their corresponding values using lowercase letters. We also use the standard notations for extracting a subset of the dimensions of a random variable
\begin{align} \label{eq:X_S}
X_S \defeq \{X_i:i\in S\}, \qquad X_{-j} \defeq \{X_i:i\neq j\}.
\end{align}

We start by reviewing the pseudo loglikelihood function \citep{Besag1974} associated with the data $D$ \eqref{eq:data},
\begin{align}
   p\ell_n(\theta\,; D) &\defeq \sum_{i=1}^n \sum_{j=1}^m\log p_{\theta}(X^{(i)}_j|X^{(i)}_{-j})\label{eq:pl}.
\end{align}
The maximum pseudo likelihood estimator (mple) $\hat\theta_n^\text{mpl}$ is
consistent i.e., $\hat\theta_n^\text{mpl}\to\theta_0$ with probability 1, but possesses considerably higher asymptotic variance than the mle's $(nI(\theta_0))^{-1}$. Its main advantage is that it does not require the computation of the normalization term as it cancels out in the probability ratio defining conditional distributions
\begin{eqnarray}
p_{\theta}(X_j|X_{-j})=p_{\theta}(X_j|\{X_k:k\neq j\})=\frac{p_{\theta}(X)}{\sum_{x_j} p_{\theta}(X_1,\ldots,X_{j-1},X_j=x_j,X_{j+1},\ldots,X_m)}.
\end{eqnarray}

The mle and mple represent two different ways of resolving the tradeoff between asymptotic variance and computational complexity. The mle has low asymptotic variance but high computational complexity while the mple has higher asymptotic variance but low computational complexity. It is desirable to obtain additional estimators realizing alternative resolutions of the accuracy complexity tradeoff.  To this end we define the stochastic composite likelihood whose maximization provides a family of consistent estimators with statistical accuracy and computational complexity spanning the entire accuracy-complexity spectrum.

Stochastic composite likelihood generalizes the likelihood and pseudo
likelihood functions by constructing an objective function that is a stochastic sum of likelihood objects. We start by defining the notion of $m$-pairs and likelihood objects and then proceed to stochastic composite likelihood.
\begin{defn}
An $m$-pair $(A,B)$ is a pair of sets $A,B\subset\{1,\ldots,m\}$ satisfying $A\neq\emptyset=A\cap B$. The likelihood object associated with an $m$-pair $(A,B)$ and $X$ is $S_{\theta}(A,B)\defeq\log p_{\theta}(X_A|X_B)$ where $X_S$ is defined in \eqref{eq:X_S}. The composite loglikelihood function \citep{Lindsay1988}
is a collection of likelihood objects defined by a finite sequence of $m$-pairs $(A_1,B_1),\ldots,(A_k,B_k)$
\begin{align}
  c\ell_n(\theta\,;D) 
  &\defeq \sum_{i=1}^n \sum_{j=1}^k \log     p_{\theta}(X^{(i)}_{A_j}|X^{(i)}_{B_j}).\label{eq:cl}
\end{align}
\end{defn}

There is a certain lack of flexibility associated with the composite likelihood framework as each likelihood object is either selected or not for the entire sample $X^{(1)},\ldots,X^{(n)}$. There is no allowance for some objects to be selected more frequently than others. For example, available computational resources may allow the computation of the loglikelihood for 20\% of the samples, and the pseudo-likelihood for the remaining 80\%. In the case of composite likelihood if we select the full-likelihood component (or the pseudo-likelihood or any other likelihood object) then this component is applied to all samples indiscriminately. 

In SCL, different likelihood objects  $S_{\theta}(A_j,B_j)$ may be selected for different samples with the possibility of some likelihood objects being selected for only a small fraction of the data samples. The selection may be non-coordinated, in which case each component is selected or not independently of the other components. Or it may be coordinated in which case the selection of one component depends on the selection of the other ones. For example, we may wish to avoid selecting a pseudo likelihood component for a certain sample $X^{(i)}$ if the full likelihood component was already selected for it. 

Another important advantage of stochastic selection is that the discrete parameterization of \eqref{eq:cl} defined by the sequence $(A_1,B_1),\ldots,(A_k,B_k)$ is less convenient for theoretical analysis. Each component is either selected or not, turning the problem of optimally selecting components into a hard combinatorial problem. The stochastic composite likelihood, which is defined below, enjoys continuous parameterization leading to more convenient optimization techniques and convergence analysis. 

\begin{defn} \label{def:scl}
Consider a finite sequence of $m$-pairs $(A_1,B_1),\ldots,(A_k,B_k)$, a dataset $D=(X^{(1)},\ldots,X^{(n)})$, $\beta\in\R^k_+$, and $m$ iid   binary random vectors $Z^{(1)},\ldots,Z^{(m)}\iid P(Z)$ with $\lambda_j\defeq \E(Z_j)>0$.
The stochastic composite loglikelihood (scl) is
\begin{align} \label{eq:scl}
sc\ell_n(\theta\,;D) &\defeq 
	\frac{1}{n} \sum_{i=1}^n m_{\theta}(X^{(i)},Z^{(i)}), \quad \text{where}\quad \\ 
m_{\theta}(X,Z) &\defeq\sum_{j=1}^k \beta_j Z_j  \log p_{\theta}(X_{A_j}|X_{B_j}). \label{eq:mDef}
\end{align}
\end{defn}
In other words, the scl is a stochastic extension of \eqref{eq:cl} where for each sample $X^{(i)}, i=1,\ldots,n$, the likelihood objects $S(A_1,B_1),\ldots,S(A_k,B_k)$ are either selected or not, depending on the values of the binary random variables $Z^{(i)}_1,\ldots,Z^{(i)}_m$ and weighted by the constants $\beta_1,\ldots,\beta_m$. Note that $Z^{(i)}_j$ may in general depend on $Z^{(i)}_r$ but not on $Z^{(l)}_r$ or on $X^{(i)}$.

When we focus on examining different models for $P(Z)$ we sometimes parameterize it, for example by $\lambda$ i.e., $P_{\lambda}(Z)$. This reuse of $\lambda$ (it is also used in Definition~\ref{def:scl}) is a notational abuse. We accept it, however, as in most of the cases that we consider $\lambda_1,\ldots,\lambda_k$ from Definition~\ref{def:scl}  either form the parameter vector for $P(Z)$ or are part of it. 

Some illustrative examples follow. 
\begin{description}
\item[Independence.] Factorizing $P_{\lambda}(Z_1,\ldots,Z_k)=\prod_j P_{\lambda_j}(Z_j)$ leads to $Z^{(i)}_j\sim \text{Ber}(\lambda_j)$ with complete independence among the indicator variables. For each sample $X^{(i)}$, each likelihood object $S(A_j,B_j)$ is selected or not independently with probability $\lambda_j$. 
\item[Multinomial.] A multinomial model $Z\sim \text{Mult}(1,\lambda)$ implies that for each sample $Z^{(i)}$ a multivariate Bernoulli experiment is conducted with precisely one likelihood object being selected depending on the selection probabilities $\lambda_1,\ldots,\lambda_k$.
\item[Product of Multinomials.] A product of multinomials is formed by a partition of the dimensions to $l$ disjoint subsets $\{1,\ldots,m\}=C_1\cup \cdots C_l$ where $Z_{C_i}\sim \text{Mult}(1,(\lambda_j:j\in C_i))$ i.e., 
\[ P(Z)=\prod_{i=1}^c P_i\left(\{Z_j:j\in C_i\}\right), \quad \text{ where }P_i \text{ is }\text{Mult}(1,(\lambda_j:j\in C_l)).\]
\item[Loglinear Models.] The distribution $P(Z)$ follows a hierarchical loglinear model \citep{Bishop1975}. This case subsumes the other cases above.
\end{description}

In analogy to the mle and the mple, the maximum scl estimator (mscle)
$\hat\theta_n^{\text{msl}}$ estimates $\theta_0$ by maximizing the scl
function. In contrast to the loglikelihood and pseudo loglikelihood functions, the scl function and its maximizer are random variables that depend on the indicator variables $Z^{(1)},\ldots,Z^{(n)}$ in addition to the data $D$. As such, its behavior should be summarized by examining the limit $n\to\infty$. Doing so eliminates the dependency on particular realizations of $Z^{(1)},\ldots,Z^{(n)}$ in favor of the the expected frequencies $\lambda_j=\E_{P(Z)}Z_j$ which are non-random constants.

The statistical accuracy and computational complexity of the msl estimator are continuous functions of the parameters $(\beta,\lambda)$ (components weights and selection probabilities respectively) which vary continuously throughout their domain $(\lambda,\beta)\in \Lambda\times \R_+^k$. Choosing appropriate values of $(\lambda,\beta)$ retrieves the special cases of mle, mple, maximum composite likelihood with each selection being associated with a distinct statistical accuracy and computational complexity. The scl framework allows selections of many more values of $(\lambda,\beta)$ realizing a wide continuous spectrum of estimators, each resolving the accuracy-complexity tradeoff differently. 

We include below a demonstration of the scl framework in a simple low dimensional case. In the following sections we discuss in detail the statistical behavior of the mscle and its computational complexity. We conclude the paper with several experimental studies. 

\subsection{Boltzmann Machine Example}
Before proceeding we illustrate the SCL framework using a simple example involving a Boltzmann machine \citep{Jordan1999}. We   consider in detail three SCL policies: full likelihood (FL), pseudo-likelihood (PL), and a stochastic combination of first and second order pseudo-likelihood  with the first order components ($p(X_i|X_{-i})$) selected with probability $\lambda$ and the second order components ($p(X_i,X_j|X_{\{i,j\}^c})$) with probability $1-\lambda$.

Denoting the number of (binary) graph nodes by $m$, the number of examples by $n$, the computational complexity of the FL function (FLOP\footnote{FLOP stands for the number of floating point operations.} counts) is $O\left(\begin{pmatrix}m\\2\end{pmatrix}(2^m+n)\right)$ (loglikelihood) and $O\left(\begin{pmatrix}m\\2\end{pmatrix}^22^m+n\begin{pmatrix}m\\2\end{pmatrix}\right)$ (loglikelihood gradient). The exponential growth in $m$ prevents such computations for large graphs. 

The $k$-order PL function offers a practical alternative to FL (1-order PL correspond to the traditional pseudo-likelihood and 2-order is its analog with second order components $p(X_{\{i,j\}}|X_{\{i,j\}^c})$). The complexity of computing the corresponding SCL function is 
$O\left(\begin{pmatrix}m\\2\end{pmatrix} \left(\begin{pmatrix}m\\k\end{pmatrix}2^k+n\right)\right)$ (for the objective function) and $O\left(\begin{pmatrix}m\\k\end{pmatrix}\begin{pmatrix}m\\2\end{pmatrix}^22^k+n\begin{pmatrix}m\\2\end{pmatrix}\right)$ (for the gradient). The slower complexity growth of the $k$-order PL (polynomial in $m$ instead of exponential) is offset by its reduced statistical accuracy, which we measure using the normalized asymptotic variance 
\begin{align} \label{eq:relEff}
\text{eff}(\hat\theta_n) = \frac{\det (\text{Asymp Var}(\hat\theta_n))}{\det (\text{Asymp Var}(\hat\theta_n^{\text{mle}}))}
\end{align}
which is bounded from below by 1 (due to Cramer Rao lower bound) and its deviation from 1 reflects its inefficiency relative to the MLE.

The MLE thus achieves the best accuracy but it is computationally intractable. The first order and second order PL have higher asymptotic variance but are easier to compute. The SCL framework enables adding many more estimators filling in the gaps between ML, 1-order PL, 2-order PL, etc. 

We illustrate three SCL functions in the context of a simple Boltzmann machine (five binary nodes, fourteen samples $X^{(1)},\ldots,X^{(14)}$, $\theta^{\text{true}}=(-1,-1,-1,-1,-1,1,1,1,1,1)$) in Figure~\ref{fig:policies}. The top box refers to the full likelihood policy. For each of the fourteen samples, the FL component is computed and their aggregation forms the SCL function which in this case equals the loglikelihood. The selection of the FL component for each sample is illustrated using a diamond box. The numbers under the boxes reflect the FLOP counts needed to compute the components and the total complexity associated with computing the entire SCL or loglikelihood is listed on the right. As mentioned above, the normalized asymptotic variance \eqref{eq:relEff} is 1. 

The pseudo-likelihood function \eqref{eq:pl} is illustrated in the second box where each row correspond to one of the five PL components. As each of the five PL component is selected for each of the samples we have diamond boxes covering the entire $5\times 14$ array. The shade of the diamond boxes reflects the complexity required to compute them enabling an easy comparison to the FL components in the top of the figure (note how the FL boxes are much darker than the PL boxes). The numbers at the bottom of each column reflect the FLOP marginal count for each of the fourteen samples and the numbers to the right of the rows reflect the FLOP marginal count for each of the PL components. In this case the FLOP count is less than half the FLOP count of the FL in top box (this reduction in complexity obtained by replacing FL with PL will increase dramatically for graphs with more than 5 nodes) but the asymptotic variance is 83\% higher\footnote{The asymptotic variance of SCL functions is computed using formulas derived in the next section}. 

The third SCL policy reflects a stochastic combination of first and second order pseudo likelihood components. Each first order component is selected with probability $\lambda$ and each second order component is selected with probability $1-\lambda$. The result is a collection of 5 1-order PL components and 10 2-order components with only some of them selected for each of the fourteen samples. Again diamond boxes correspond to selected components which are shaded according to their FLOP complexity. The per-component FLOP marginals and per example FLOP marginals are listed as the bottom row and right-most column. The total complexity is somewhere between FL and PL and the asymptotic variance is reduced from the PL's 183\% to 148\%.

\begin{figure}
\centering
\setlength{\tabcolsep}{0pt}
{ \scriptsize \begin{tabular}{|l|*{14}{c}cc|}\hline
 & $X^{(1)}$& $X^{(2)}$& $X^{(3)}$& $X^{(4)}$& $X^{(5)}$& $X^{(6)}$& $X^{(7)}$& $X^{(8)}$& $X^{(9)}$& $X^{(10)}$& $X^{(11)}$& $X^{(12)}$& $X^{(13)}$& $X^{(14)}$&& \\\hline \hline
 FL &&&&&&&&&&&&&&&&\\
$X_1,,\ldots,X_5$ & 
\sq{$\diamond$}{1} & \sq{$\diamond$}{1} &
\sq{$\diamond$}{1} & \sq{$\diamond$}{1} &
\sq{$\diamond$}{1} & \sq{$\diamond$}{1} &
\sq{$\diamond$}{1} & \sq{$\diamond$}{1} &
\sq{$\diamond$}{1} & \sq{$\diamond$}{1} &
\sq{$\diamond$}{1} & \sq{$\diamond$}{1} &
\sq{$\diamond$}{1} & \sq{$\diamond$}{1} & & 4620\\ 
& 
330 & 330 & 
330 & 330 & 
330 & 330 & 
330 & 330 & 
330 & 330 & 
330 & 330 & 
330 & 330 & & 4620\\
&&&&&&&&&&&&&&&&\\
Complexity  & 4620&&&&&&&&&&&&&&&\\ 
Norm Asym Var  & 1&&&&&&&&&&&&&&&\\ 
\hline \hline
PL &&&&&&&&&&&&&&&&\\
$X_1|X_{-1}$ & 
\sq{$\diamond$}{93} & \sq{$\diamond$}{93} &
\sq{$\diamond$}{93} & \sq{$\diamond$}{93} &
\sq{$\diamond$}{93} & \sq{$\diamond$}{93} &
\sq{$\diamond$}{93} & \sq{$\diamond$}{93} &
\sq{$\diamond$}{93} & \sq{$\diamond$}{93} &
\sq{$\diamond$}{93} & \sq{$\diamond$}{93} &
\sq{$\diamond$}{93} & \sq{$\diamond$}{93} & & 308\\
$X_2|X_{-2}$ & 
\sq{$\diamond$}{93} & \sq{$\diamond$}{93} &
\sq{$\diamond$}{93} & \sq{$\diamond$}{93} &
\sq{$\diamond$}{93} & \sq{$\diamond$}{93} &
\sq{$\diamond$}{93} & \sq{$\diamond$}{93} &
\sq{$\diamond$}{93} & \sq{$\diamond$}{93} &
\sq{$\diamond$}{93} & \sq{$\diamond$}{93} &
\sq{$\diamond$}{93} & \sq{$\diamond$}{93} & & 308\\
$X_3|X_{-3}$ & 
\sq{$\diamond$}{93} & \sq{$\diamond$}{93} &
\sq{$\diamond$}{93} & \sq{$\diamond$}{93} &
\sq{$\diamond$}{93} & \sq{$\diamond$}{93} &
\sq{$\diamond$}{93} & \sq{$\diamond$}{93} &
\sq{$\diamond$}{93} & \sq{$\diamond$}{93} &
\sq{$\diamond$}{93} & \sq{$\diamond$}{93} &
\sq{$\diamond$}{93} & \sq{$\diamond$}{93} & & 308\\
$X_4|X_{-4}$ & 
\sq{$\diamond$}{93} & \sq{$\diamond$}{93} &
\sq{$\diamond$}{93} & \sq{$\diamond$}{93} &
\sq{$\diamond$}{93} & \sq{$\diamond$}{93} &
\sq{$\diamond$}{93} & \sq{$\diamond$}{93} &
\sq{$\diamond$}{93} & \sq{$\diamond$}{93} &
\sq{$\diamond$}{93} & \sq{$\diamond$}{93} &
\sq{$\diamond$}{93} & \sq{$\diamond$}{93} & & 308\\
$X_5|X_{-5}$ & 
\sq{$\diamond$}{93} & \sq{$\diamond$}{93} &
\sq{$\diamond$}{93} & \sq{$\diamond$}{93} &
\sq{$\diamond$}{93} & \sq{$\diamond$}{93} &
\sq{$\diamond$}{93} & \sq{$\diamond$}{93} &
\sq{$\diamond$}{93} & \sq{$\diamond$}{93} &
\sq{$\diamond$}{93} & \sq{$\diamond$}{93} &
\sq{$\diamond$}{93} & \sq{$\diamond$}{93} & & 308\\ 
& 
110 & 110 & 
110 & 110 & 
110 & 110 & 
110 & 110 & 
110 & 110 & 
110 & 110 & 
110 & 110 & &1540\\ 
&&&&&&&&&&&&&&&&\\
Complexity  & 1540&&&&&&&&&&&&&&&\\ 
Norm Asym Var  & 1.83&&&&&&&&&&&&&&&\\ 
\hline \hline
0.7PL+0.3PL2 &&&&&&&&&&&&&&&&\\
$X_1|X_{-1}$ &  
  & \sq{$\diamond$}{93} &
\sq{$\diamond$}{93} &\sq{$\diamond$}{93} &
\sq{$\diamond$}{93} & & 
& & 
\sq{$\diamond$}{93} & \sq{$\diamond$}{93} &
\sq{$\diamond$}{93} & \sq{$\diamond$}{93} &
& & & 176 \\
$X_2|X_{-2}$ & 
& \sq{$\diamond$}{93} &
& \sq{$\diamond$}{93} &
\sq{$\diamond$}{93} & &
\sq{$\diamond$}{93} & \sq{$\diamond$}{93} &
\sq{$\diamond$}{93} & \sq{$\diamond$}{93} & 
\sq{$\diamond$}{93} & & 
\sq{$\diamond$}{93} & \sq{$\diamond$}{93} & &220\\
$X_3|X_{-3}$ & 
\sq{$\diamond$}{93} & \sq{$\diamond$}{93} &
& & 
& \sq{$\diamond$}{93} &
\sq{$\diamond$}{93} & \sq{$\diamond$}{93} &
& \sq{$\diamond$}{93} &
\sq{$\diamond$}{93} & \sq{$\diamond$}{93} & 
\sq{$\diamond$}{93} & \sq{$\diamond$}{93} & &220\\
$X_{4}|X_{-4}$ &
& &
\sq{$\diamond$}{93} & &
& \sq{$\diamond$}{93} &
\sq{$\diamond$}{93} & & 
& & 
\sq{$\diamond$}{93} & \sq{$\diamond$}{93} & 
\sq{$\diamond$}{93} & \sq{$\diamond$}{93} & &154\\
$X_5|X_{-5}$ & 
\sq{$\diamond$}{93} & &
&&
\sq{$\diamond$}{93} & \sq{$\diamond$}{93} &
\sq{$\diamond$}{93} & \sq{$\diamond$}{93} &
\sq{$\diamond$}{93} & &
\sq{$\diamond$}{93} & \sq{$\diamond$}{93} &
& \sq{$\diamond$}{93} & &198\\
$X_{\{1,2\}}|X_{\{1,2\}^c}$ &
& &
\sq{$\diamond$}{75} & \sq{$\diamond$}{75} &
& \sq{$\diamond$}{75} & 
&&
&&
\sq{$\diamond$}{75} &  &
&&& 164 \\
$X_{\{1,3\}}|X_{\{1,3\}^c}$ &
\sq{$\diamond$}{75} & &
\sq{$\diamond$}{75} & &
\sq{$\diamond$}{75} & &
&&
\sq{$\diamond$}{75} & &
& \sq{$\diamond$}{75} &
&&&205\\
$X_{\{1,4\}}|X_{\{1,4\}^c}$ &
&&&&
\sq{$\diamond$}{75} & \sq{$\diamond$}{75} & \sq{$\diamond$}{75} & 
&& \sq{$\diamond$}{75} &
&&&&&164\\
$X_{\{1,5\}}|X_{\{1,5\}^c}$ &
\sq{$\diamond$}{75} &  & 
&&&&
\sq{$\diamond$}{75} &  & & & 
\sq{$\diamond$}{75} &  \sq{$\diamond$}{75} &  && & 164\\
$X_{\{2,3\}}|X_{\{2,3\}^c}$ &
&&&\sq{$\diamond$}{75} &  \sq{$\diamond$}{75} &  & \sq{$\diamond$}{75} & 
&& \sq{$\diamond$}{75} & \sq{$\diamond$}{75} &  &&& & 205\\
$X_{\{2,4\}}|X_{\{2,4\}^c}$ &
& \sq{$\diamond$}{75} &  \sq{$\diamond$}{75} &  &&& \sq{$\diamond$}{75} &  \sq{$\diamond$}{75} &  & \sq{$\diamond$}{75} &  \sq{$\diamond$}{75} &  \sq{$\diamond$}{75} &  & && 287\\
$X_{\{2,5\}}|X_{\{2,5\}^c}$ &
\sq{$\diamond$}{75} &  & & \sq{$\diamond$}{75} &  & \sq{$\diamond$}{75} &  & \sq{$\diamond$}{75} &  &&&&&&&164\\
$X_{\{3,4\}}|X_{\{3,4\}^c}$ &
\sq{$\diamond$}{75} &  &&&&&&\sq{$\diamond$}{75} & &&&&&&&82\\
$X_{\{3,5\}}|X_{\{3,5\}^c}$ &
&&\sq{$\diamond$}{75} &&\sq{$\diamond$}{75} &\sq{$\diamond$}{75} &&\sq{$\diamond$}{75} &&&&&&& &164\\
$X_{\{4,5\}}|X_{\{4,5\}^c}$ &
&&&&&& \sq{$\diamond$}{75} & \sq{$\diamond$}{75} & \sq{$\diamond$}{75} & \sq{$\diamond$}{75} & & \sq{$\diamond$}{75} & &&& 205\\
& 208 & 107 & 208 & 167 & 230 & 230 & 293 & 271 & 148 & 230 & 274 & 252 & 66 & 88 & &2772 \\
&&&&&&&&&&&&&&&&\\
Complexity  & 2772&&&&&&&&&&&&&&&\\ 
Norm Asym Var  & 1.48&&&&&&&&&&&&&&&\\ \hline
\end{tabular}}
\caption{Sample runs of three different SCL policies for 14 examples  $X^{(1)},\ldots,X^{(14)}$ drawn from a 5 binary node Boltzmann machine ($\theta^{\text{true}}=(-1,-1,-1,-1,-1,1,1,1,1,1)$). The policies are full likelihood (FL, top), pseudo-likelihood (PL, middle), and a stochastic combination of first and second order pseudo-likelihood  with the first order components selected with probability 0.7 and the second order components with probability 0.3 (bottom). 
 \vspace{0.1in} \newline
The sample runs for the policies are illustrated by placing a diamond box in table entries corresponding to selected likelihood objects (rows corresponding to likelihood objects and columns to $X^{(1)},\ldots,X^{(14)}$). The FLOP counts of each likelihood object determines the shade of the diamond boxes while the total FLOP counts per example and per likelihood objects are displayed as table marginals (bottom row and right column for each policy). We also display the total FLOP count and the normalized asymptotic variance \eqref{eq:relEff}.
\vspace{0.1in} \newline 
Even in the simple case of 5 nodes, FL is the most complex policy with PL requiring a third of the FL computation. 0.7PL+0.3PL2 is somewhere in between. The situation is reversed for the estimation accuracy-FL achieves the lowest possible normalized asymptotic variance of 1, PL is almost twice that, and 0.7PL+0.3PL2 somewhere in the middle. The SCL framework spans the accuracy-complexity spectrum. Choosing the  right $\lambda$ value obtains an estimator that is suits  available computational resources and required accuracy.}
\label{fig:policies}
\end{figure}
  
Additional insight may be gained at this point by considering Figure~\ref{fig:compAccPlot} which plots several SCL estimators as points in the plane whose $x$ and $y$ coordinates correspond to normalized asymptotic variance and computational complexity respectively. We turn at this point to considering the statistical properties of the SCL estimators. 

\section{Consistency and Asymptotic Variance of $\hat\theta_n^{\text{msl}}$} \label{sec:stat}
A nice property of the SCL framework is enabling mathematical characterization of the statistical properties of the estimator $\hat\theta_n^{\text{msl}}$. In this section we examine the conditions for consistency of the mscle and its asymptotic distribution and in the next section we consider robustness. The propositions below constitute novel generalizations of some well-known results in classical statistics. Proofs may be found in Appendix~\ref{sec:proofs}. For simplicity, we assume that $X$ is discrete and $p_{\theta}(x)>0$.
  
\begin{defn} \label{def:identifiability}
A sequence of $m$-pairs $(A_1,B_1),\ldots,(A_k,B_k)$ ensures identifiability of $p_{\theta}$ if the map $\{p_{\theta}(X_{A_j}|X_{B_j}): j=1,\ldots,k\}\mapsto p_{\theta}(X)$ is injective. In other words, there exists only a single collection of conditionals $\{p_{\theta}(X_{A_j}|X_{B_j}): j=1,\ldots,k\}$ that
does not contradict the joint $p_{\theta}(X)$.
\end{defn}

\begin{prop} \label{prop:consistency}
Let $\Theta\subset \R^r$ be an open set, $p_{\theta}(x)>0$ and  continuous and smooth in $\theta$, and $(A_1,B_1),\ldots,(A_k,B_k)$ be a sequence of $m$-pairs for which $\{(A_j,B_j):\forall j \text{ such that } \lambda_j>0\}$ ensures identifiability. Then the sequence of SCL maximizers is strongly consistent i.e.,
\begin{align}
P\left(\lim_{n\to\infty} \hat\theta_n=\theta_0\right)=1.
\end{align}
\end{prop}

The above proposition indicates that to guarantee consistency, the sequence of $m$-pairs needs to satisfy Definition~\ref{def:identifiability}. It can be shown that a selection equivalent to the pseudo likelihood function, i.e.,
\begin{align} 
\mathcal{S}=\{(A_1,B_1),\ldots,(A_m,B_m)\} \quad \text{where} \quad
A_i=\{i\}, B_i=\{1,\ldots,m\}\setminus A_i \label{eq:plPieces}
\end{align} 
ensure identifiability and consequently the consistency of the mscle estimator. Furthermore, every selection of $m$-pairs that subsumes $\mathcal{S}$ in  \eqref{eq:plPieces} similarly guarantees identifiability and consistency.

The proposition below establishes the asymptotic normality of the mscle $\hat\theta_n$. The asymptotic variance enables the comparison of scl functions with different parameterizations $(\lambda,\beta)$. 

\begin{prop} \label{prop:asympVar} 
Making the assumptions of Proposition \ref{prop:consistency} as well as convexity of $\Theta\subset\R^r$ we have the following convergence in distribution
\begin{align}
   \sqrt{n}(\hat\theta_n^{\text{msl}}-\theta_0) \tood N\left(0,\Upsilon \Sigma \Upsilon\right)
\end{align}
where
\begin{align}
\Upsilon^{-1}&=\sum_{j=1}^k \beta_j\lambda_j \Var_{\theta_0} (\nabla S_{\theta_0}(A_j,B_j)) \\
\Sigma&=\Var_{\theta_0}\left(\sum_{j=1}^k\beta_j\lambda_j \nabla S_{\theta_0}(A_j,B_j)\right).
\end{align}
\end{prop}
The notation $\Var_{\theta_0}(Y)$ represents the covariance matrix of the random vector $Y$ under $p_{\theta_0}$ while the notations $\toop,\tood$ in the proof below denote convergences in probability and in distribution \citep{Ferguson1996}. $\nabla$ represents the gradient vector with respect to $\theta$. 

When $\theta$ is a vector the asymptotic variance is a matrix. To facilitate comparison between different estimators we follow the convention of using the determinant, and in some cases the trace, to measure the statistical accuracy. See \citep{Serfling1980} for some heuristic arguments for doing so. Figures~\ref{fig:policies},\ref{fig:boltzmann},\ref{fig:compAccPlot} provide the asymptotic variance for some SCL estimators and describe how it can be used to gain insight into which estimator to use.

The statistical accuracy of the SCL estimator depends on $\beta$ (weight parameters) and $\lambda$ (selection parameter). It is thus desirable to use the results in this section in determining what values of $\beta,\lambda$ to use. Directly using the asymptotic variance is not possible in practice as it depends on the unknown quantity $\theta_0$. However, it is possible to estimate the asymptotic variance using the training data. We describe this in Section~\ref{sec:beta}.

\section{Robustness of $\hat\theta_n^{\text{msl}}$}\label{sec:robust}

We observed in our experiments (see Section~\ref{sec:experiments}) that the SCL estimator sometimes performs better on a held-out test set than did the maximum likelihood estimator. This phenomenon seems to be in contradiction to the fact that the asymptotic variance of the MLE is lower than that of the SCL maximizer. This is explained by the fact that in some cases the true model generating the data does not lie within the parametric family $\{p_{\theta}:\theta\in\Theta\}$ under consideration. For example, many graphical models (HMM, CRF, LDA, etc.) make conditional independence assumptions that are often violated in practice. In such cases the SCL estimator acts as a regularizer achieving better test set performance than the non-regularized MLE. We provide below a theoretical account of this phenomenon using the language of $m$-estimators and statistical robustness. Our notation follows the one in \citep{Vaart1998}.

We assume that the model generating the data is outside the model family $P(X)\not\in \{p_{\theta}:\theta\in\Theta\}$ and we augment $m_{\theta}(X,Z)$ in \eqref{eq:mDef} with
\begin{align*}
\psi_{\theta}(X,Z) &\defeq \nabla m_{\theta}(X,Z)\\
\dot\psi_{\theta}(X,Z) &\defeq \nabla^2 m_{\theta}(X,Z) \quad \text{(matrix of second order derivatives)}\\
\Psi_n(\theta) &\defeq\frac{1}{n}\sum_{i=1}^n \psi_{\theta}(X^{(i)},Z^{(i)}).
\end{align*}

Proposition~\ref{prop:robustConsist} below generalizes the consistency result by asserting that $\hat\theta_n\to\theta_0$ where $\theta_0$ is the point on $\{p_{\theta}:\theta\in\Theta\}$ that is closest to the true model $P$, as defined by 
\begin{align}\label{eq:KLproj}
\theta_0=\argmax_{\theta\in\Theta}M(\theta) \quad \text{where} \quad M(\theta)\defeq-\sum_{j=1}^k \beta_j\lambda_j D(P(X_{A_j}|X_{B_j})||p_{\theta}(X_{A_j}|X_{B_j})),
\end{align}
or equivalently, $\theta_0$ satisfies
\begin{align} \label{eq:defTheta0}
\E_{P(X)}\E_{P(Z)} \psi_{\theta_0}(X,Z)=0.
\end{align}
When the scl function reverts to the loglikelihood function, $\theta_0$ becomes the KL projection of the true model $P$ onto the parametric family $\{p_{\theta}:\theta\in\Theta\}$.

\begin{prop} \label{prop:robustConsist}
Assuming the conditions in Proposition~\ref{prop:consistency} as well as $\sup_{\theta:\|\theta-\theta_0\|\geq \epsilon} M(\theta)<M(\theta_0)$ for all $\epsilon>0$  we have 
$\hat\theta_n^{\text{msl}}\to\theta_0$ as $n\to\infty$ with probability 1. 
\end{prop}

The added condition maintains that $\theta_0$ is a well separated  maximum point of $M$. In other words it asserts that only values close to $\theta_0$ may yield a value of $M$ that is close to the maximum $M(\theta_0)$. This condition is satisfied in the case of most exponential family models. 

\begin{prop} \label{prop:robustVar1}
Assuming the conditions of Proposition~\ref{prop:asympVar} as well as 
$\E_{P(X)}\E_{P(Z)} \|\psi_{\theta_0}(X,Z)\|^2 < \infty$, $\E_{P(X)}\E_{P(Z)} \dot\psi_{\theta_0}(X)$ exists and is non-singular, $|\ddot\Psi_{ij}| = |\partial^2\psi_{\theta}(x)/\partial\theta_i\theta_j|<g(x)$ for all $i,j$ and $\theta$ in a neighborhood of $\theta_0$ for some integrable $g$, we have 
\begin{align} \label{eq:asymp}
\sqrt{n}(\hat\theta_n-\theta_0) &= -(\E_{P(X)}\E_{P(Z)} \dot\psi_{\theta_0})^{-1} \frac{1}{\sqrt{n}}\sum_{i=1}^n\psi_{\theta_0}(X^{(i)},Z^{(i)})+o_P(1)\\
&\text{or equivalently} \nonumber \\
\hat\theta_n &= \theta_0  -(\E_{P(X)}\E_{P(Z)} \dot\psi_{\theta_0})^{-1} \frac{1}{n}\sum_{i=1}^n\psi_{\theta_0}(X^{(i)},Z^{(i)}) + o_P\left(\frac{1}{\sqrt{n}}\right). \label{eq:asymp2}
\end{align}
\end{prop}
Above, $f_n=o_P(g_n)$ means $f_n/g_n$ converges to 0 with probability 1. 
\begin{cor}  \label{corr:robustVar} Assuming the conditions specified in Proposition~\ref{prop:robustVar1} we have
\begin{align}
\sqrt{n}(\hat\theta_n-\theta_0) &\tood N(0,(\E_{P(X)}\E_{P(Z)} \dot\psi_{\theta_0})^{-1} (\E_{P(X)}\E_{P(Z)} \psi_{\theta_0}\psi_{\theta_0}^{\top})(\E_{P(X)}\E_{P(Z)} \dot\psi_{\theta_0})^{-1}). \label{eq:normality}
\end{align}
\end{cor}

Equation \eqref{eq:asymp2} means that asymptotically, $\hat\theta_n$ behaves as $\theta_0$ plus the average of iid RVs. As mentioned in \citep{Vaart1998} this fact may be used to obtain a convenient expression for the asymptotic influence function, which measures the effect of adding a new observation to an existing large dataset. Neglecting the remainder in \eqref{eq:asymp} we have
\begin{align}
\mathcal{I}(x,z) &\defeq \hat\theta_n(X^{(1)},\ldots,X^{(n-1)},x,Z^{(1)},\ldots,Z^{(n-1)},z)-\hat\theta_{n-1}(X^{(1)},\ldots,X^{(n-1)},Z^{(1)},\ldots,Z^{(n-1)})
\nonumber \\ &\approx
 -(\E_{P(X)}\E_{P(Z)} \dot\psi_{\theta_0})^{-1} \left( \frac{1}{n} \sum_{i=1}^{n-1}\psi_{\theta_0}(X^{(i)},Z^{(i)})+ \frac{1}{n} \psi_{\theta_0}(w,z)-\frac{1}{n-1} \sum_{i=1}^{n-1} \psi_{\theta_0}(X^{(i)},Z^{(i)})\right) \nonumber \\ 
&= -(\E_{P(X)}\E_{P(Z)} \dot\psi_{\theta_0})^{-1}  \frac{1}{n} \psi_{\theta_0}(w,z) +(\E_{P(X)}\E_{P(Z)} \dot\psi_{\theta_0})^{-1} \frac{1}{n(n-1)} \sum_{i=1}^{n-1}\psi_{\theta_0}(X^{(i)},Z^{(i)}) \nonumber \\
&=  - \frac{1}{n} (\E_{P(X)}\E_{P(Z)} \dot\psi_{\theta_0})^{-1}  \psi_{\theta_0}(w,z) + o_P\left(\frac{1}{n}\right). \label{eq:infFun}
\end{align}

Corollary~\ref{corr:robustVar} and Equation~\ref{eq:infFun} measure the statistical behavior of the estimator when the true distribution is outside the model family. In these cases it is possible that a computationally efficient SCL maximizer will result in higher statistical accuracy as well. This ``win-win'' situation of improving in both accuracy and complexity over the MLE is confirmed by our experiments in Section~\ref{sec:experiments}. 

\section{Stochastic Composite Likelihood for Markov Random Fields}
\label{sec:comp}

Markov random fields (MRF) are some of the more popular statistical models for complex high dimensional data. Approaches based on pseudo likelihood and composite likelihood are naturally well-suited in this case due to the cancellation of the normalization term in the probability ratios defining conditional distributions.  More specifically, a MRF with respect to a graph $G=(V,E)$, $V=\{1,\ldots,m\}$
with a clique set $\mathcal{C}$ is given by the following exponential family
model
\begin{align} \label{eq:expModel}
   P_{\theta}(x)&= \exp\left(\sum_{C\in\mathcal{C}} \theta_{C}
   f_C(x_C)-\log Z(\theta)\right),\nonumber\\ & Z(\theta)=\sum_x \exp\left(\sum_{C\in\mathcal{C}} \theta_c f_C(x_C)\right).
\end{align}
The primary bottlenecks in obtaining the maximum likelihood are the
computations $\log Z(\theta)$ and $\nabla \log Z(\theta)$. Their computational
complexity is exponential in the graph's treewidth and for many cyclic graphs,
such as the Ising model or the Boltzmann machine, it is exponential in $|V|=m$.

In contrast, the conditional distributions that form the composite likelihood of \eqref{eq:expModel} are given by (note the cancellation of $Z(\theta)$)
\begin{align} 
   P_{\theta}(x_A|x_{B}) &= \label{eq:condProb} 
   \frac{       \sum\limits_{x_{(A\cup B)^c}'} \exp\left(\sum_{C\in\mathcal{C}} \theta_{C} f_C((x_A,x_B,x_{(A\cup B)^c}')_C) \right)}
   {
       \sum\limits_{x_{(A\cup B)^c}'} \sum\limits_{x_A''} \exp\left(\sum\limits_{C\in\mathcal{C}} \theta_{C} f_C((x_A'',x_B,x_{(A\cup B)^c}')_C)\right)
   }.
\end{align}
whose computation is substantially faster. Specifically, The computation of \eqref{eq:condProb} depends on the size of the sets $A$ and $(A\cup B)^c$ and their intersections with the cliques in $\mathcal{C}$. In general, selecting small $|A_j|$ and $B_j=(A_j)^c$ leads to efficient computation of the composite likelihood and its gradient. For example, in the case of $|A_j|=l, |B_j|=m-l$
with $l\ll m$ we have that $k\leq m!/(l!(m-l)!)$ and the complexity of
computing the $c\ell(\theta)$ function and its gradient may be shown to require time that is at most exponential in $l$ and polynomial in $m$.

\section{Automatic Selection of $\beta$} \label{sec:beta}

As Proposition \ref{prop:asympVar} indicates, the weight vector $\beta$ and
selection probabilities $\lambda$ play an important role in the statistical
accuracy of the estimator through its asymptotic variance. The computational
complexity, on the other hand, is determined by $\lambda$ independently of
$\beta$. Conceptually, we are interested in resolving the accuracy-complexity
tradeoff jointly for both $\beta,\lambda$ before estimating $\theta$ by
maximizing the scl function.  However, since the computational complexity depends only on $\lambda$ we propose the following simplified problem: Select $\lambda$ based on available computational resources, and then given $\lambda$, select the $\beta$ (and $\theta$) that will achieve optimal statistical accuracy. 

Selecting $\beta$ that minimizes the asymptotic variance is somewhat ambiguous
as $\Upsilon\Sigma\Upsilon$ in Proposition~\ref{prop:asympVar} is an $r\times
r$ positive semidefinite matrix. A common solution is to consider the
determinant as a one dimensional measure of the size of the variance matrix\footnote{See \citep{Serfling1980} for a heuristic discussion motivating this measre.}, and minimize
\begin{align}
   J(\beta) &= \log \det (\Upsilon\Sigma\Upsilon)  
    =\log\det\Sigma + 2\log \det \Upsilon.   \label{eq:betaobjective}
\end{align}

A major complication with selecting $\beta$ based on the optimization of \eqref{eq:betaobjective} is that it depends on the true
parameter value $\theta_0$ which is not known at training time. This may be resolved, however, by noting that \eqref{eq:betaobjective} is composed of covariance matrices under $\theta_0$ which may be estimated using empirical covariances over the training set. To facilitate fast computation of the optimal $\beta$ we also propose to 
replace the determinant in \eqref{eq:betaobjective} with the product of the digaonal elements. Such an approximation is motivated by Hadamard's inequality (which states that for symmetric matrices $\det(M)\le\prod_i M_{ii}$) and by Ger\v{s}gorin's circle theorem (see below). This approximation works well in practice as we observe in the experiments section. We also note that the procedure described below involves only simple statisics that may be computed on the fly and does not contribute significant additional computation (nor do they require significant memory). 

More specifically, we denote $K^{(ij)}=\Cov_{\theta_0} (\nabla S_{\theta_0}(A_i,B_i), \nabla
S_{\theta_0}(A_j,B_j))$ with entries $K^{(ij)}_{st}$, and approximate the $\log\det$ terms in \eqref{eq:betaobjective} using 
\begin{align}
\log\det \Upsilon&=-\log\det\sum_{j=1}^k \beta_j\lambda_j K^{(jj)} 
	\approx - 	\sum_{l=1}^r\log\sum_{j=1}^k \beta_j\lambda_j K^{(jj)}_{ll}\label{eq:approxUpsilon}\\
\log\det\Sigma &= \log\det\Var_{\theta_0}\left(\sum_{j=1}^k\beta_j\lambda_j \nabla S_{\theta_0}(A_j,B_j)\right) 
	=\log\det\sum_{i=1}^k\sum_{j=1}^k \beta_i\lambda_i \beta_j\lambda_j K^{(ij)}\nonumber\\
	&\approx	\sum_{l=1}^r\log\sum_{i=1}^k\sum_{j=1}^k \beta_i\lambda_i \beta_j\lambda_j K^{(ij)}_{ll}.\label{eq:approxSigma}
\end{align}

We denote (assuming $A$ is a $n \times n$ matrix) for $i \in \{1, \ldots, n\}$,  $R_i(A) = \sum_{j\neq{i}}
\left|A_{ij}\right|$ and let $D(A_{ii}, R_i(A))$ ($D_i$
where unambiguous) be the closed disc centered at $A_{ii}$ with radius
$R_i(A)$. Such a disc is called a Ger\v{s}gorin disc. The result below states that for matrices that are close to diagonal, the eigenvalues are close to the diagonal elements making our approximation accurate.
\begin{theorem}[Ger\v{s}gorin's circle theorem e.g., \citep{Horn1990}]
Every eigenvalue of $A$ lies within at least one of the Ger\v{s}gorin discs
$D(A_{ii}, R_i(A)).$ Furthermore, if the union of $k$ discs is disjoint from the
union of the remaining $n-k$ discs, then the former union contains exactly $k$ and
the latter $n-k$ eigenvalues of $A.$ \label{thm:gersgorin}
\end{theorem}

The following algorithm solves for $\theta,\beta$ jointly using alternating optimization. The second optimization problem with respect to $\beta$ is done using the approximation above and may be computed without much additional computation. In practice we found that such an approach lead to a selection of $\beta$ that is close to the optimal $\beta$ (see Sec.~\ref{sec:chunking} and Figures~\ref{fig:chunk_boltzchain_heuristic},
\ref{fig:chunk_crf_heuristic} for results).  
\begin{algorithm}[H]
\begin{algorithmic}[1]
	\REQUIRE $X$, $\beta_0$, and $\gamma$
	\STATE $i \leftarrow 1$
	\STATE $\beta \leftarrow \beta_0$
	\WHILE{$i<\textrm{MAXITS}$}
		\STATE $\theta \leftarrow \argmin sc\ell(X,\lambda,\beta)$\label{lin:argminscl}
		\IF{converged}
			\RETURN $\theta$
		\ELSE
			\STATE $\beta\leftarrow\argmin \mathcal{J}(X,\lambda,\theta,\gamma)$\label{lin:argminbeta}
			\STATE $i \leftarrow i + 1$
		\ENDIF
	\ENDWHILE
	\RETURN\FALSE
\end{algorithmic}
\caption{Calculate $\hat\theta^{msl}$}
\end{algorithm}

\section{Experiments} \label{sec:experiments}

We demonstrate the asymptotic properties of $\hat\theta_n^{\text{msl}}$ and explore the complexity-accuracy tradeoff for three different models-Boltzmann machine, linear Boltzmann MRF and conditional random fields. In terms of datasets, we consider synthetic data as well as datasets from sentiment prediction and text chunking domains.

\subsection{Toy Example:  Boltzmann Machines}\label{sec:boltzmann}

We illustrate the improvement in asymptotic variance of the mscle associated
with adding higher order likelihood components with increasing probabilities in
context of the Boltzmann machine 
\begin{align} \label{eq:BoltzmannMachine}
p_{\theta}(x)=\exp\left(\sum_{i<j}\theta_{ij}x_i
x_j-\log\psi(\theta)\right),\quad x \in\{0,1\}^m.
\end{align} 
To be able to accurately compute the asymptotic variance we use  $m=5$ with $\theta$ being a ${5 \choose 2}$
dimensional vector with half the components $+1$ and half $-1$. 
Since the asymptotic variance of  $\hat\theta_n^{\text{msl}}$ is a matrix we summarize
its size using either its trace or determinant.

Figure~\ref{fig:boltzmann} displays the asymptotic variance, relative to the
minimal variance of the mle, for the cases of full likelihood (FL), pseudo
likelihood ($|A_j|=1$) $\text{PL}_1$, stochastic combination of pseudo
likelihood and 2nd order pseudo likelihood ($|A_j|=2$) components $\alpha
\text{PL}_2 + (1-\alpha)\text{PL}_1$, stochastic combination of 2nd order
pseudo likelihood and 3rd order pseudo likelihood ($|A_j|=3$) components
$\alpha \text{PL}_3 + (1-\alpha)\text{PL}_2$, and stochastic combination of 3rd
order pseudo likelihood and 4th order pseudo likelihood ($|A_j|=4$) components
$\alpha \text{PL}_4 + (1-\alpha)\text{PL}_3$. 

The graph demonstrates the computation-accuracy tradeoff as follows: (a) pseudo
likelihood is the fastest but also the least accurate, (b) full likelihood is
the slowest but the most accurate, (c) adding higher order components reduces
the asymptotic variance but also requires more computation, (d) the variance
reduces with the increase in the selection probability $\alpha$ of the higher
order component, and (e) adding 4th order components brings the variance very
close the lower limit and with each successive improvement becoming smaller and
smaller according to a law of diminishing returns.

Figure~\ref{fig:compAccPlot} displays the asymptotic accuracy and complexity for different SCL policies for $m=9$. We see how taking different linear combinations of pseudo likelihood orders spans a continuous spectrum of accuracy-complexity resolutions. The lower part of the diagram is the boundary of the achievable region (the optimal but unachievable place is the bottom left corner). SCL policies that lie to the right and top of that boundary may be improved by selecting a policy below and to the left of it.

\begin{figure}
\begin{center}
{
	\psfrag{x1}{$\alpha$}
	\psfrag{x2}{$\alpha$}
	\psfrag{y1}{\scriptsize $\text{tr}(\Var(\hat\theta^{\text{msl}}))\,\,/\,\,\text{tr}(\Var(\hat\theta^{\text{ml}}))$}
	\psfrag{y2}{\scriptsize $\det(\Var(\hat\theta^{\text{msl}}))\,\,/\,\,\det(\Var(\hat\theta^{\text{ml}}))$}
	\includegraphics[scale=0.6]{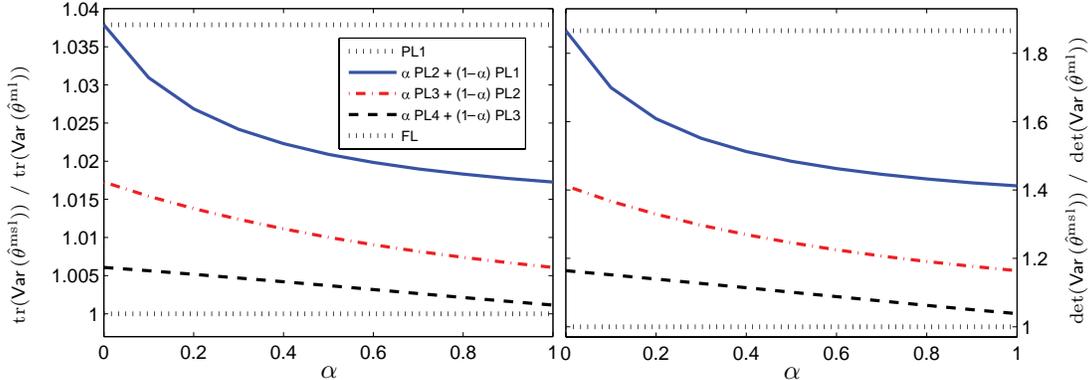}
}
\end{center}
\vspace{-.12in}
\caption{Asymptotic variance matrix, as measured by trace (left) and
determinant (right), as a function of the selection probabilities for different stochastic versions of the scl function.}
\label{fig:boltzmann}
\end{figure}

\begin{figure}\centering
\includegraphics[scale=0.4]{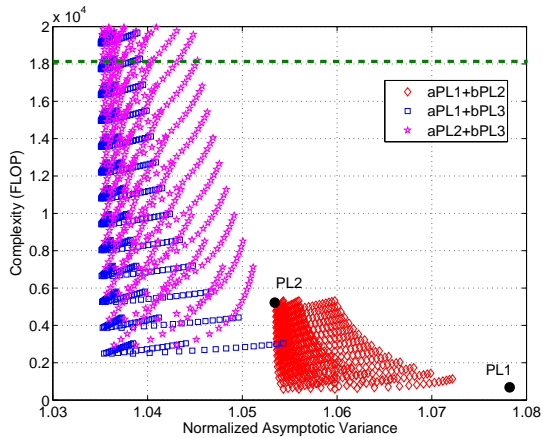}
\caption{Computation-accuracy diagram for three SCL families:$\lambda_1\beta_1\text{PL1}+\lambda_2(1-\beta_1)\text{PL2}$, 
$\lambda_1\beta_1 \text{PL1}+\lambda_2(1-\beta_1)\text{PL3}$, 
$\lambda_1\beta_1 \text{PL2}+\lambda_2(1-\beta_1)\text{PL3}$ (for multiple values of $\lambda_1,\lambda_2,\beta_1$) for the Boltzmann machine with 9 binary nodes. The pure policies PL1 and PL2 are indicated by black circles and the computational complexity of the full likelihood indicated by a dashed line (corresponding normalized asymptotic variance is 1). As the number of nodes increase the computational cost increase dramatically, in particular for the full likelihood and to a lesser extend for the pseudo likelihood policies.
}  
\label{fig:compAccPlot}
\end{figure}

\subsection{Local Sentiment Prediction}\label{sec:localsent}

Our first real world dataset experiment involves local sentiment prediction using a conditional MRF model. The dataset consisted of 249 movie review documents having an average of 30.5 sentences each with an average of 12.3 words from a 12633 word vocabulary.  Each sentence was manually labeled as one of five sentimental designations: very negative, negative, objective, positive, or very positive. As described in \citep{Mao2007b} (where more infomration may be found) we considered the task of predicting the local sentiment flow within these documents using regularized conditional random fields (CRFs) (see Figure~\ref{fig:crfgm} for a graphical diagram of the model in the case of four sentences). 

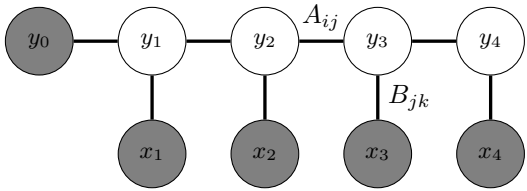
\begin{figure}
\centering
\begin{tikzpicture}[scale=1.5,auto,swap]
	\node[vertexobs] (y0) at (0,0){\small$y_0$};
	\foreach \pre/\cur/\A/\B in {0/1//, 1/2//, 2/3/A_{ij}/B_{jk}, 3/4//} {
		\node[vertex] (y\cur)  at (\cur,0)  {$y_\cur$};
		\path[edge] (y\pre) -- node[above] {$\A$} (y\cur);
		\node[vertexobs] (x\cur)  at (\cur,-1)  {$x_\cur$};
		\path[edge] (y\cur) -- node[right] {$\B$} (x\cur);
	}
\end{tikzpicture}
\caption{ Graphical representation of a four token conditional random
field (CRF).  $A$, $B$ are positive weight matrices and represent
state-to-state transitions and state-to-observation outputs. Shading indicates the variable is conditioned upon while no shading indicates the variable is
generated by the model.}
\label{fig:crfgm}
\end{figure}

Figure \ref{fig:localsent_crf_cont} shows the contour plots of train and test loglikelihood as a function of the scl parameters: weight $\beta$ and selection probability
$\lambda$.  The likelihood components were mixtures of full and pseudo
($|A_j|=1$) likelihood (rows 1,3) and pseudo and 2nd order pseudo $(|A_j|=2$)
likelihood (rows 2,4). $A_j$ identifies a set of labels corresponding to
adjacent sentences over which the probabilistic query is evaluated.  Results
were averaged over 100 cross validation iterations with 50\% train-test split.
We used BFGS quasi-Newton method for maximizing the regularized scl functions. The figure demonstrates how the train loglikelihood increases with
increasing the weight and selection probability of full likelihood in rows 1,3 and of 2nd order pseudo likelihood in rows 2,4. This increase in train loglikelihood is also correlated with an increase in computational complexity as higher order likelihood components require more computation. Note however, that the test set behavior in the third and fourth rows shows an improvement in prediction accuracy associated with decreasing the influence of full likelihood in favor of pseudo likelihood. The fact that this happens for weak regularization $\sigma^2=10$ indicates that lower order pseudo likelihood has a regularization effect which improves prediction accuracy when the model is not regularized enough. We have encountered this phenomenon in other experiments as well and we will discuss it further in the following subsections.

\begin{figure}
\centering
\vspace{-.2in}
\begin{tabular}{cc}
\vspace{-.02in}
   \includegraphics[width=.40\textwidth]{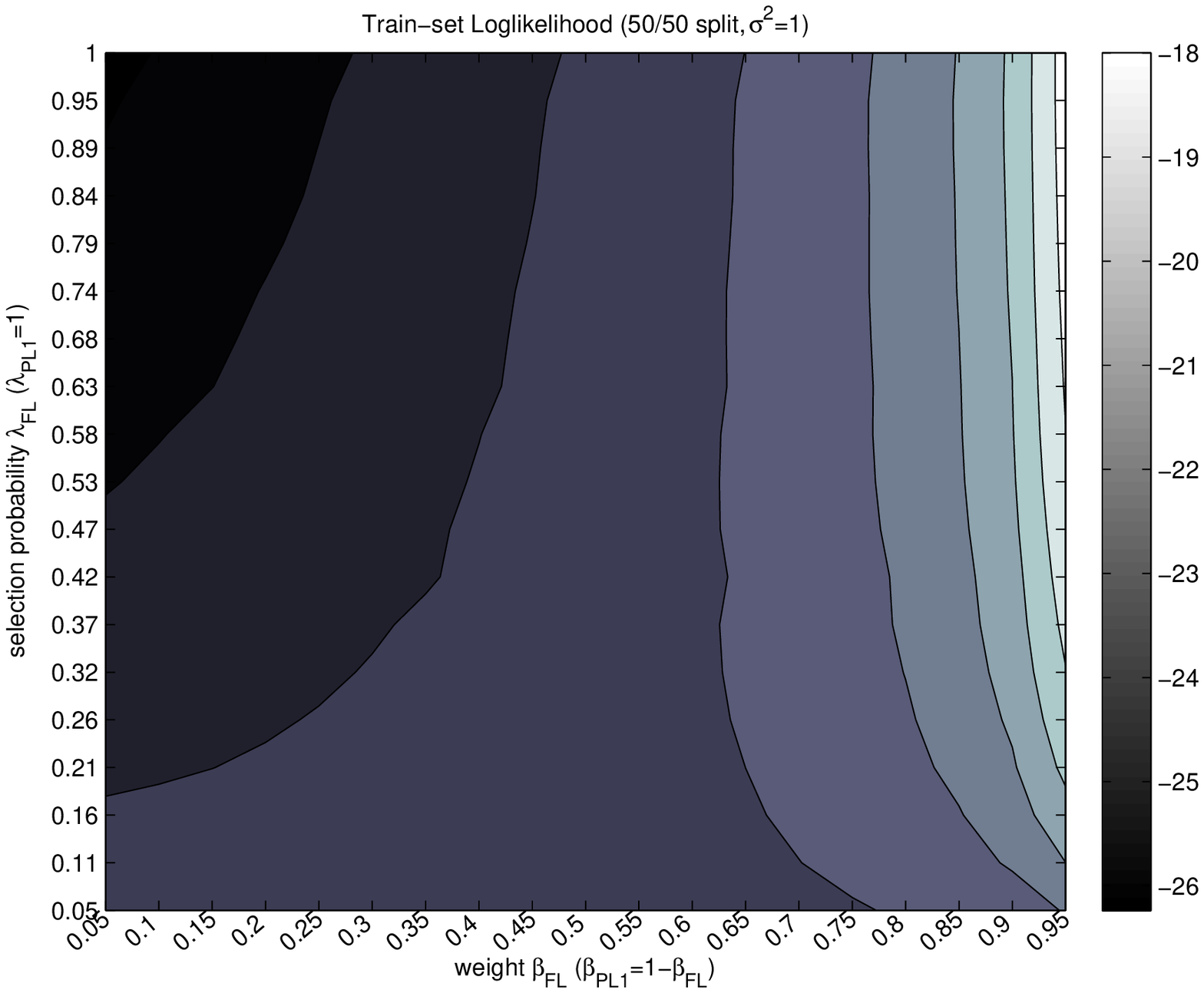} &     \vspace{-.055in}
   \includegraphics[width=.409\textwidth]{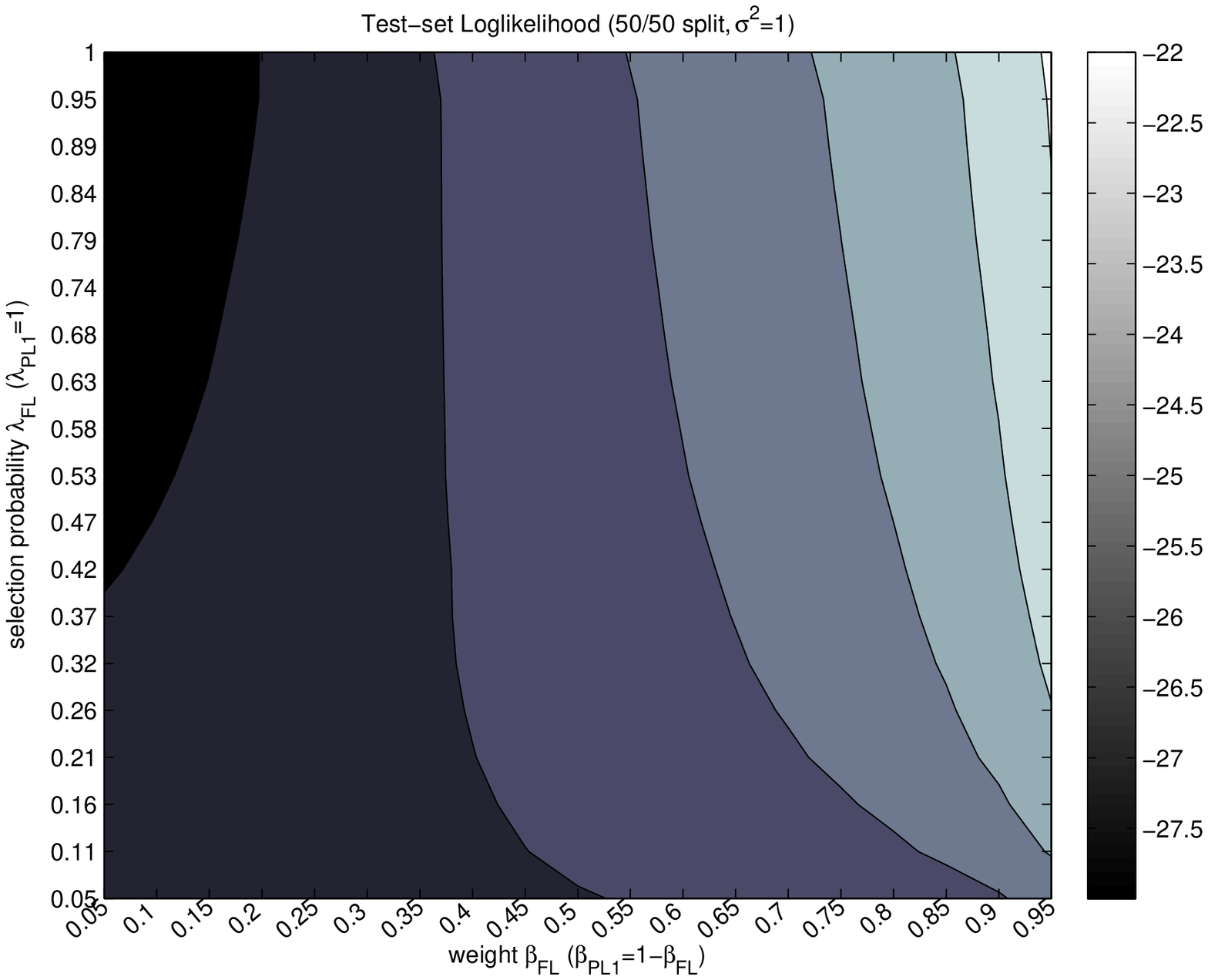}\\	\vspace{-.02in}
   \includegraphics[width=.40\textwidth]{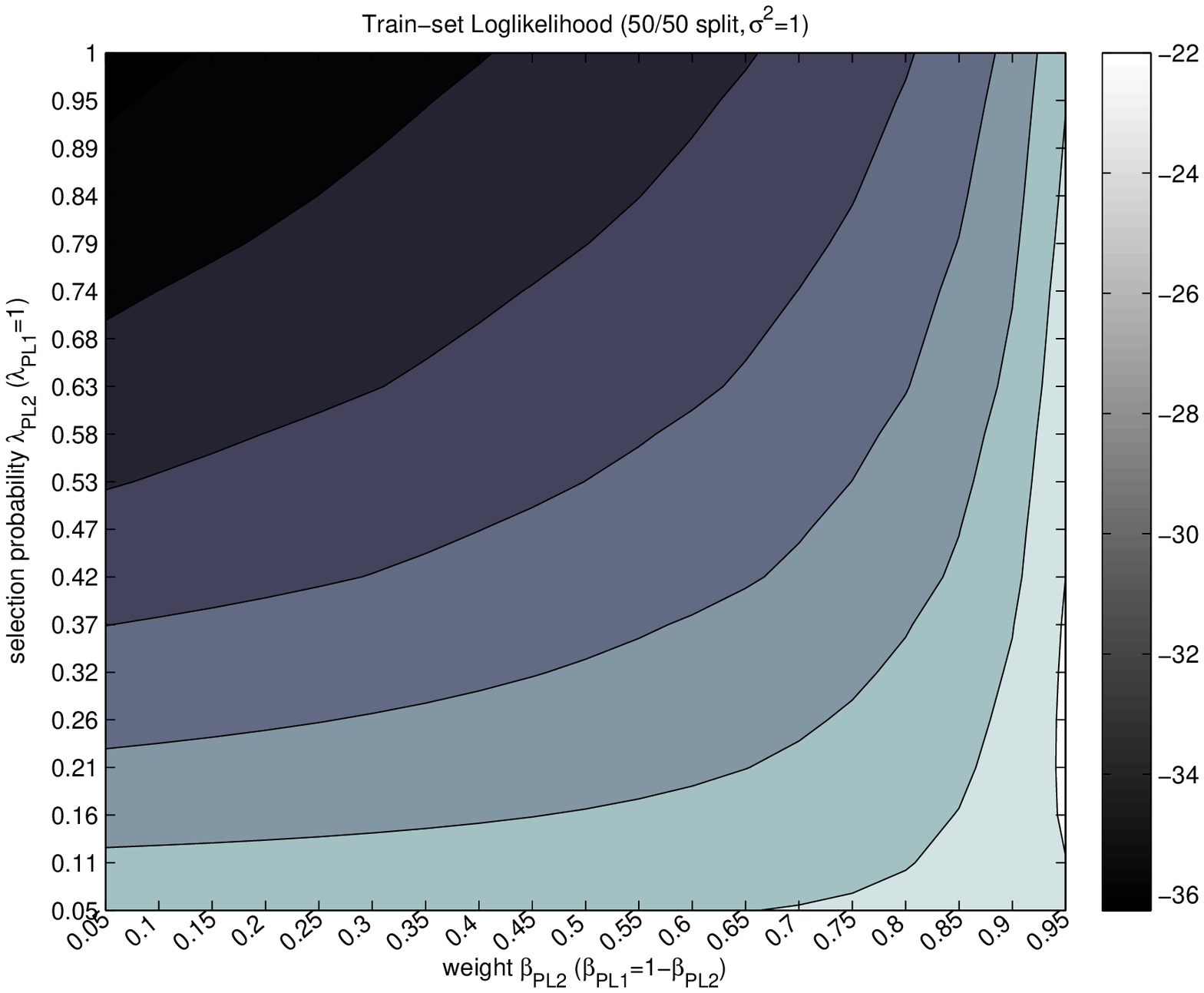} &   \vspace{-.055in}
   \includegraphics[width=.409\textwidth]{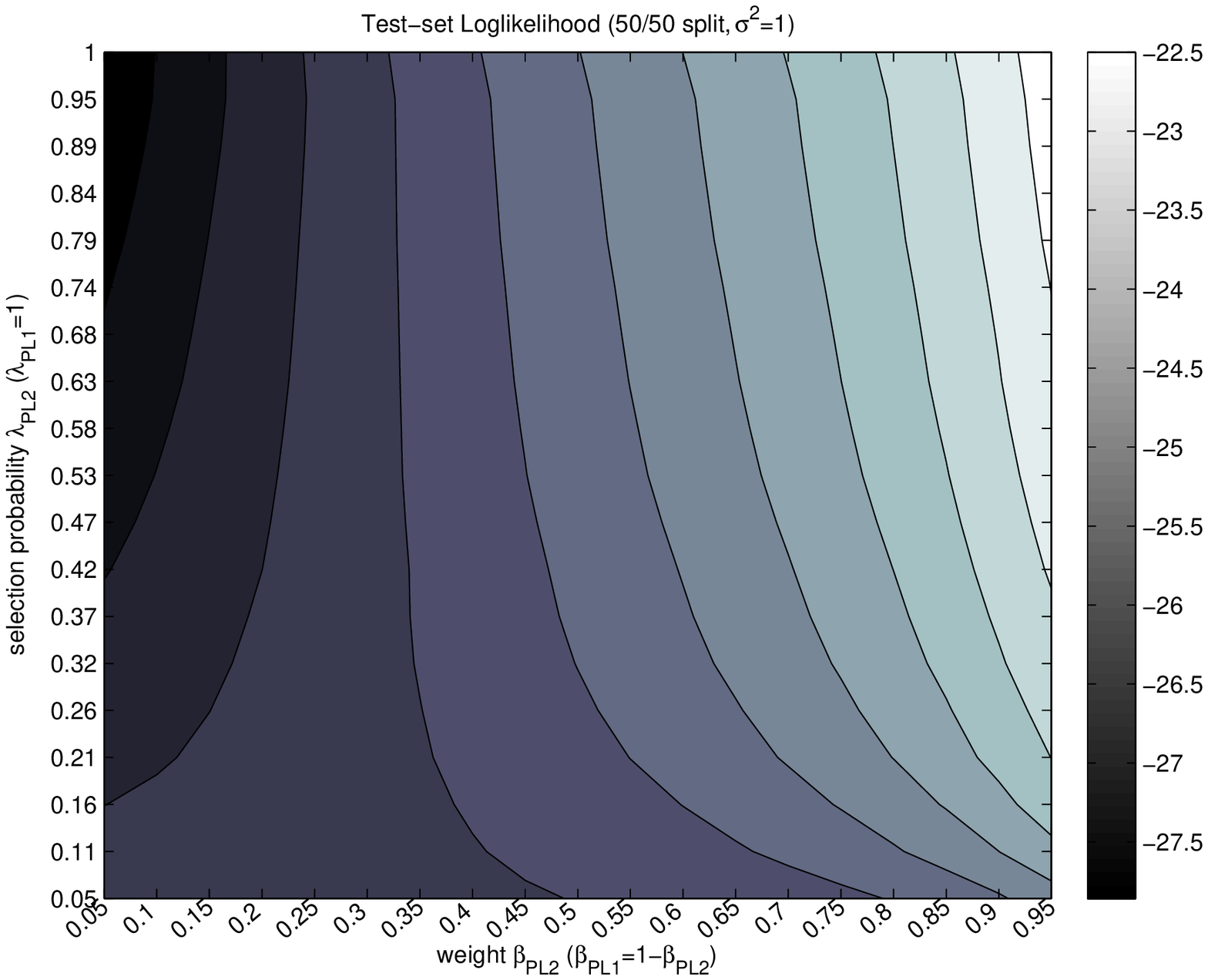}\\	\vspace{-.02in}
   \includegraphics[width=.40\textwidth]{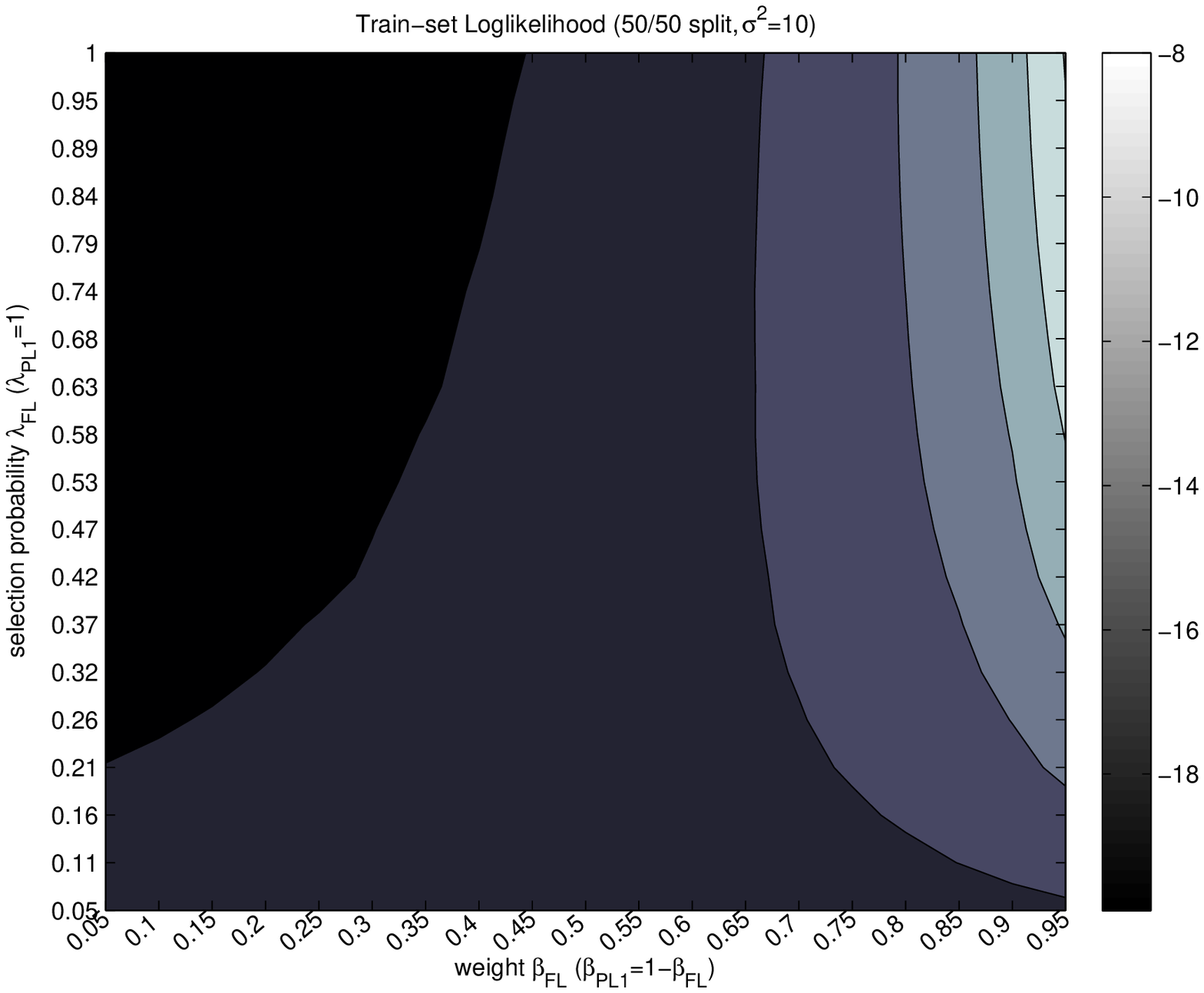} &   \vspace{-.055in}
   \includegraphics[width=.409\textwidth]{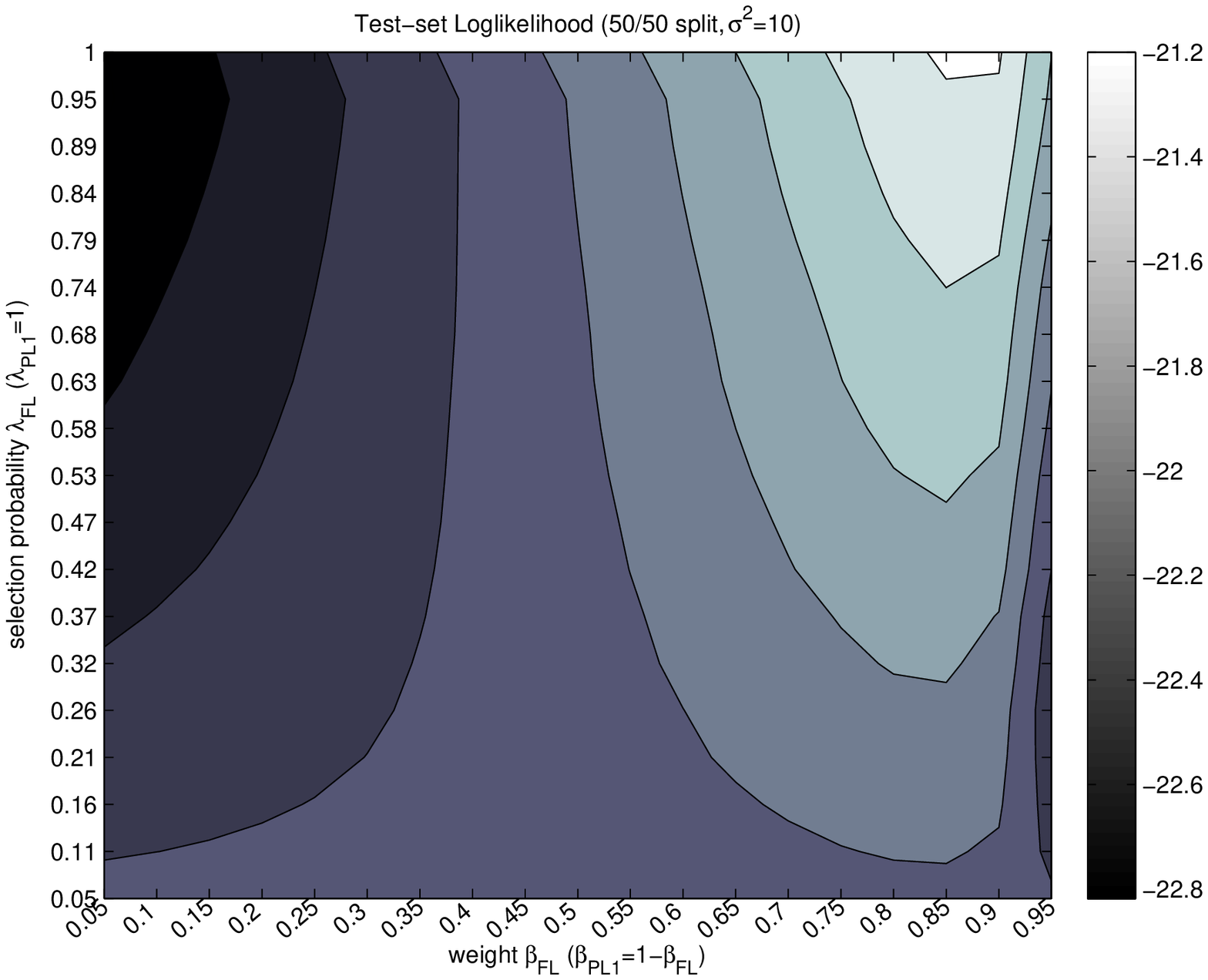}\\	\vspace{-.02in}
   \includegraphics[width=.40\textwidth]{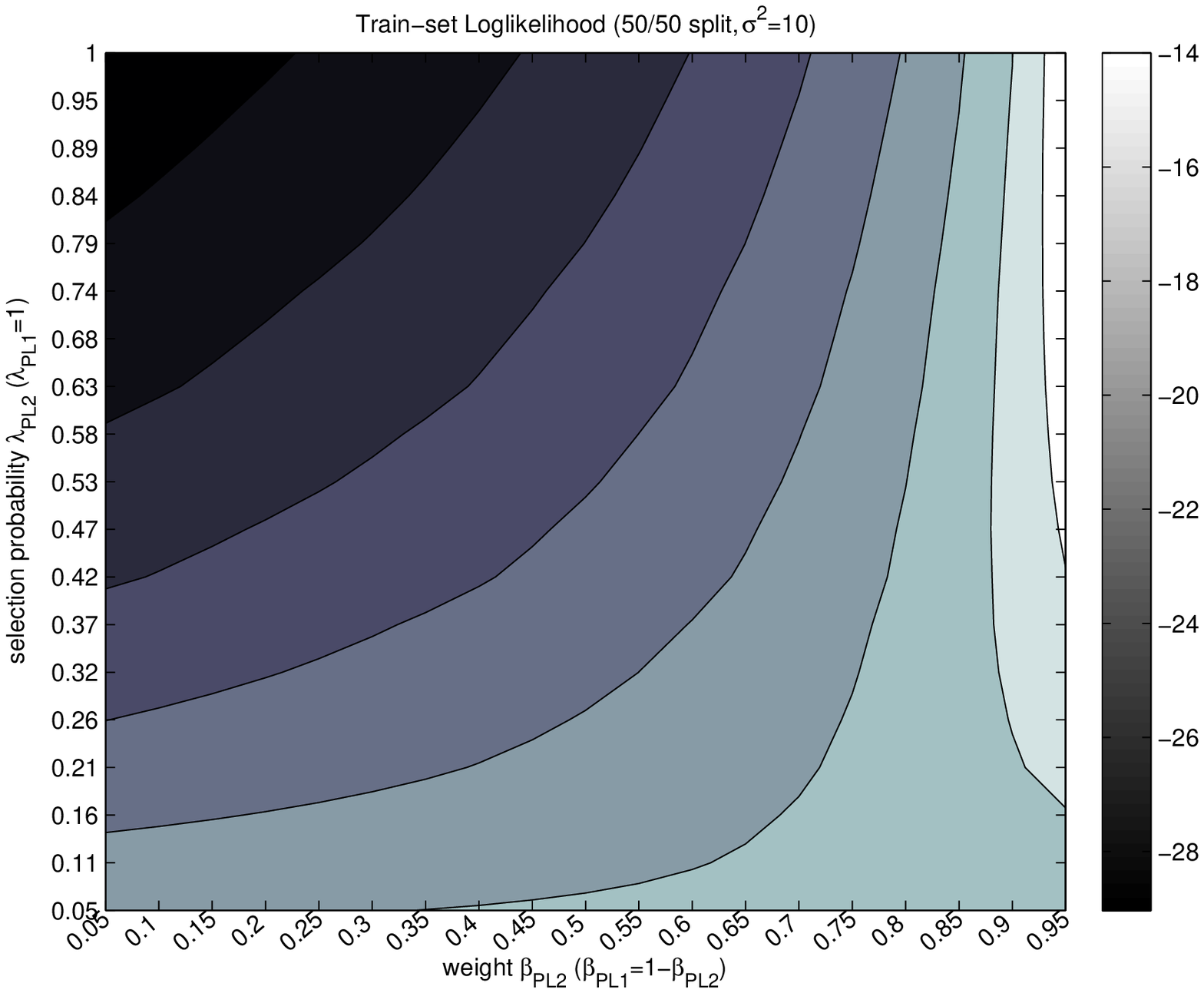} & \vspace{-.055in}
   \includegraphics[width=.409\textwidth]{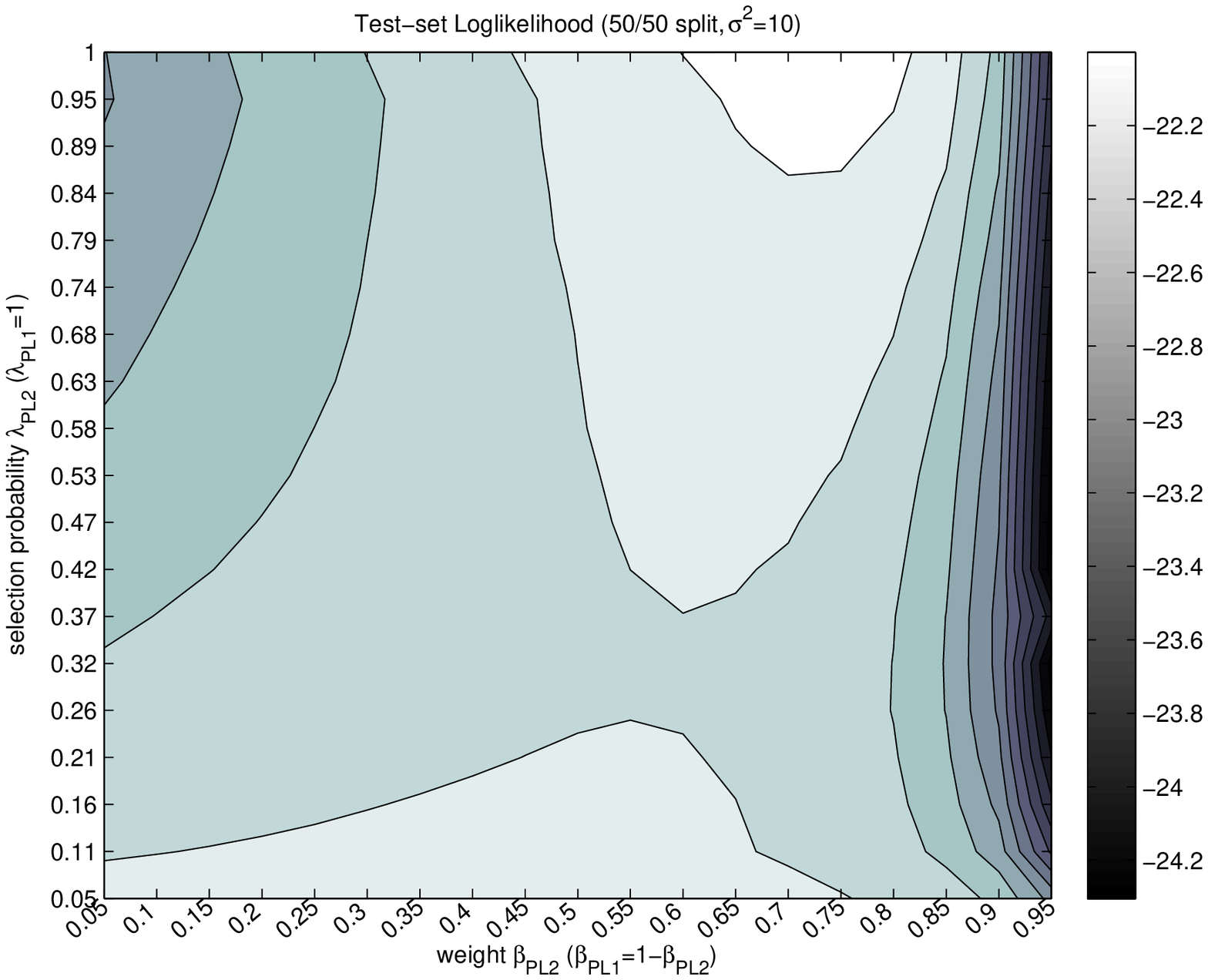}	\vspace{-.02in}
\end{tabular}
\caption{Train (left) and test (right) loglikelihood contours for maximum scl
estimators for the CRF model. $L_2$ regularization parameters are $\sigma^2=1$
(rows 1,2) and $\sigma^2=10$ (rows 3,4). Rows 1,3 are stochastic mixtures
of full (FL) and pseudo (PL1) loglikelihood components while rows 2,4
are PL1 and 2nd order pseudo likelihood (PL2).}
\label{fig:localsent_crf_cont}
\end{figure}

Figure~\ref{fig:tradeoff} displays the complexity and negative loglikelihoods
(left:train, right:test) of different scl estimators, sweeping through
$\lambda$ and $\beta$, as points in a two dimensional space. The shaded area
near the origin is unachievable as no scl estimator can achieve high accuracy
and low computation at the same time. The optimal location in this 2D plane is
the curved boundary of the achievable region with the exact position on that
boundary depending on the required solution of the computation-accuracy
tradeoff. 

\begin{figure}
\centering
\begin{tabular}{cc}
         \includegraphics[width=.5\textwidth]{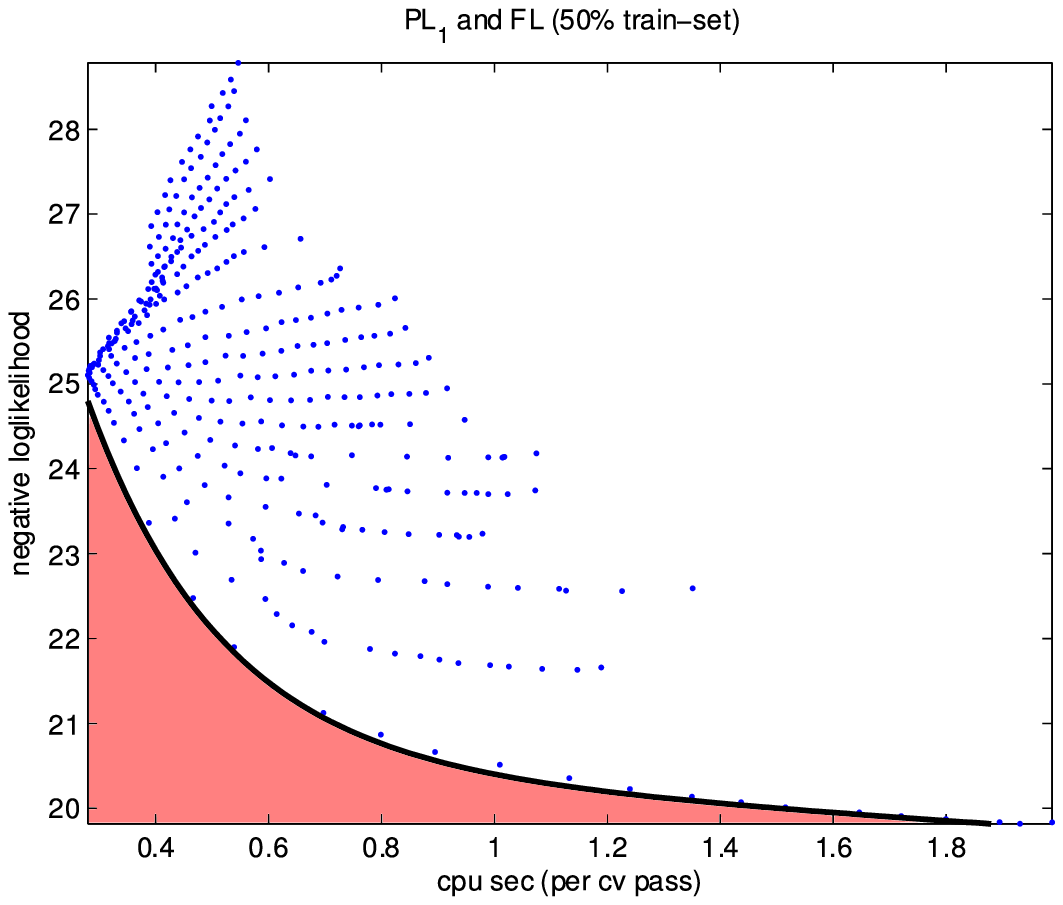} & \hspace{-0.45in}
   \includegraphics[width=.5\textwidth]{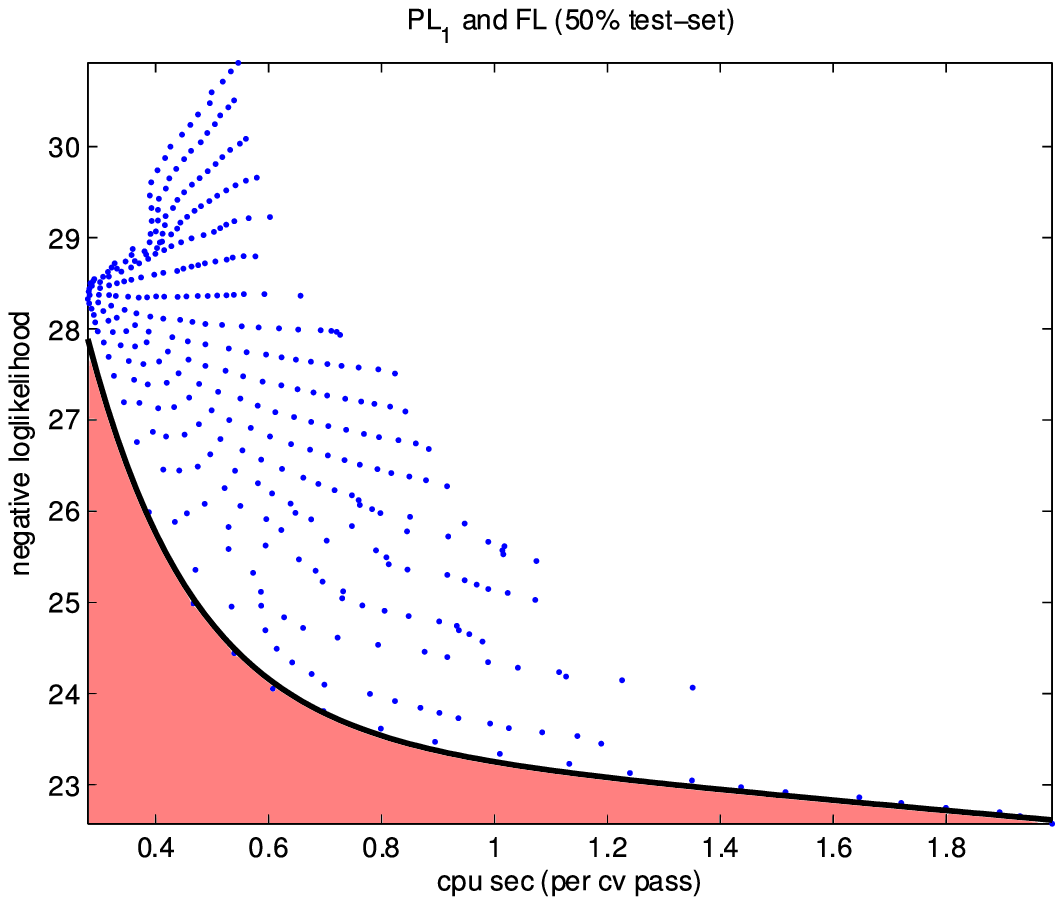}
\end{tabular}
\vspace{-.05in}
\caption{ Scatter plot representing complexity and negative loglikelihood
(left:train, right:test) of scl functions for CRFs with regularization
parameter $\sigma^2=1/2$. The points represent different stochastic
combinations of full and pseudo likelihood components. The shaded region
represents impossible accuracy/complexity demands.}
\label{fig:tradeoff}
\end{figure}

\subsection{Text Chunking}\label{sec:chunking}

This experiment consists of using sequential MRFs to divide sentences into
``text chunks,'' i.e., syntactically correlated sub-sequences, such as noun and
verb phrases. Chunking is an crucial step towards full parsing.  For
example\footnote{Taken from the CoNLL-2000 shared task site,
\url{http://www.cnts.ua.ac.be/conll2000/chunking/}.}, the sentence:
\begin{quote}
\small
He reckons the current account deficit will narrow to only \# 1.8 billion in September.
\end{quote}
could be divided as:
\begin{quote}
\small
[NP \textcolor{red}{He} ]
[VP \textcolor{green}{reckons} ]
[NP \textcolor{red}{the current account deficit} ]
[VP \textcolor{green}{will narrow} ]
[PP \textcolor{blue}{to} ]
[NP \textcolor{red}{only \# 1.8 billion} ]
[PP \textcolor{blue}{in} ]
[NP \textcolor{red}{September} ]. 
\end{quote}
where NP, VP, and PP indicate noun phrase, verb phrase, and prepositional phrase.

\begin{figure}
\centering
\includegraphics[trim=0.0mm 0.0mm 0.0mm 0.0mm,clip,width=.40\textwidth]{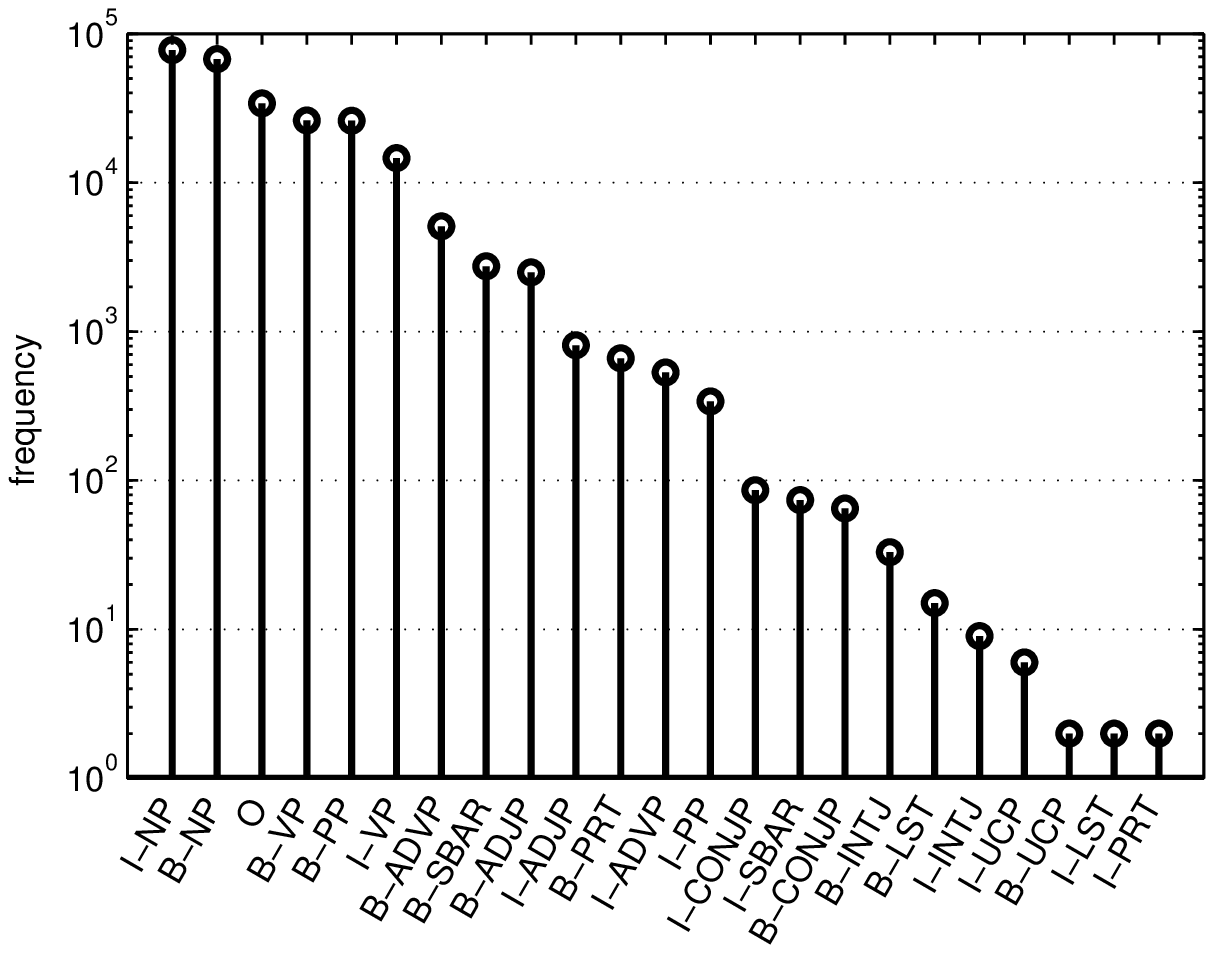}
\vspace{-1.5em}
\caption{ Label counts in CoNLL-2000 dataset.}
\label{fig:conll2000labels}
\end{figure}
We used the publicly available CoNLL-2000 shared task dataset.  It consists of
labeled partitions of a subset of the Wall Street Journal (WSJ) corpus.  Our
training sets consisted of sampling 100 sentences without replacement from the
the CoNLL-2000 training set (211,727 tokens from WSJ sections 15-18).  The test
set was the same as the CoNLL-2000 testing
partition (47,377 tokens from  WSJ section 20).  Each of the possible 21,589
tokens, i.e., words, numbers, punctuation, etc., are tagged by one of 11 chunk
types and an O label indicating the token is not part of any chunk.  Chunk
labels are prepended with flags indicating that the token begins (B-) or is
inside (I-) the phrase. Figure~\ref{fig:conll2000labels} lists all labels and
respective frequencies.  In addition to labeled tokens, the dataset contains a
part-of-speech (POS) column.  These tags were automatically generated by the Brill tagger and must be incorporated into any model/feature set accordingly.

In the following, we explore this task using various scl selection polices on two related, but fundamentally different sequential MRFs: Boltzmann chain MRFs and CRFs.

\subsubsection{Boltzmann Chain MRF}\label{sec:chunk_boltzchain}
Boltzmann chains are a generative MRF that are closely related to hidden Markov models (HMM).  See \citep{MacKay1996} for a discussion on the relationship between Boltzmann chain MRFs and HMMs. We consider SCL components of the form  $\Pr(X_2,Y_2|Y_1,Y_3)$,  $\Pr(X_2,X_3,Y_2,Y_3|Y_1,Y_4)$ which we refer to as first and second order pesudo likelihood (with higher order components generalizing in a straightforward manner).

\begin{figure}
\centering
\begin{tikzpicture}[scale=1.5,auto,swap]
	\node[vertexobs] (y0) at (0,0){\small$y_0$};
	\foreach \pre/\cur/\A/\B in {0/1//, 1/2//, 2/3/A_{ij}/B_{jk}, 3/4//} {
		\node[vertex] (y\cur)  at (\cur,0)  {$y_\cur$};
		\path[edge] (y\pre) -- node[above] {\small $\A$} (y\cur);
		\node[vertex] (x\cur)  at (\cur,-1)  {$x_\cur$};
		\path[edge] (y\cur) -- node[right] {\small $\B$} (x\cur);
	}
\end{tikzpicture}
\caption{ Graphical representation of a four token Boltzmann chain. $A$, $B$
are positive weight matrices and represent preference in particular
state-to-state transitions and state-to-feature emissions. Only the start state
is conditioned upon while all others are generative.}
\label{fig:boltzchaingm}
\end{figure}
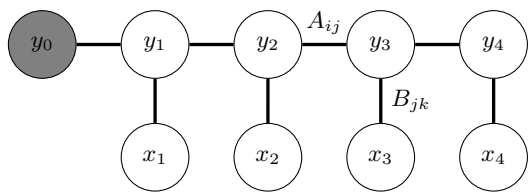

The nature of the Boltzmann chain constrains our feature set to only encode the
particular token present at each position, or time index.  In doing so we avoid
having to model additional dependencies across time steps and dramatically
reduce computational complexity.  Although scl is precisely motivated by high
treewidth graphs, we wish to include the full likelihood for demonstrative
purposes--in practice, this is often not possible.  Although POS tags are
available we do not include them in these features since the dependence they
share on neighboring tokens and other POS tags is unclear.  For these reasons
our time-sliced feature vector, $x_i$, has only a single-entry one and
cardinality matching the size of the vocabulary (21,589 tokens).  

As is common practice, we curtail overfitting through a $L_2$ regularizer,
$\exp\{-\frac{1}{2\sigma^{2}} ||\theta||^2_2\}$, which is is strong when
$\sigma^2$ is small and weak when $\sigma^2$ is large.  We consider $\sigma^2$
a hyper-parameter and select it through cross-validation, unless noted
otherwise.  More often though, we show results for several representative
$\sigma^2$ to demonstrate the roles of $\lambda$ and $\beta$ in
$\hat{\theta}_n^{msl}$.

Figures \ref{fig:chunk_boltzchain_pl1fl_cont} and
\ref{fig:chunk_boltzchain_pl1fl_beta} depict train and test negative
log-likelihood, i.e., perplexity, for the scl estimator
$\hat{\theta}_{100}^{msl}$ with a pseudo/full likelihood selection policy
(PL1/FL).  As is our convention, weight $\beta$ and selection probability
$\lambda$ correspond to the higher order component, in this case full
likelihood.  The lower order pseudo likelihood component is always selected and
has weight $1-\beta$. As expected the test set perplexity dominates the
train-set perplexity.  As was the situation in Sec.~\ref{sec:localsent}, we
note that the lower order component serves to regularize the full-likelihood,
as evident by the abnormally large $\sigma^2$.

We next demonstrate the effect of using a 1st order/2nd order pseudo likelihood
selection policy (PL1/PL2).  Recall, our notion of pseudo likelihood never
entails conditioning on $x$, although in principle it could. Figures
\ref{fig:chunk_boltzchain_pl1pl2_cont} and
\ref{fig:chunk_boltzchain_pl1pl2_beta} show how the policy responds to varying
both $\lambda$ and $\beta$.  Figure \ref{fig:chunk_boltzchain_complexity}
depicts the empirical tradeoff between accuracy and complexity.
Figure~\ref{fig:chunk_boltzchain_heuristic} highlights the effectiveness of the
$\beta$ heuristic. See captions for additional comments.

\begin{figure}
\centering
\begin{tabular}{cc}
\includegraphics[trim=0mm 0mm 0mm 0mm,clip,totalheight=.25\textheight]{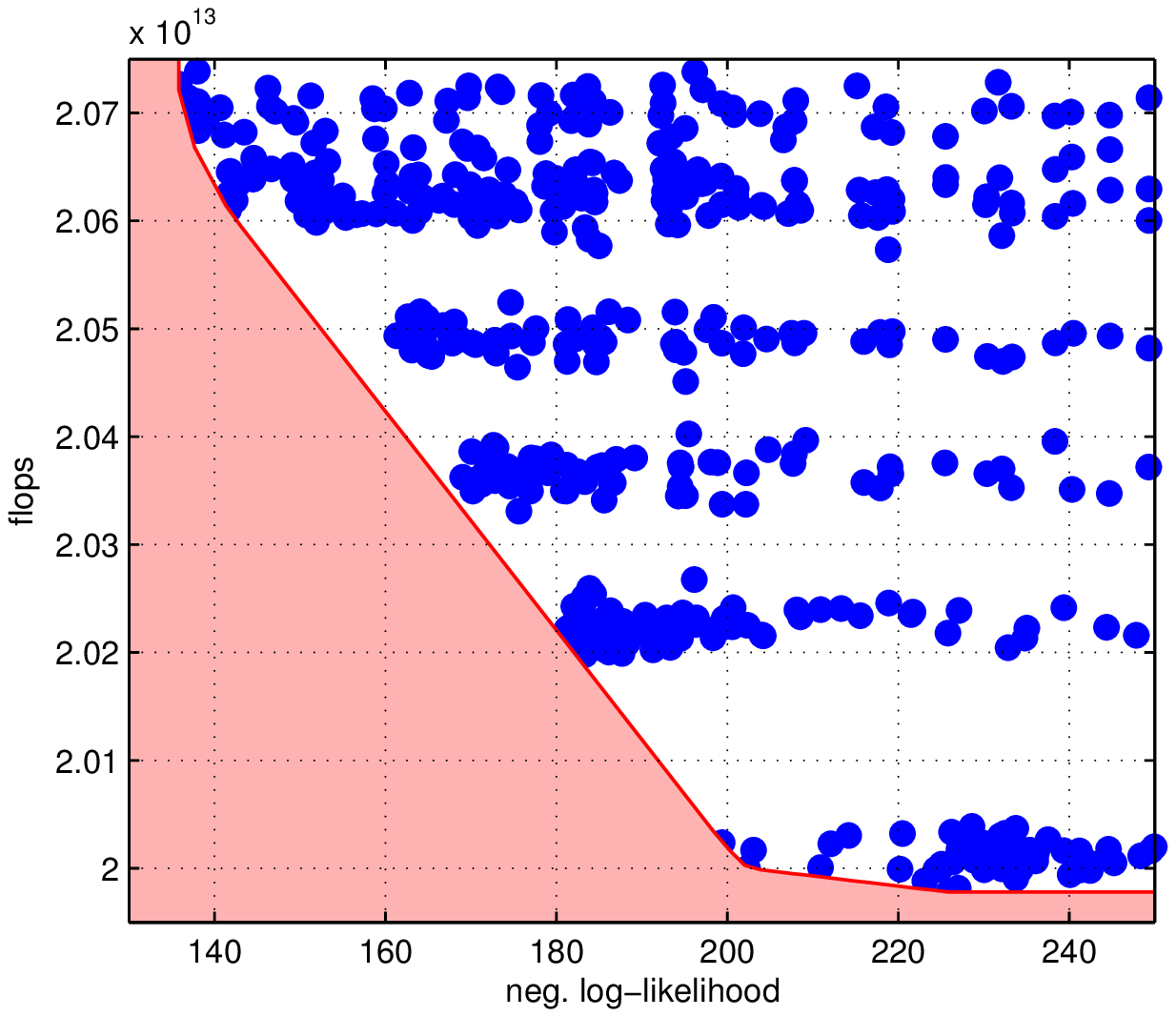}  &  \includegraphics[trim=0mm 0mm 0mm 0mm,clip,totalheight=.25\textheight]{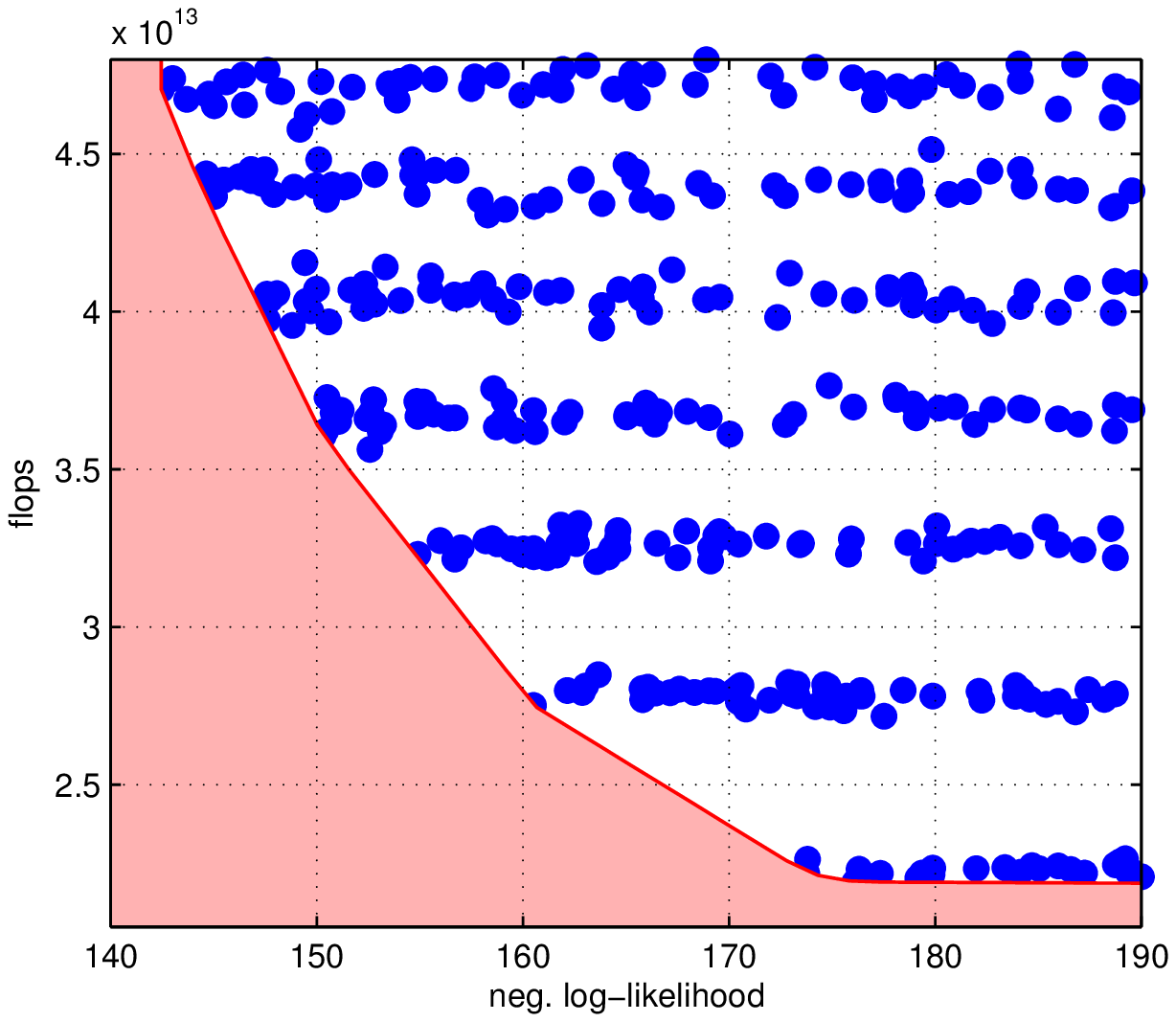} \end{tabular}
\caption{
	Accuracy and complexity tradeoff for the Boltzmann chain MRF with PL1/FL
	(left) and PL1/PL2 (right) selection policies.  Each point represents the negative loglikelihood (perplexity) and the 
	number of flops required to evaluate the composite likelihood and its
	gradient under a particular instantiation of the selection policy.  The
	shaded region represents empirically unobtainable combinations of
	computational complexity and accuracy.  
}\label{fig:chunk_boltzchain_complexity}
\end{figure}

\begin{figure}[htb!]
\vspace{-3.00em}
\hspace{-0.75cm}
\begin{tabular}{cccc}
\includegraphics[trim= 0.0mm 10.2mm 0mm 9.5mm,clip,totalheight=.157\textheight]{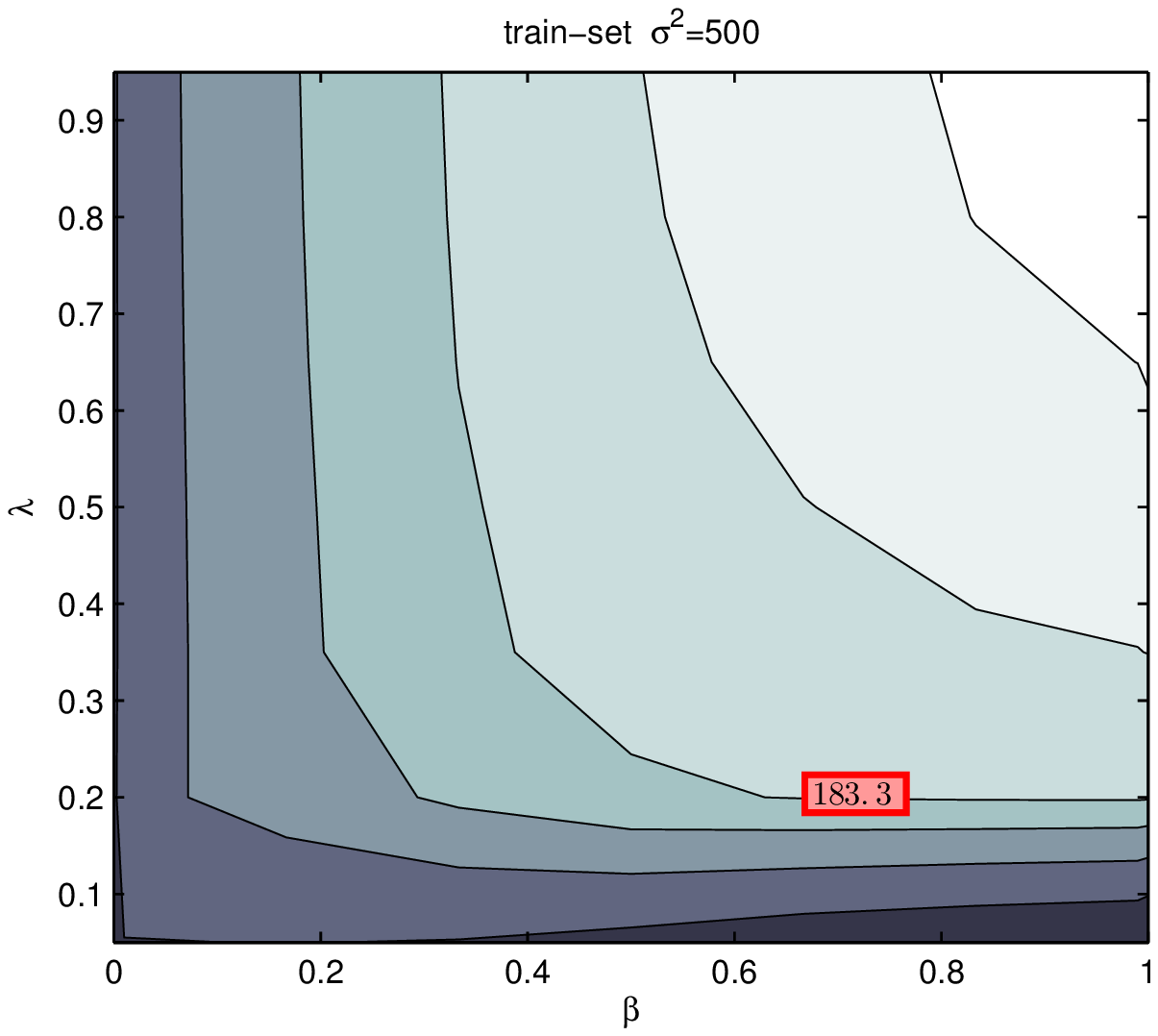}  &  \hspace{-5.5mm}
\includegraphics[trim=12.0mm 10.2mm 0mm 9.5mm,clip,totalheight=.157\textheight]{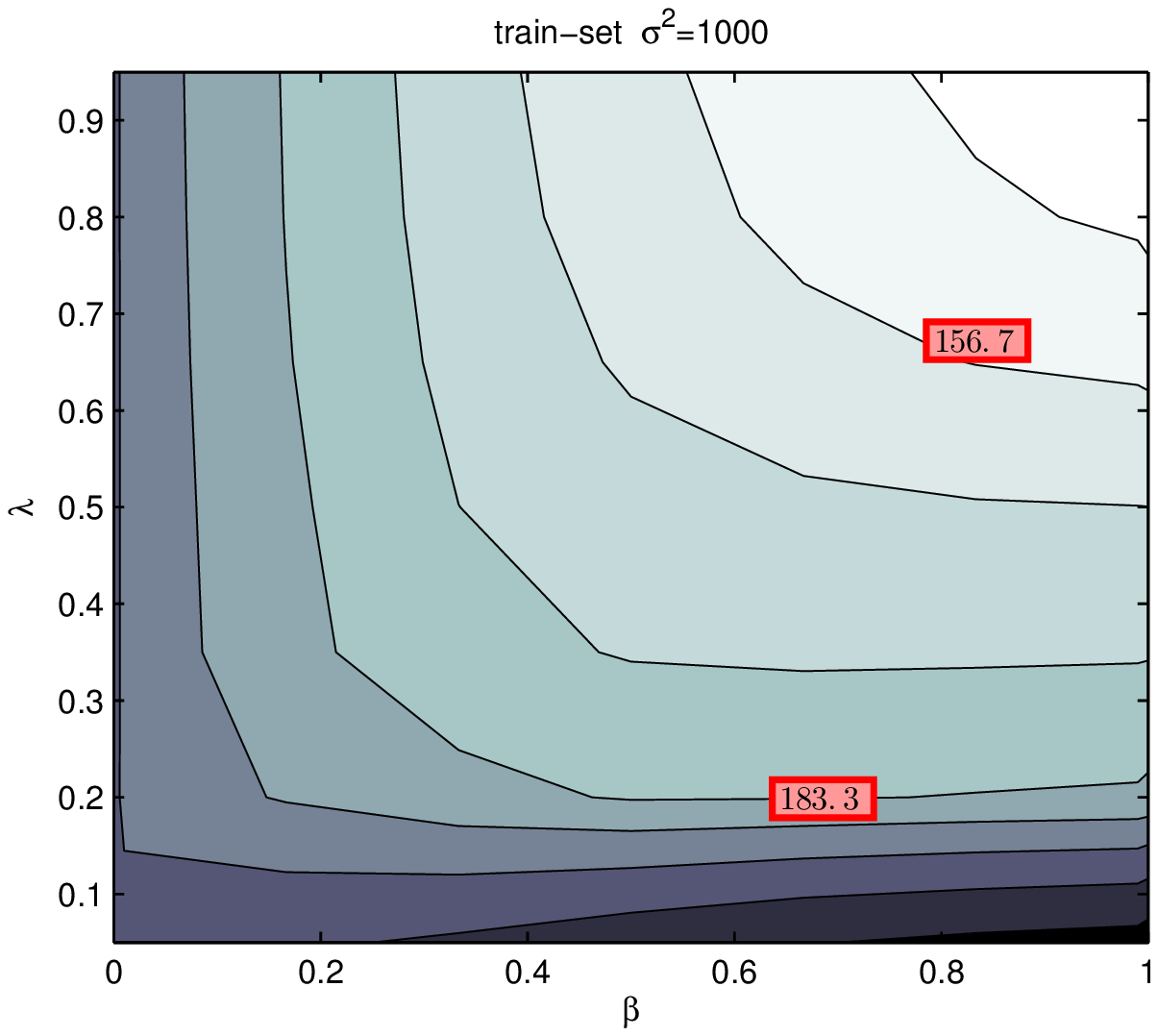} &  \hspace{-5.5mm}
\includegraphics[trim=12.0mm 10.2mm 0mm 9.5mm,clip,totalheight=.157\textheight]{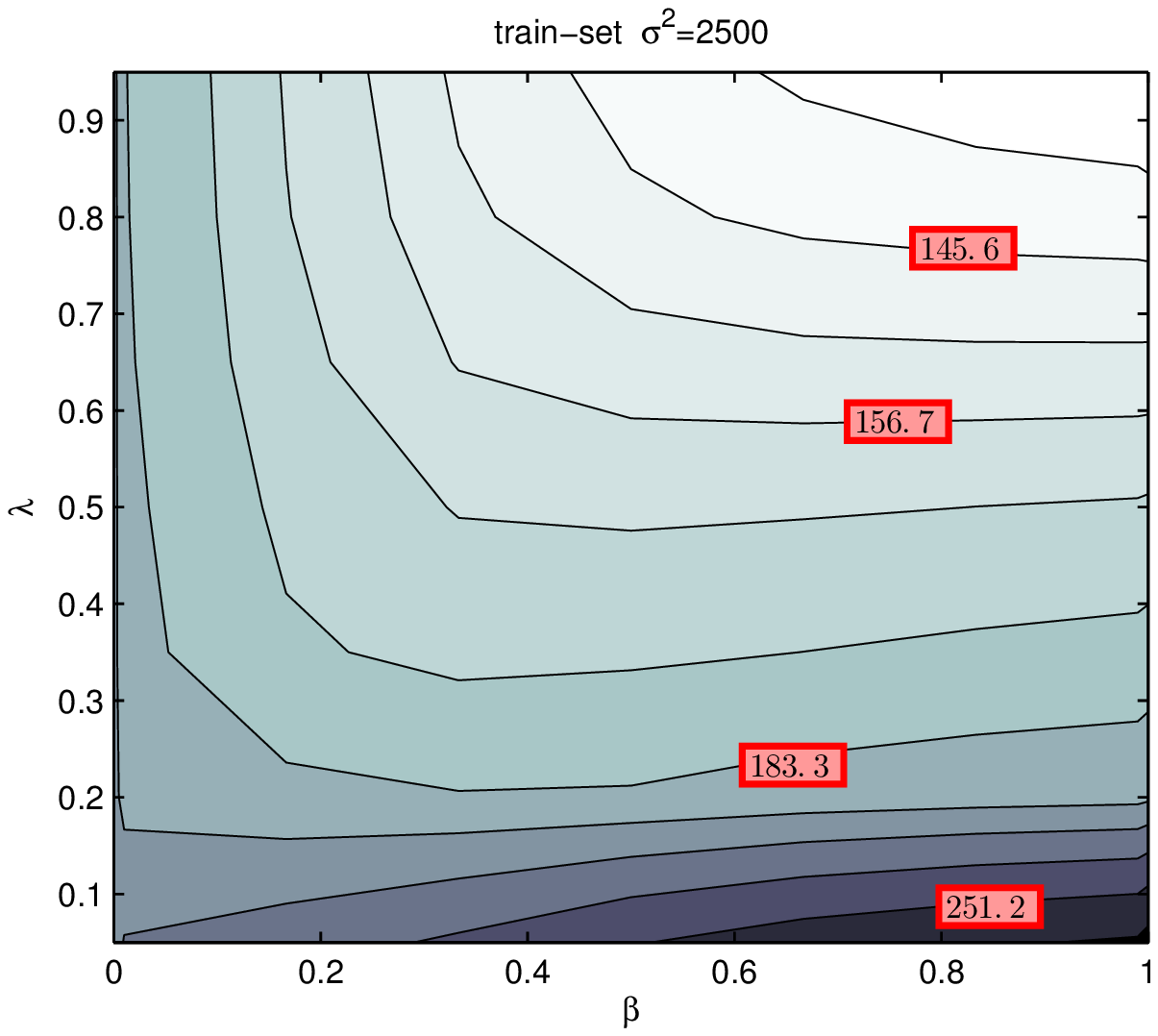} &  \hspace{-5.5mm}
\includegraphics[trim=12.0mm 10.2mm 0mm 9.5mm,clip,totalheight=.157\textheight]{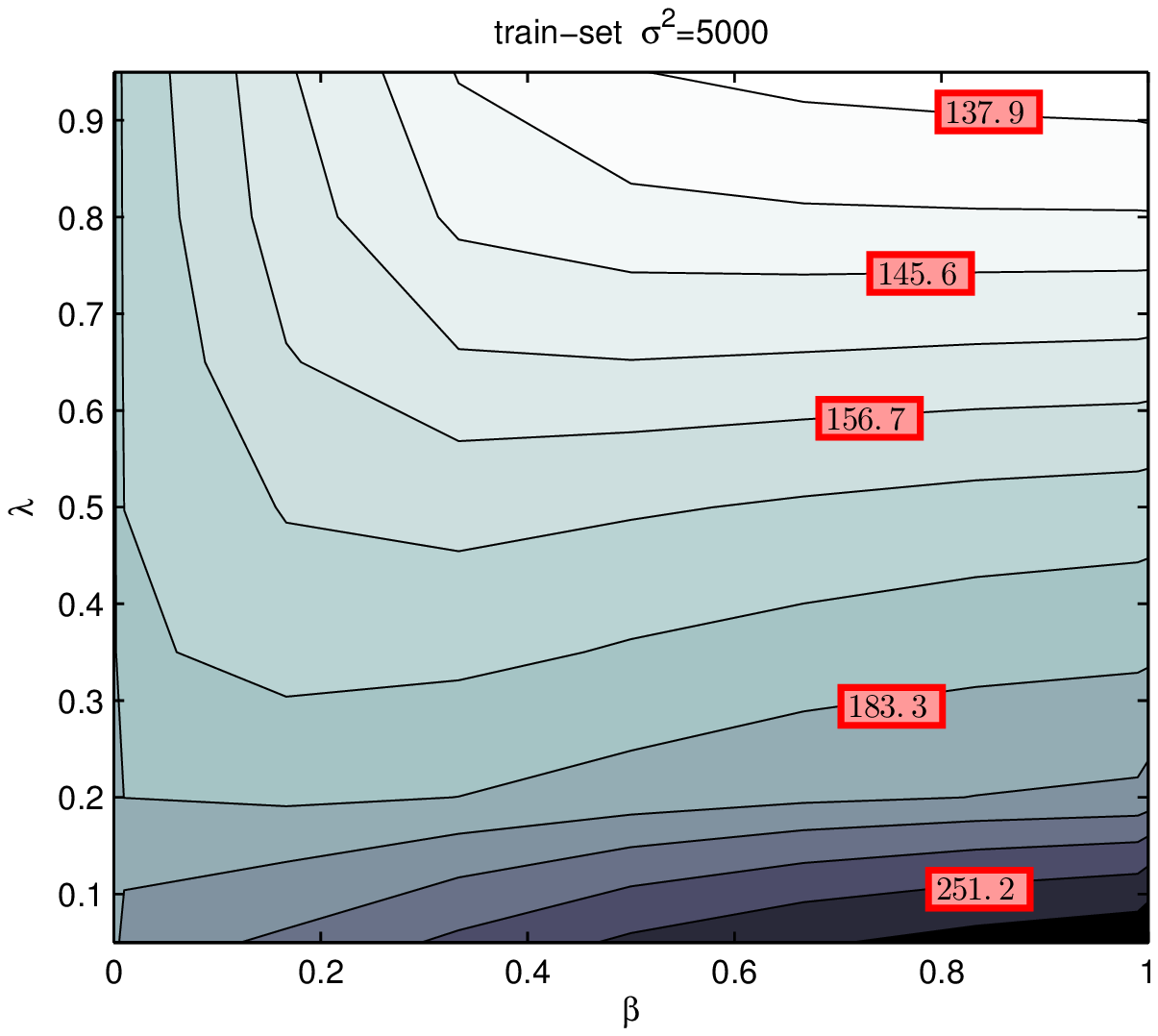} \\
\includegraphics[trim= 0.0mm 0mm 0mm 9.5mm,clip,totalheight=.175\textheight]{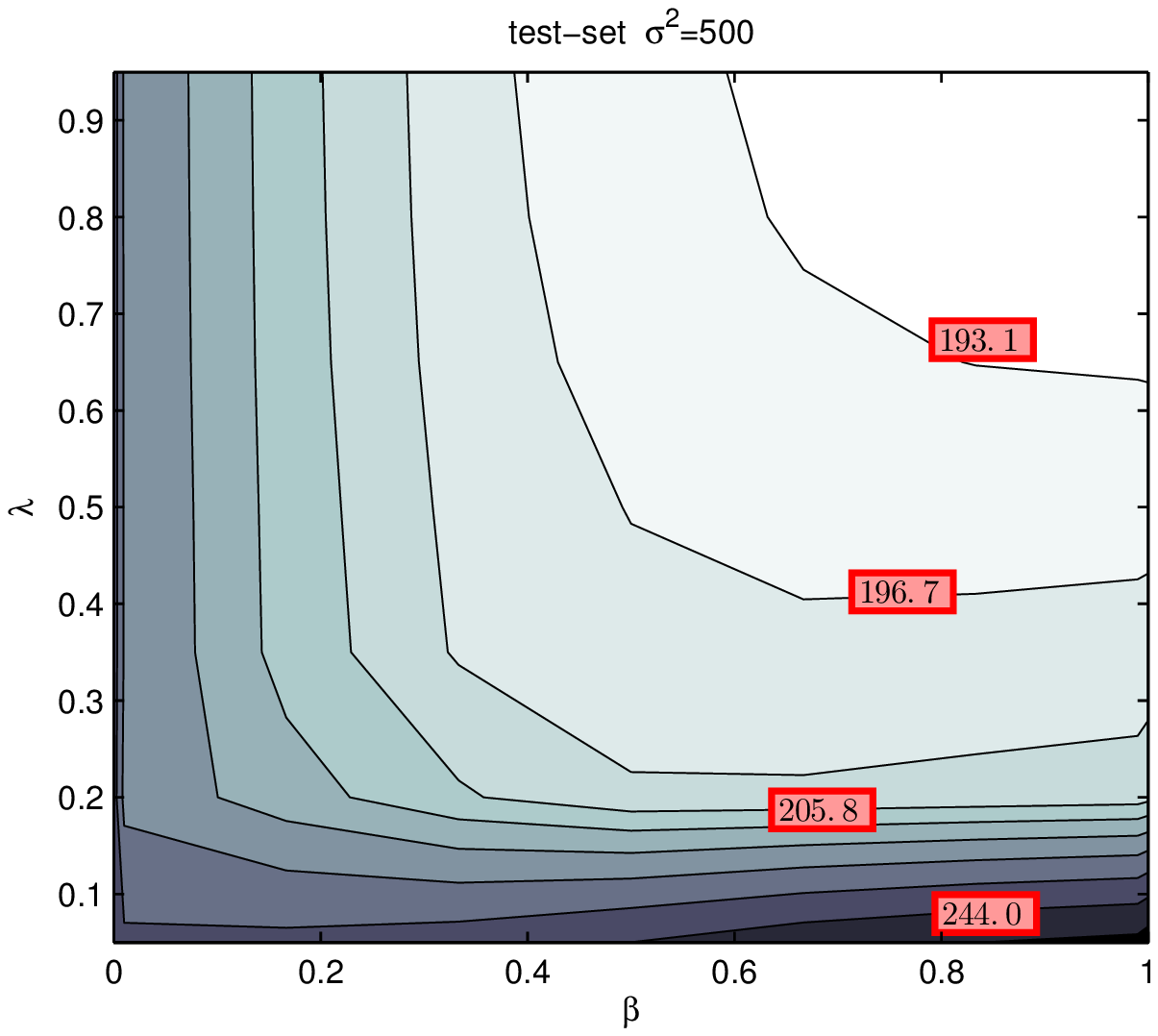}     &  \hspace{-5.5mm}
\includegraphics[trim=12.0mm 0mm 0mm 9.5mm,clip,totalheight=.175\textheight]{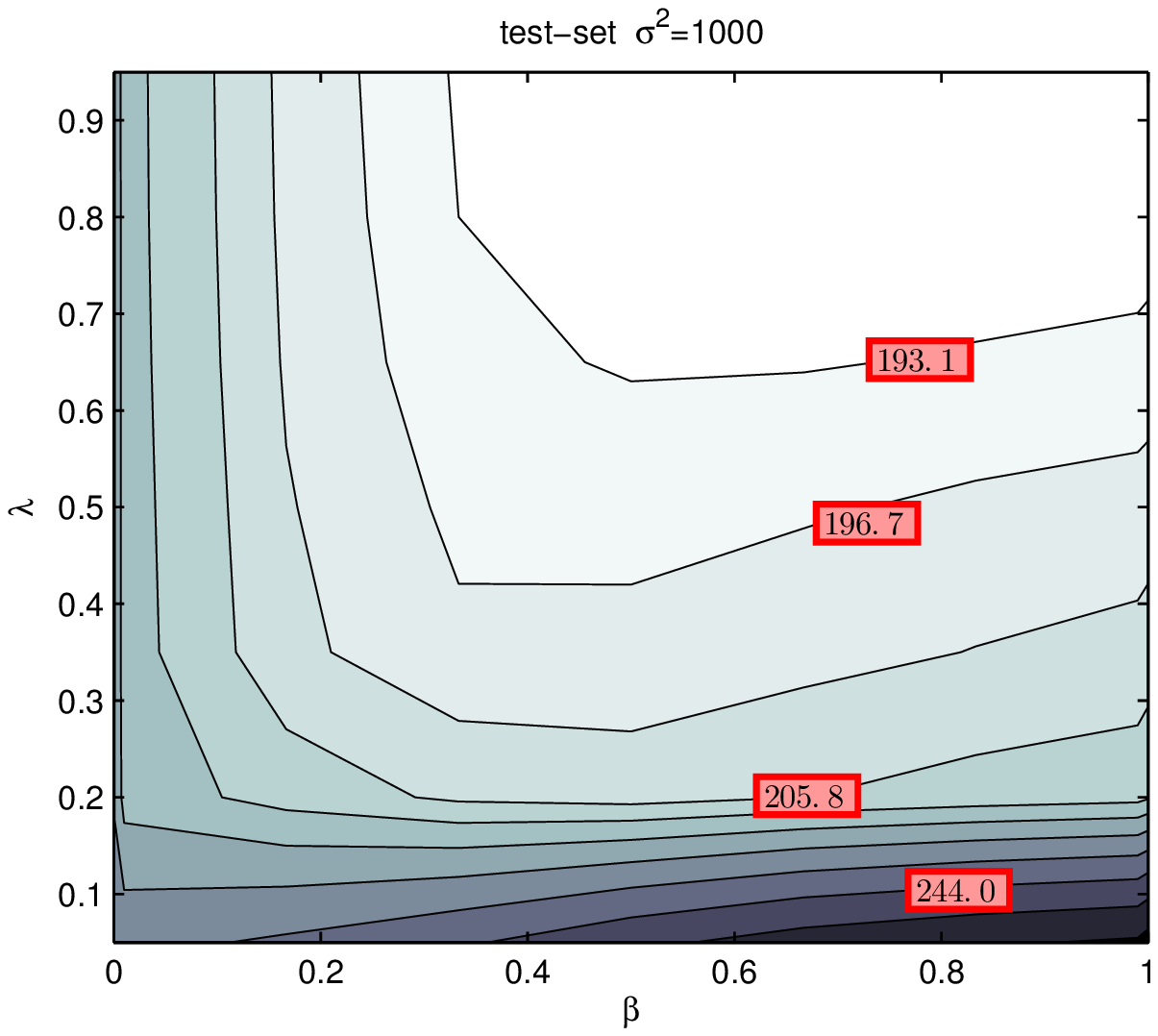}    &  \hspace{-5.5mm}
\includegraphics[trim=12.0mm 0mm 0mm 9.5mm,clip,totalheight=.175\textheight]{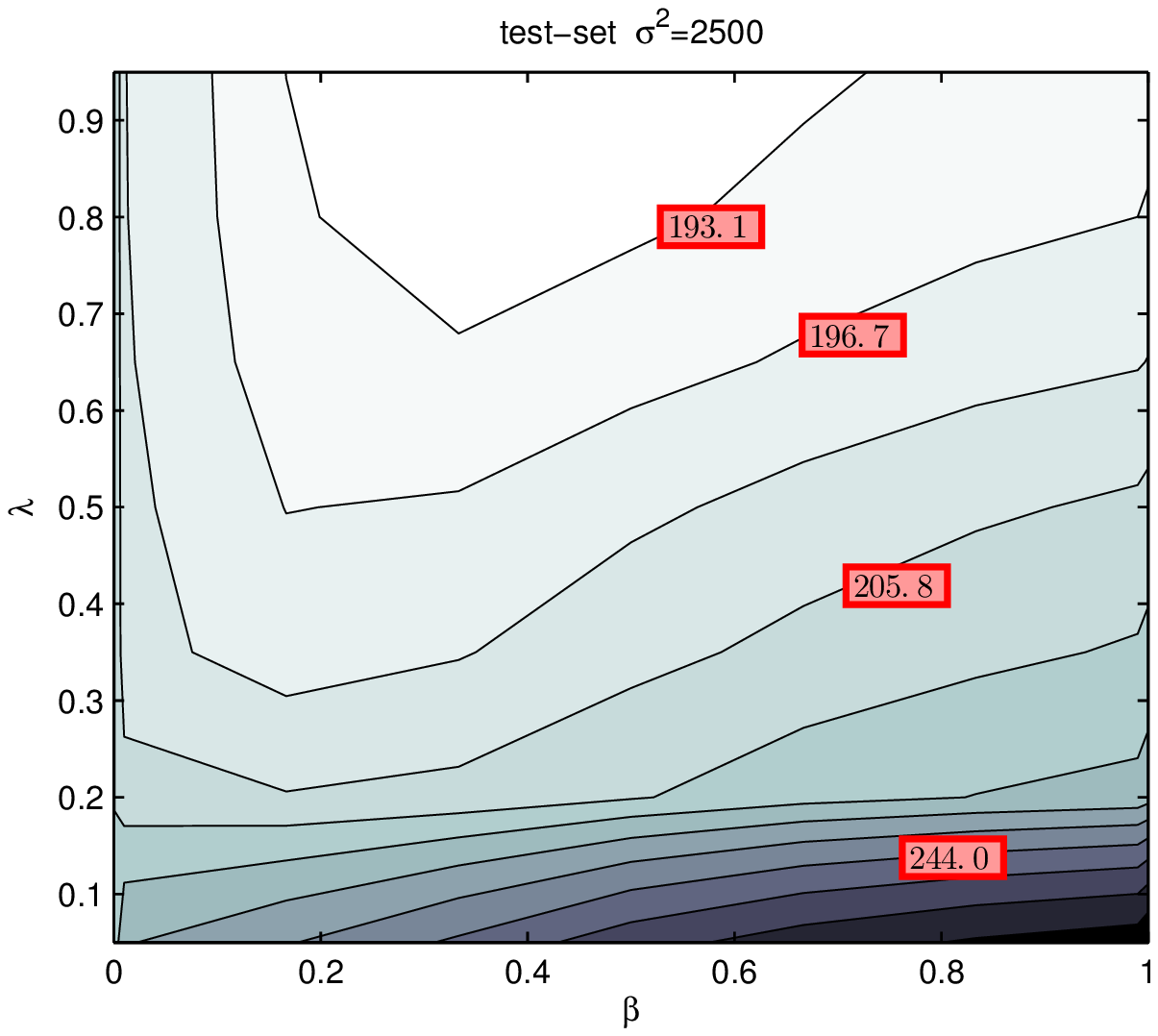}    &  \hspace{-5.5mm}
\includegraphics[trim=12.0mm 0mm 0mm 9.5mm,clip,totalheight=.175\textheight]{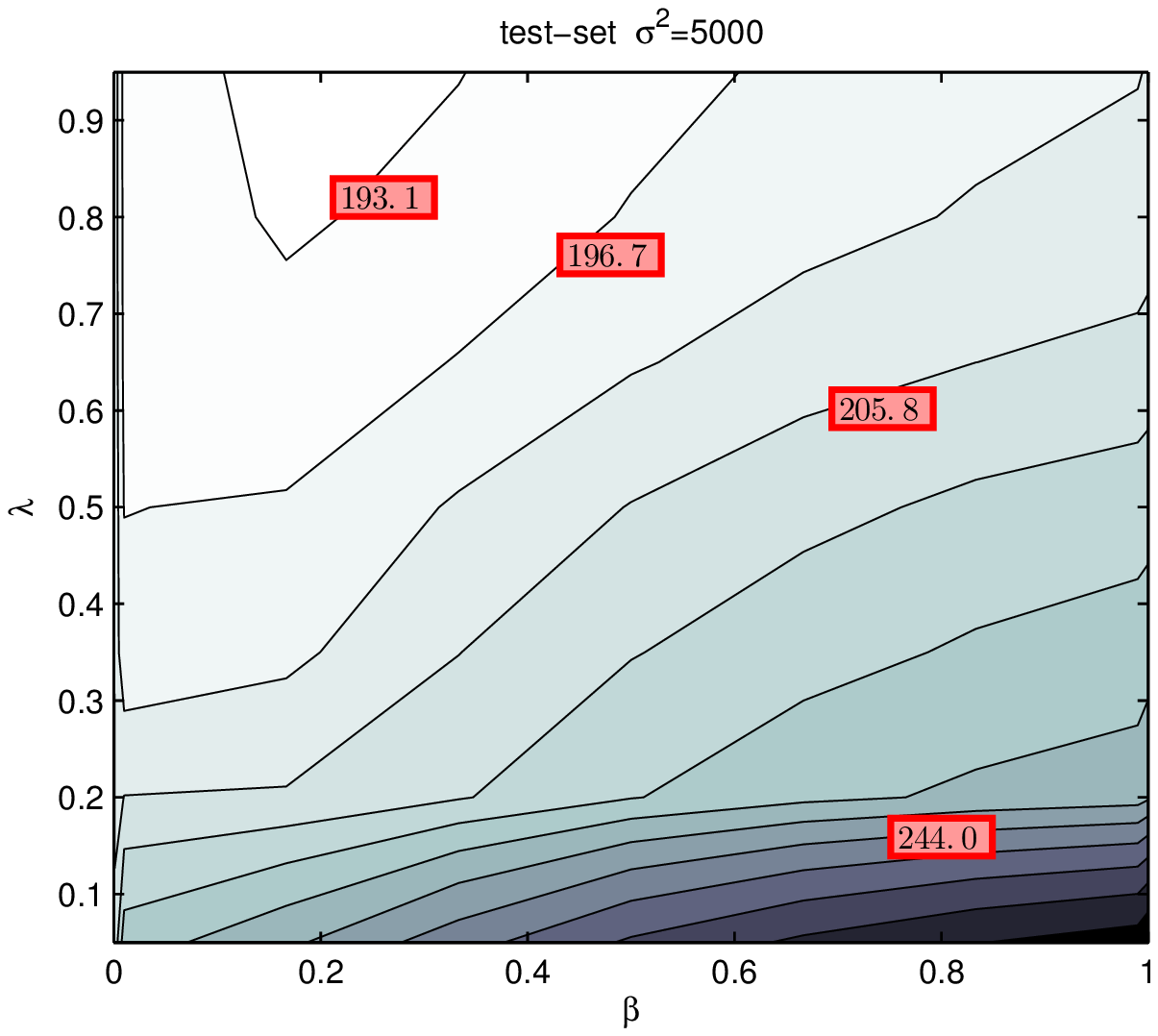}
\end{tabular}
\vspace{-1.5em}
\caption{
	Train set (top) and test set (bottom) negative log-likelihood (perplexity)
	for the Boltzmann chain MRF with pseudo/full likelihood selection policy
	(PL1/FL).  The x-axis, $\beta$, corresponds to relative weight placed on FL
	and and the y-axis, $\lambda$, corresponds to the probability of selecting
	FL.  PL1 is selected with probability 1 and weight $1-\beta$.  Contours and
	labels are fixed across columns.  Results averaged over several
	cross-validation folds, i.e., resampling both the train set and the PL1/FL
	policy.  Columns from left to right correspond to weaker regularization,
	$\sigma^2=\{500, 1000, 2500, 5000\}$.  The best achievable test set
	perplexity is about 190.
	\vspace{1em}
	\newline
	Unsurprisingly the test set perplexity dominates the train set perplexity
	at each $\sigma^2$ (column).  For a desired level of accuracy (contour)
	there exists a computationally favorable regularizer.  Hence
	$\hat{\theta}^{msl}_n$ acts as both a regularizer and mechanism for
	controlling accuracy and complexity.
}\label{fig:chunk_boltzchain_pl1fl_cont}

\hspace{-0.75cm}
\begin{tabular}{cccc}
\includegraphics[trim= 0.0mm 10.2mm 0mm 8.2mm,clip,totalheight=.160\textheight]{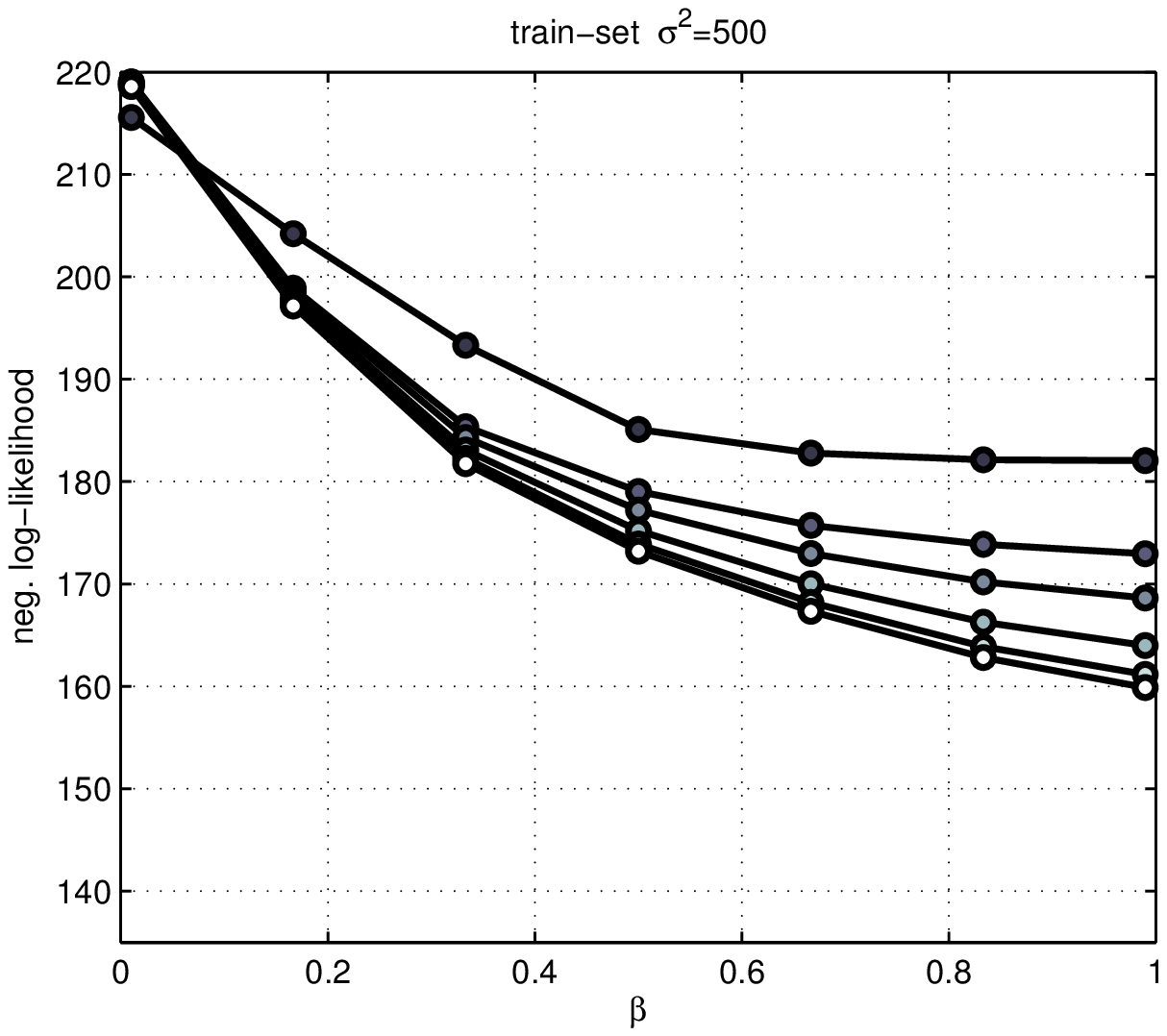}  &   \hspace{-5.5mm}
\includegraphics[trim=12.5mm 10.2mm 0mm 8.2mm,clip,totalheight=.160\textheight]{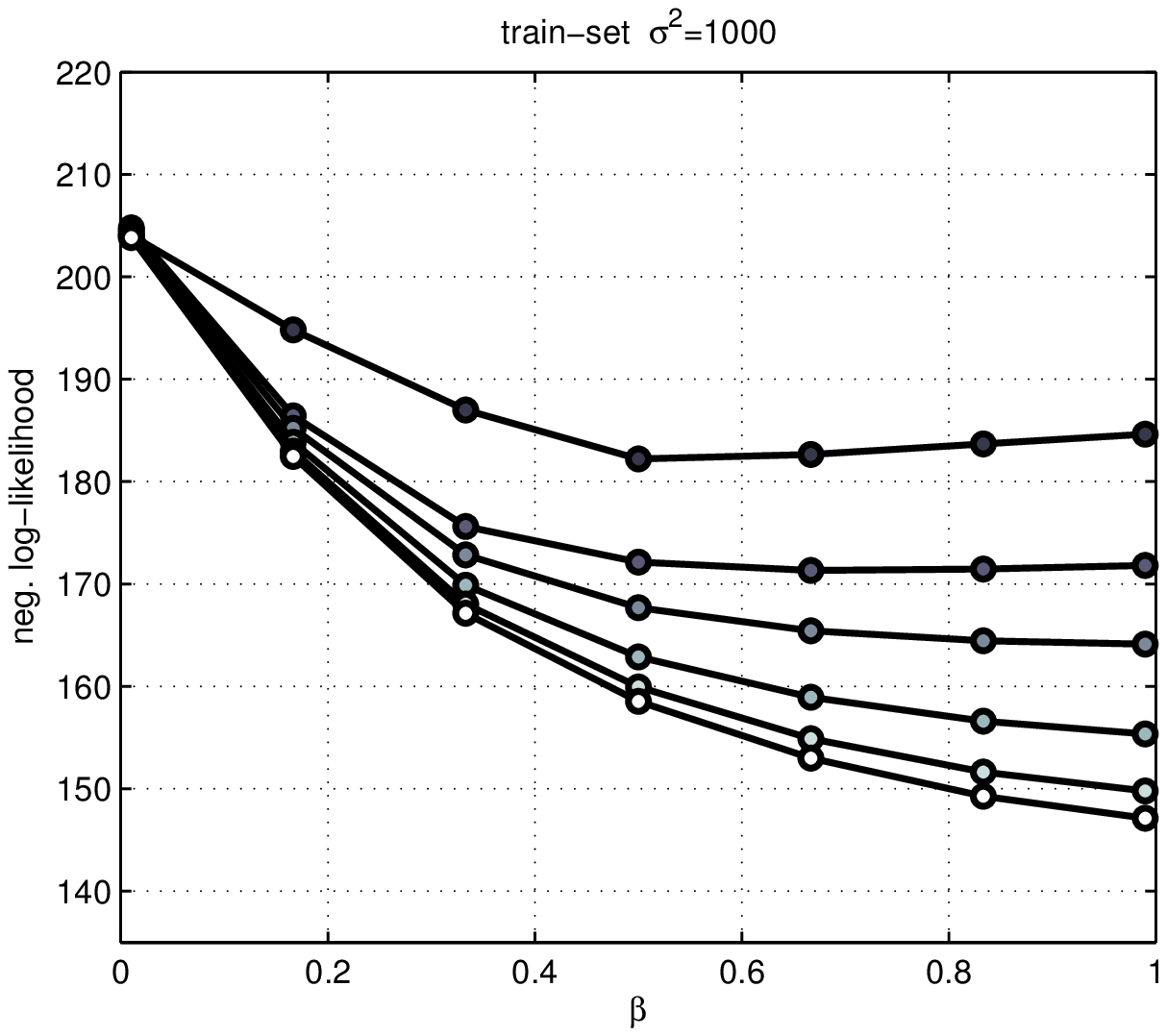} &   \hspace{-5.5mm}
\includegraphics[trim=12.5mm 10.2mm 0mm 8.2mm,clip,totalheight=.160\textheight]{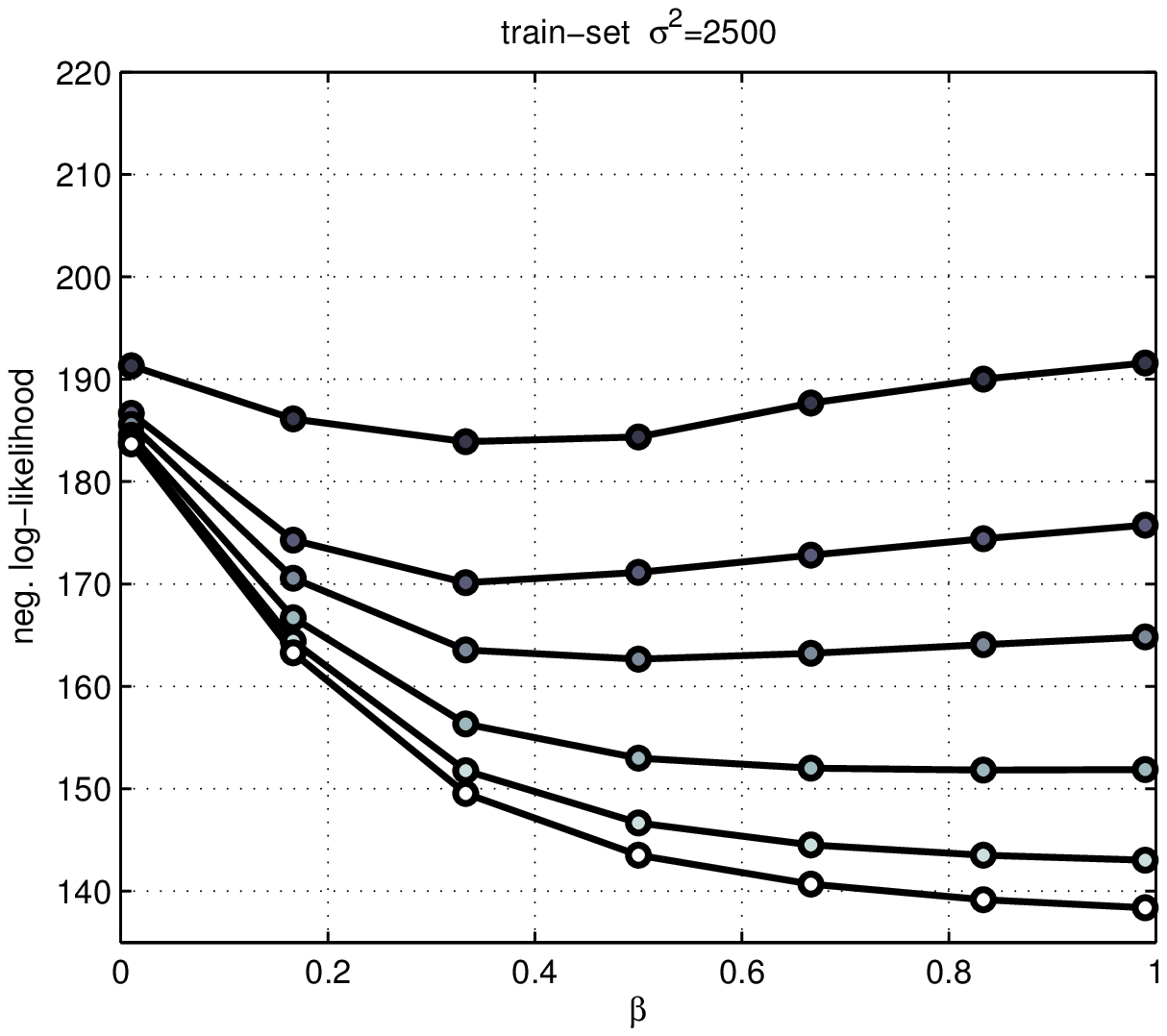} &   \hspace{-5.5mm}
\includegraphics[trim=12.5mm 10.2mm 0mm 8.2mm,clip,totalheight=.160\textheight]{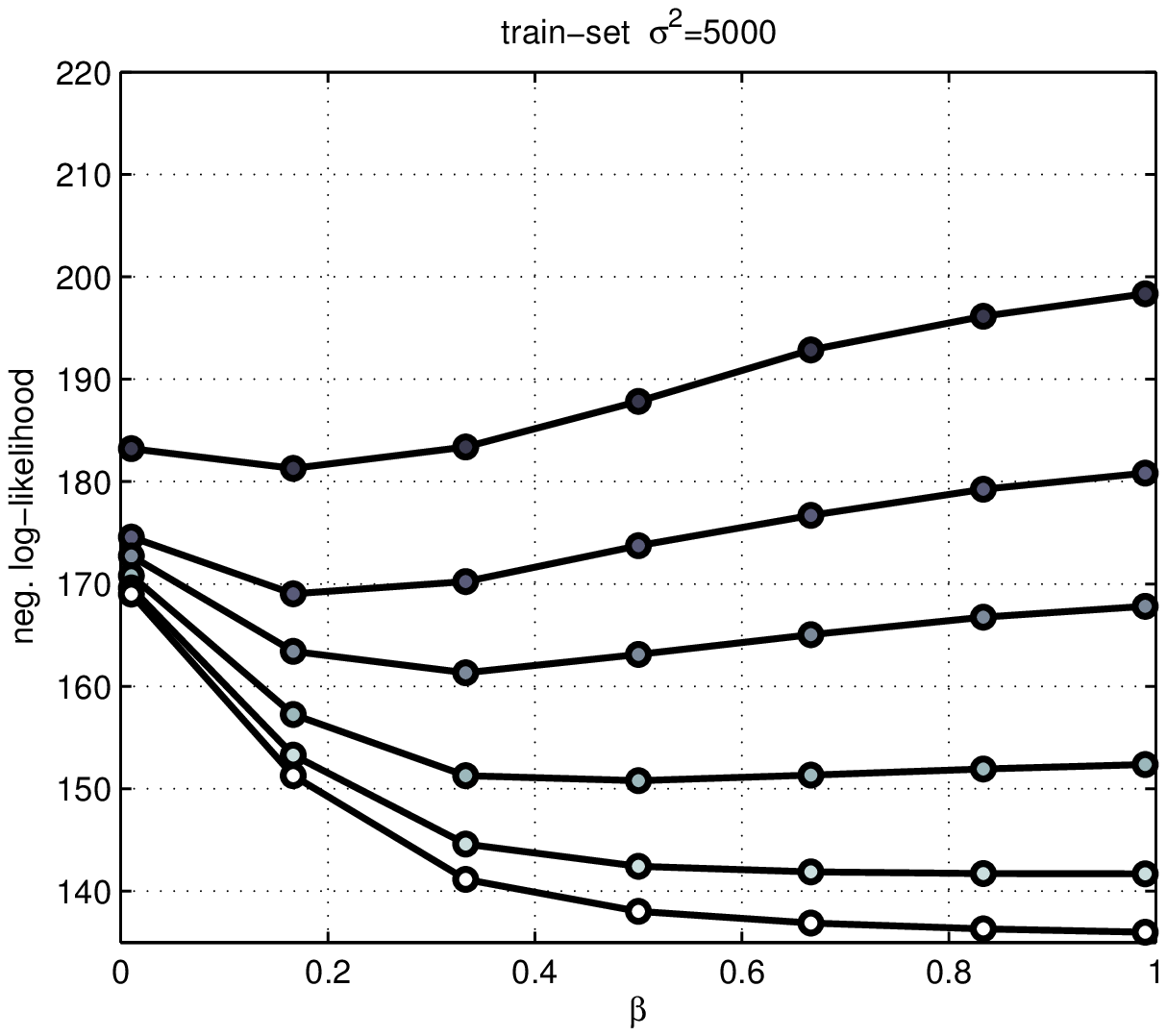} \\
\includegraphics[trim= 0.0mm 0mm 0mm 8.2mm,clip,totalheight=.178\textheight]{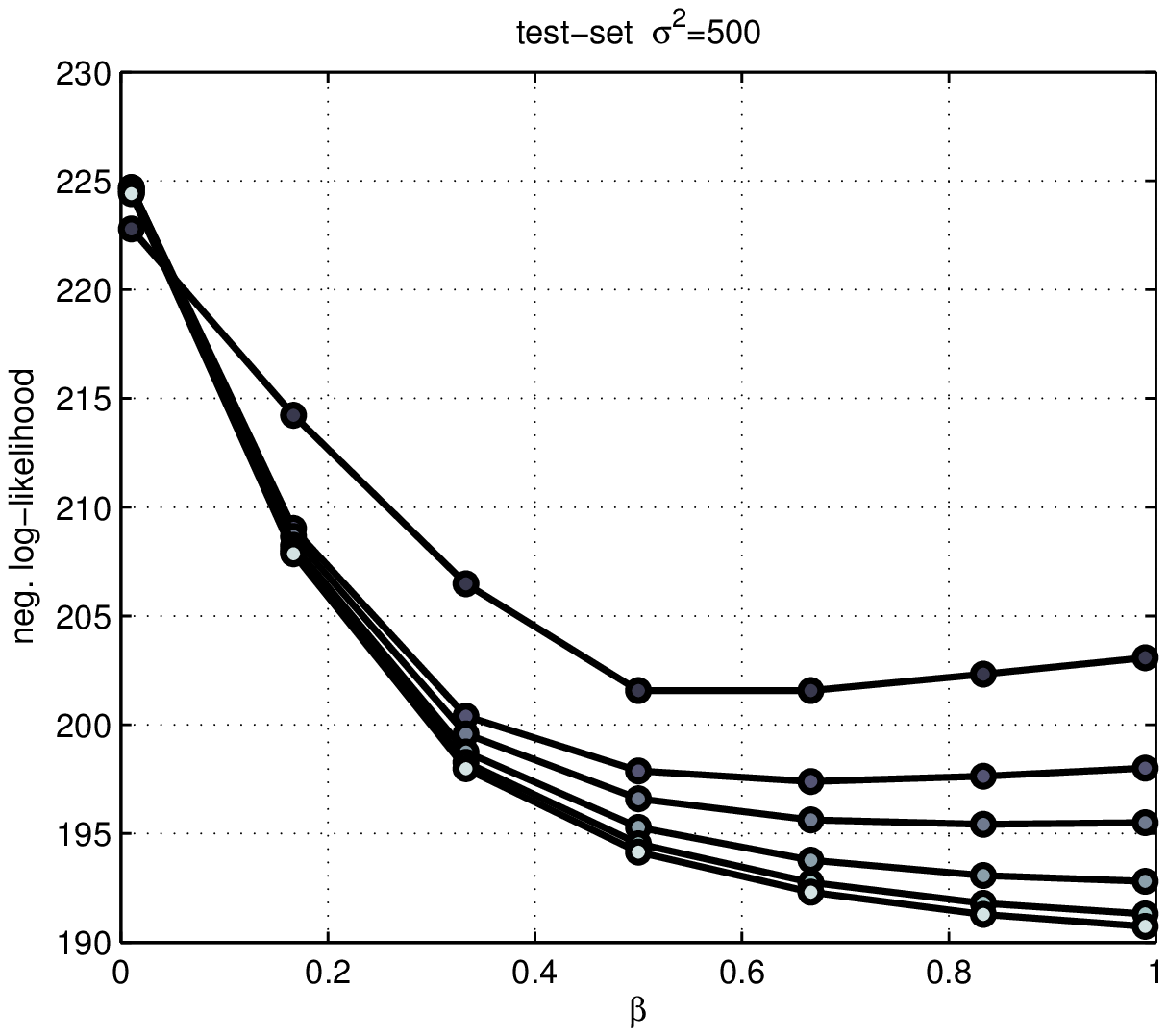}  &   \hspace{-5.5mm}
\includegraphics[trim=12.5mm 0mm 0mm 8.2mm,clip,totalheight=.178\textheight]{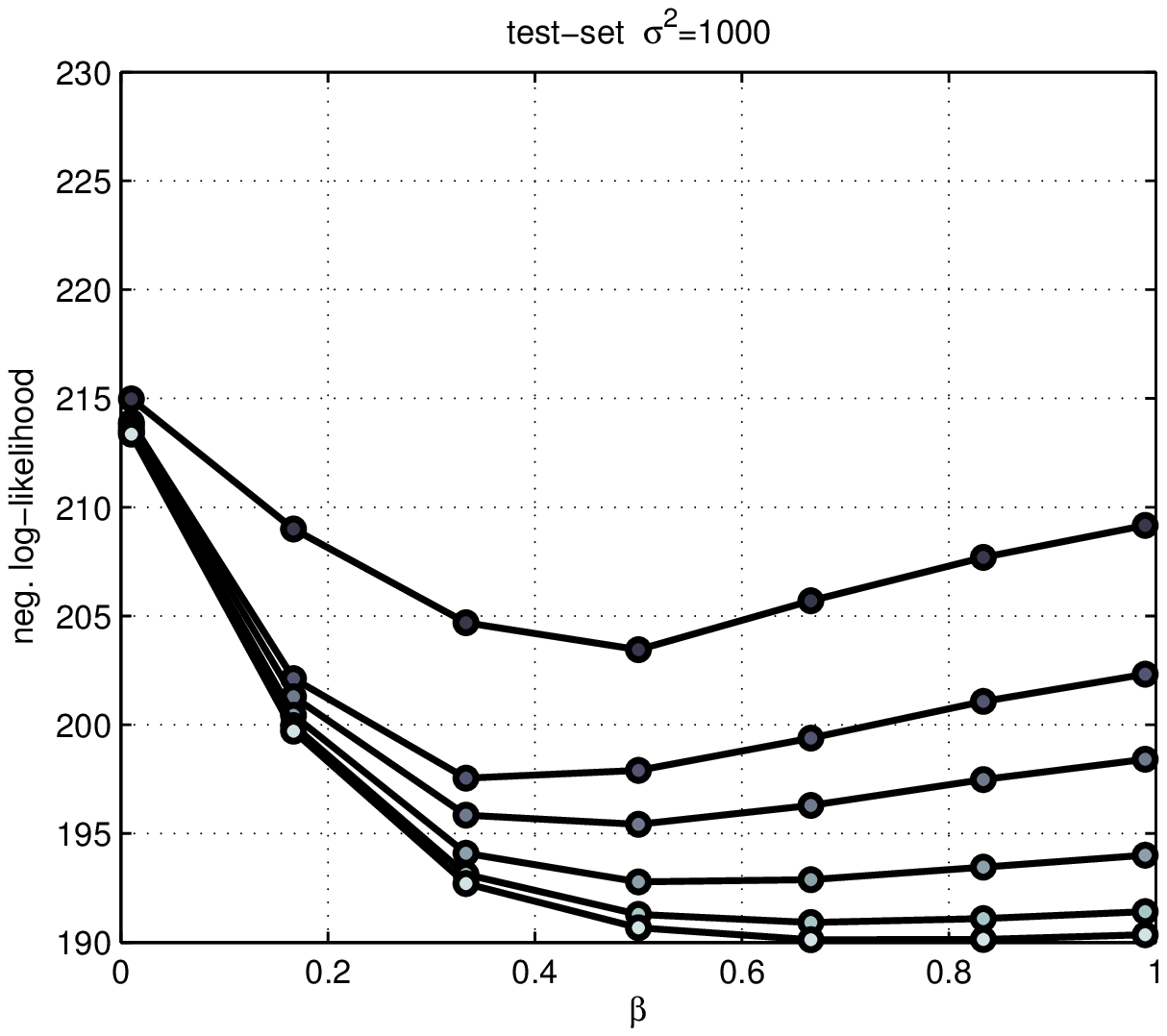} &   \hspace{-5.5mm}
\includegraphics[trim=12.5mm 0mm 0mm 8.2mm,clip,totalheight=.178\textheight]{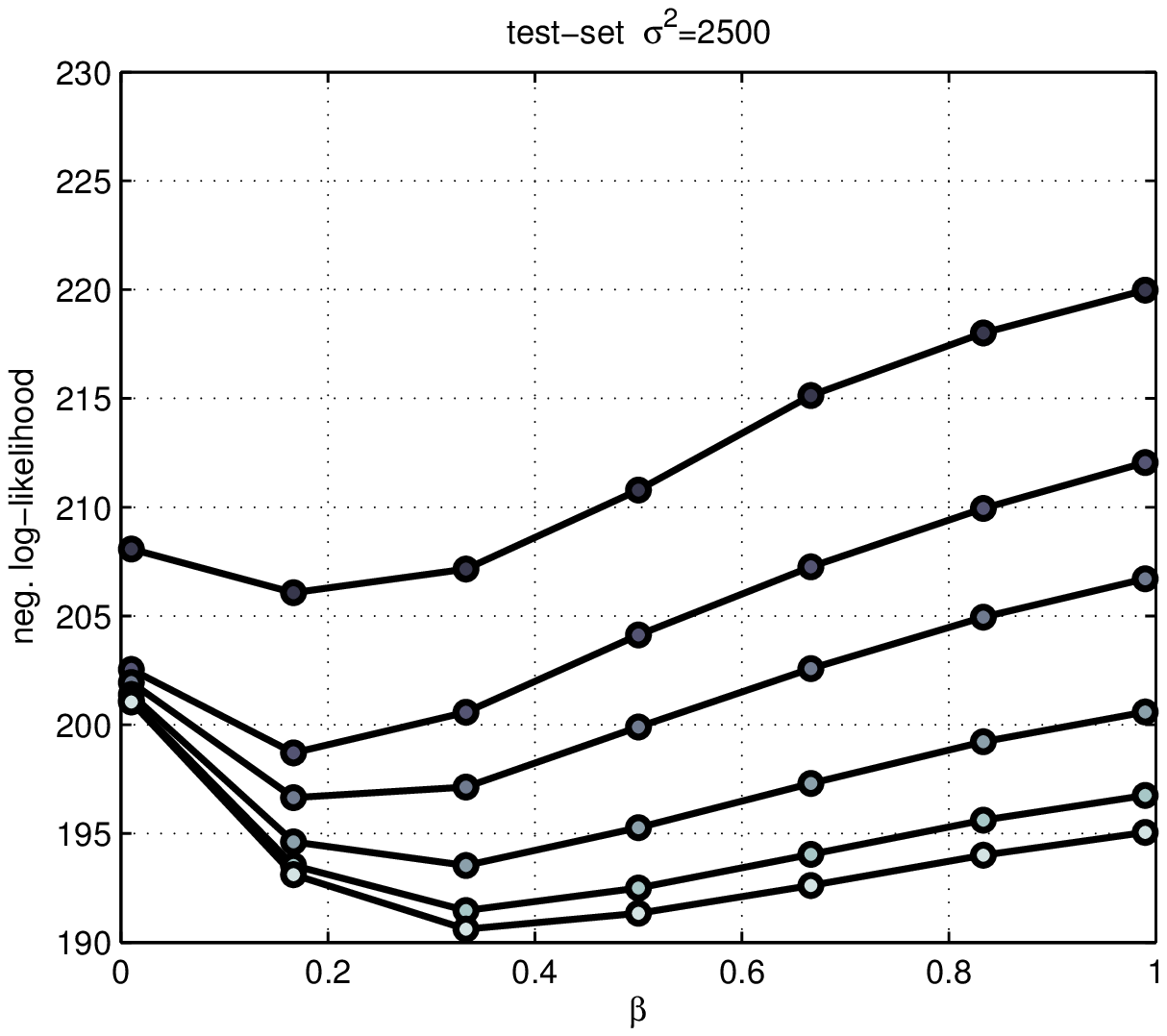} &   \hspace{-5.5mm}
\includegraphics[trim=12.5mm 0mm 0mm 8.2mm,clip,totalheight=.178\textheight]{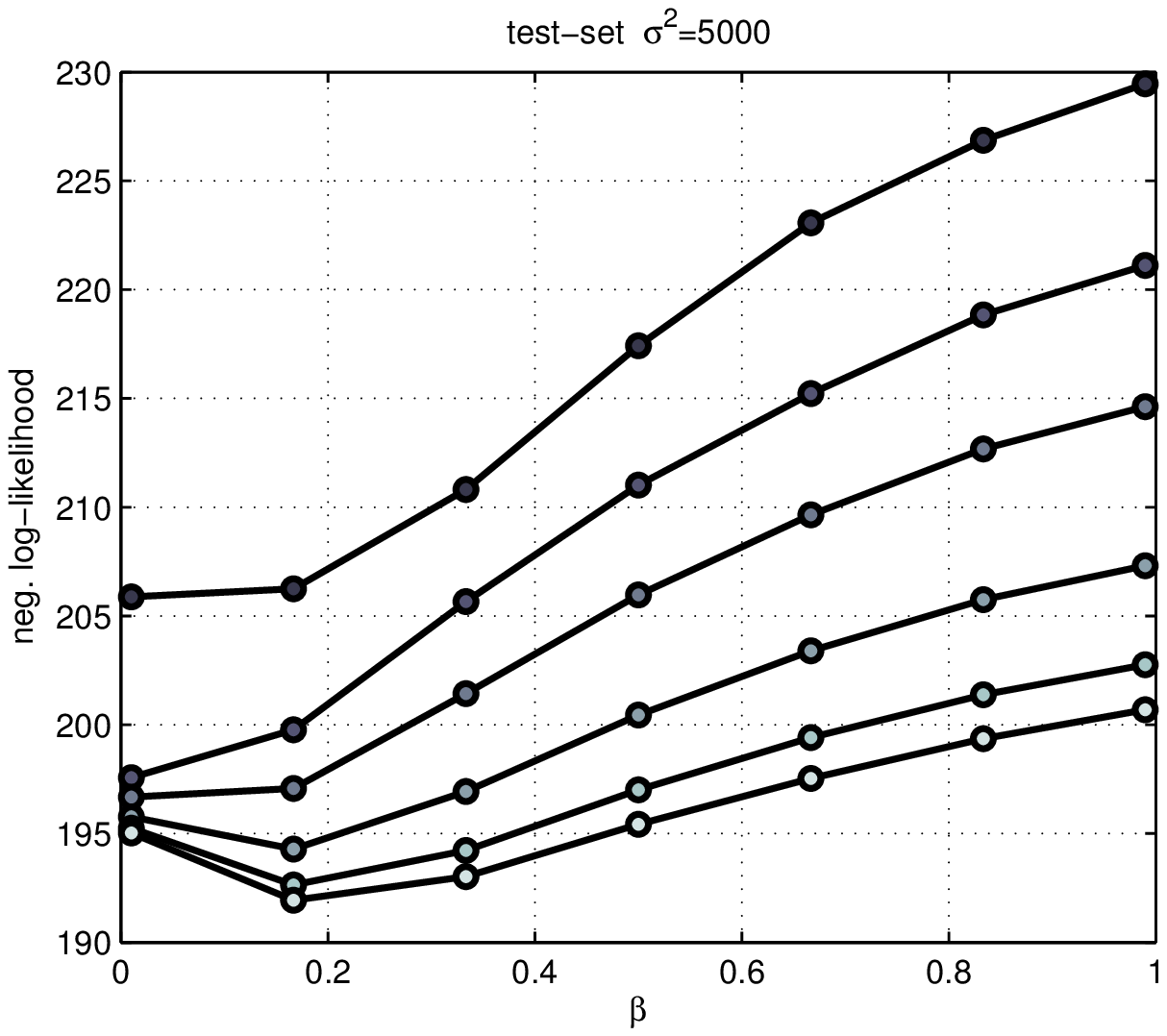} 
\end{tabular}
\vspace{-1.5em}
\caption{
	Train set and test set perplexities for the Boltzmann chain MRF with PL1/FL
	selection policy (see above layout description).  The x-axis is again
	$\beta$ and the y-axis perplexity.  Lighter shading indicates FL is
	selected with increasing frequency.  Note that as the regularizer is
	weakened the the range in perplexity spanned by $\lambda$ increases and the
	lower bound decreases. This indicates that the approximating power of
	$\hat{\theta}^{msl}_n$ increases when unencumbered by the regularizer and
	highlights its secondary role as a regularizer.
}\label{fig:chunk_boltzchain_pl1fl_beta}

\end{figure}

\begin{figure}[ht!]
\vspace{-3.0em}
\hspace{-0.75cm}
\begin{tabular}{cccc}
\includegraphics[trim= 0.0mm 10.2mm 0mm 9.5mm,clip,totalheight=.157\textheight]{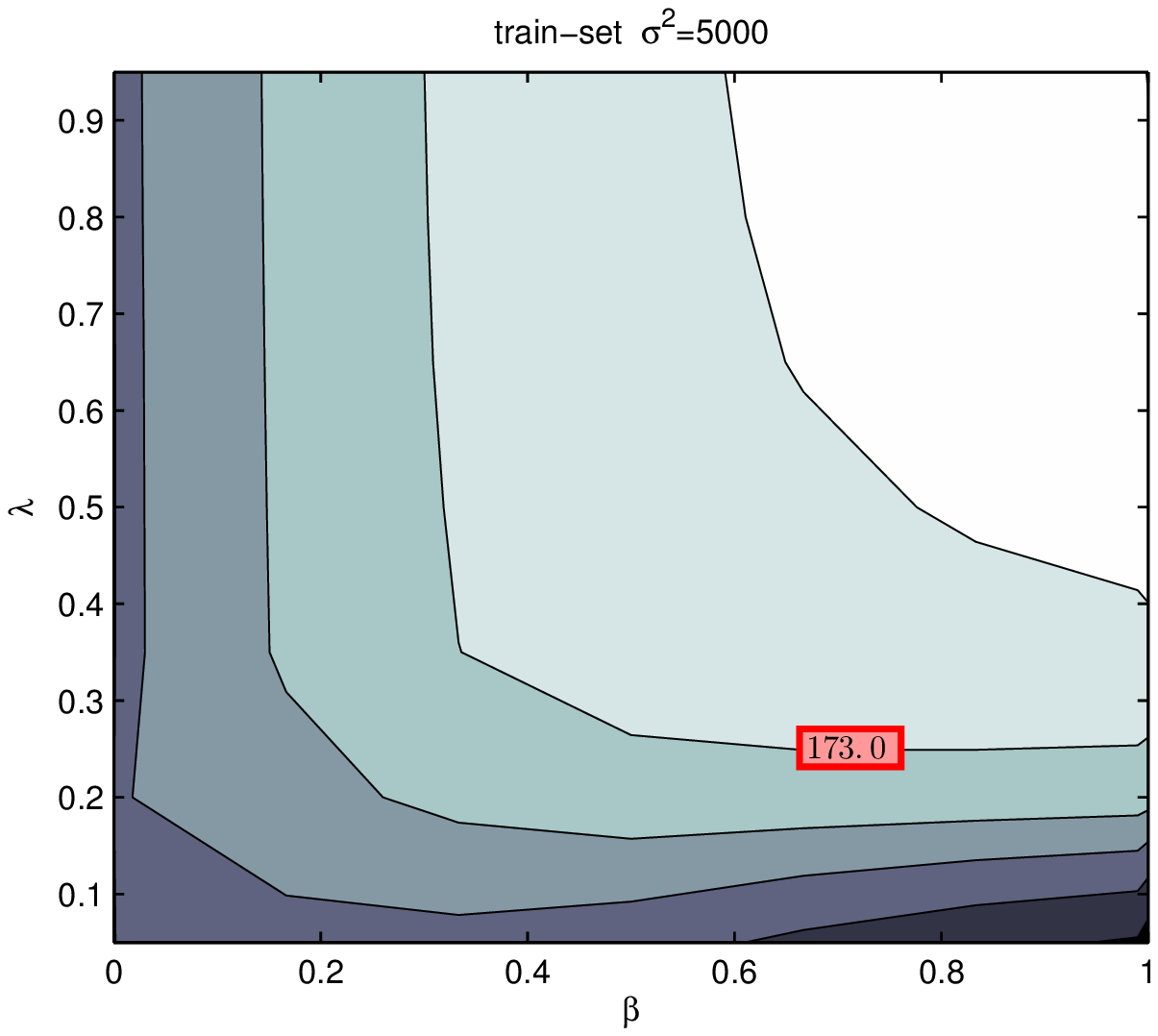}  &  \hspace{-5.5mm}
\includegraphics[trim=12.0mm 10.2mm 0mm 9.5mm,clip,totalheight=.157\textheight]{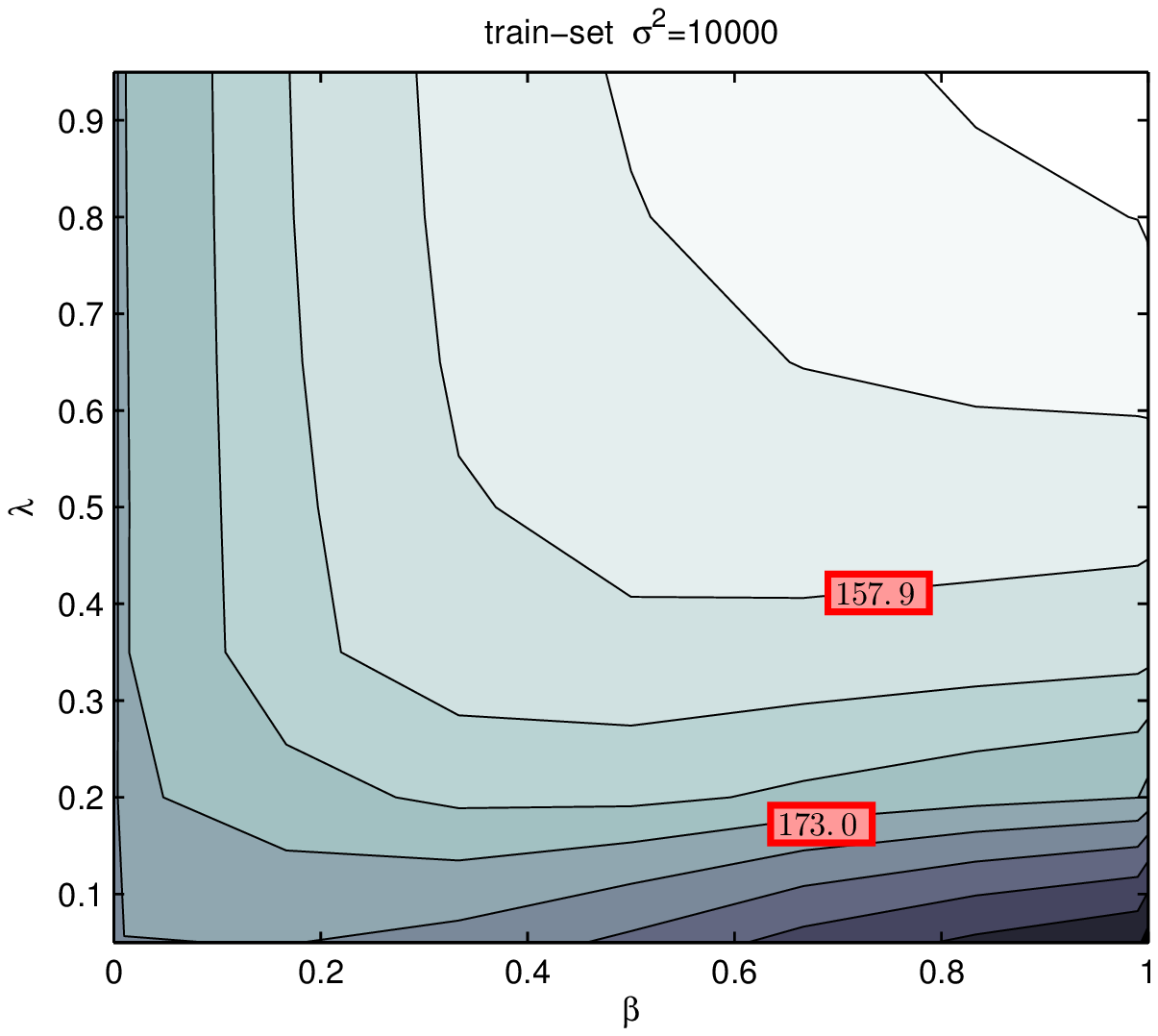} &  \hspace{-5.5mm}
\includegraphics[trim=12.0mm 10.2mm 0mm 9.5mm,clip,totalheight=.157\textheight]{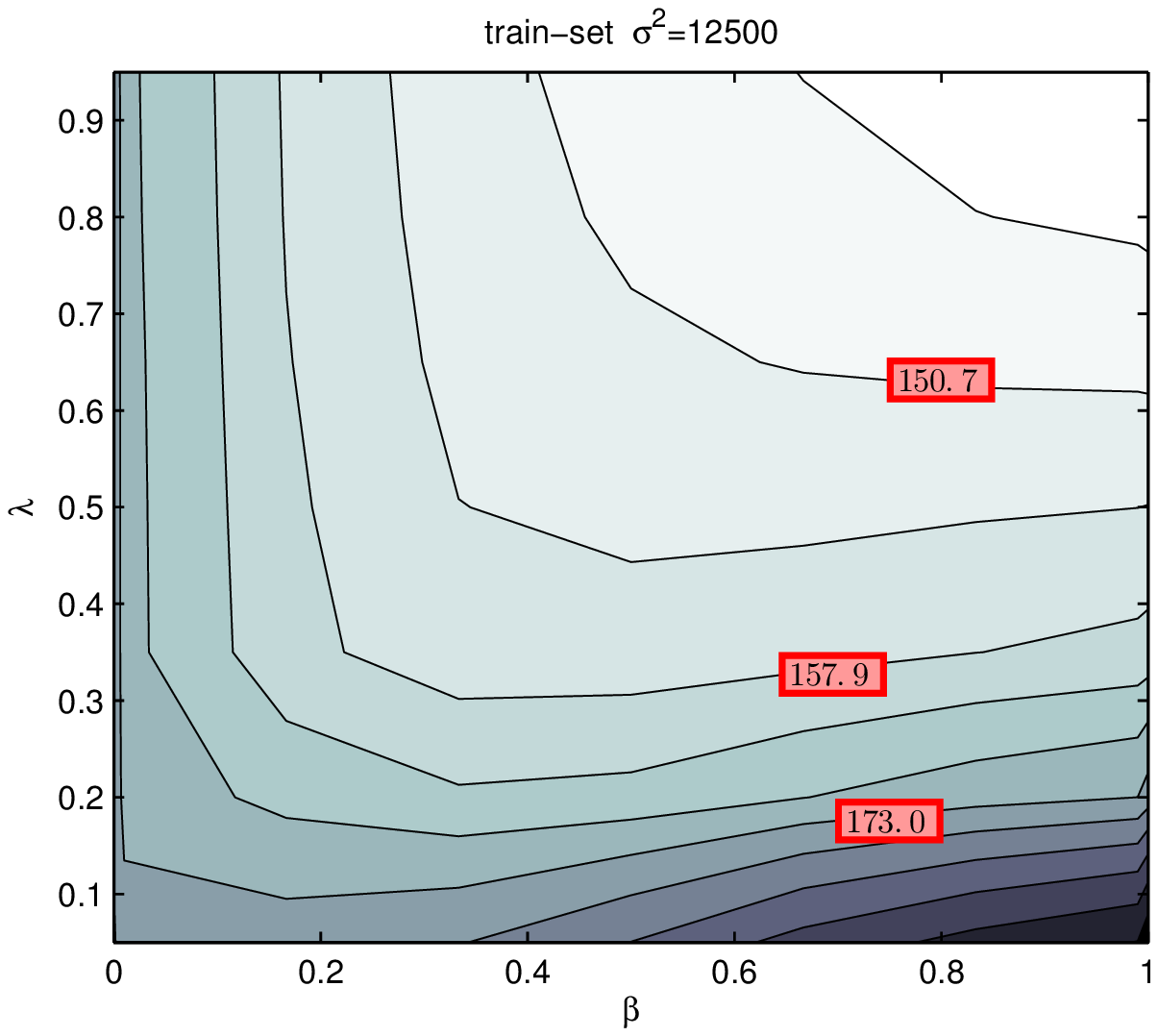} &  \hspace{-5.5mm}
\includegraphics[trim=12.0mm 10.2mm 0mm 9.5mm,clip,totalheight=.157\textheight]{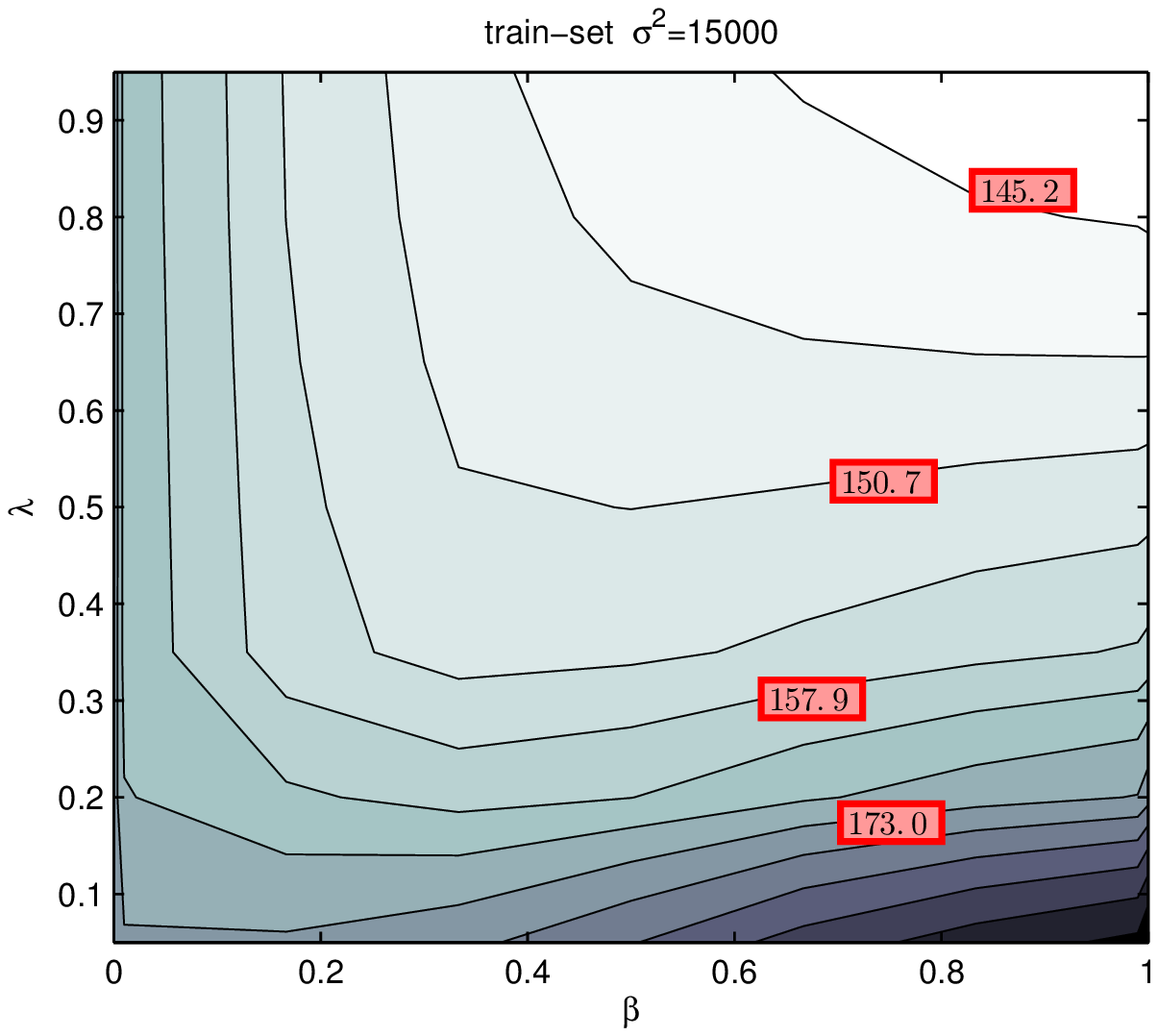} \\
\includegraphics[trim= 0.0mm 0mm 0mm 9.5mm,clip,totalheight=.175\textheight]{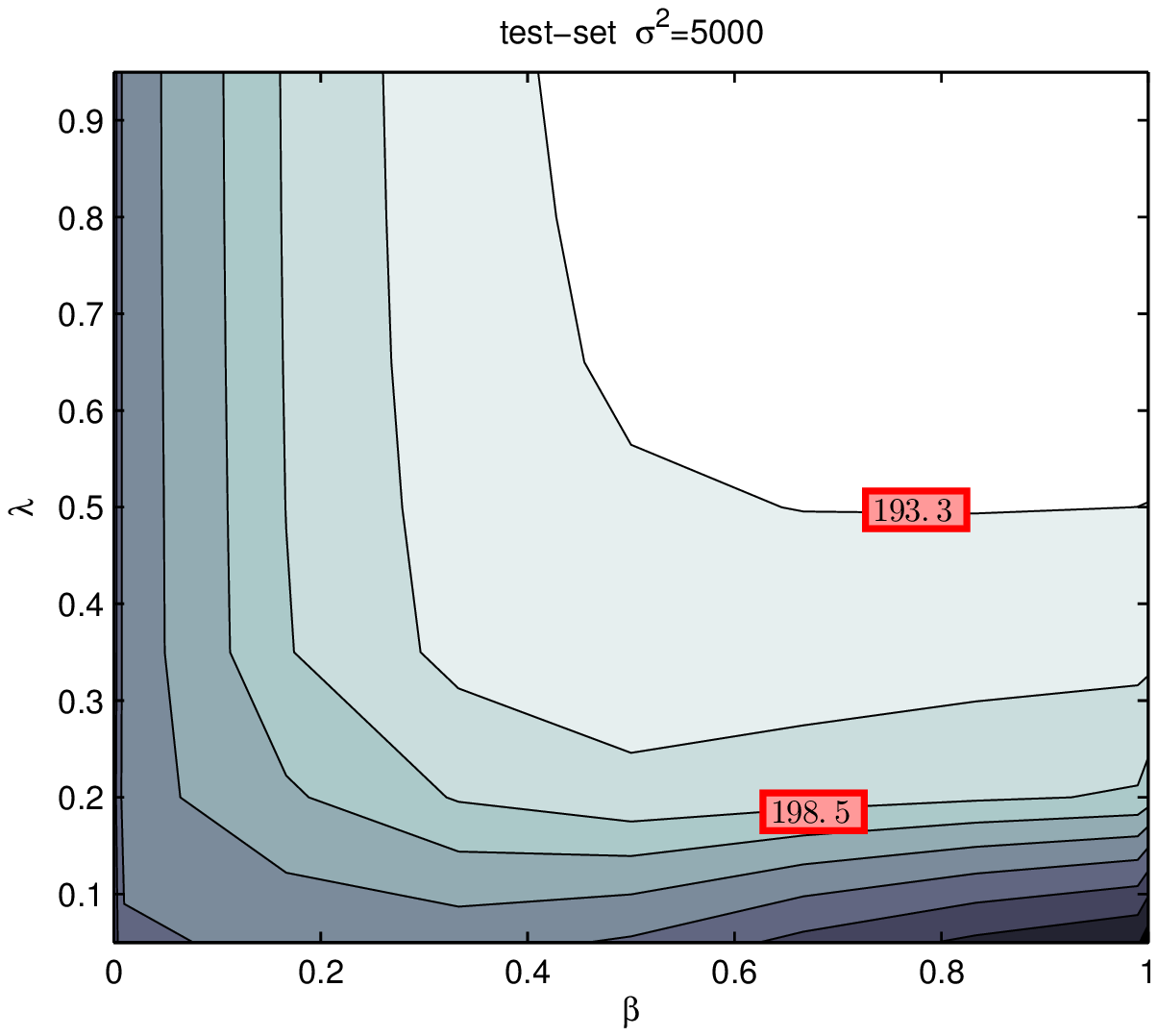}     &  \hspace{-5.5mm}
\includegraphics[trim=12.0mm 0mm 0mm 9.5mm,clip,totalheight=.175\textheight]{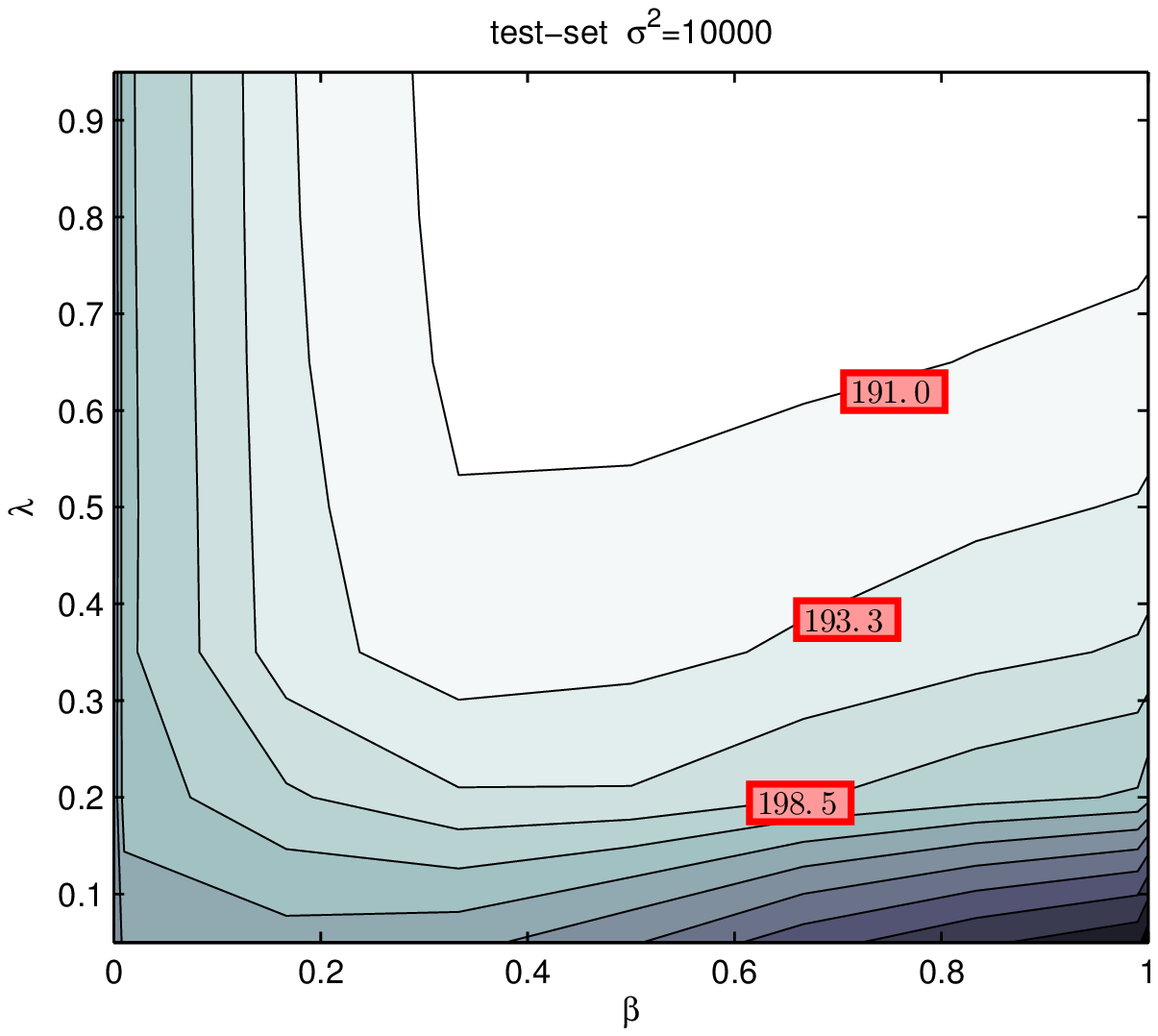}    &  \hspace{-5.5mm}
\includegraphics[trim=12.0mm 0mm 0mm 9.5mm,clip,totalheight=.175\textheight]{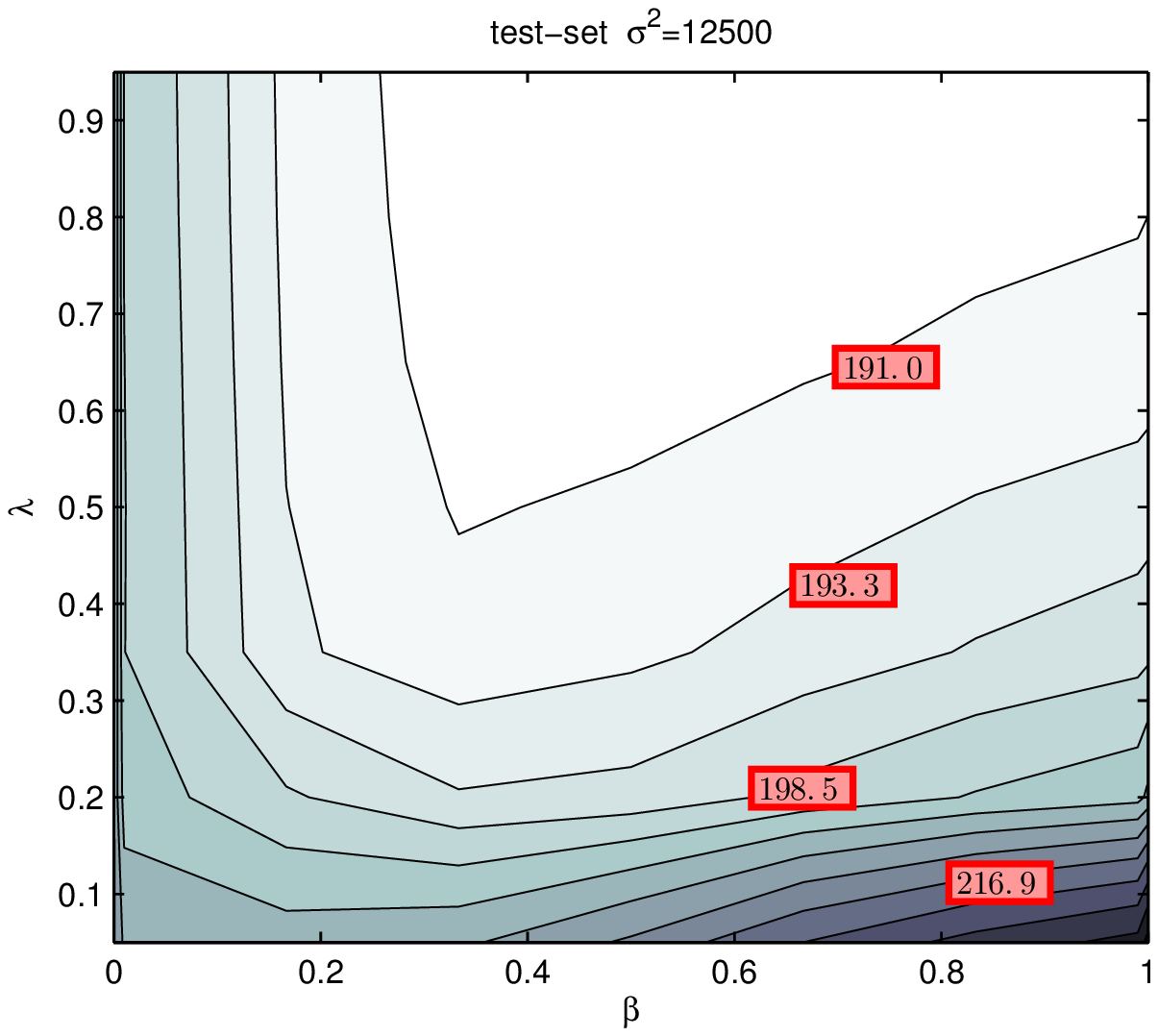}    &  \hspace{-5.5mm}
\includegraphics[trim=12.0mm 0mm 0mm 9.5mm,clip,totalheight=.175\textheight]{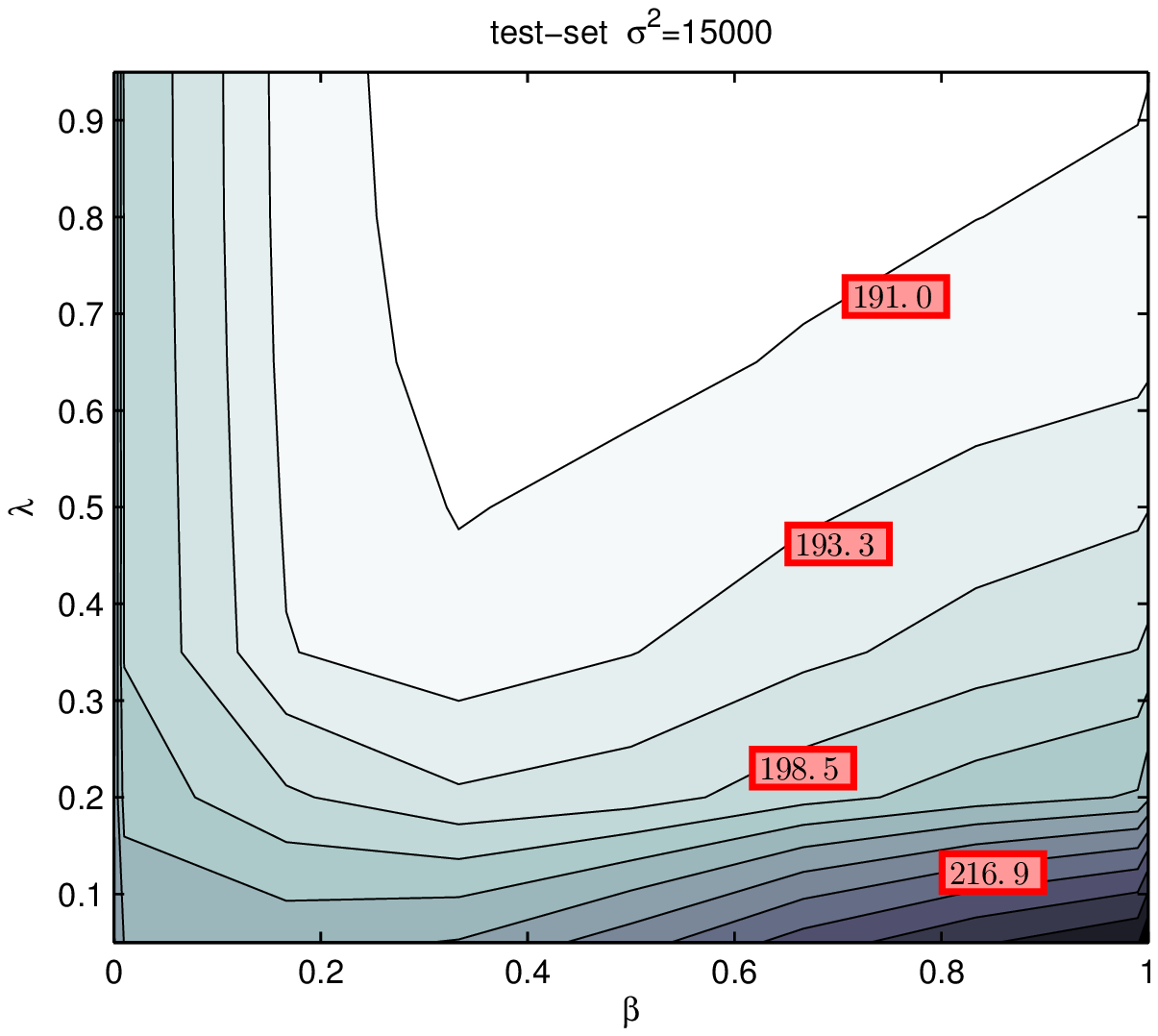}
\end{tabular}
\vspace{-1.5em}
\caption{
	Train set (top) and test set (bottom) perplexity for the Boltzmann chain
	MRF with 1st/2nd order pseudo likelihood selection policy (PL1/PL2).
	The x-axis corresponds to PL2 weight and the y-axis the probability of its
	selection.  PL1 is selected with probability 1 and weight $1-\beta$.
	Columns from left to right correspond to $\sigma^2=\{5000, 10000, 12500,
	15000\}$.  See Figure~\ref{fig:chunk_boltzchain_pl1fl_cont} for more
	details.  The best achievable test set perplexity is about 189.5.
	\vspace{1em}
	\newline
	In comparing these results to PL1/FL, we note that the test set contours
	exhibit less perplexity for larger areas.  In particular, perplexity is
	lower at smaller $\lambda$ values, meaning a computational saving over
	PL1/FL at a given level of accuracy.
}\label{fig:chunk_boltzchain_pl1pl2_cont}

\vspace{0.5em}
\hspace{-0.75cm}
\begin{tabular}{cccc}
\includegraphics[trim= 0.0mm 10.2mm 0mm 8.2mm,clip,totalheight=.160\textheight]{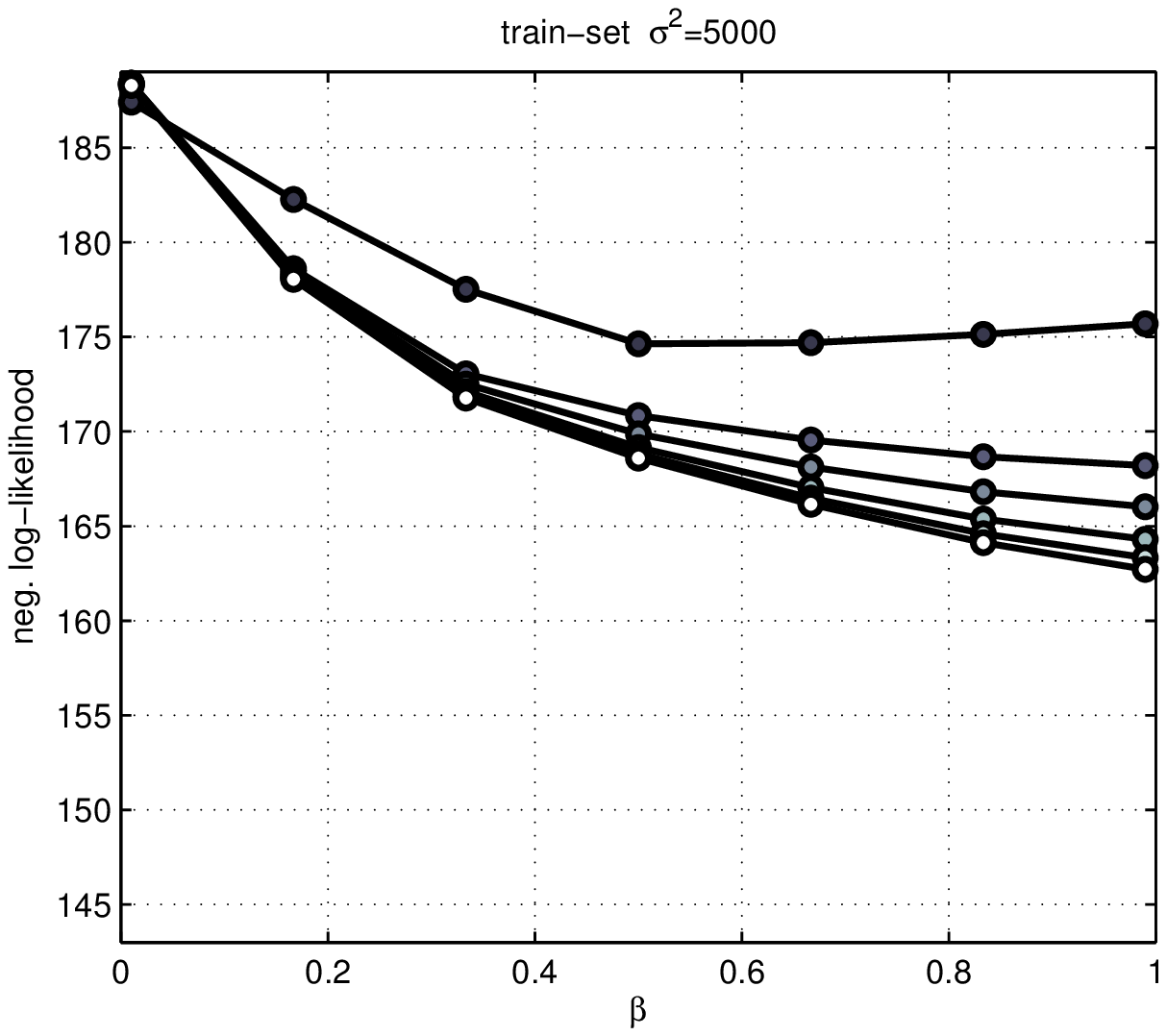}  &   \hspace{-5.5mm}
\includegraphics[trim=12.5mm 10.2mm 0mm 8.2mm,clip,totalheight=.160\textheight]{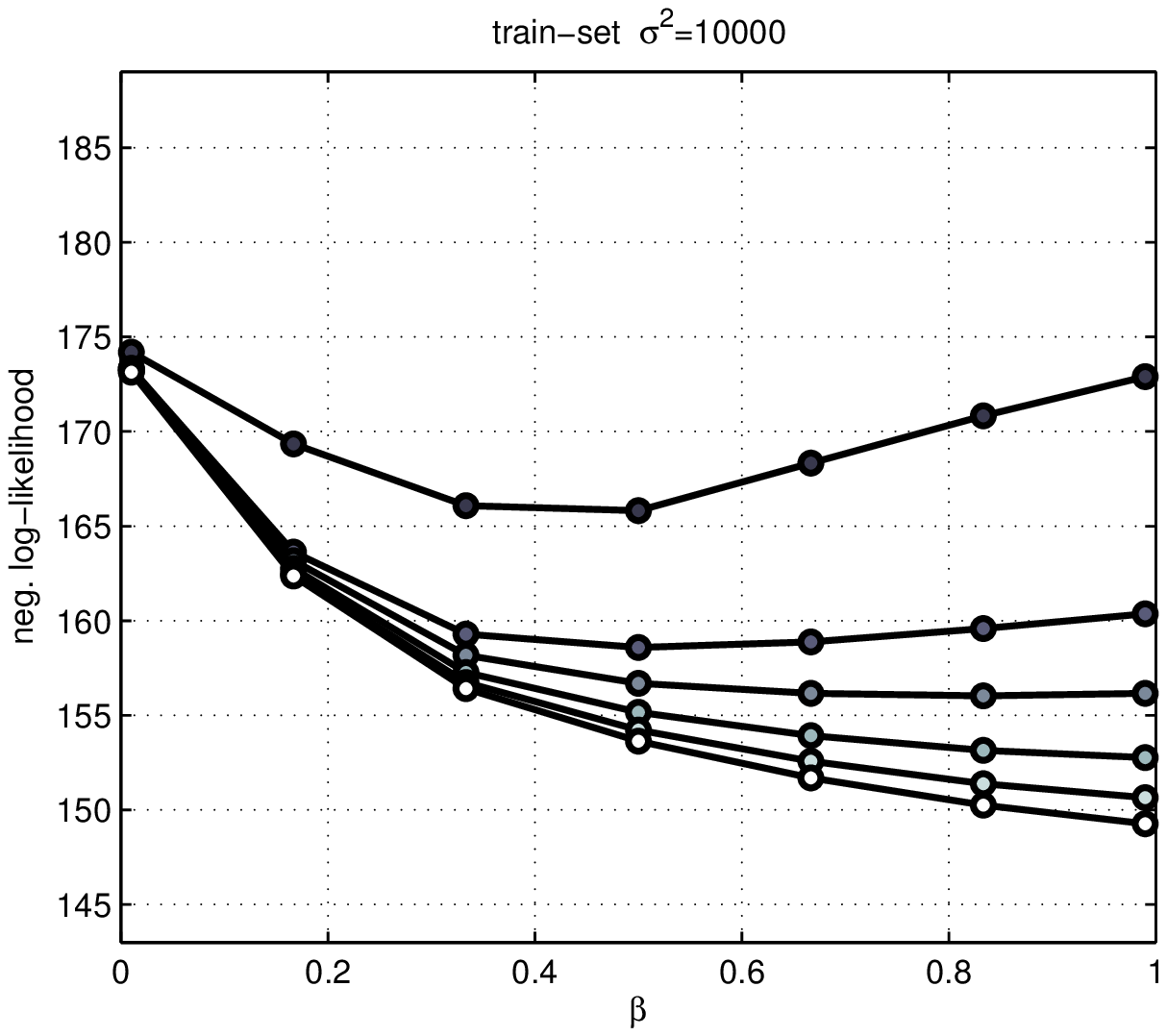} &   \hspace{-5.5mm}
\includegraphics[trim=12.5mm 10.2mm 0mm 8.2mm,clip,totalheight=.160\textheight]{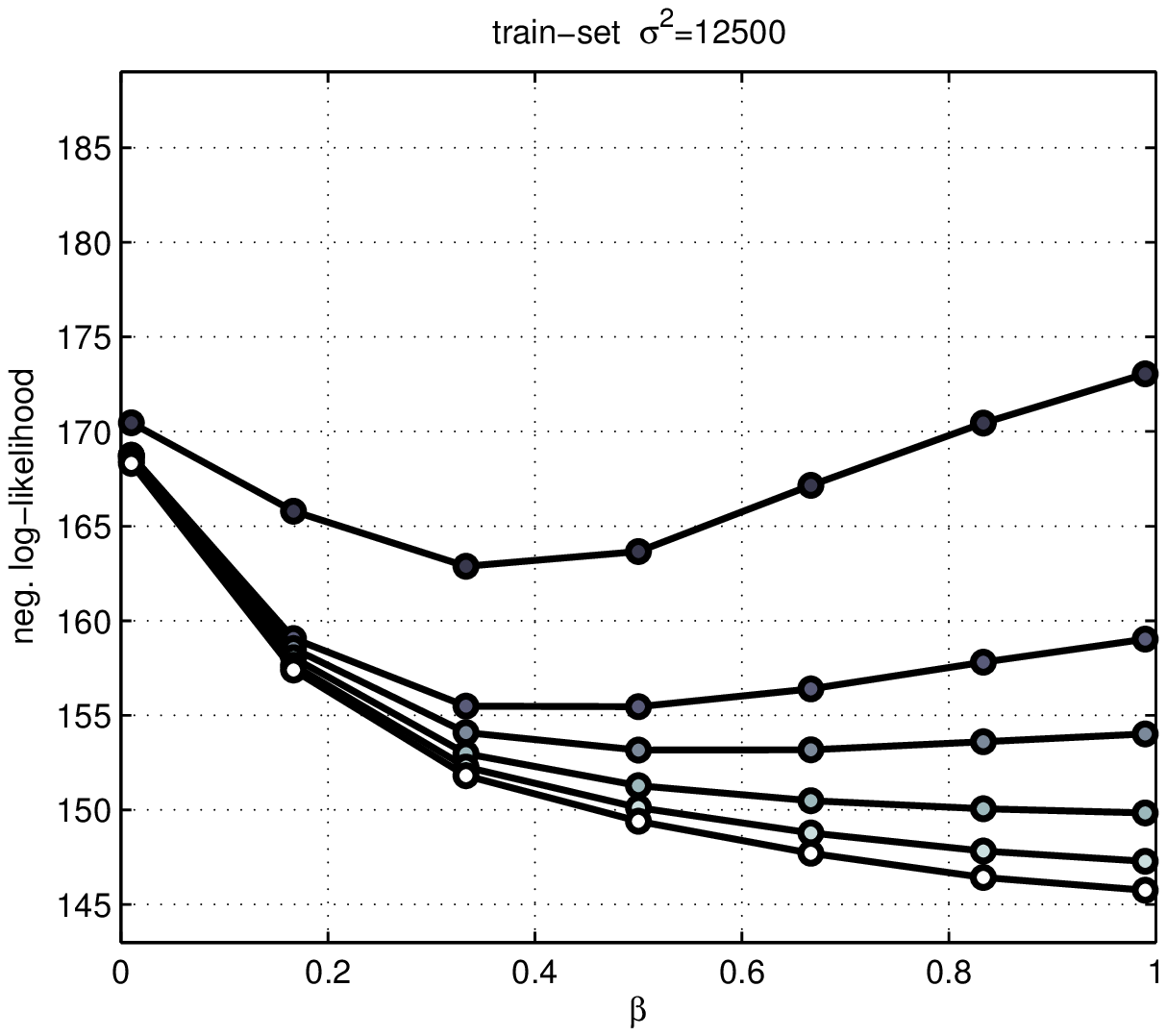} &   \hspace{-5.5mm}
\includegraphics[trim=12.5mm 10.2mm 0mm 8.2mm,clip,totalheight=.160\textheight]{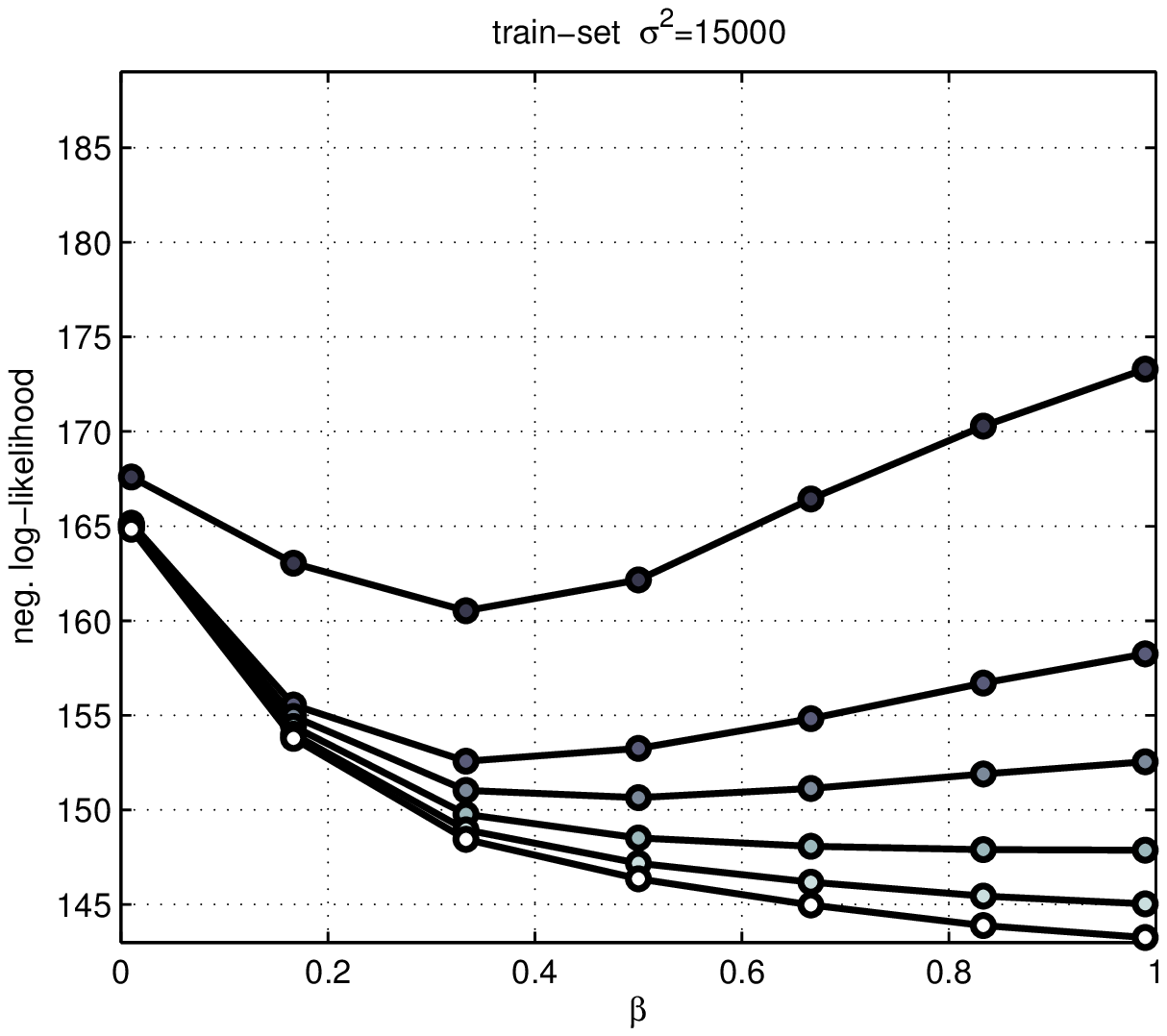} \\
\includegraphics[trim= 0.0mm 0mm 0mm 8.2mm,clip,totalheight=.178\textheight]{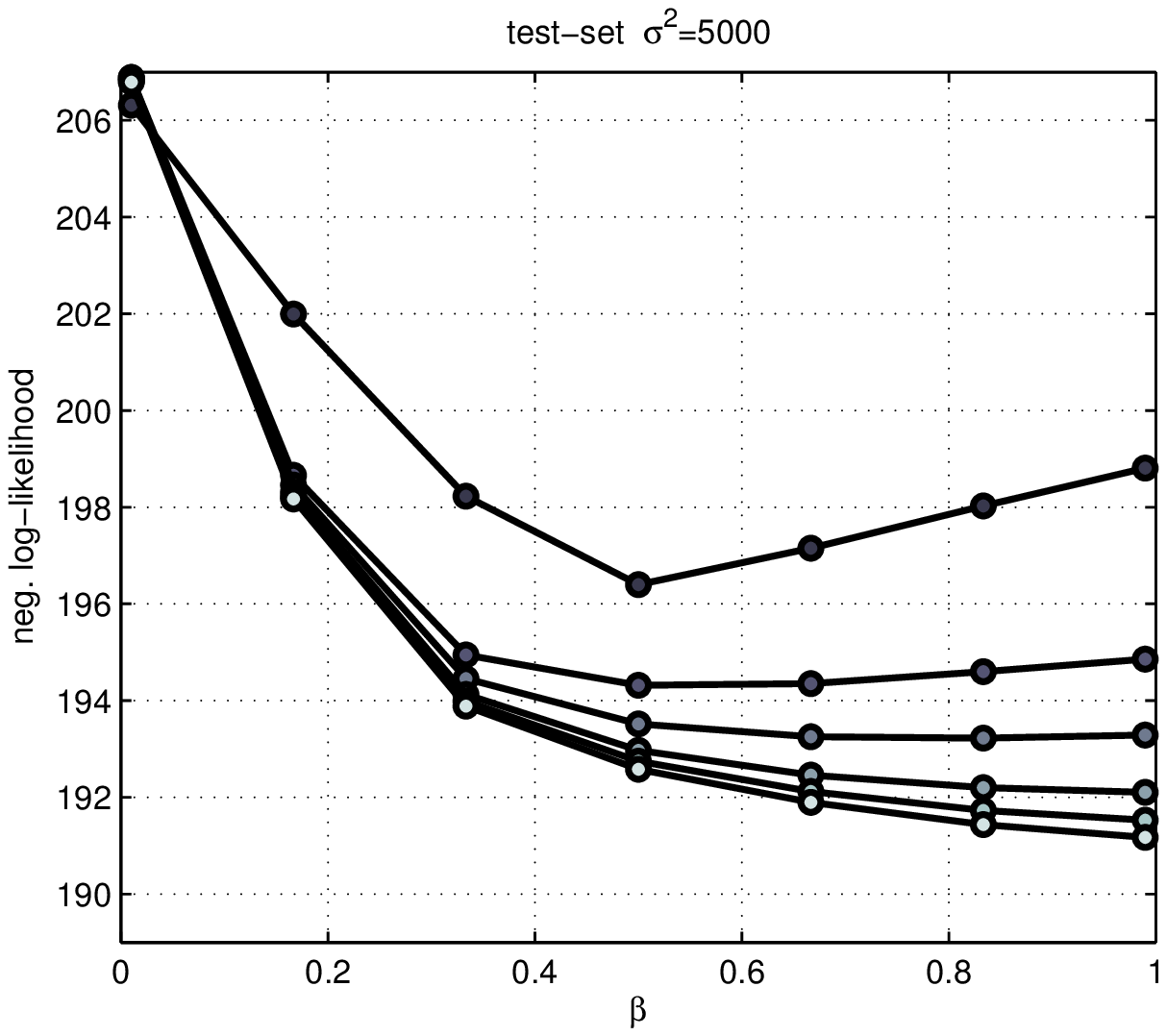}     &   \hspace{-5.5mm}
\includegraphics[trim=12.5mm 0mm 0mm 8.2mm,clip,totalheight=.178\textheight]{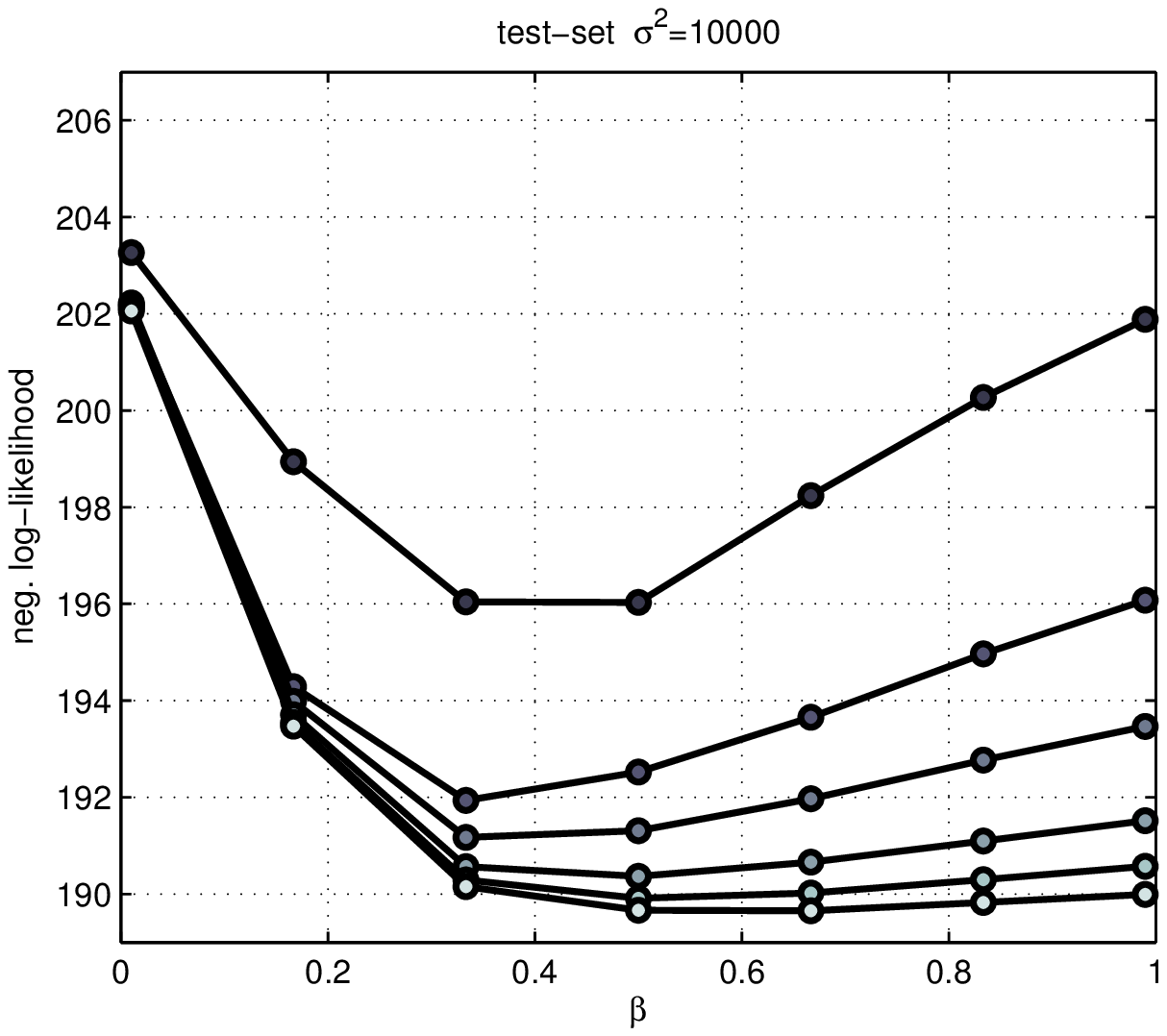}    &   \hspace{-5.5mm}
\includegraphics[trim=12.5mm 0mm 0mm 8.2mm,clip,totalheight=.178\textheight]{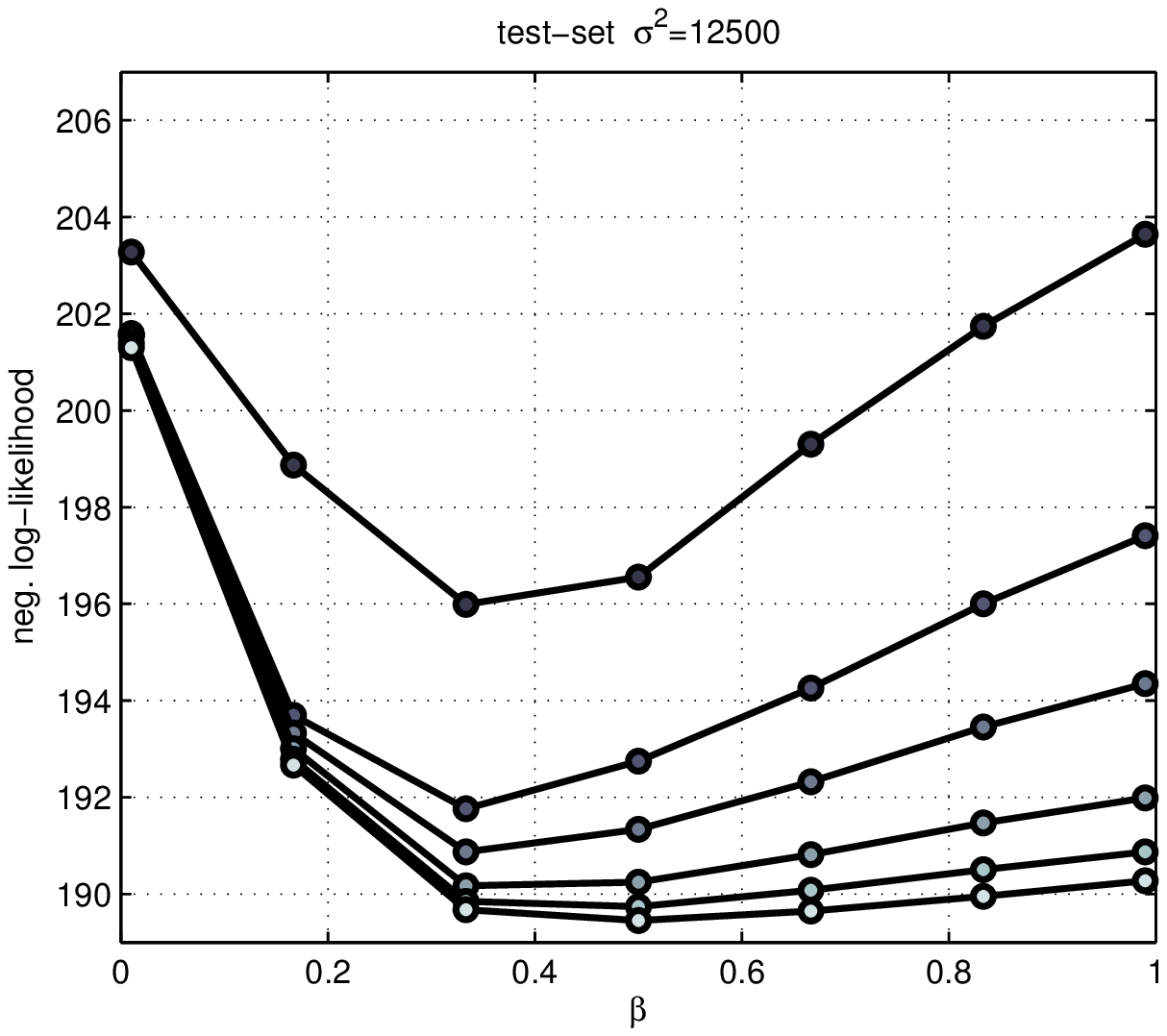}    &   \hspace{-5.5mm}
\includegraphics[trim=12.5mm 0mm 0mm 8.2mm,clip,totalheight=.178\textheight]{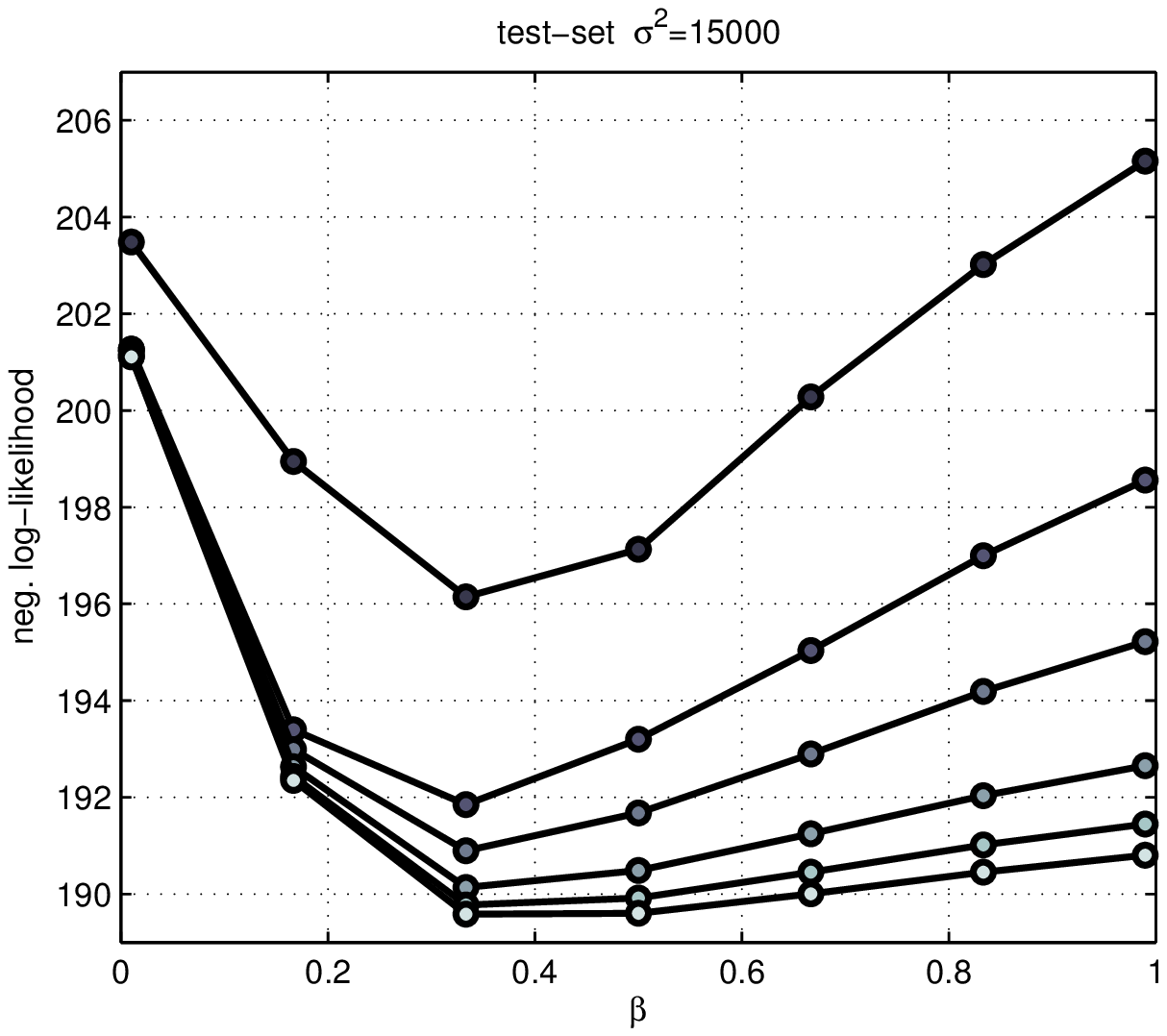} 
\end{tabular}
\vspace{-1.5em}
\caption{
	Train (top) and test (bottom) perplexities for the Boltzmann chain MRF with
	PL1/PL2 selection policy (x-axis:PL2 weight, y-axis:perplexity; see above
	and previous).
		\vspace{1em}
	\newline
	PL1/PL2 outperforms PL1/FL test perplexity at $\sigma^2=5000$ and continues
	to show improvement with weaker regularizers.  This is perhaps surprising
	since the previous policy includes FL as a special case, i.e.,
	$(\lambda,\beta)=(1,1)$.  We speculate that the regularizer's indirect
	connection to the training samples precludes it from preventing certain
	types of overfitting.  See Sec.~\ref{sec:winwin} for more discussion.
}\label{fig:chunk_boltzchain_pl1pl2_beta}

\vspace{-0.5em}
\end{figure}

\begin{figure}[ht!]
\hspace{-0.5cm}
\begin{tabular}{ccc}
\includegraphics[trim=0mm 5.5mm 0mm 0.2mm,clip,totalheight=.20\textheight]{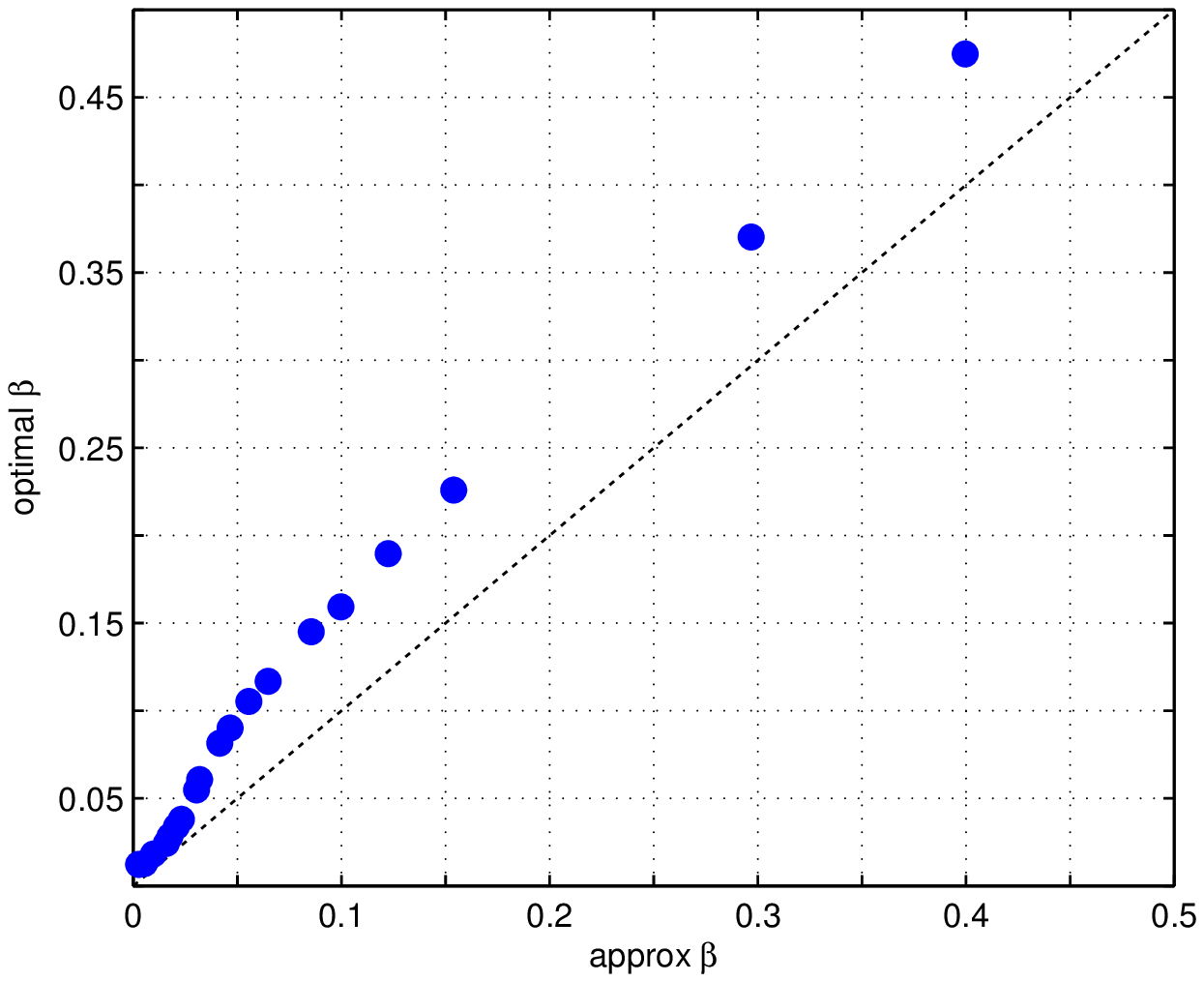}  &  \hspace{-3.5mm}
\includegraphics[trim=5.5mm 5.5mm 0mm 6.2mm,clip,totalheight=.20\textheight]{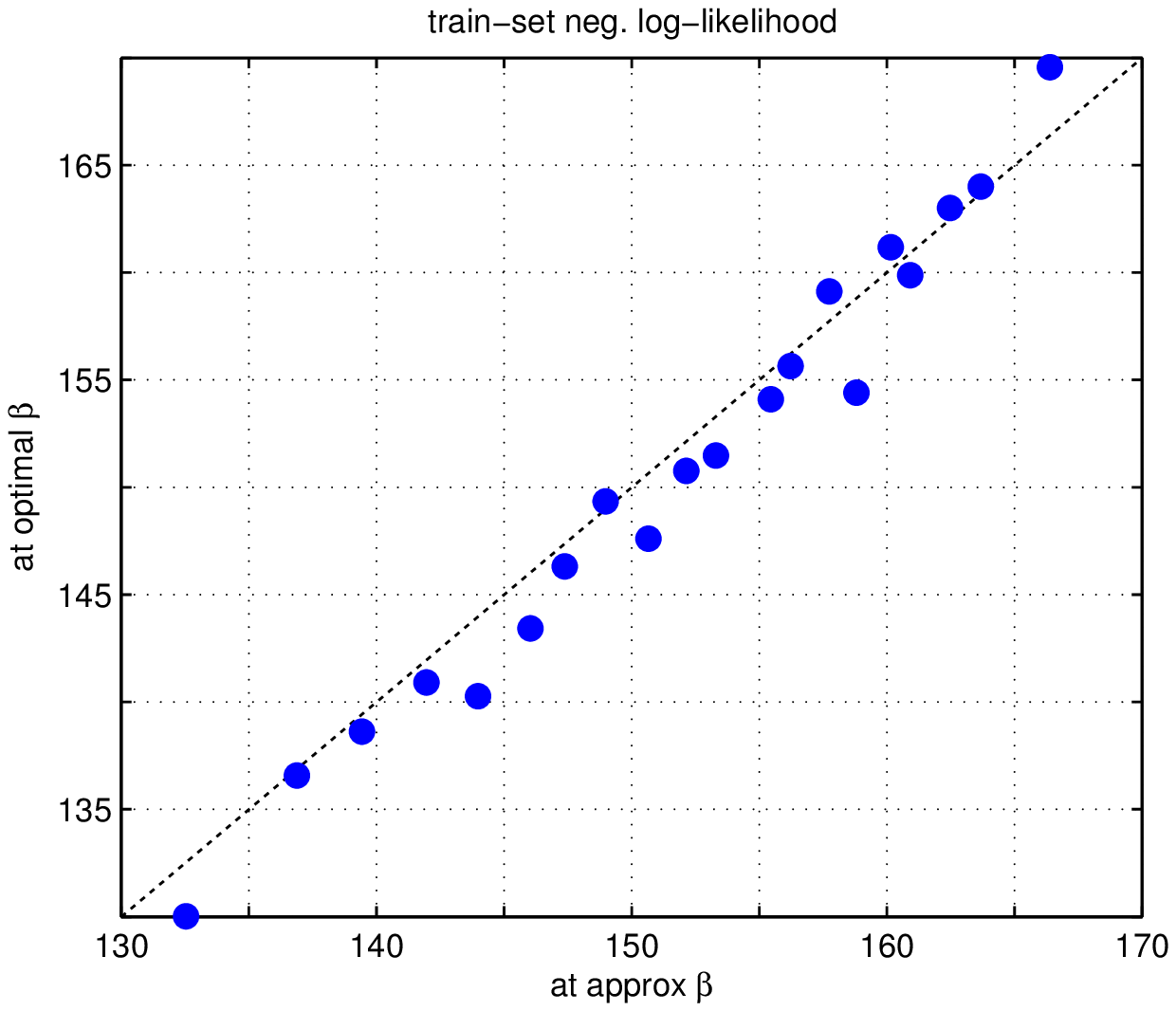}   &  \hspace{-5.5mm}
\includegraphics[trim=5.5mm 5.5mm 0mm 6.2mm,clip,totalheight=.20\textheight]{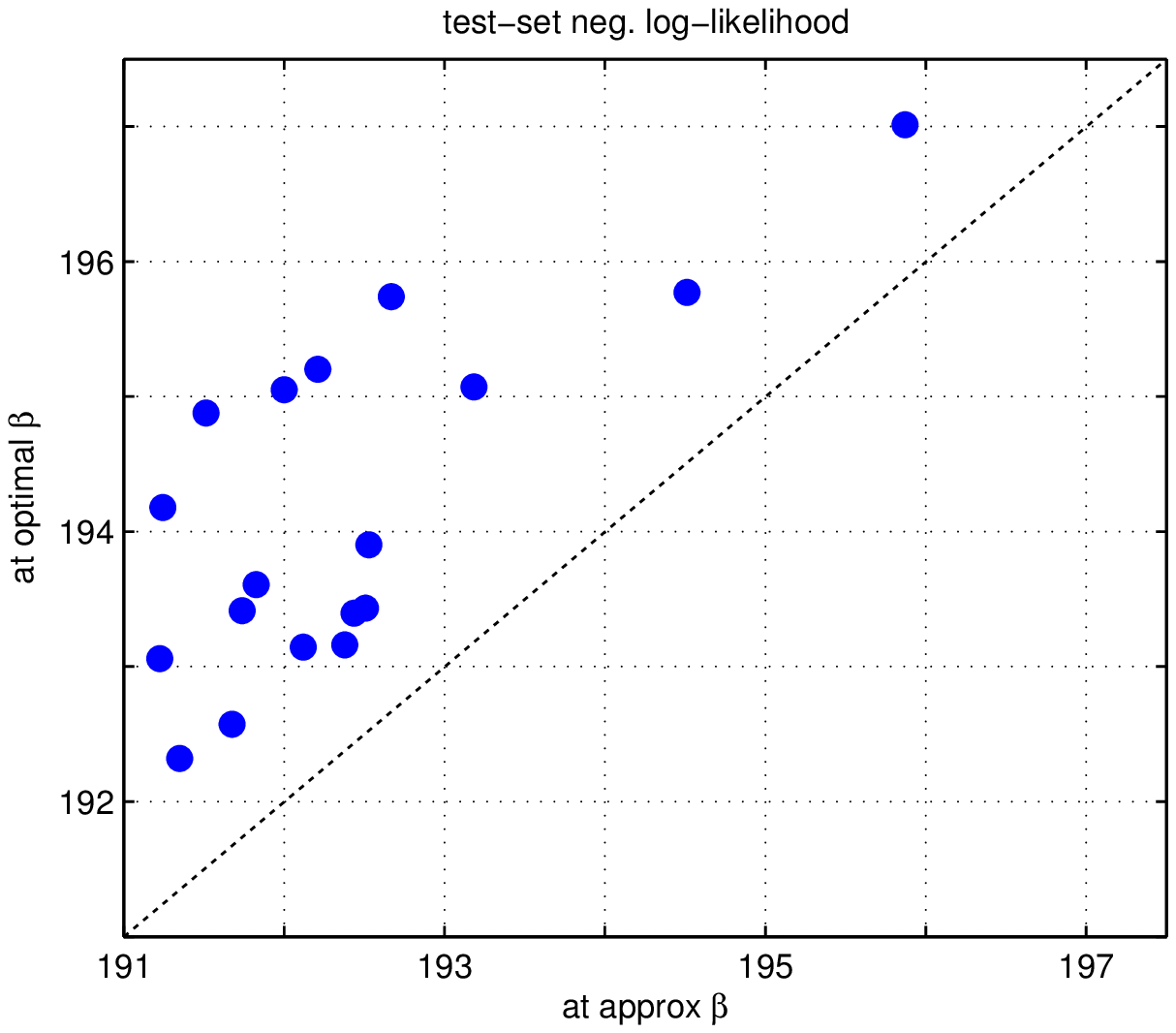}   \\
\includegraphics[trim=0mm 0.0mm 0mm 0.2mm,clip,totalheight=.21\textheight]{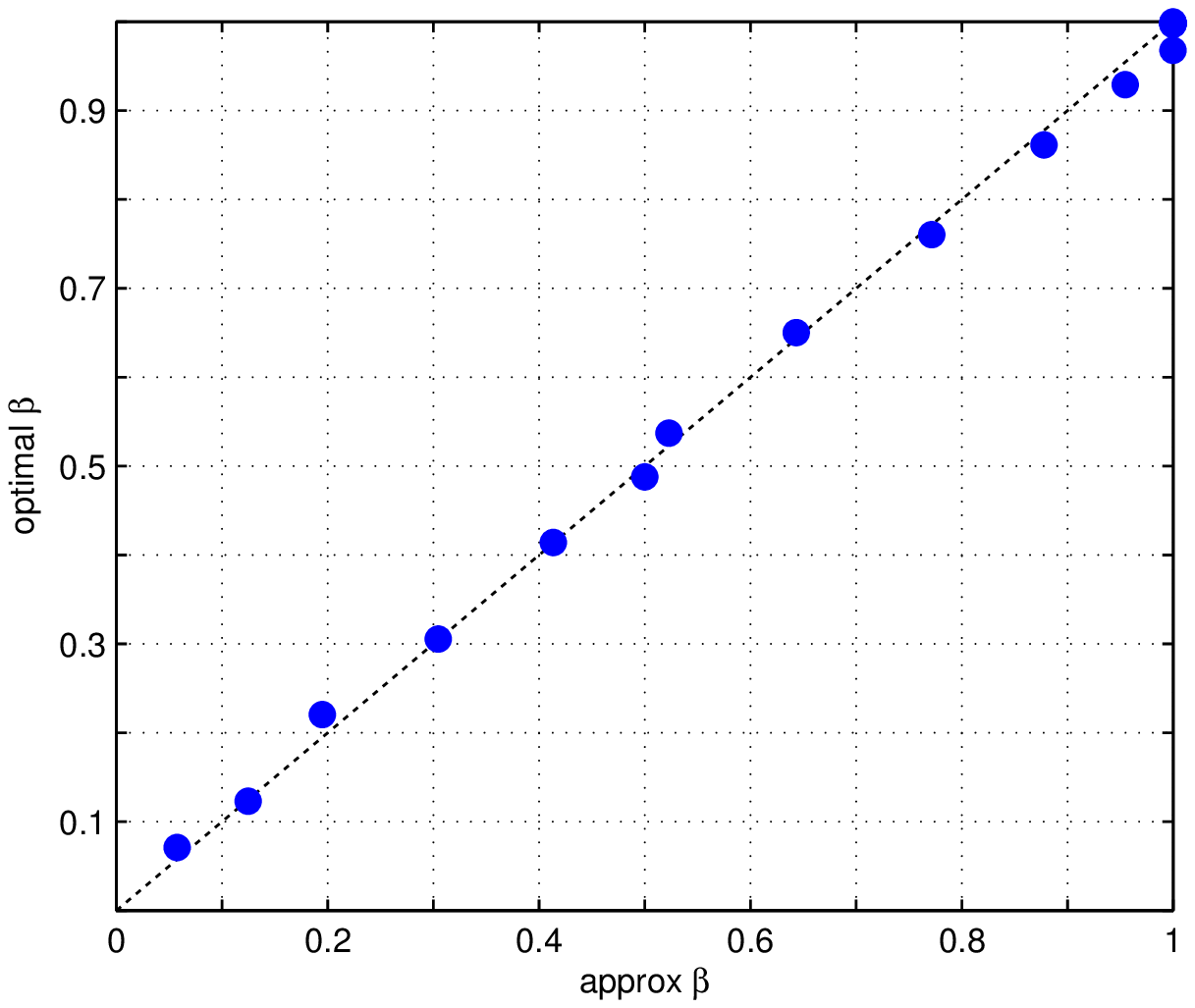} &  \hspace{-3.5mm}
\includegraphics[trim=7.6mm 0.0mm 0mm 6.2mm,clip,totalheight=.211\textheight]{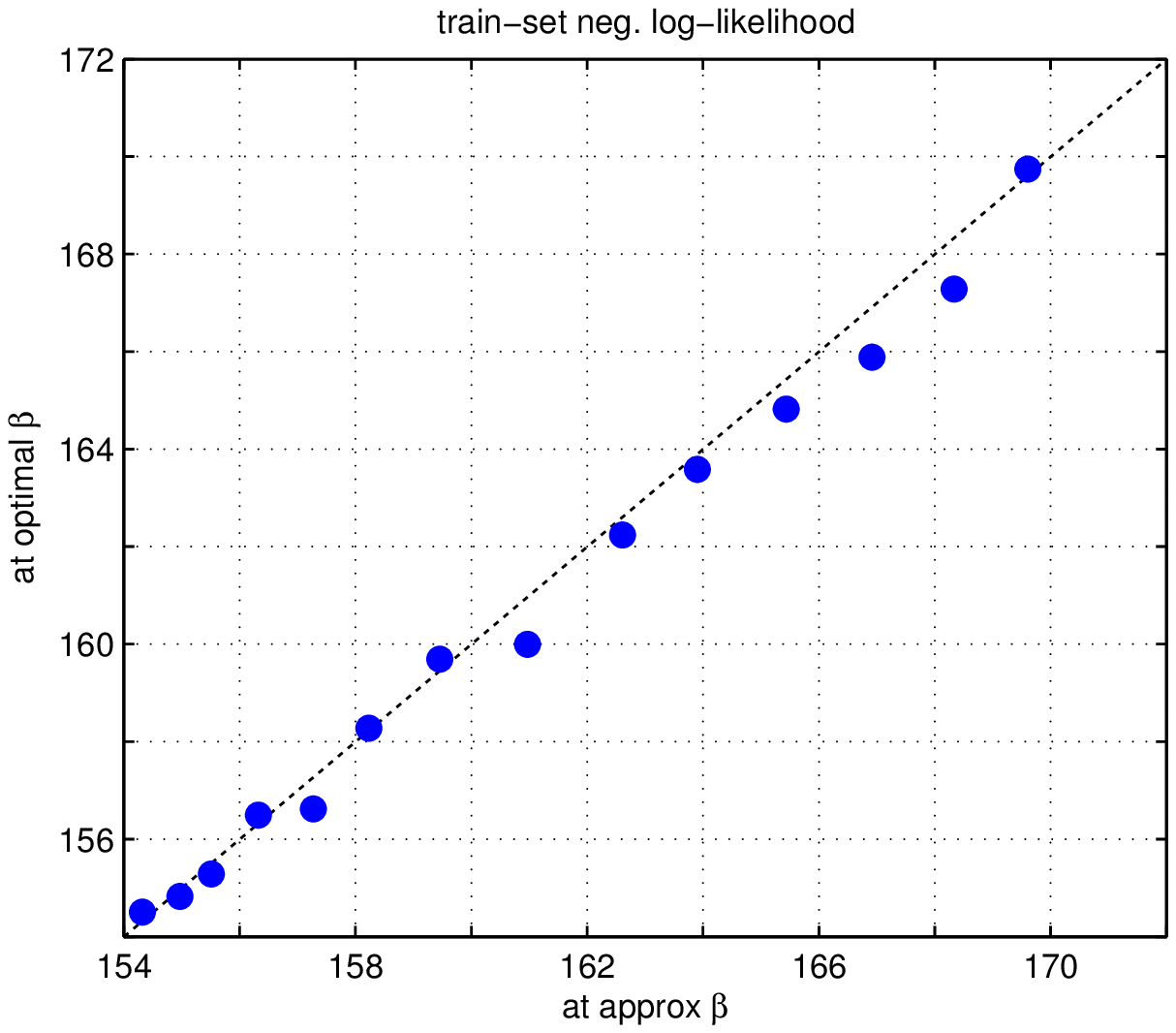} &  \hspace{-5.5mm}
\includegraphics[trim=5.8mm 0.0mm 0mm 6.2mm,clip,totalheight=.211\textheight]{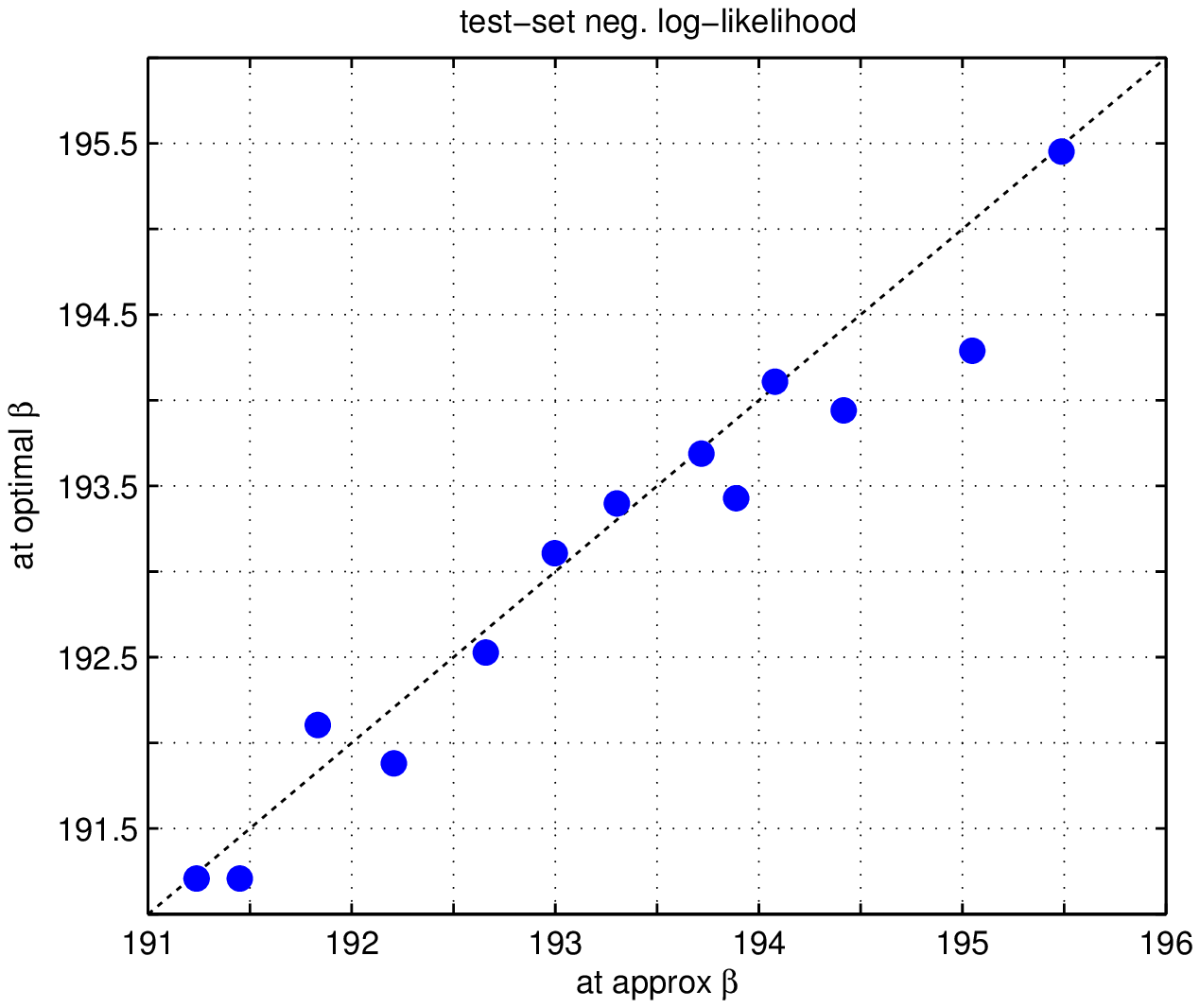}
\end{tabular}
\vspace{-1.5em}
\caption{
	Demonstration of the effectiveness of the $\beta$ heuristic, i.e., using
	$\hat{\theta}^{msl}$ as a plug-in estimate for $\theta_0$ to periodically
	re-estimate $\beta$ during gradient descent.  Results are for the Boltzmann
	chain with PL1/FL (top) and PL1/PL2 (bottom) selection policies.  The
	x-axis is the value at the heuristically found $\beta$ and the y-axis the
	value at the optimal $\beta.$ The optimal $\beta$ was found be evaluating
	over a $\beta$ grid and choosing that with the smallest train set
	perplexity.  The first column depicts the best performing $\beta$ against
	the heuristic $\beta$.  The second and third columns depict the training
	and testing perplexities (resp.) at the best performing $\beta$ and
	heuristically found $\beta$.  For all three columns, we assess the
	effectiveness of the heuristic by its nearness to the diagonal (dashed
	line).
	\vspace{1em}
	\newline
	For the PL1/PL2 policy the heuristic closely matched the optimal (all
	bottom row points are on diagonal).  The heuristic out-performed the
	optimal on the test set and had slightly higher perplexity on the training
	set.  It is a positive result, albeit somewhat surprising, and is
	attributable to either coarseness in the grid or improved generalization by
	accounting for variability in $\hat{\theta}^{msl}$.
}\label{fig:chunk_boltzchain_heuristic}
\end{figure}

\subsubsection{CRFs}\label{sec:chunk_crf}
Conditional random fields are the discriminative counterpart of Boltzmann
chains (cf.\ Figures~\ref{fig:crfgm} and \ref{fig:boltzchaingm}).  Since $x$ is
not jointly modeled with $y$, we are free to include features with
non-independence across time steps without significantly increasing the
computational complexity.  Here our notion of pseudo likelihood is more
traditional, e.g., $\Pr(Y_2|Y_1,Y,3,X_2)$ and $\Pr(Y_2,Y_3|Y_1,Y,4,X_2,X_3)$
are valid 1st and 2nd order pseudo likelihood components.

We employ a subset of the features outlined in \cite{Sha2003} which proved
competitive for the CoNLL-2000 shared task.  Our feature vector was based on
seven feature categories, resulting in a total of 273,571 binary features
(i.e., $\sum_i f_i(x_t)=7$).  The feature categories consisted of word
unigrams, POS unigrams, word bigrams (forward and backward), and POS bigrams
(forward and backward) as well as a stopword indicator (and its complement) as
defined by \cite{lewis04rcv}.  The set of possible feature/label pairs is much
larger than our set--we use only those features supported by the CoNLL-2000
dataset, i.e., those which occur at least once.  Thus we modeled 297,041
feature/label pairs and 847 transitions for a total of 297,888
parameters.  As before, we use the $L_2$ regularizer,
$\exp\{-\frac{1}{2\sigma^{2}} ||\theta||^2_2\}$, which is is stronger when
$\sigma^2$ is small and weak when $\sigma^2$ is large.

We demonstrate learning on two selection policies: pseudo/full likelihood
(Figures \ref{fig:chunk_crf_pl1fl_cont} and \ref{fig:chunk_crf_pl1fl_beta}) and
1st/2nd order pseudo likelihood (Figures \ref{fig:chunk_crf_pl1pl2_cont} and
\ref{fig:chunk_crf_pl1pl2_beta}).  In both selection polices we note a
significant difference from the Boltzmann chain, $\beta$ has less impact on
both train and test perplexity.  Intuitively, this seems reasonable as the
component likelihood range and variance are constrained by the conditional
nature of CRFs.  Figure \ref{fig:chunk_crf_complexity} demonstrates the
empirical accuracy/complexity tradeoff and Figure \ref{fig:chunk_crf_heuristic}
depicts the effectiveness of the $\beta$ heuristic.  See captions for further
comments.

\begin{figure}
\centering
\begin{tabular}{ccc}
\includegraphics[trim=0mm 0mm 0mm 0mm,clip,totalheight=.25\textheight]{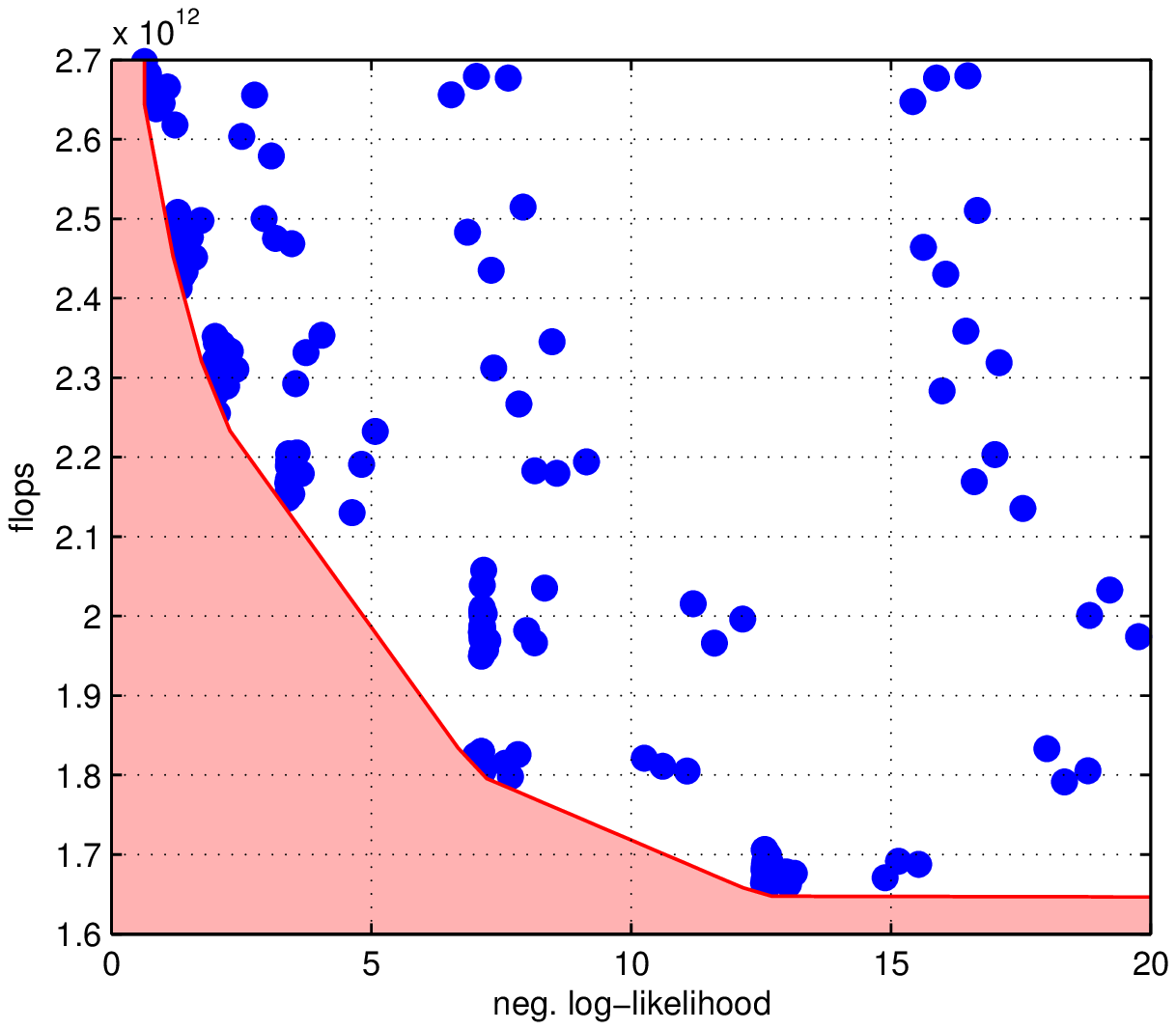} & \includegraphics[trim=0mm 0mm 0mm 0mm,clip,totalheight=.25\textheight]{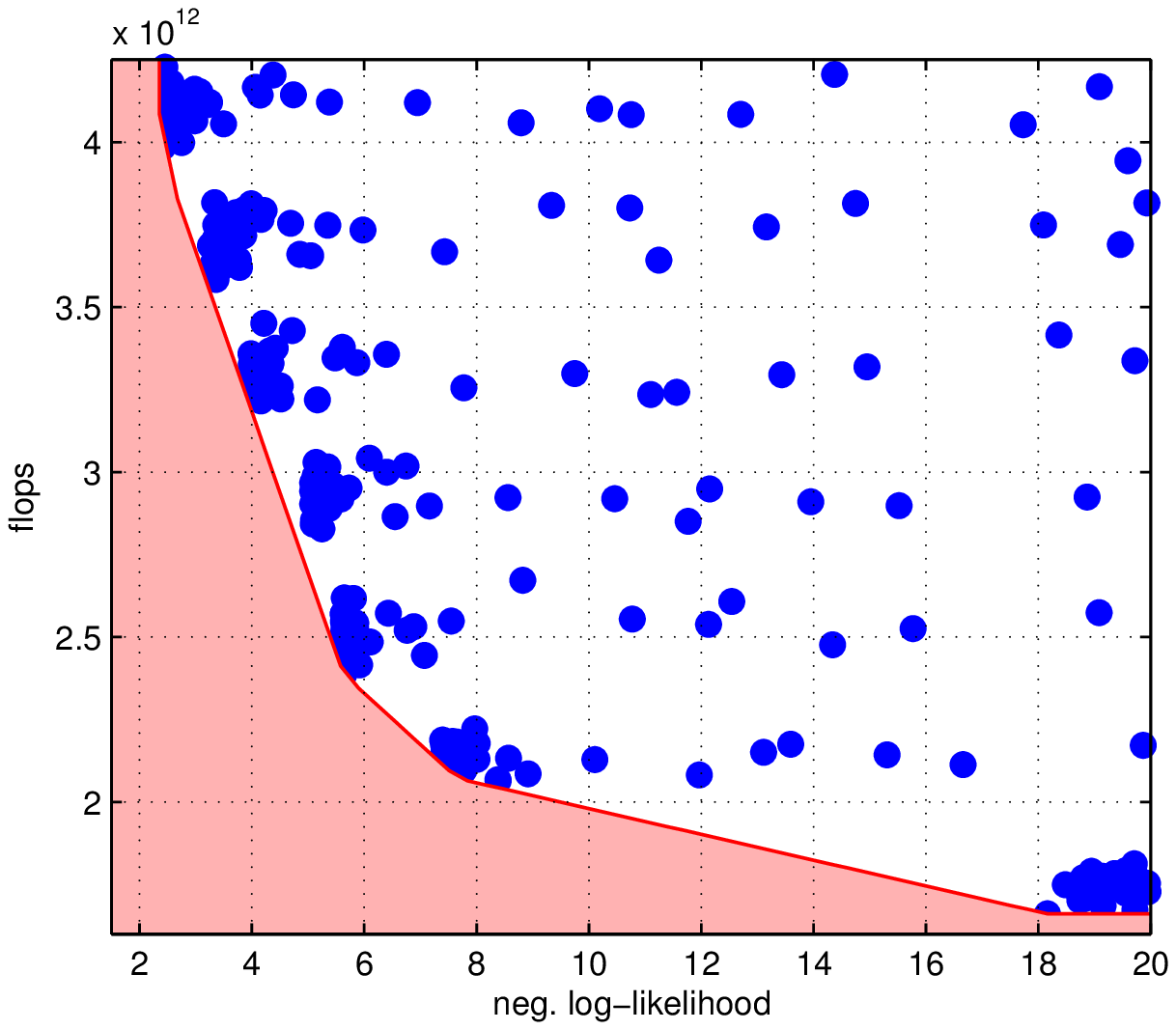} \end{tabular}
\caption{
	Accuracy and complexity tradeoff for the CRF with PL1/FL (left) and PL1/PL2
	(right) selection policies.  Each point represents the negative loglikelihood (perplexity) and the number of flops
	required to evaluate the composite likelihood and its gradient under a particular instance of the selection policy.  The shaded region represents empirically unobtainable combinations of computational complexity and accuracy.  $\sigma^2$.  
}\label{fig:chunk_crf_complexity}
\end{figure}

\begin{figure}[ht!]
\vspace{-3.0em}
\hspace{-0.75cm}
\begin{tabular}{cccc}
\includegraphics[trim= 0.0mm 10.2mm 0mm 9.5mm,clip,totalheight=.157\textheight]{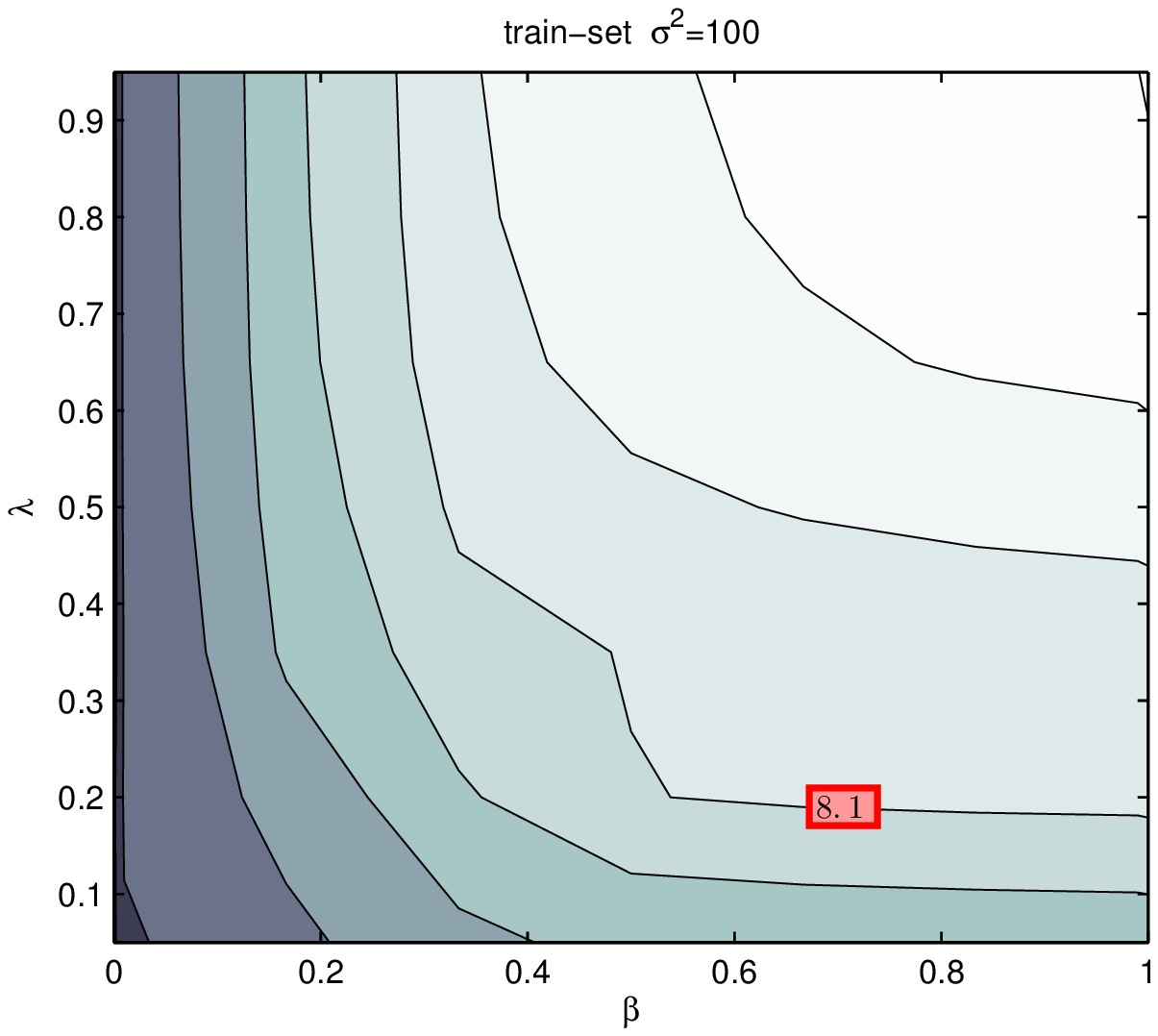}  &  \hspace{-5.5mm}
\includegraphics[trim=12.0mm 10.2mm 0mm 9.5mm,clip,totalheight=.157\textheight]{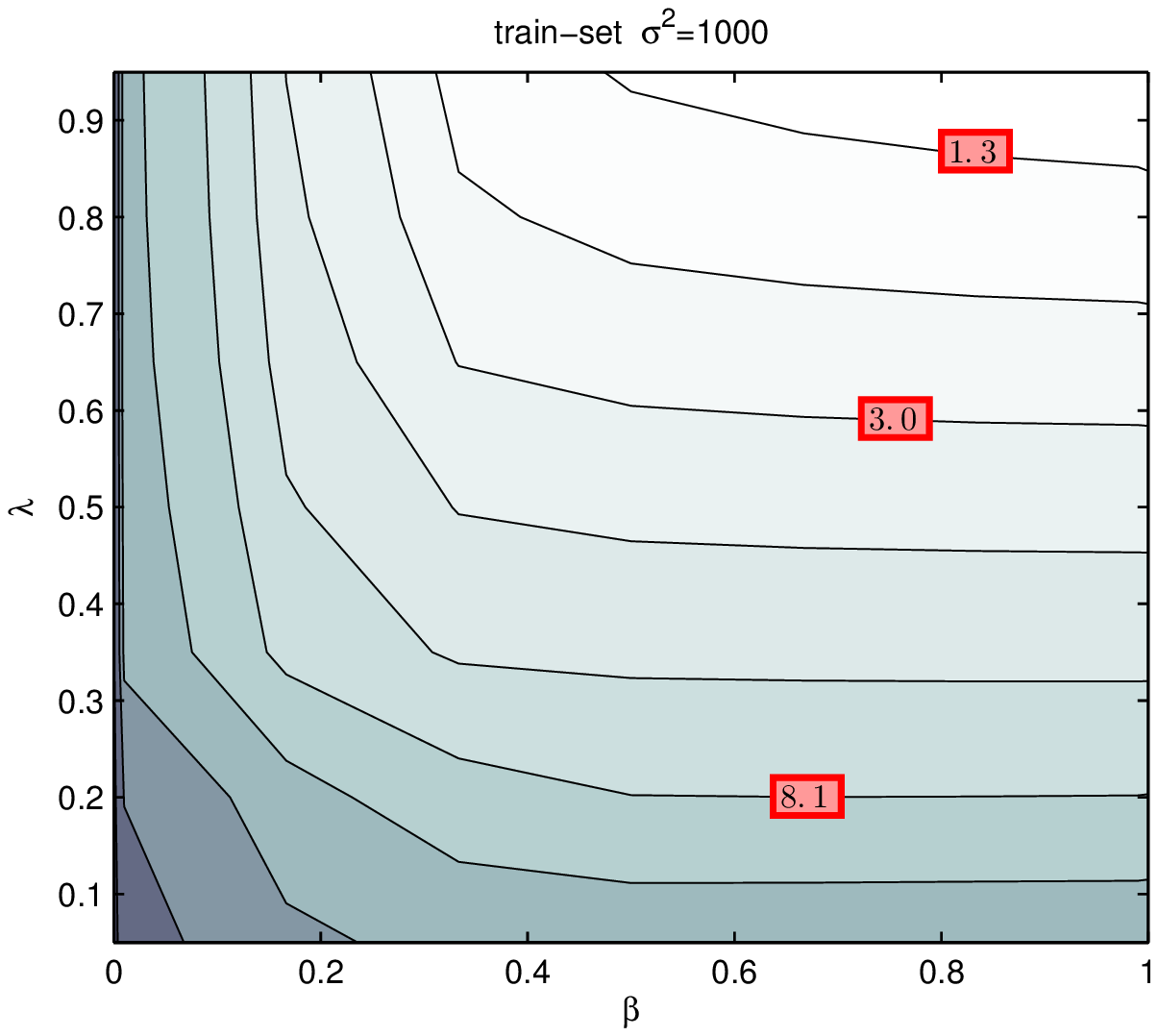} &  \hspace{-5.5mm}
\includegraphics[trim=12.0mm 10.2mm 0mm 9.5mm,clip,totalheight=.157\textheight]{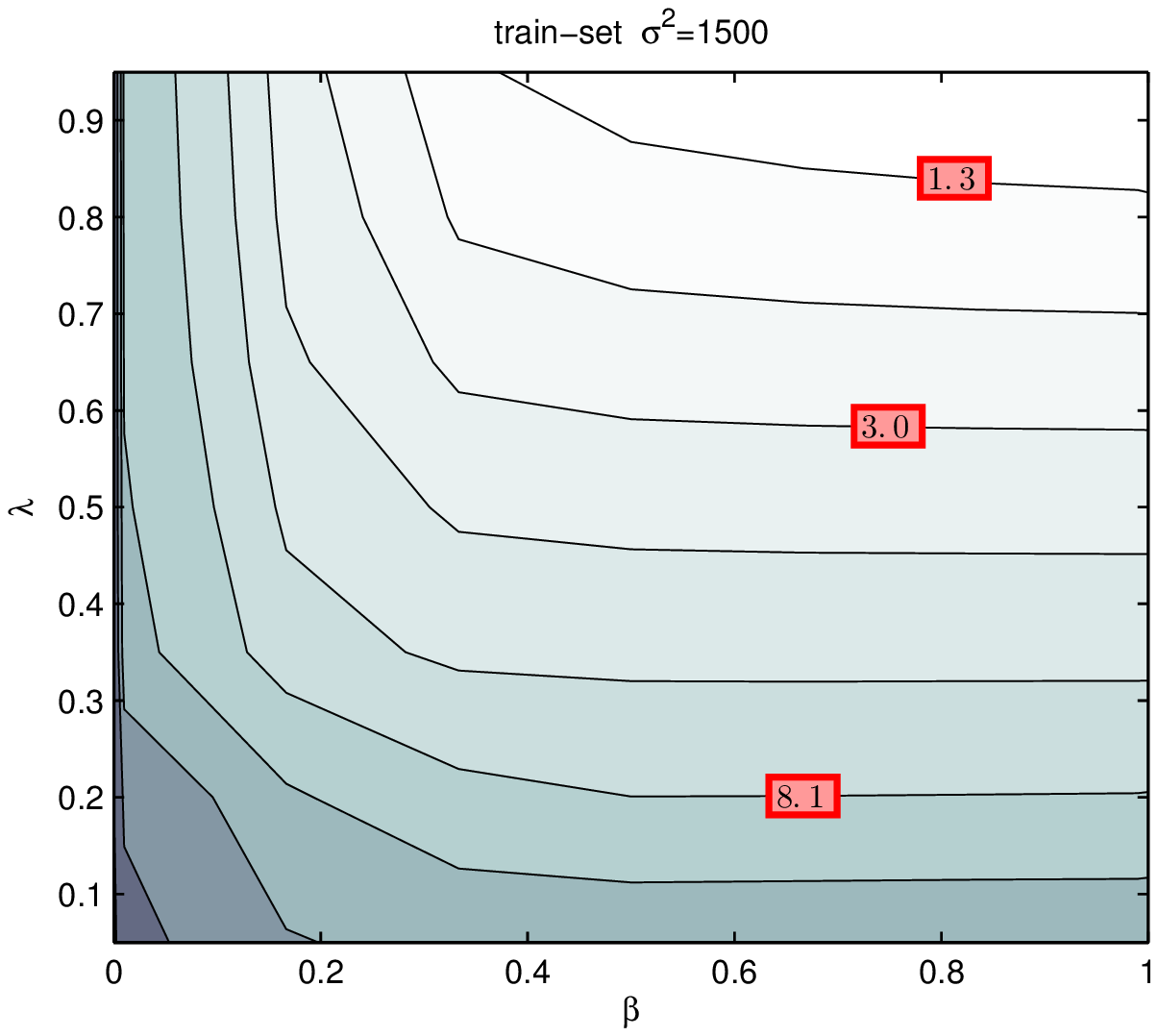} &  \hspace{-5.5mm}
\includegraphics[trim=12.0mm 10.2mm 0mm 9.5mm,clip,totalheight=.157\textheight]{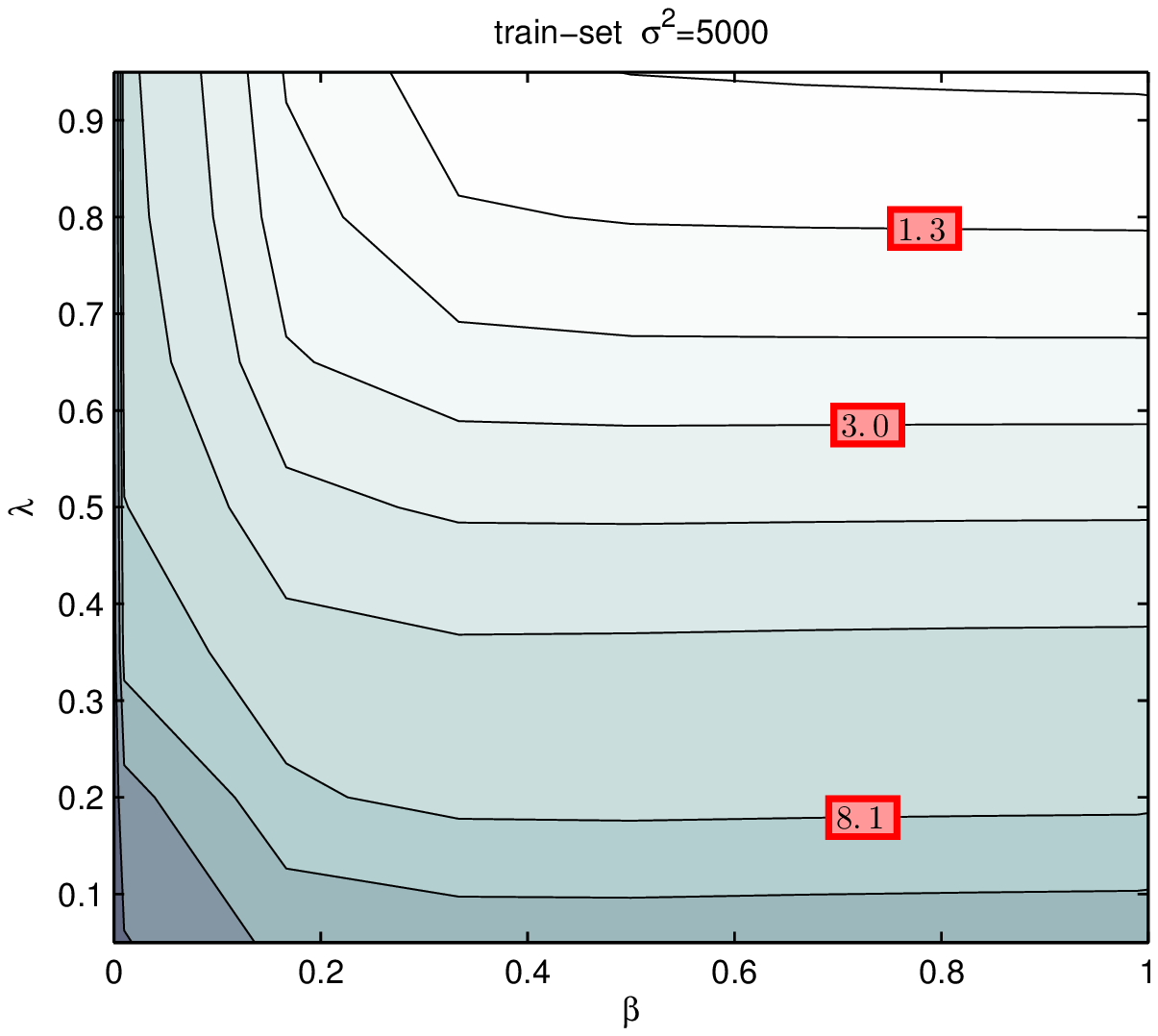} \\
\includegraphics[trim= 0.0mm 0mm 0mm 9.5mm,clip,totalheight=.175\textheight]{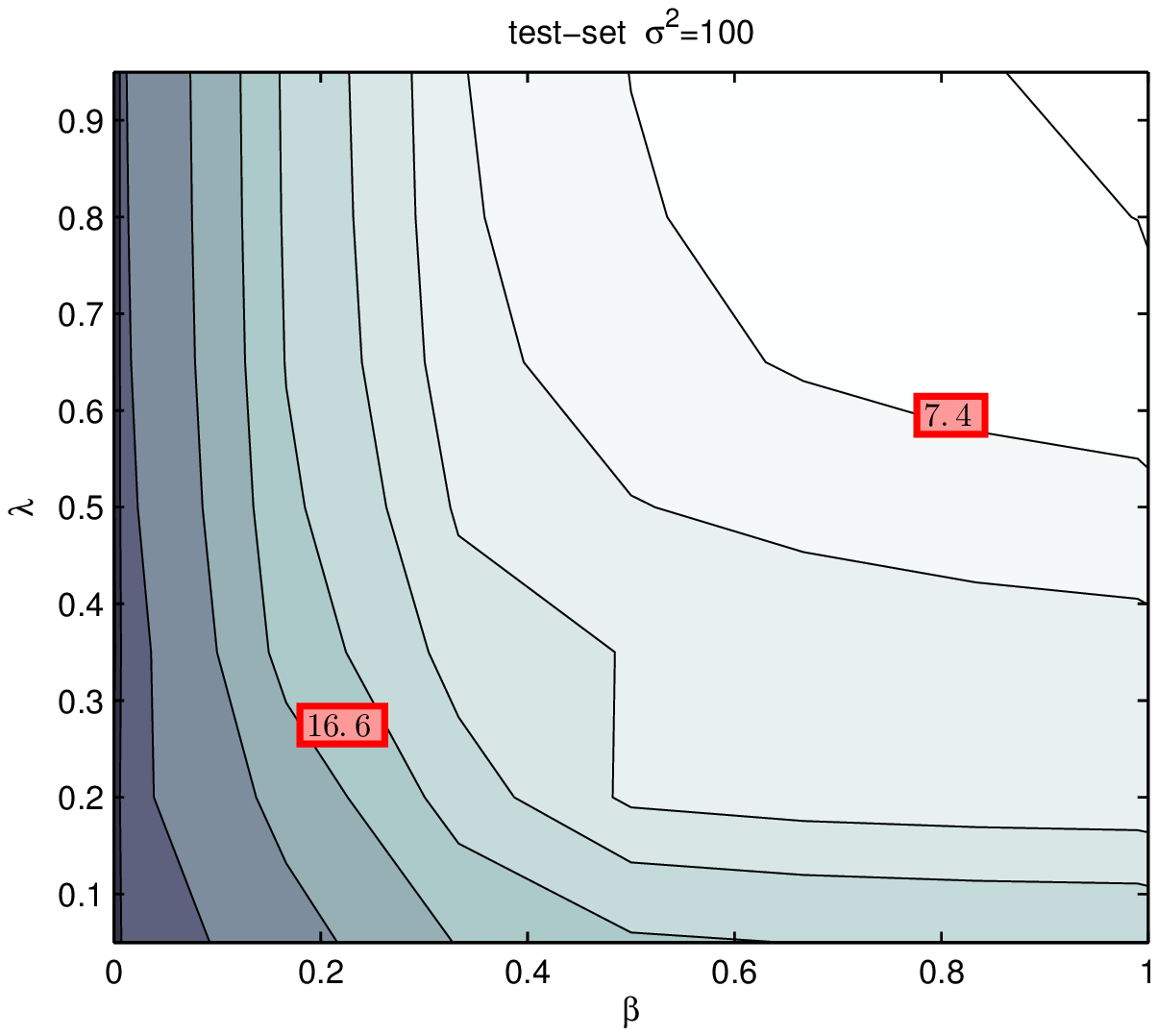}     &  \hspace{-5.5mm}
\includegraphics[trim=12.0mm 0mm 0mm 9.5mm,clip,totalheight=.175\textheight]{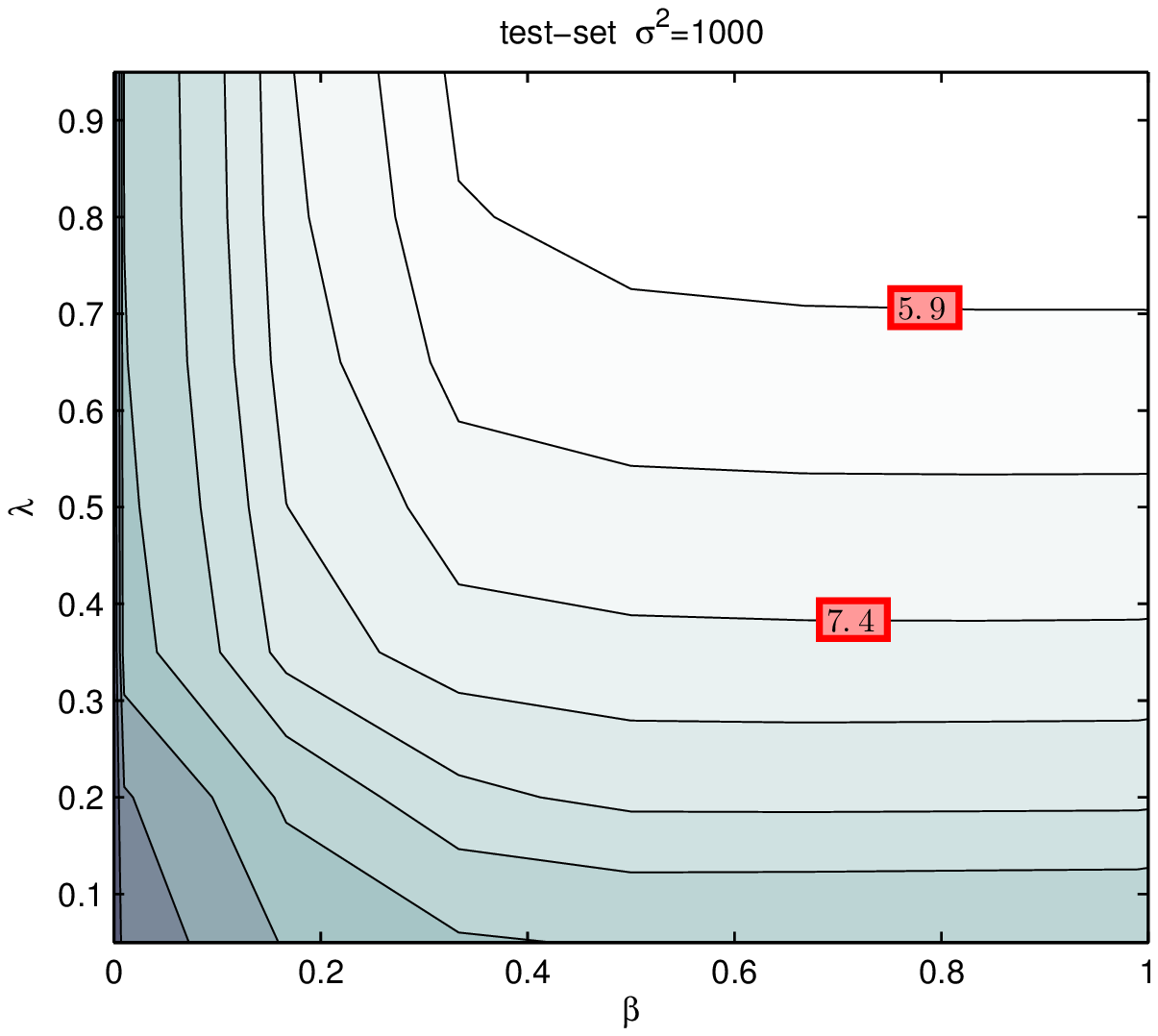}    &  \hspace{-5.5mm}
\includegraphics[trim=12.0mm 0mm 0mm 9.5mm,clip,totalheight=.175\textheight]{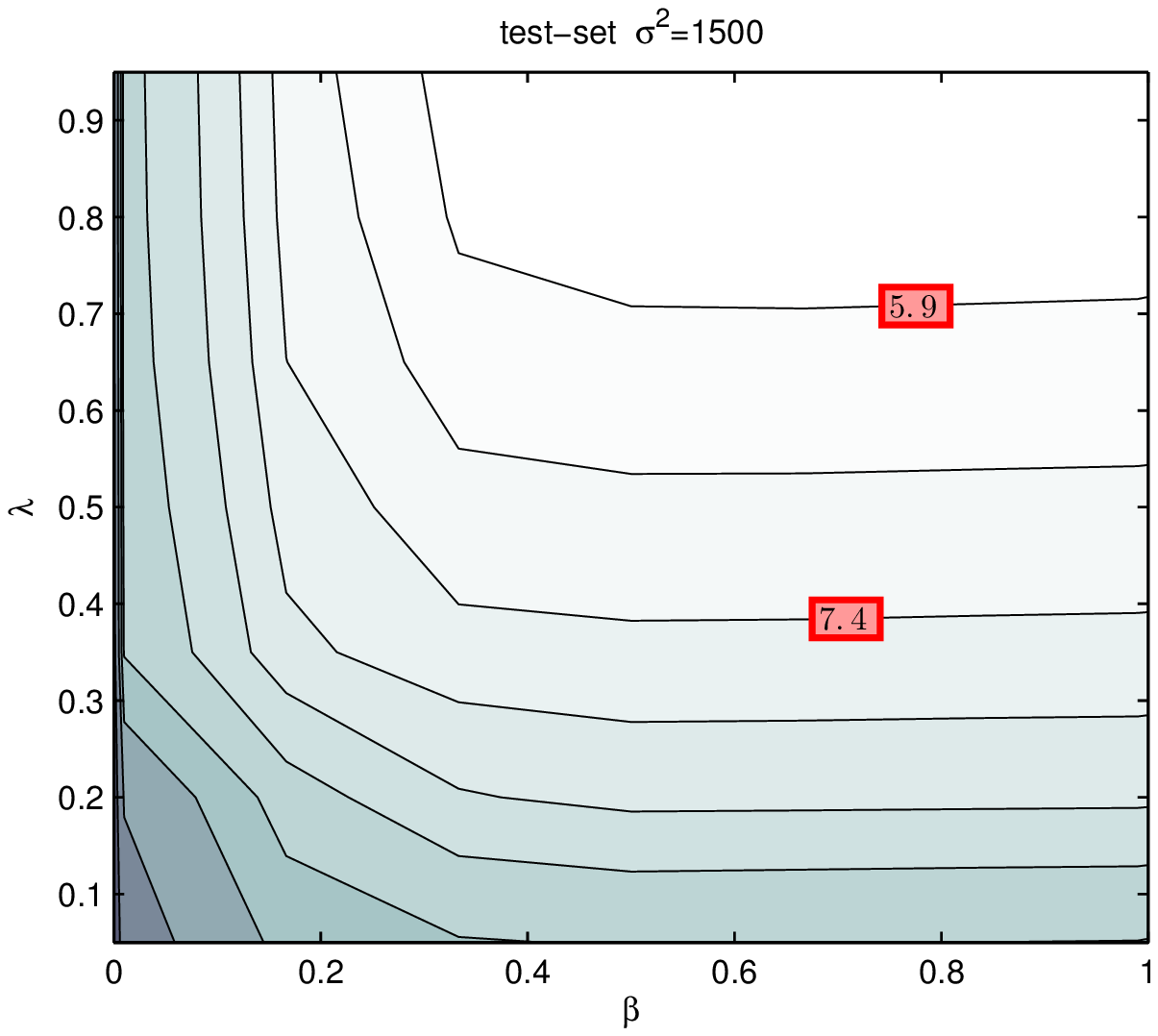}    &  \hspace{-5.5mm}
\includegraphics[trim=12.0mm 0mm 0mm 9.5mm,clip,totalheight=.175\textheight]{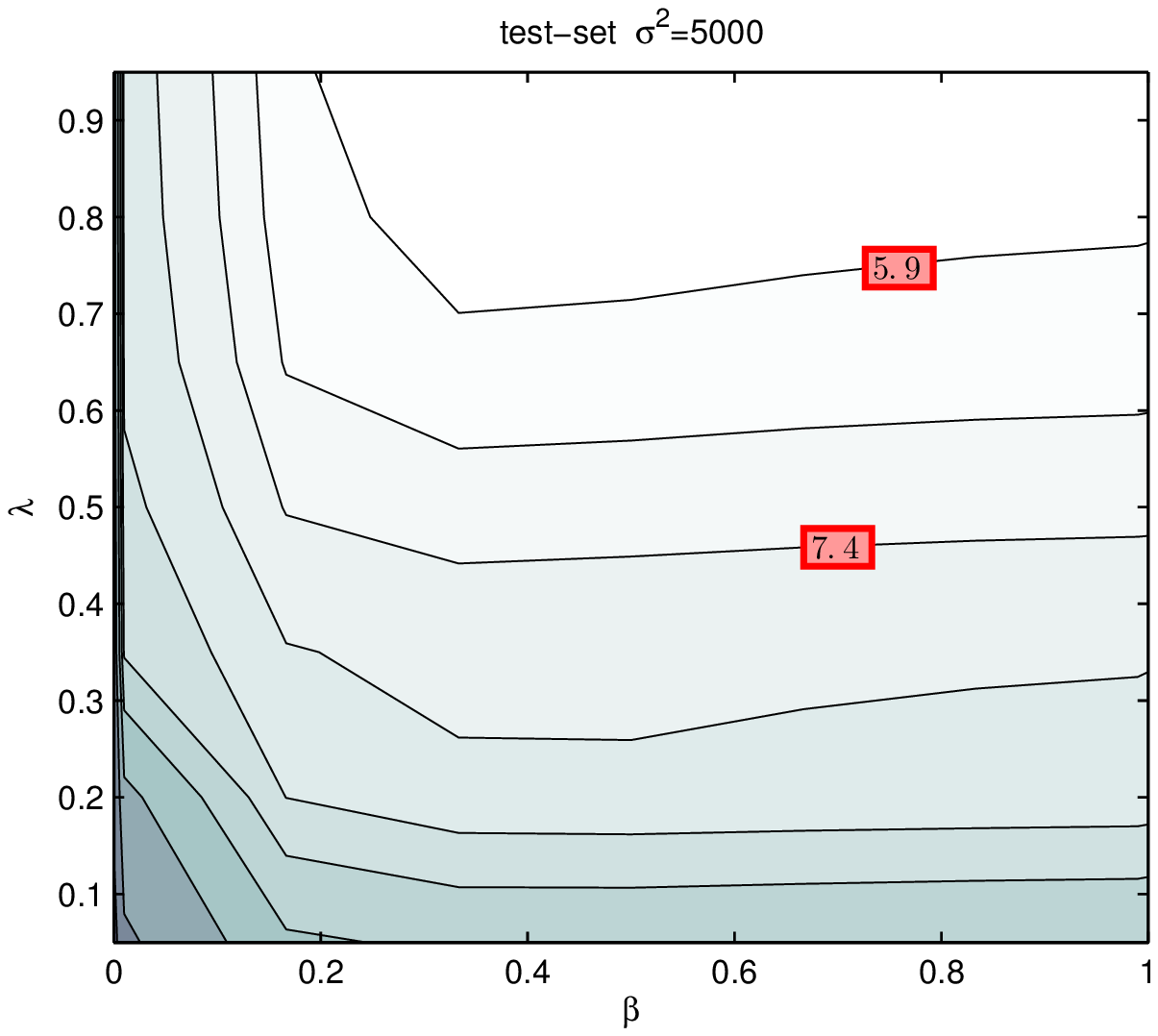}
\end{tabular}
\vspace{-1.5em}
\caption{
	Train set (top) and test set (bottom) perplexity for the CRF with
	pseudo/full likelihood selection policy (PL1/FL).  The x-axis corresponds
	to FL weight and the y-axis the probability of its selection.  PL1 is
	selected with probability 1 and weight $1-\beta$.  Columns from left to
	right correspond to $\sigma^2=\{5000, 10000, 12500, 15000\}$.  See
	Figure~\ref{fig:chunk_boltzchain_pl1fl_cont} for more details.  The best
	achievable test set perplexity is about 5.5.
	\vspace{1em}
	\newline
	Although we cannot directly compare CRFs to its generative counterpart, we
	observe some strikingly different trends.  It is immediately clear that the
	CRF is less sensitive to the relative weighting of components than is the
	Boltzmann chain.  This is partially attributable to a smaller range of the
	objective--the CRF is already conditional hence the per-component
	perplexity range is reduced.}\label{fig:chunk_crf_pl1fl_cont}

\vspace{0.5em}
\hspace{-0.75cm}
\begin{tabular}{cccc}
\includegraphics[trim= 0.0mm 9.1mm 0mm 8.2mm,clip,totalheight=.162\textheight]{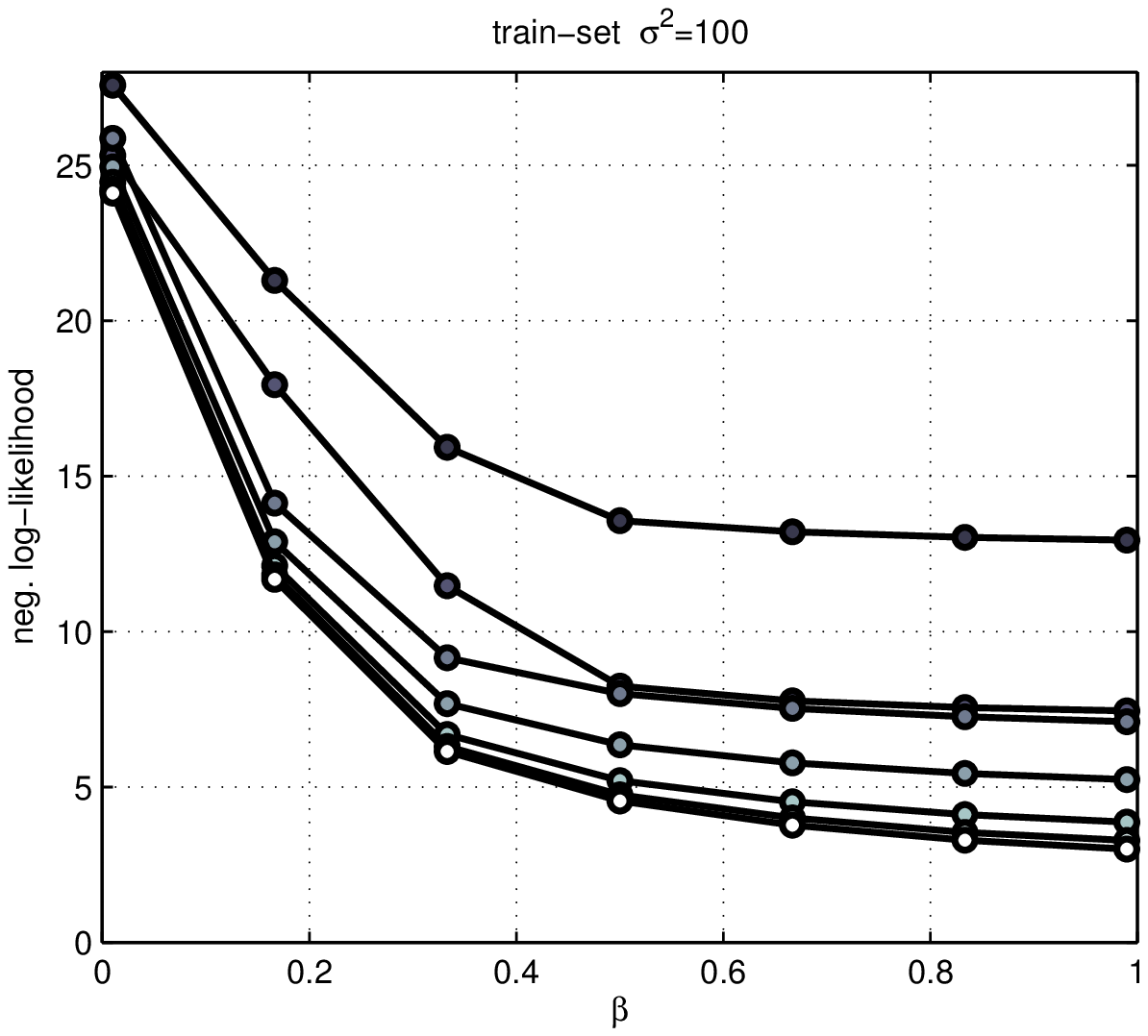}  &   \hspace{-5.5mm}
\includegraphics[trim=11.0mm 9.1mm 0mm 8.2mm,clip,totalheight=.162\textheight]{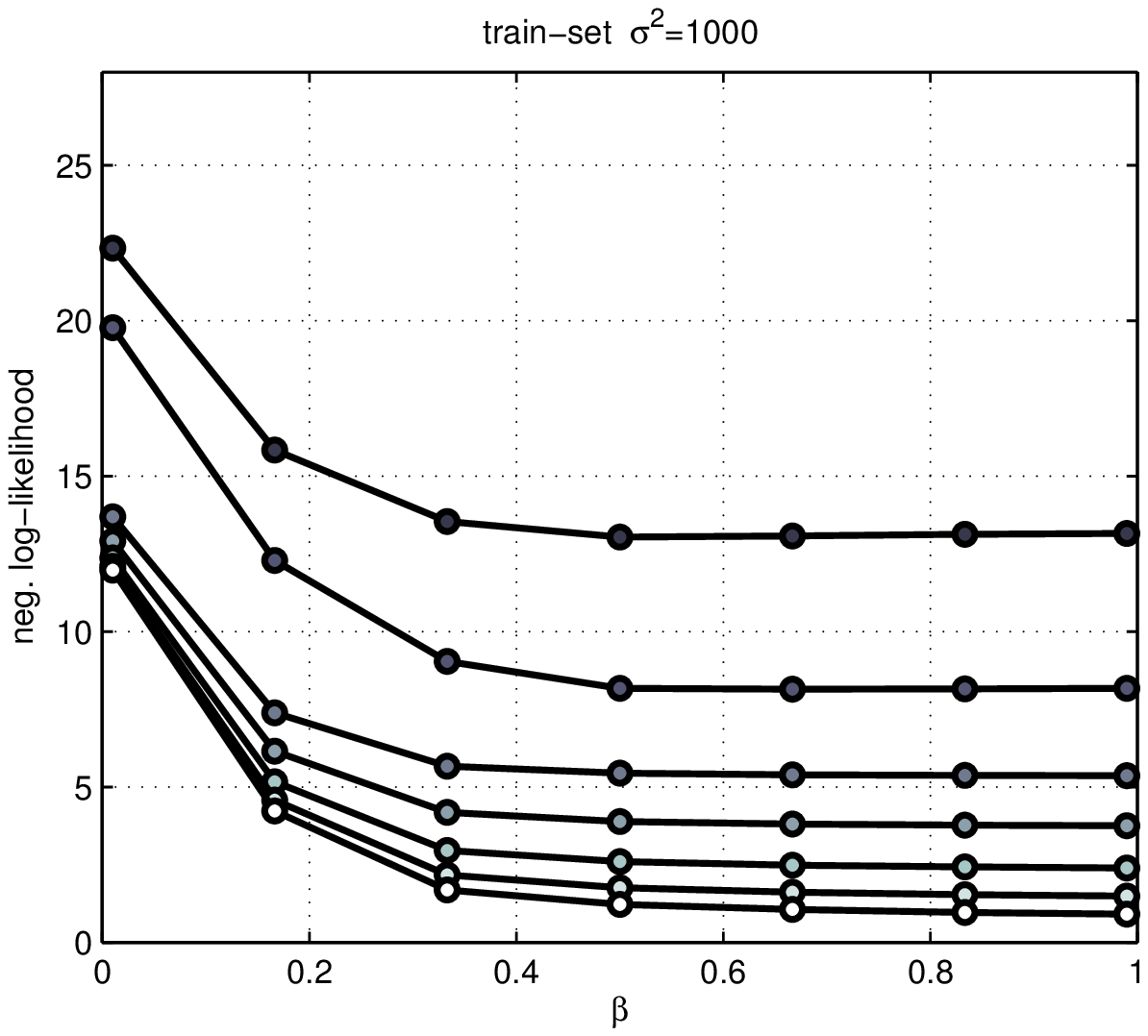} &   \hspace{-5.5mm}
\includegraphics[trim=11.0mm 9.1mm 0mm 8.2mm,clip,totalheight=.162\textheight]{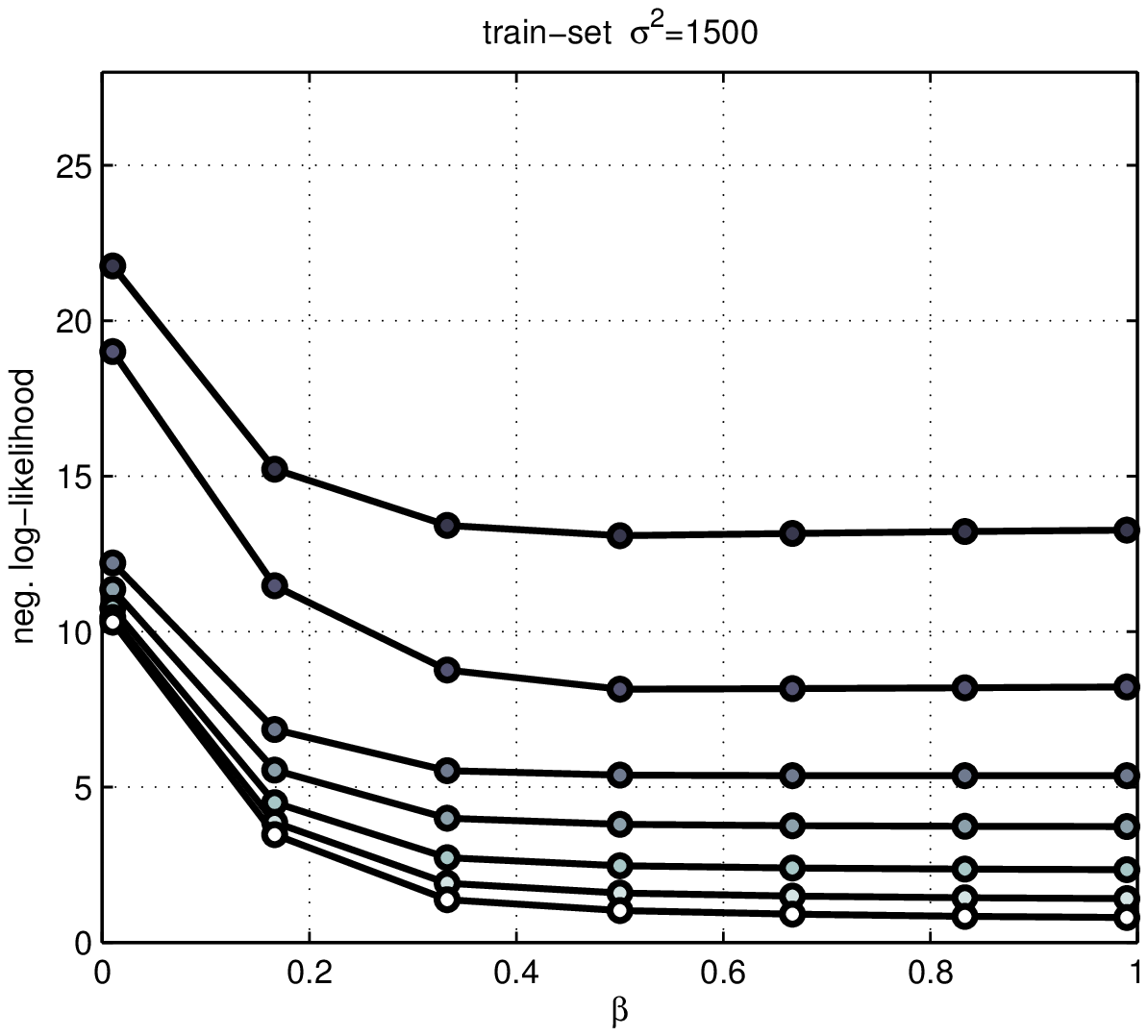} &   \hspace{-5.5mm}
\includegraphics[trim=11.0mm 9.1mm 0mm 8.2mm,clip,totalheight=.162\textheight]{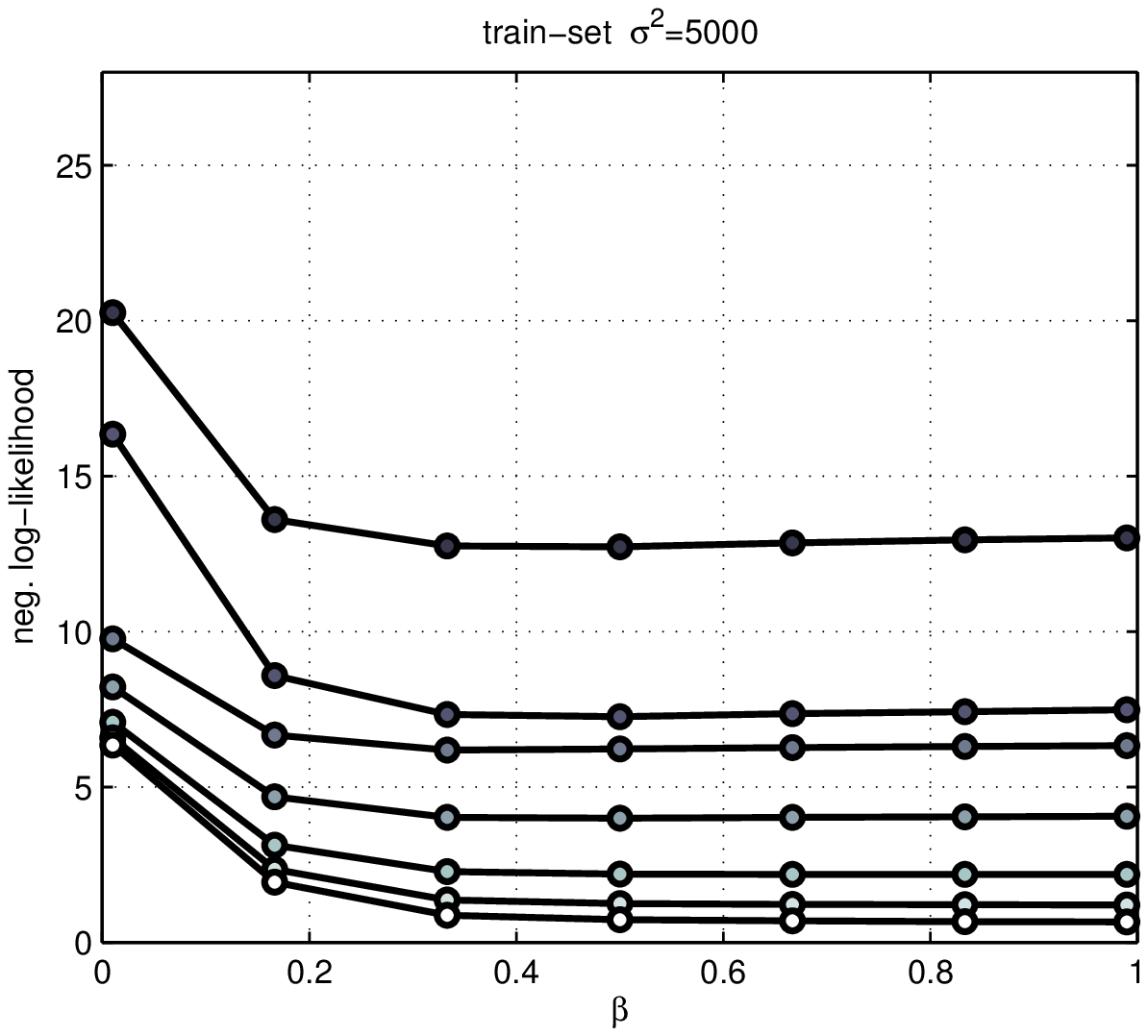} \\
\includegraphics[trim= 0.0mm 0mm 0mm 8.2mm,clip,totalheight=.178\textheight]{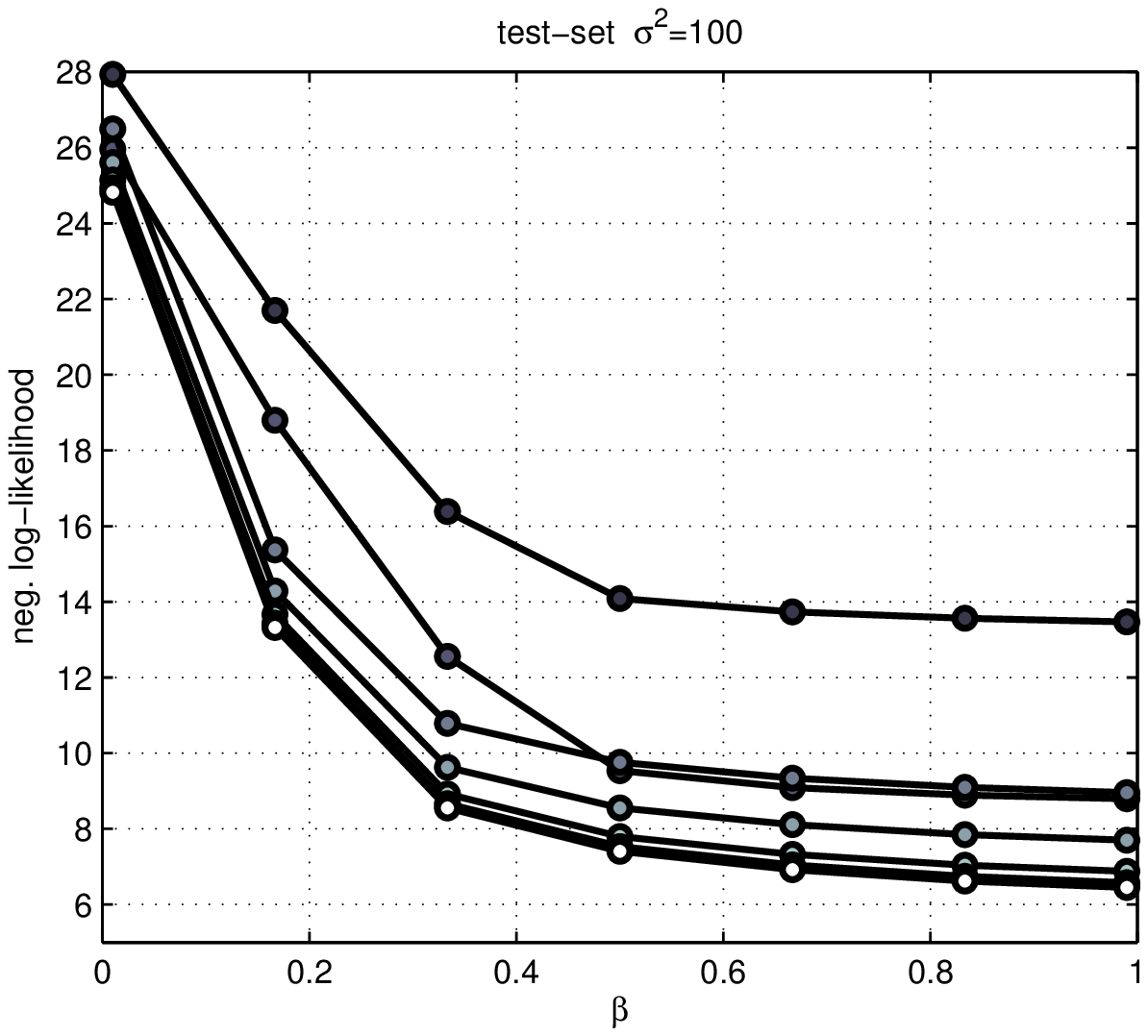}    &   \hspace{-5.5mm}
\includegraphics[trim=11.0mm 0mm 0mm 8.2mm,clip,totalheight=.178\textheight]{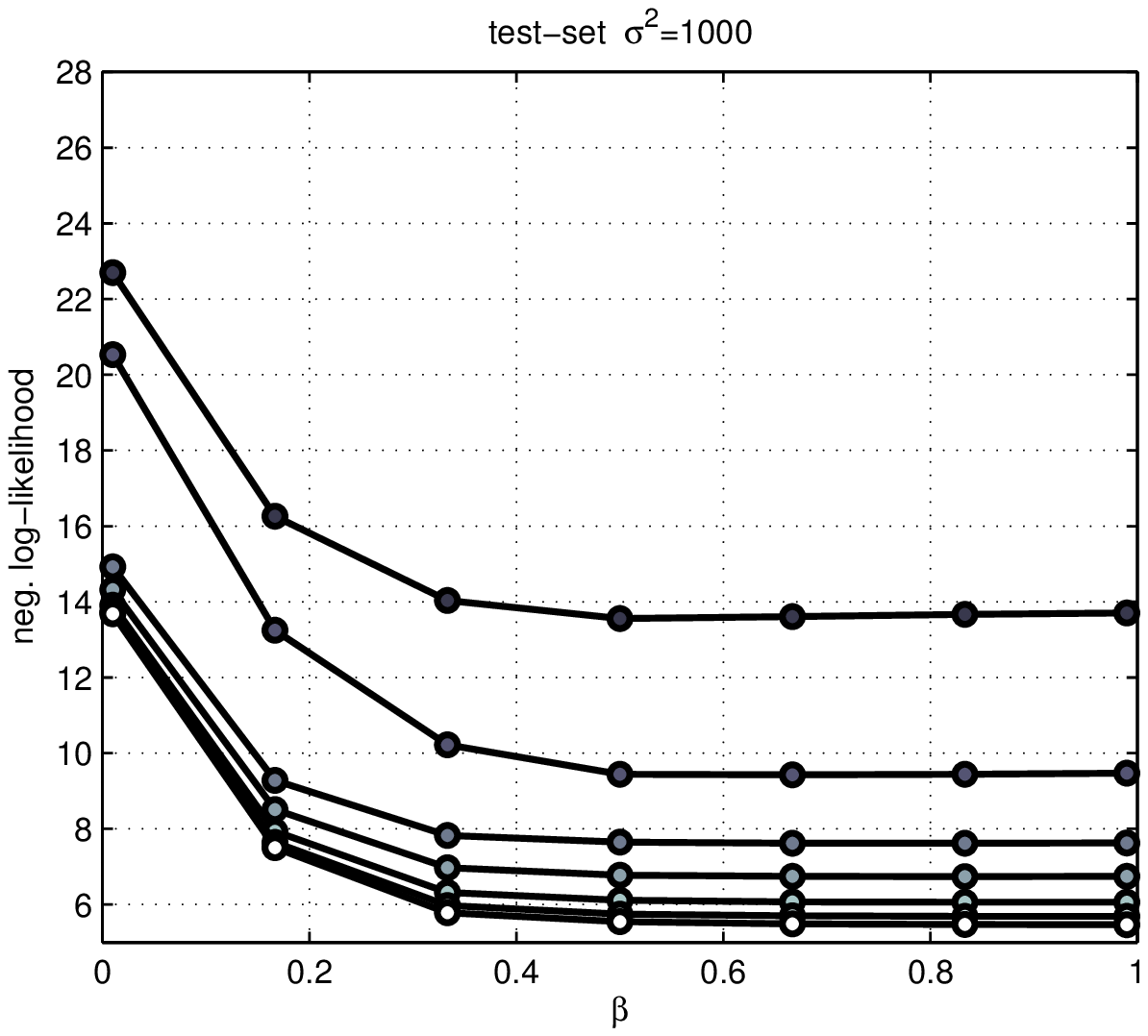}   &   \hspace{-5.5mm}
\includegraphics[trim=11.0mm 0mm 0mm 8.2mm,clip,totalheight=.178\textheight]{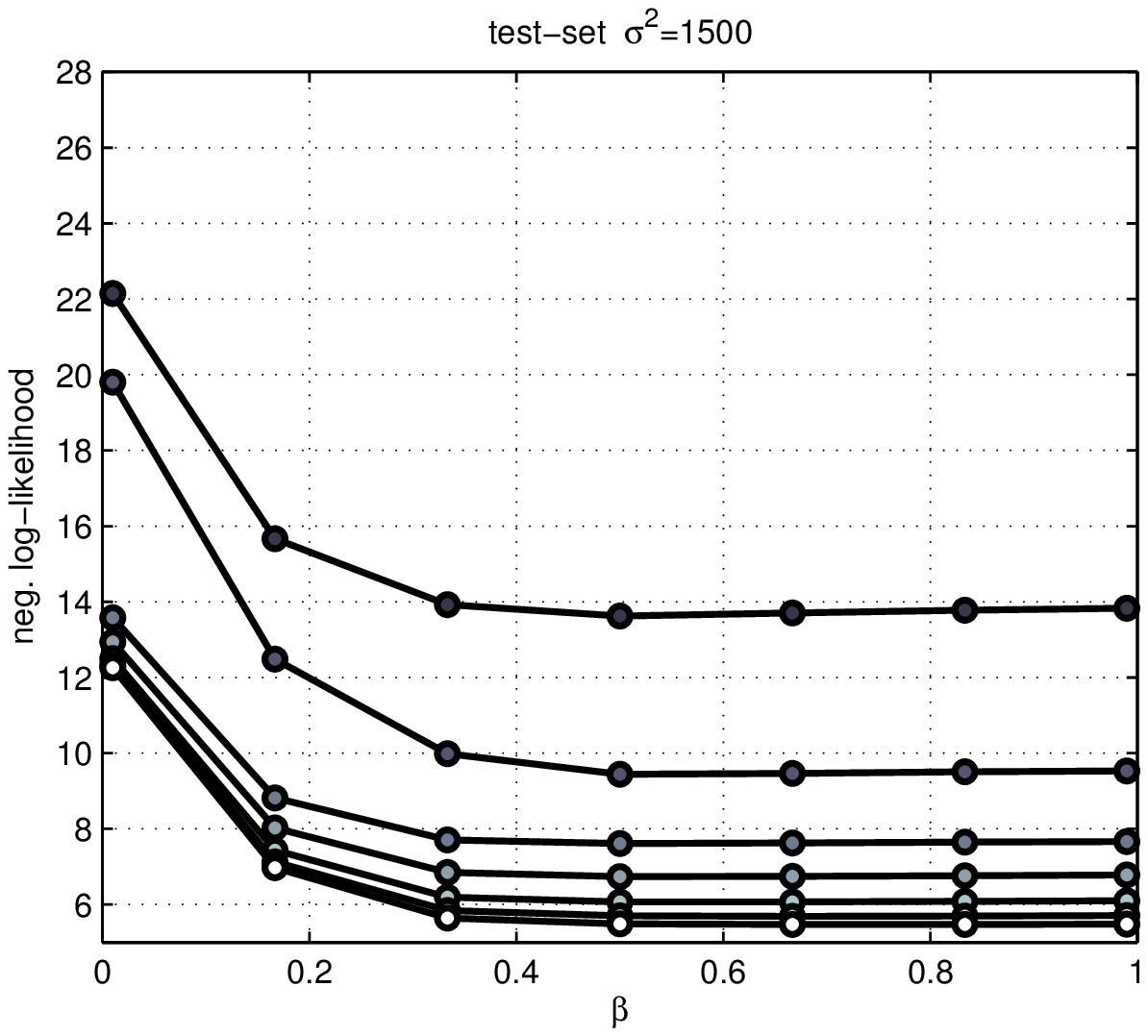}   &   \hspace{-5.5mm}
\includegraphics[trim=11.0mm 0mm 0mm 8.2mm,clip,totalheight=.178\textheight]{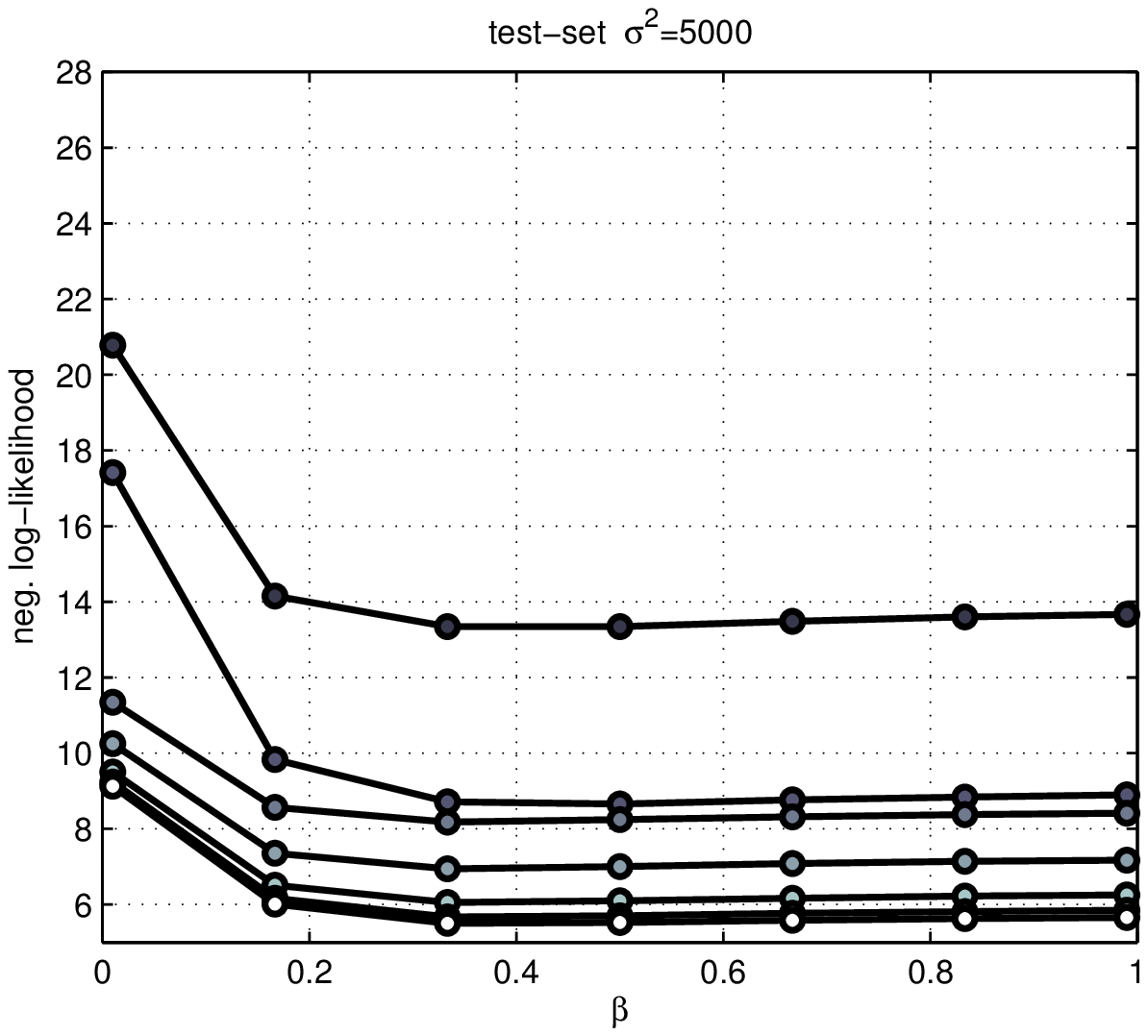} 
\end{tabular}
\vspace{-1.5em}
\caption{
	Train (top) and test (bottom) perplexities for a CRF with PL1/FL selection
	policy (x-axis:FL weight, y-axis:perplexity; see above and
	Fig.~\ref{fig:chunk_boltzchain_pl1fl_beta}).
	\vspace{1em}
	\newline
	Perhaps more evidently here than above, we note that the significance of a
	particular $\beta$ is less than that of the Boltzmann chain. However, for
	large enough $\sigma^2$, the optimal $\beta\ne 1$. This indicates the dual
	role of PL1 as a regularizer. Moreover, the left panel calls attention to
	the interplay between $\beta$, $\lambda$, and $\sigma^2$.  See
	Sec.~\ref{sec:interplay} for more discussion.
}
\label{fig:chunk_crf_pl1fl_beta}

\vspace{-1.5em}
\end{figure}

\begin{figure}[ht!]
\hspace{-0.75cm}
\begin{tabular}{cccc}
\includegraphics[trim= 0.0mm 10.2mm 0mm 9.5mm,clip,totalheight=.157\textheight]{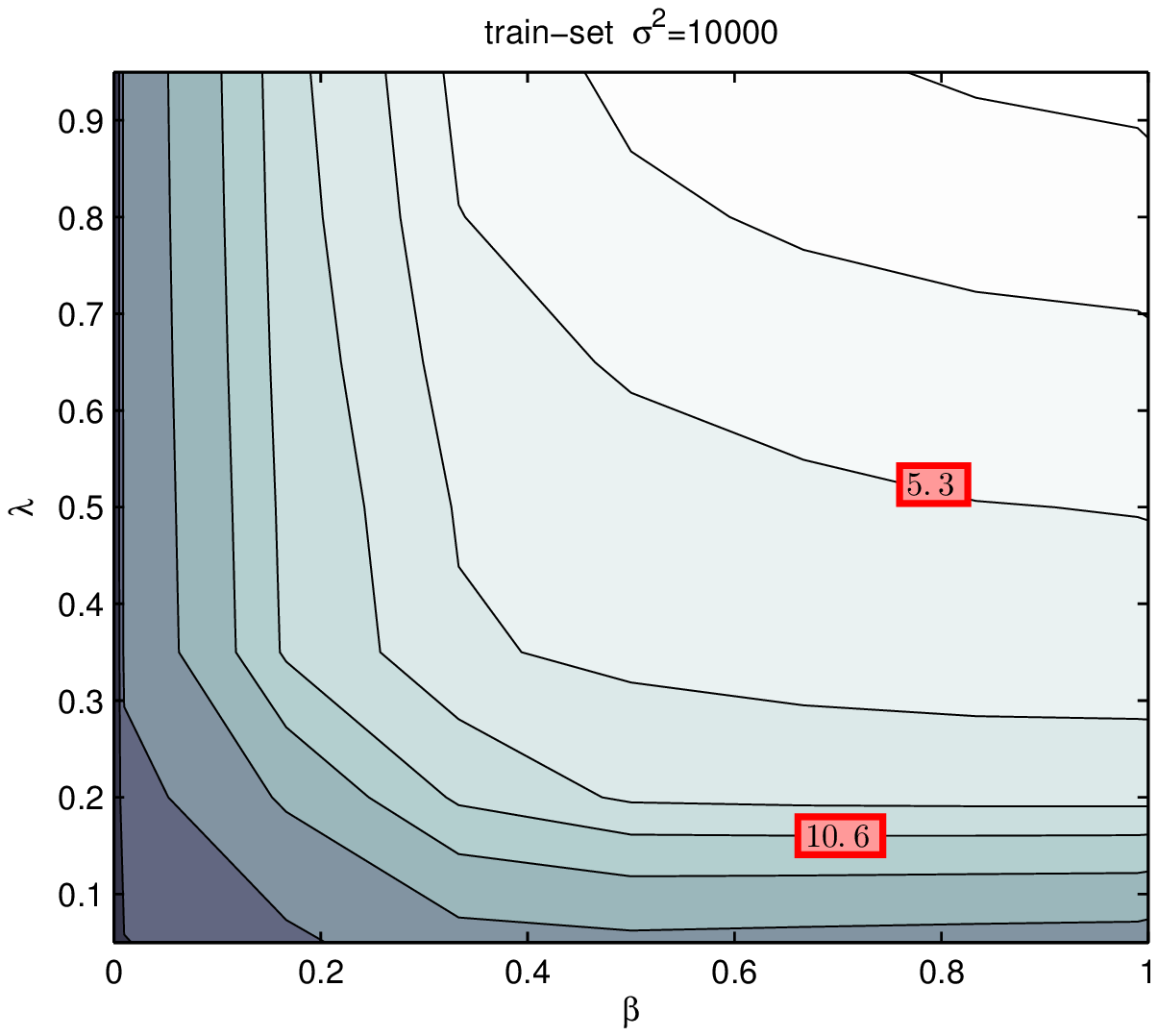} &  \hspace{-5.5mm}
\includegraphics[trim=12.0mm 10.2mm 0mm 9.5mm,clip,totalheight=.157\textheight]{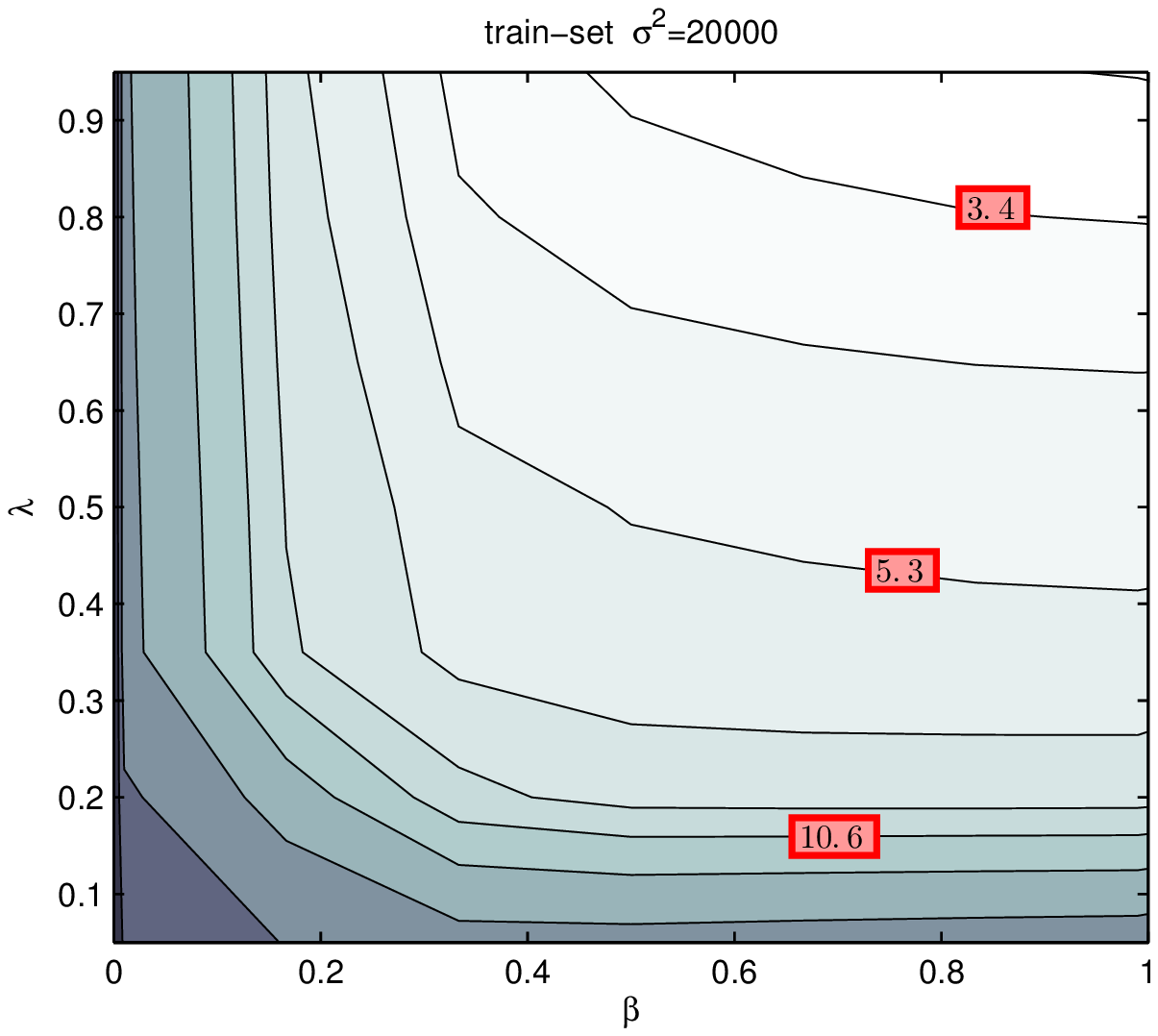} &  \hspace{-5.5mm}
\includegraphics[trim=12.0mm 10.2mm 0mm 9.5mm,clip,totalheight=.157\textheight]{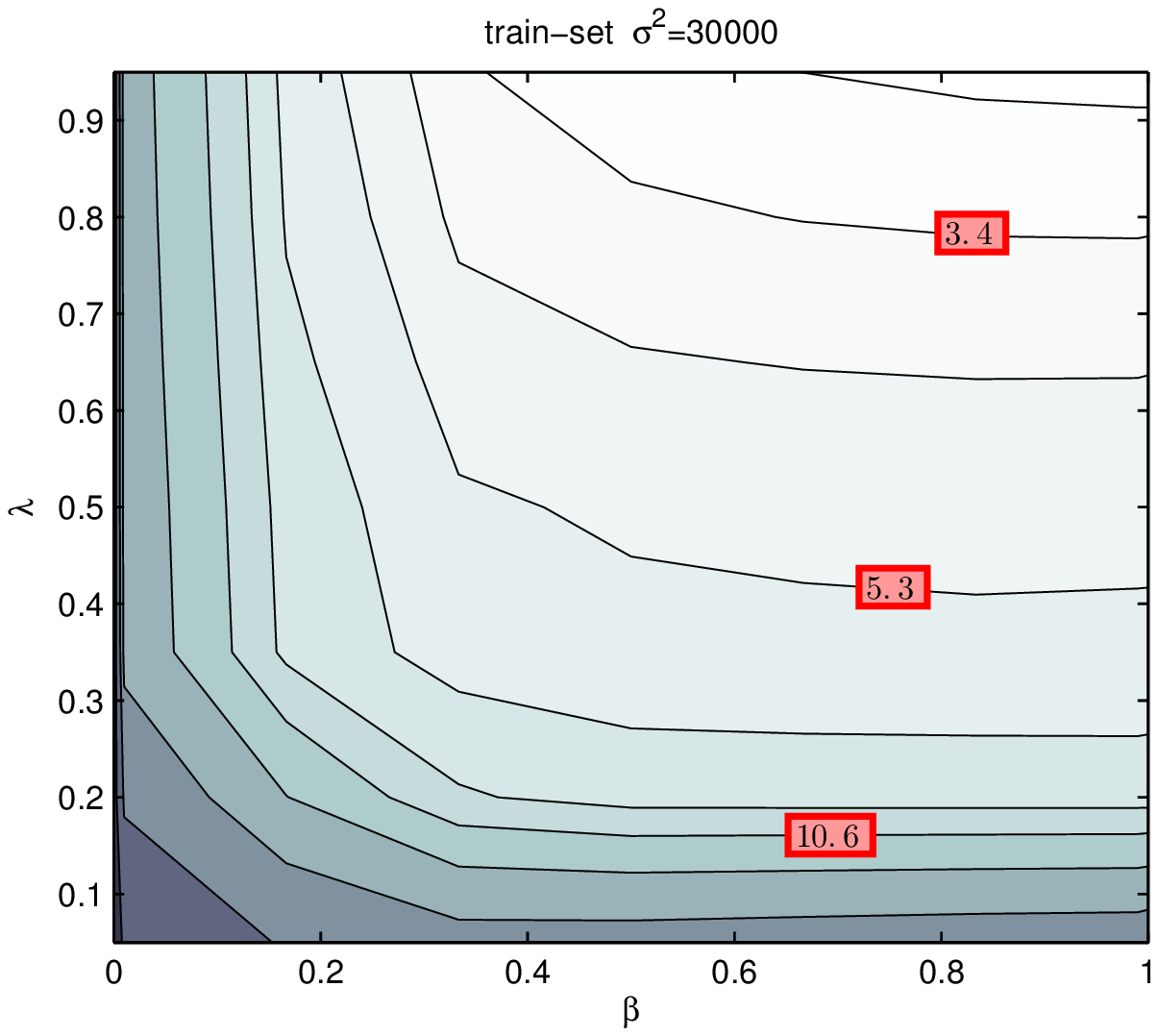} &  \hspace{-5.5mm}
\includegraphics[trim=12.0mm 10.2mm 0mm 9.5mm,clip,totalheight=.157\textheight]{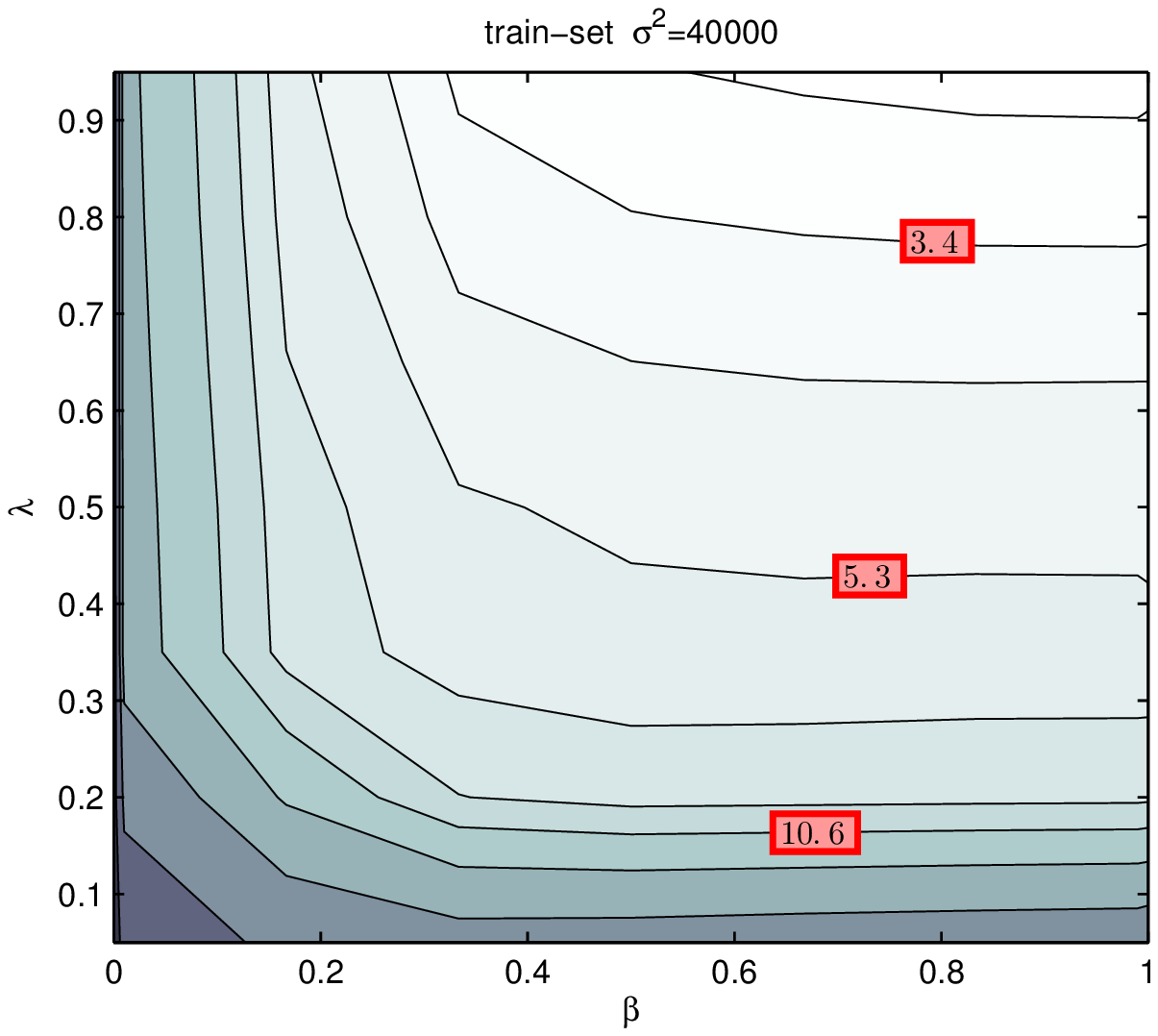} \\
\includegraphics[trim= 0.0mm 0mm 0mm 9.5mm,clip,totalheight=.175\textheight]{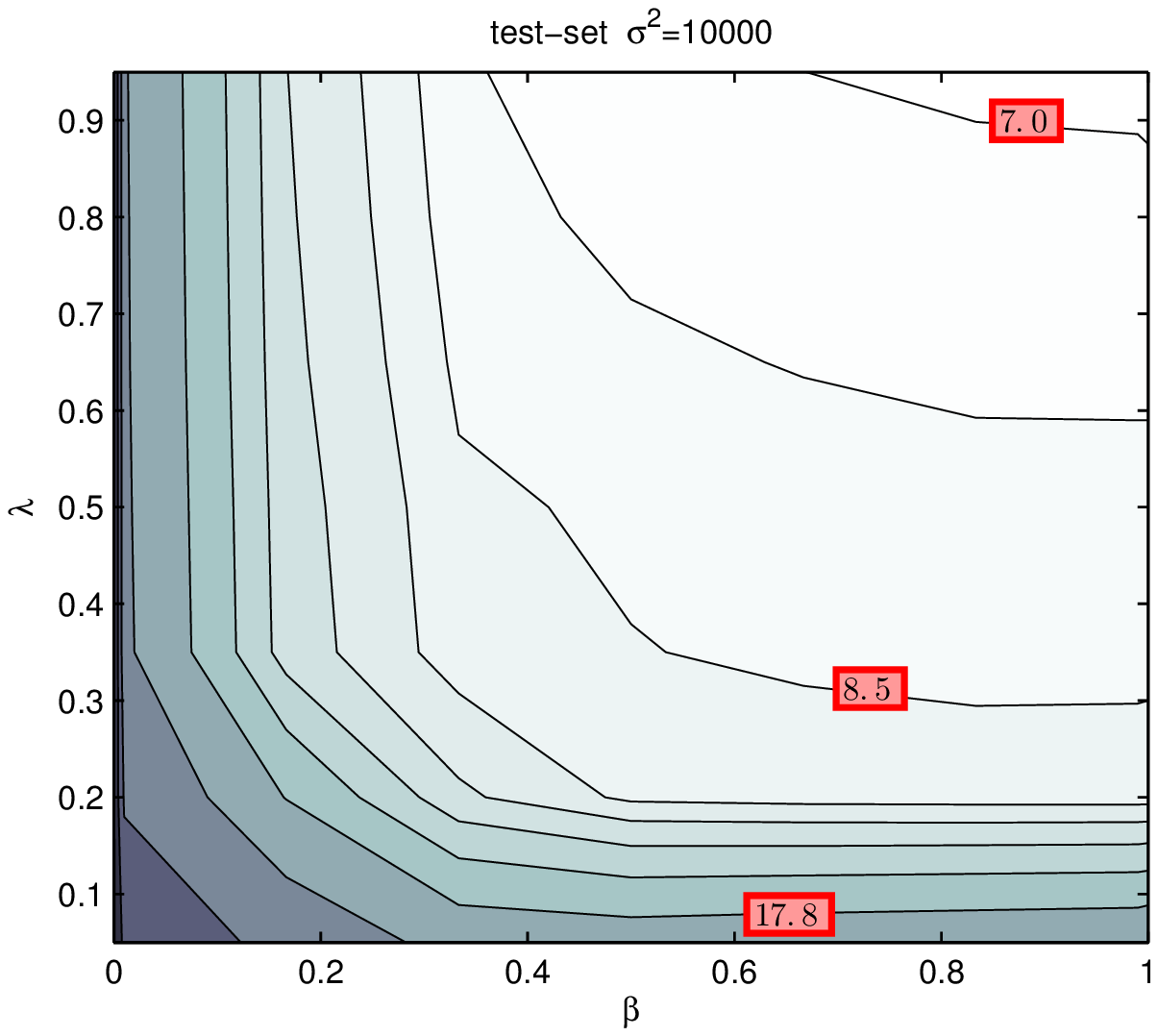}    &  \hspace{-5.5mm}
\includegraphics[trim=12.0mm 0mm 0mm 9.5mm,clip,totalheight=.175\textheight]{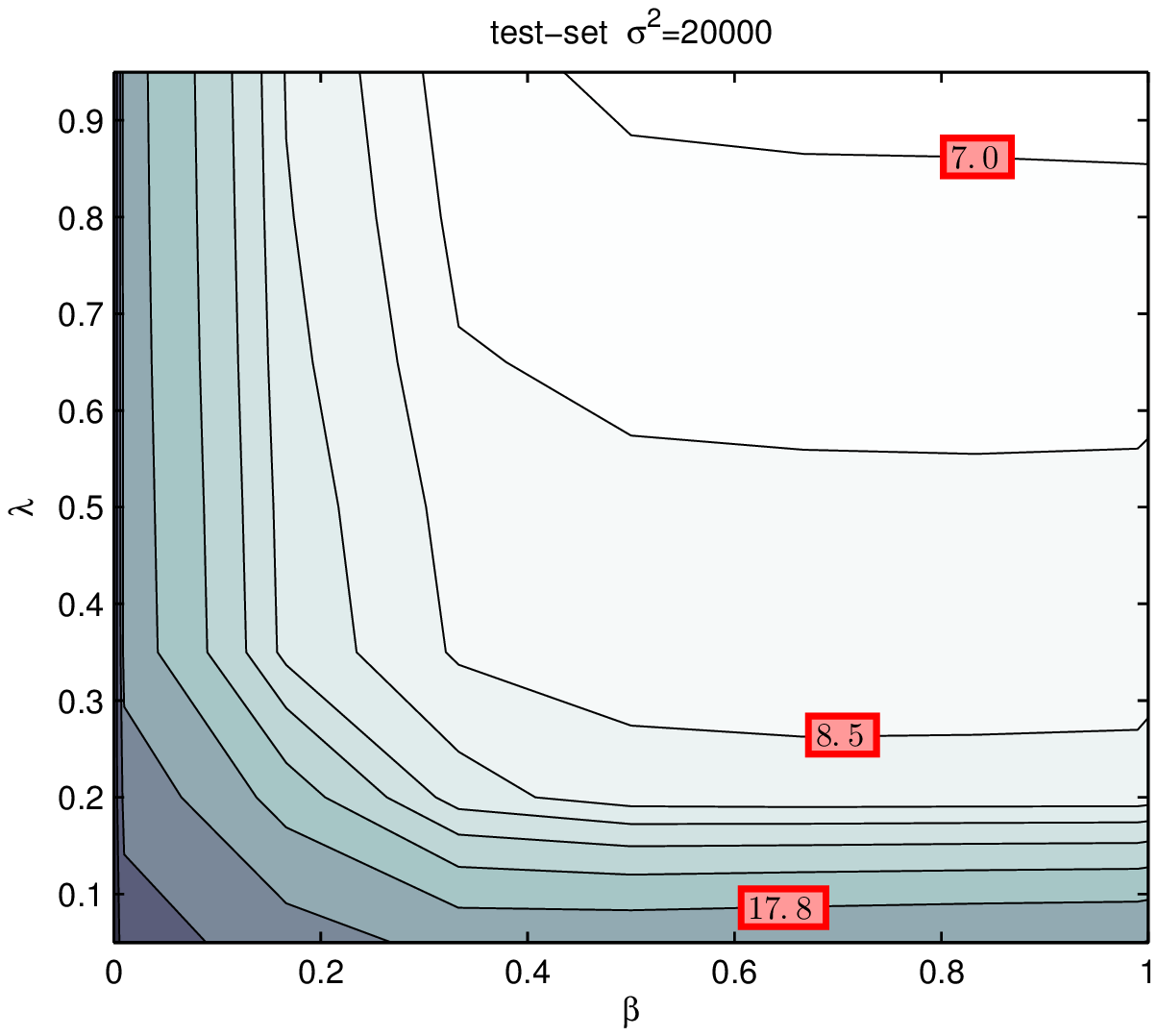}    &  \hspace{-5.5mm}
\includegraphics[trim=12.0mm 0mm 0mm 9.5mm,clip,totalheight=.175\textheight]{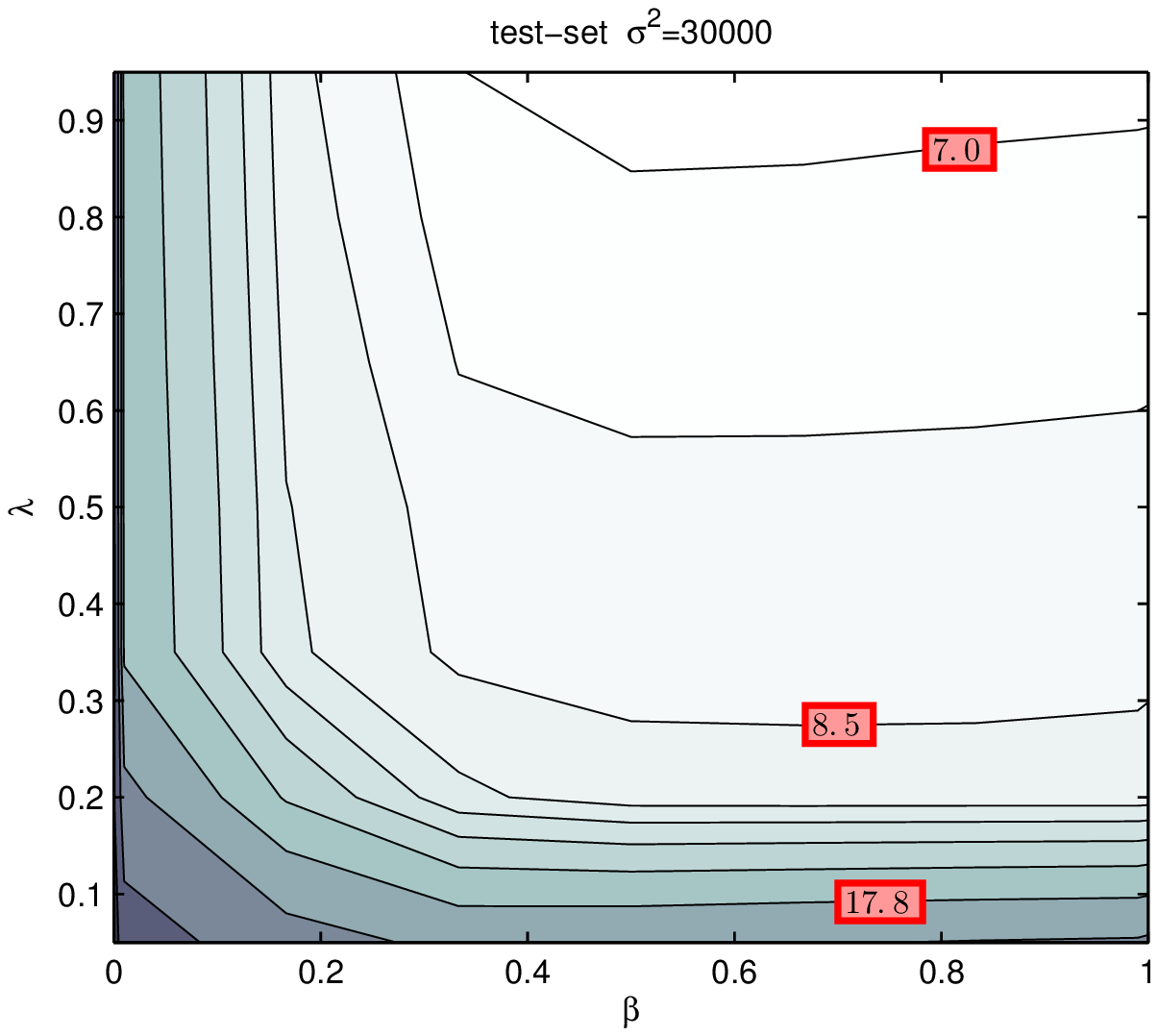}    &  \hspace{-5.5mm}
\includegraphics[trim=12.0mm 0mm 0mm 9.5mm,clip,totalheight=.175\textheight]{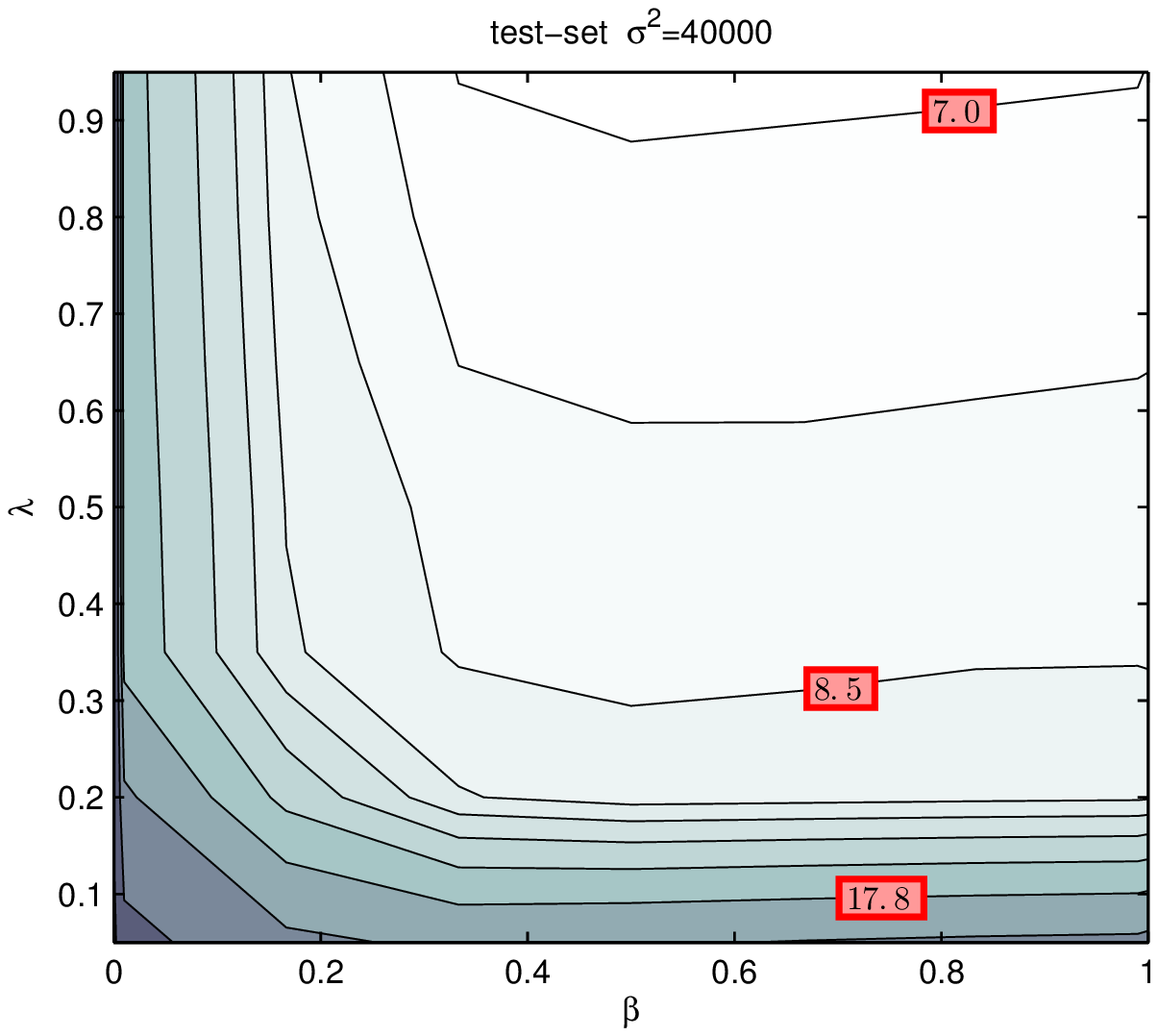} 
\end{tabular}
\vspace{-1.5em}
\caption{
	Train set (top) and test set (bottom) perplexity for a CRF with 1st/2nd
	order pseudo likelihood selection policy (PL1/PL2).  The x-axis, $\beta$,
	represents the relative weight placed on PL2 and the y-axis, $\lambda$, the
	probability of selecting PL2.  PL1 is selected with probability 1.  Columns
	from left to right correspond to weaker regularization, $\sigma^2=\{10000,
	20000, 30000, 40000\}$.  See Figure~\ref{fig:chunk_crf_pl1fl_cont} for more
	details.
}\label{fig:chunk_crf_pl1pl2_cont}

\vspace{0.5em}
\hspace{-0.75cm}
\begin{tabular}{cccc}
\includegraphics[trim=  0.0mm 9.1mm 0mm 8.2mm,clip,totalheight=.162\textheight]{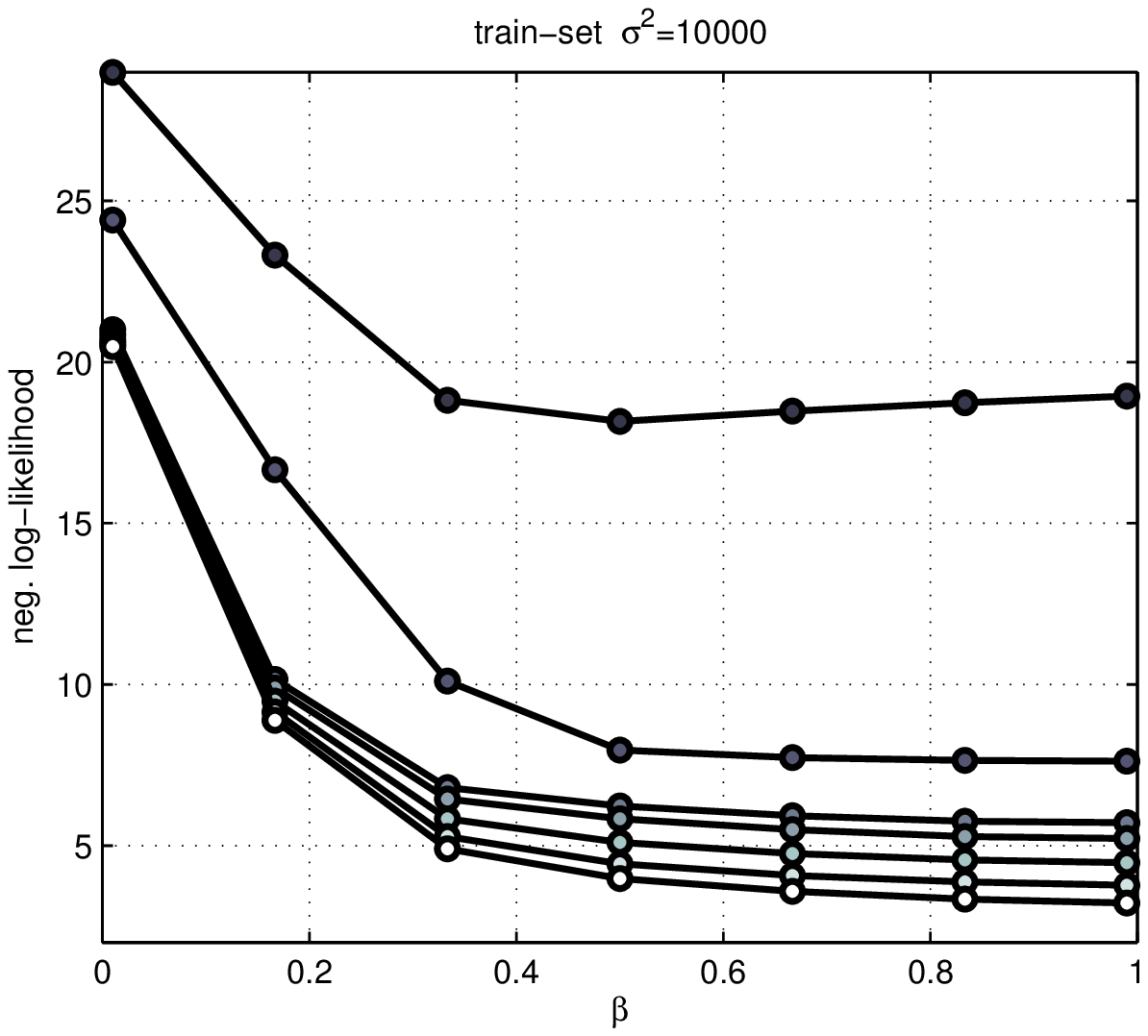} &   \hspace{-5.5mm}
\includegraphics[trim= 11.0mm 9.1mm 0mm 8.2mm,clip,totalheight=.162\textheight]{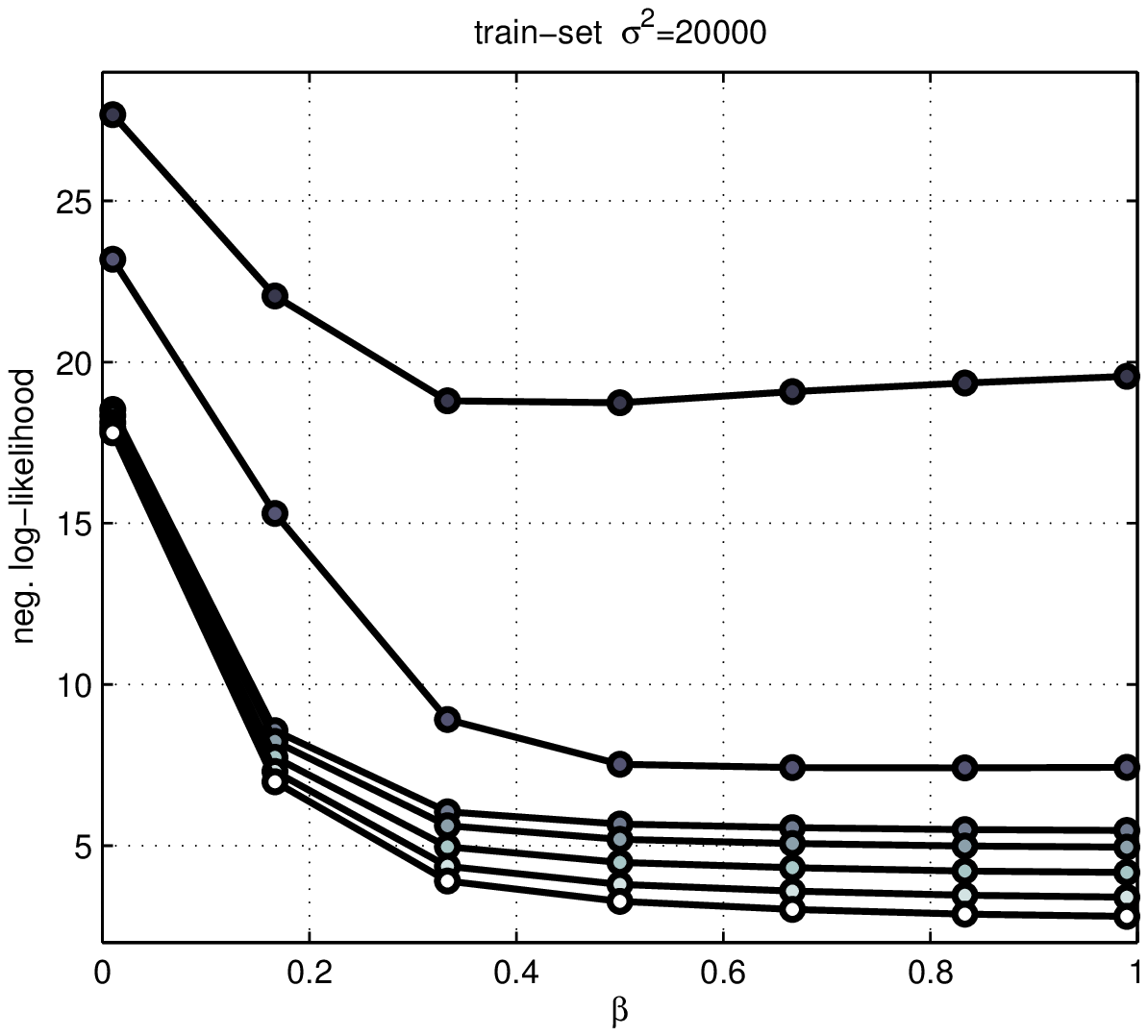} &   \hspace{-5.5mm}
\includegraphics[trim= 11.0mm 9.1mm 0mm 8.2mm,clip,totalheight=.162\textheight]{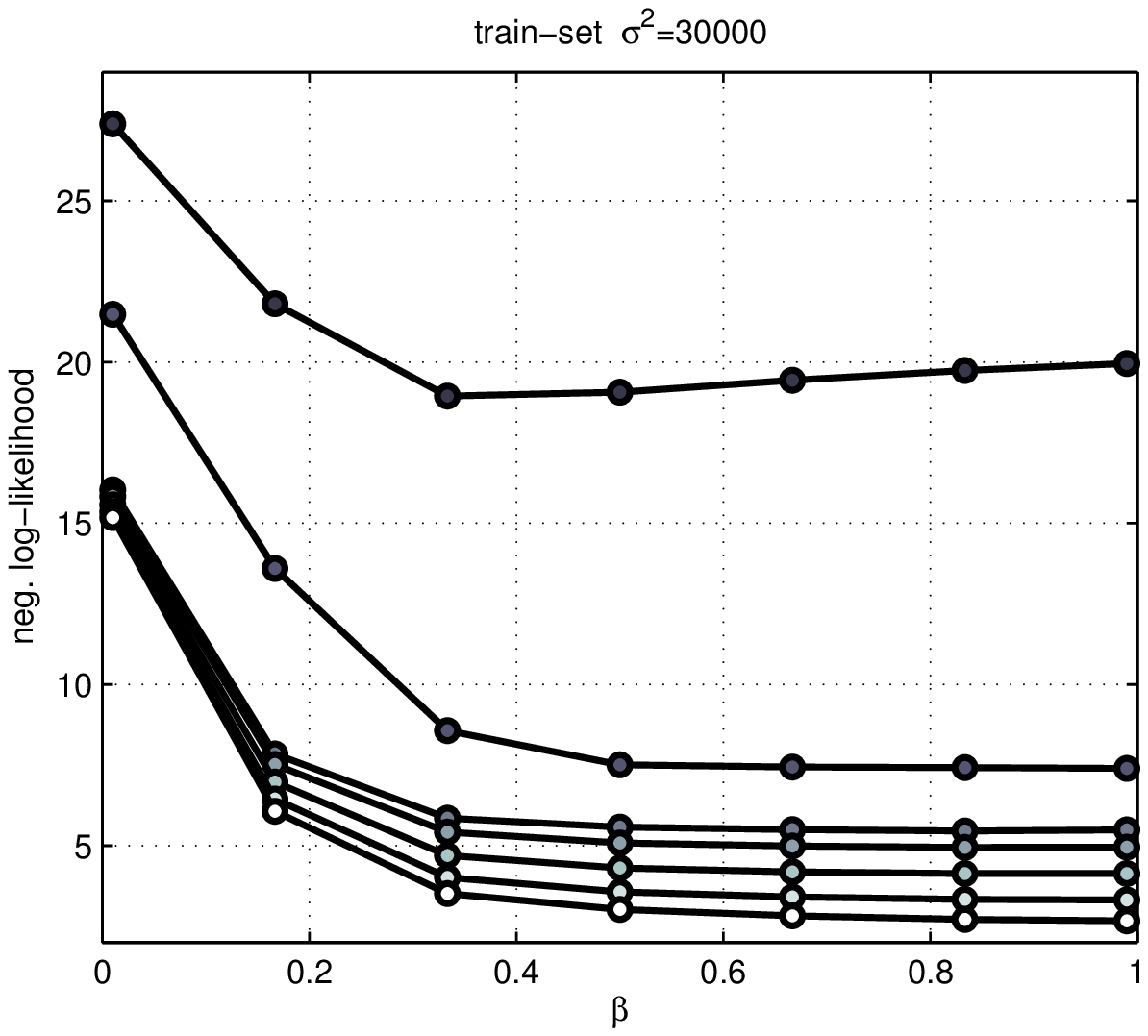} &   \hspace{-5.5mm}
\includegraphics[trim= 11.0mm 9.1mm 0mm 8.2mm,clip,totalheight=.162\textheight]{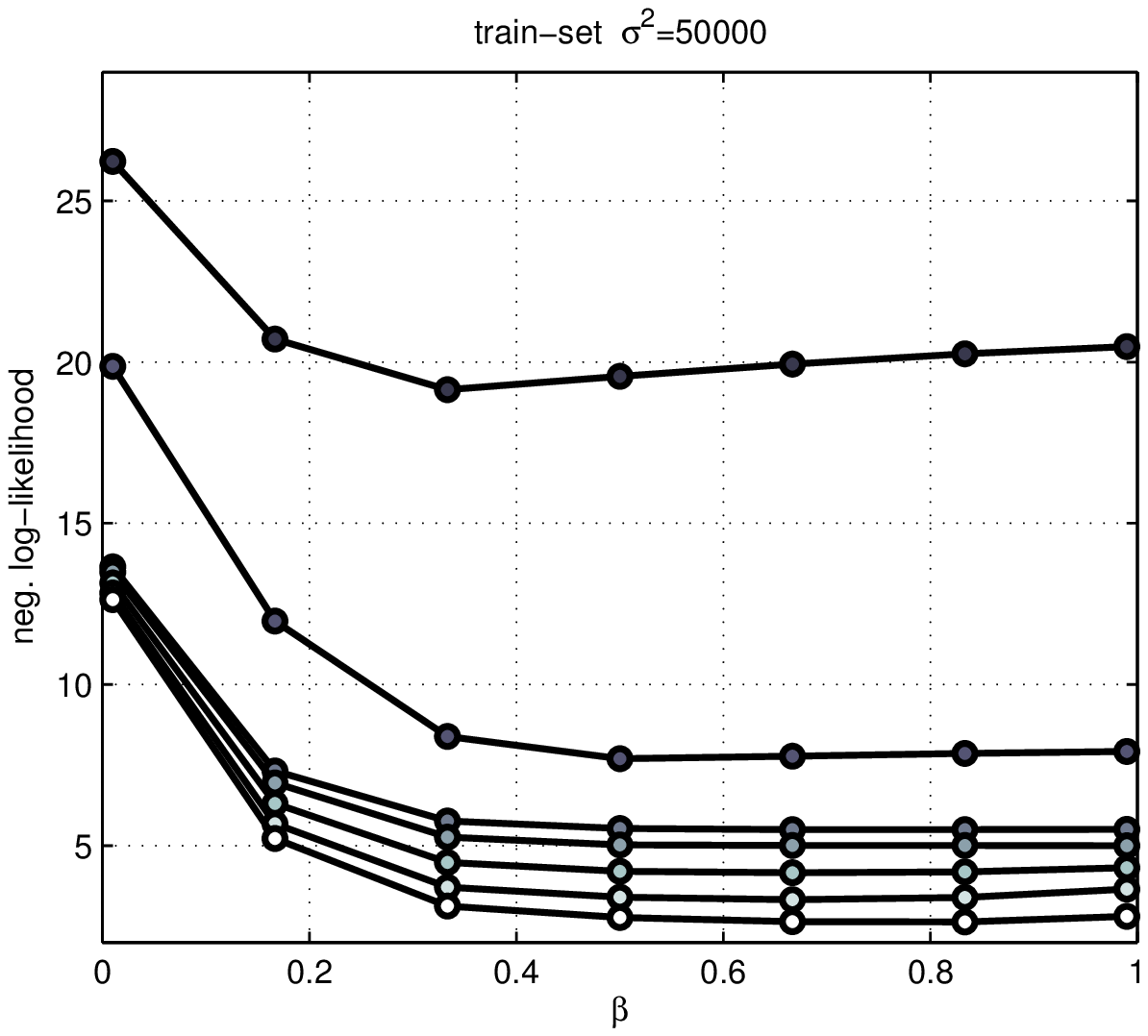} \\
\includegraphics[trim=  0.0mm 0mm 0mm 8.2mm,clip,totalheight=.178\textheight]{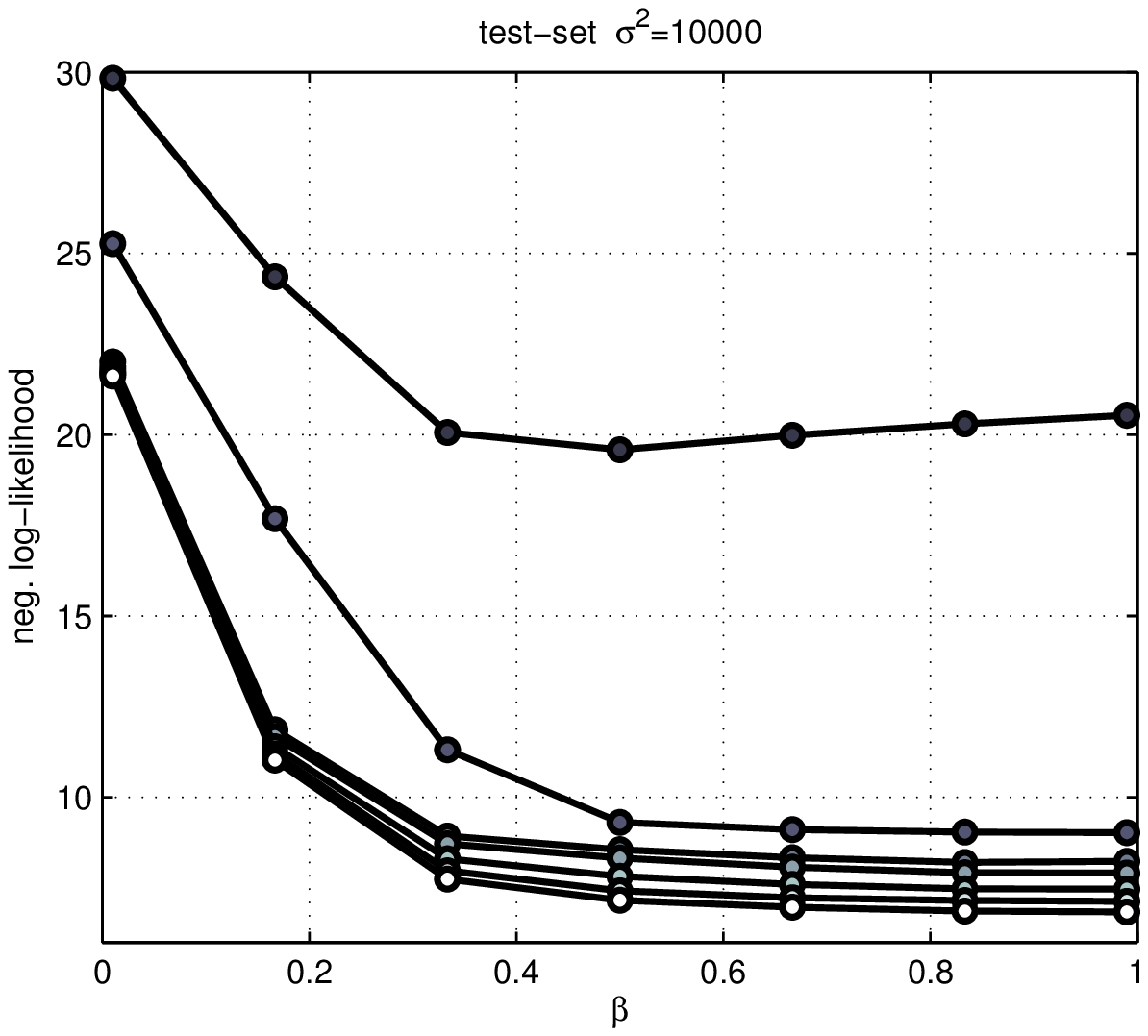}   &   \hspace{-5.5mm}
\includegraphics[trim= 11.0mm 0mm 0mm 8.2mm,clip,totalheight=.178\textheight]{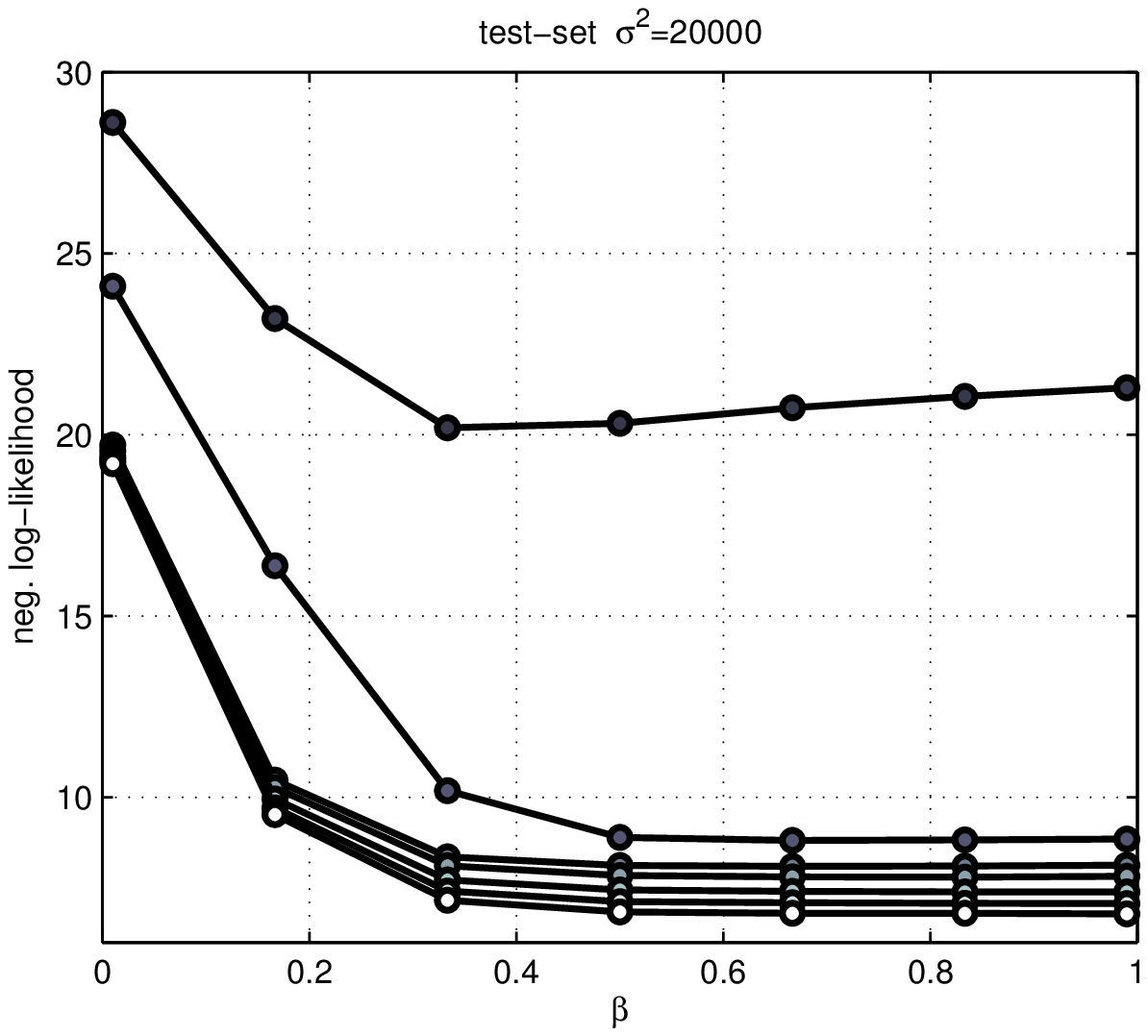}   &   \hspace{-5.5mm}
\includegraphics[trim= 11.0mm 0mm 0mm 8.2mm,clip,totalheight=.178\textheight]{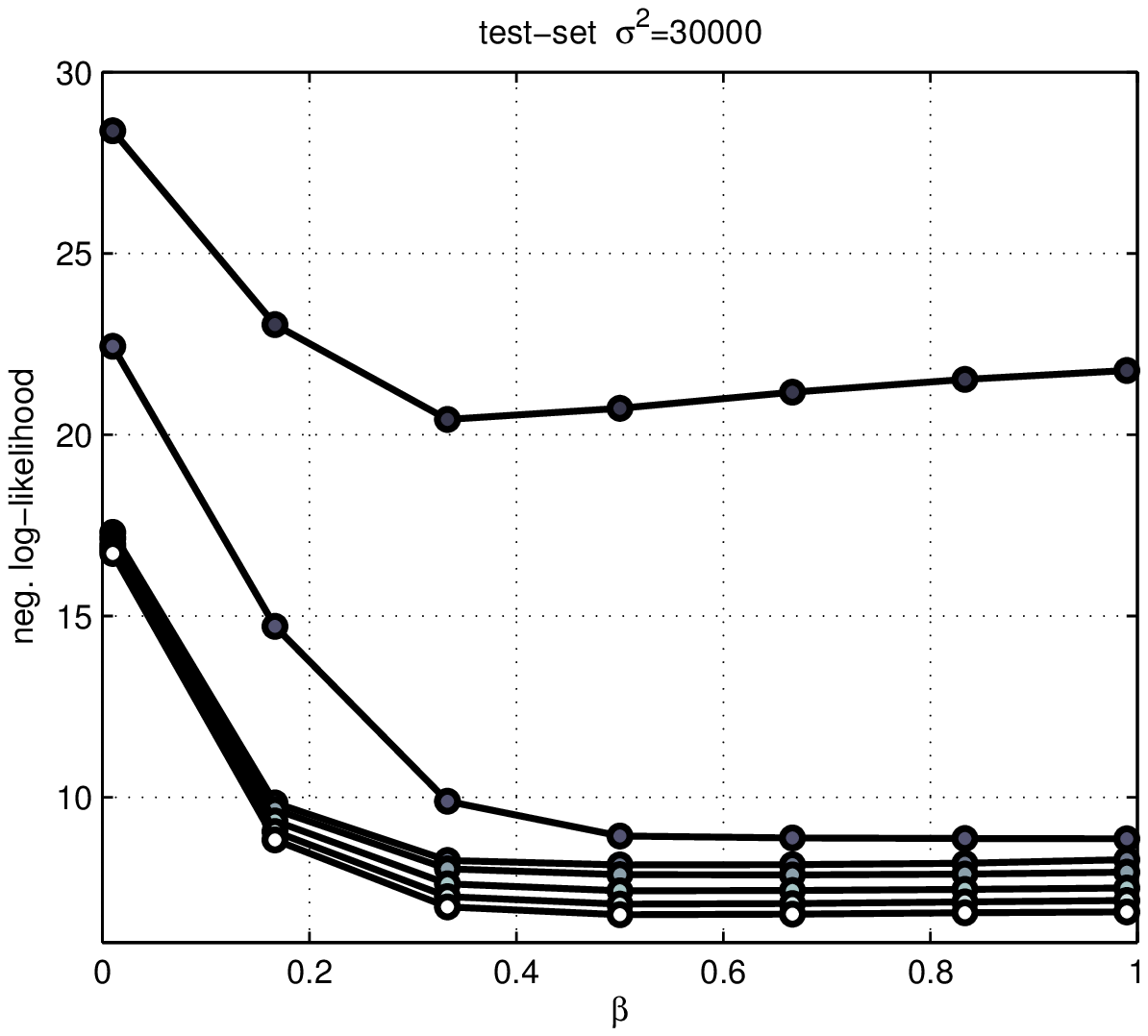}   &   \hspace{-5.5mm}
\includegraphics[trim= 11.0mm 0mm 0mm 8.2mm,clip,totalheight=.178\textheight]{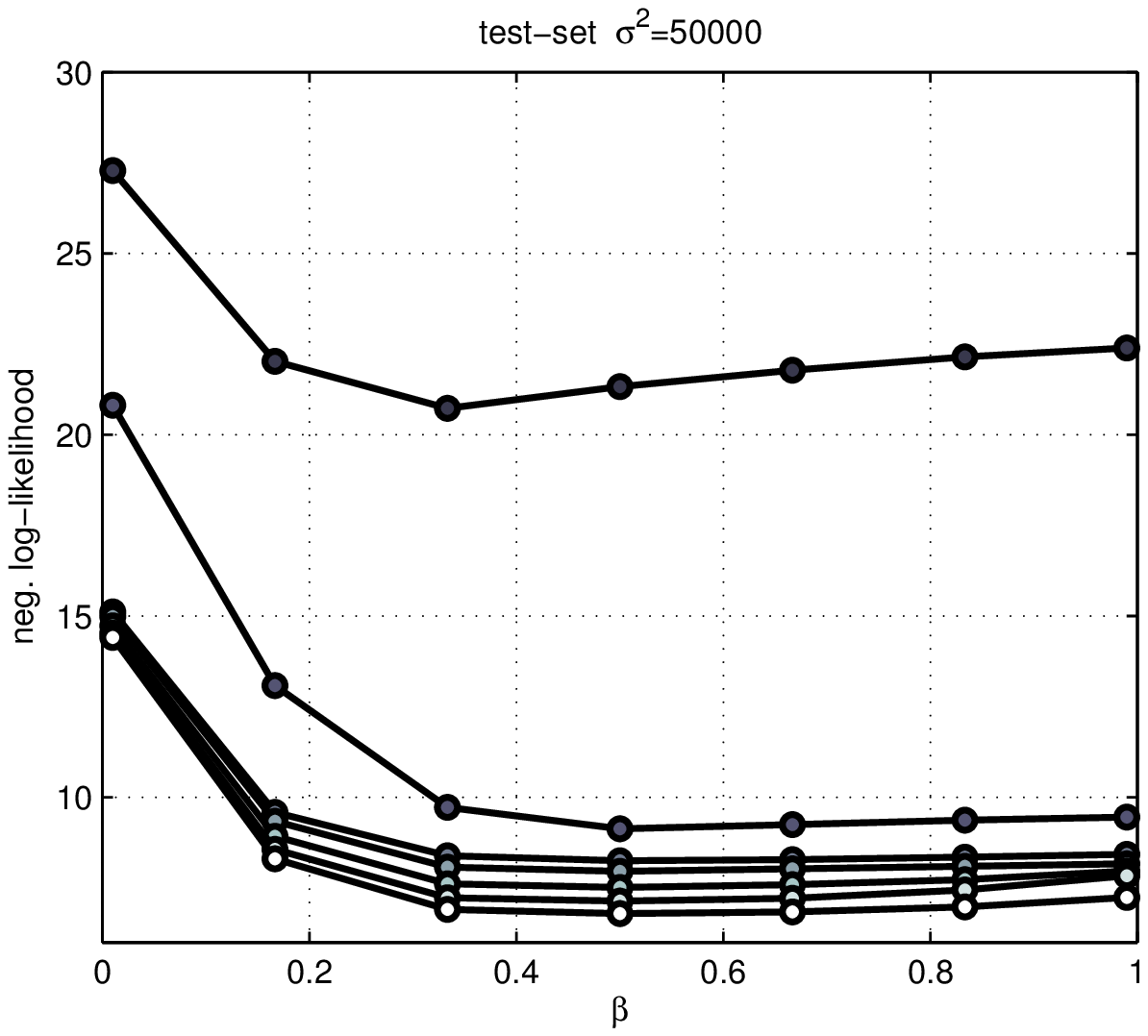} 
\end{tabular}
\vspace{-1.5em}
\caption{
	Train (top) and test (bottom) perplexities for a CRF with PL1/PL2 selection
	policy (x-axis:PL2 weight, y-axis:perplexity; see above and
	Fig.~\ref{fig:chunk_boltzchain_pl1fl_beta}).
		\vspace{1em}
	\newline
	Although increasing $\lambda$ only brings minor improvement to both the
	training and testing perplexities, it is worth noting that the test
	perplexity meets that of the PL1/FL.  Still though, the overall lack of
	resolution here suggests that smaller values of $\lambda$ would better span
	a range of perplexities and at reduced computational cost.
}\label{fig:chunk_crf_pl1pl2_beta}

\end{figure}

\begin{figure}[ht!]
\hspace{-0.5cm}
\begin{tabular}{ccc}
\includegraphics[trim=0mm 5.5mm 0mm 0.2mm,clip,totalheight=.20\textheight]{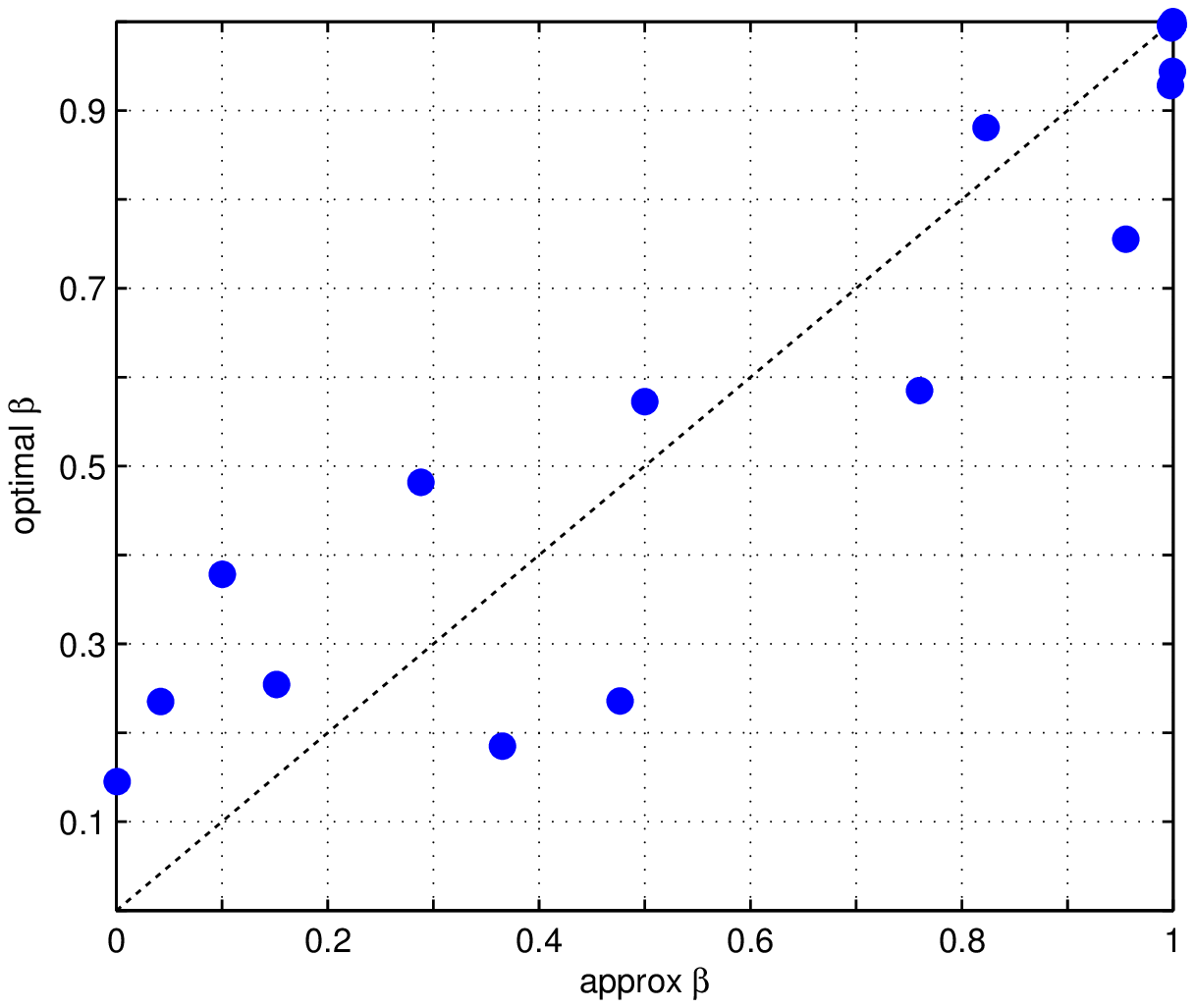}  &  \hspace{-5.0mm}
\includegraphics[trim=5mm 5.5mm 0mm 6.2mm,clip,totalheight=.20\textheight]{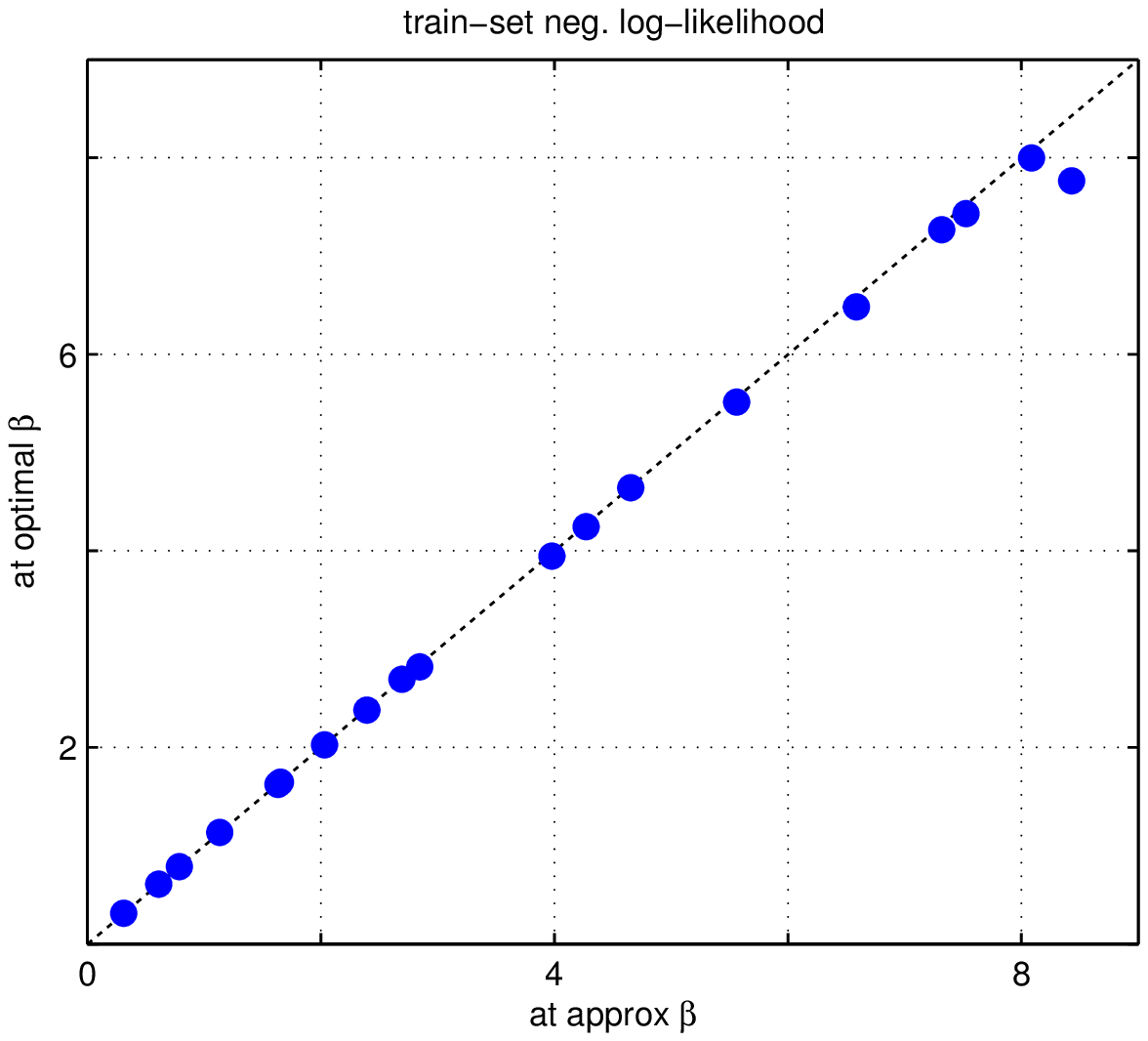}     &  \hspace{-5.5mm}
\includegraphics[trim=5mm 5.5mm 0mm 6.2mm,clip,totalheight=.20\textheight]{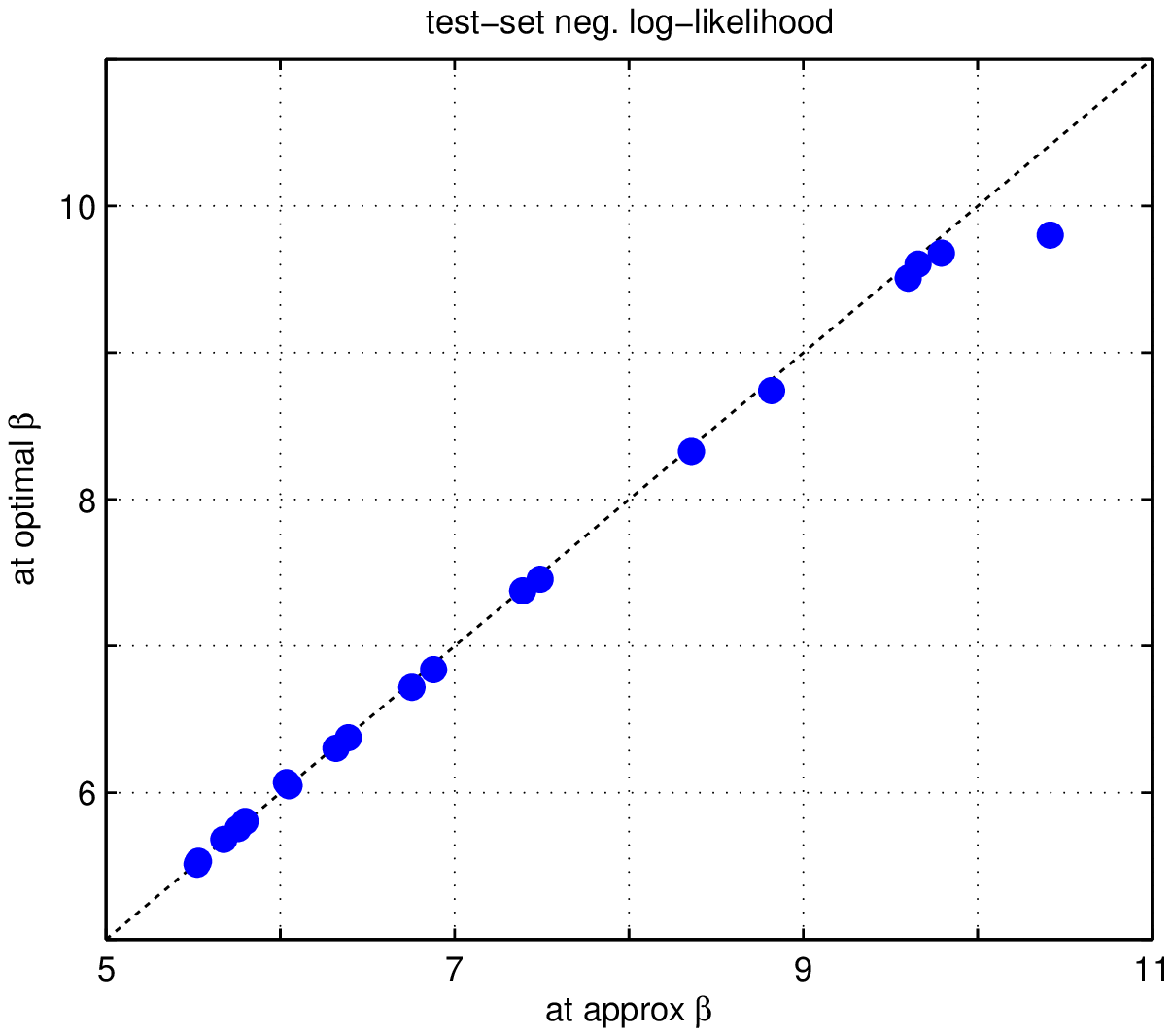}     \\
\includegraphics[trim=0mm 0.0mm 0mm 0.2mm,clip,totalheight=.21\textheight]{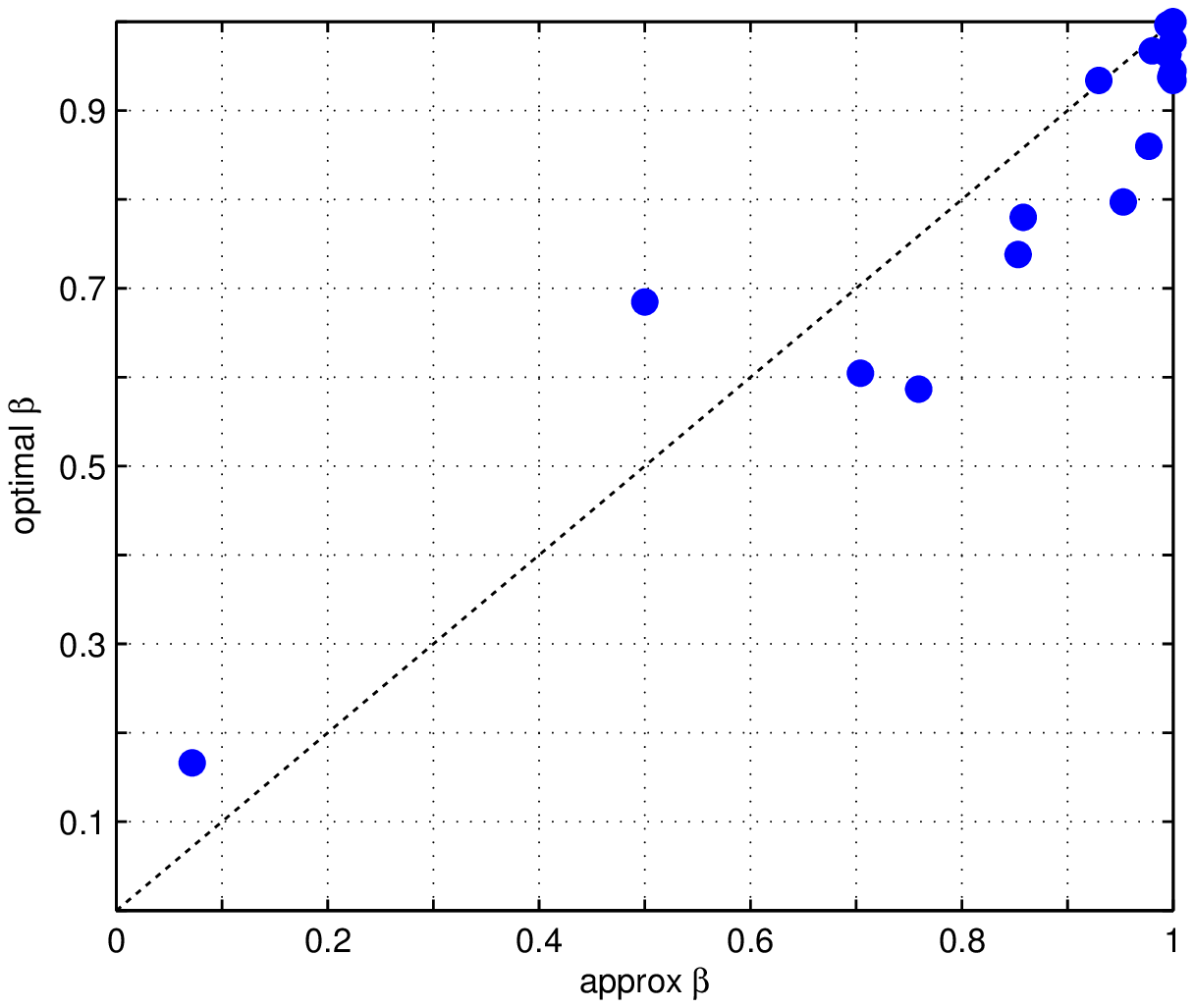} &  \hspace{-5.0mm}
\includegraphics[trim=5mm 0.0mm 0mm 6.2mm,clip,totalheight=.21\textheight]{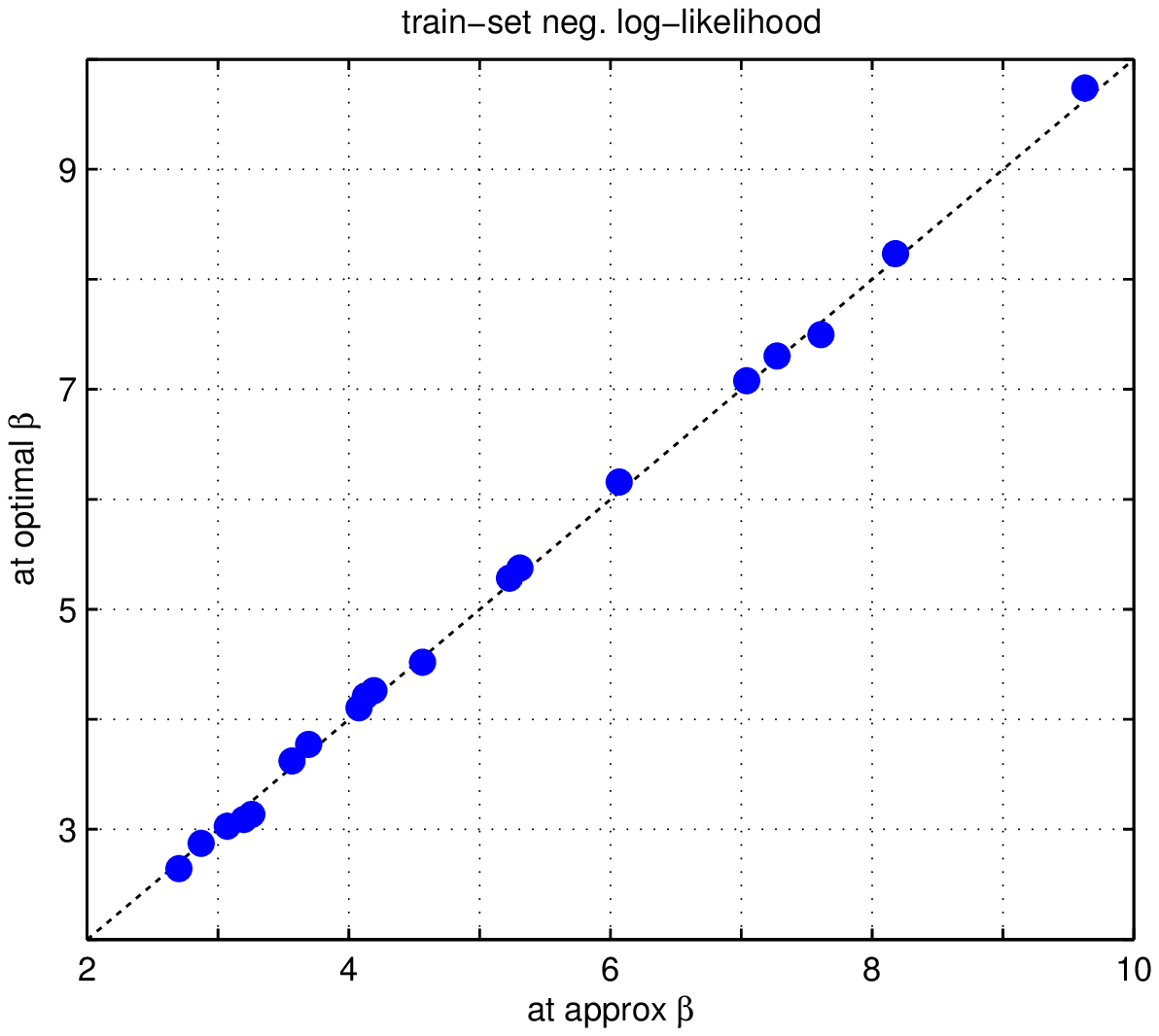}    &  \hspace{-5.5mm}
\includegraphics[trim=5mm 0.0mm 0mm 6.2mm,clip,totalheight=.21\textheight]{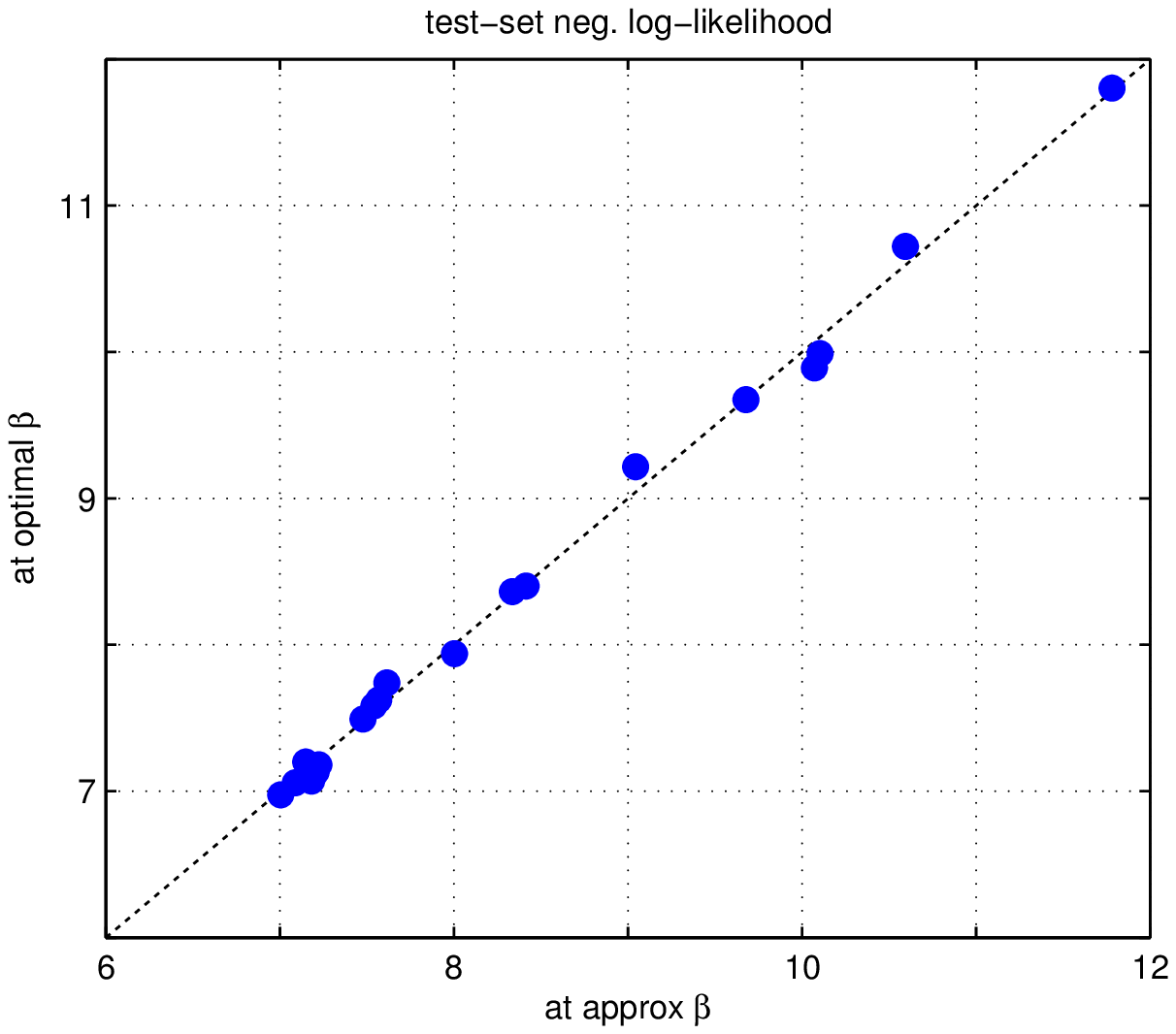}
\end{tabular}
\vspace{-1.5em}
\caption{
	Demonstration of the effectiveness of the $\beta$ heuristic.  Results are
	for the CRF with PL1/FL (top) and PL1/PL2 (bottom) selection policies.  The
	x-axis is the value at the heuristically found $\beta$ and the y-axis the
	value at the optimal $\beta.$ The first column depicts the best performing
	$\beta$ against the heuristic $\beta$.  The second and third columns depict
	the training and testing perplexities (resp.) at the best performing
	$\beta$ and heuristically found $\beta$.  For all three columns, we assess
	the effectiveness of the heuristic by its nearness to the diagonal (dashed
	line).  See Fig.~\ref{fig:chunk_boltzchain_heuristic} for more details.
	\vspace{1em}
	\newline
	The optimal and heuristic $\beta$ match train and test perplexities for
	both policies.  The actual $\beta$ value however does not seem to match as
	well as the Boltzmann chain.  However, if we note the flatness of the
	$\beta$ grid (cf. Fig.~\ref{fig:chunk_crf_pl1fl_beta} and
	\ref{fig:chunk_crf_pl1pl2_beta}) this result is unsurprising and can be
	disregarded as an indication of the heuristic's performance.
}\label{fig:chunk_crf_heuristic}
\end{figure}

\subsection{Complexity/Regularization Win-Win}\label{sec:winwin}
It is interesting to contrast the test loglikelihood behavior in the case of
mild and stronger $L_2$ regularization. In the case of weaker or no
regularization, the test loglikelihood shows different behavior than the train
loglikelihood. Adding a lower order component such as pseudo likelihood acts as
a regularizer that prevents overfitting. Thus, in cases that are prone to
overfitting reducing higher order likelihood components improves both
performance as well as complexity. This represents a win-win situation in
contrast to the classical view where the mle has the lowest variance and adding
lower order components reduces complexity but increases the variance.

In Figure \ref{fig:localsent_crf_cont} we note this phenomenon when comparing
$\sigma^2=1$ to $\sigma^2=10$ across the selection policies PL1/FL and PL1/PL2.
That is, the weaker regularization and more restrictive selection policy, i.e.,
PL1/PL2, is able to achieve comparable test set perplexity.

For the text chunking experiments, we observe a striking win-win when using the
Boltzmann chain MRF, Figures \ref{fig:chunk_boltzchain_pl1fl_cont} and
\ref{fig:chunk_boltzchain_pl1pl2_cont}. Notice that as regularization is
decreased (comparing from left to right), the contours are pulled  closer to
the x-axis.  This means that we are achieving the same perplexity at reduced
levels of computational complexity.  The CRF however, only exhibits the win-win
to a minor extent.  We delve deeper into why this is might be the case in the
following section.

\subsection{$\lambda$, $\sigma^2$ Interplay}\label{sec:interplay}

Throughout these experiments we fixed $\sigma^2$ and either swept over
$(\lambda,\beta)$ or used the heuristic to evaluate $(\lambda,\beta(\lambda))$.
Motivated by the sometimes weak win-win (cf.\ Section \ref{sec:winwin}) we now
consider how the optimal $\sigma^2$ changes as a function of $\lambda$.  In
Figure \ref{fig:chunk_crf_lamsig1} we used the $\beta$ heuristic to evaluate
train and test perplexity over a $(\lambda,\sigma^2)$ grid.  We used CRFs and
the text chunking task as outlined in Section \ref{sec:chunk_crf}.

For the PL1/FL policy, we observe that for small enough $\lambda$ the optimal
$\sigma^2$, i.e., the $\sigma^2$ with smallest test perplexity, has
considerable range. At some point there are enough samples of the higher-order
component to stabilize the choice of regularizer, noting that it is still
weaker than the optimal full likelihood regularizer.  Conversely, the PL1/PL2
regularizer has an essentially constant optimal regularizer which is relatively
much weaker.

\begin{figure}[ht!]
\centering
\begin{tabular}{cc}
\includegraphics[trim=0mm 0mm 0mm 0mm,clip,totalheight=.22\textheight]{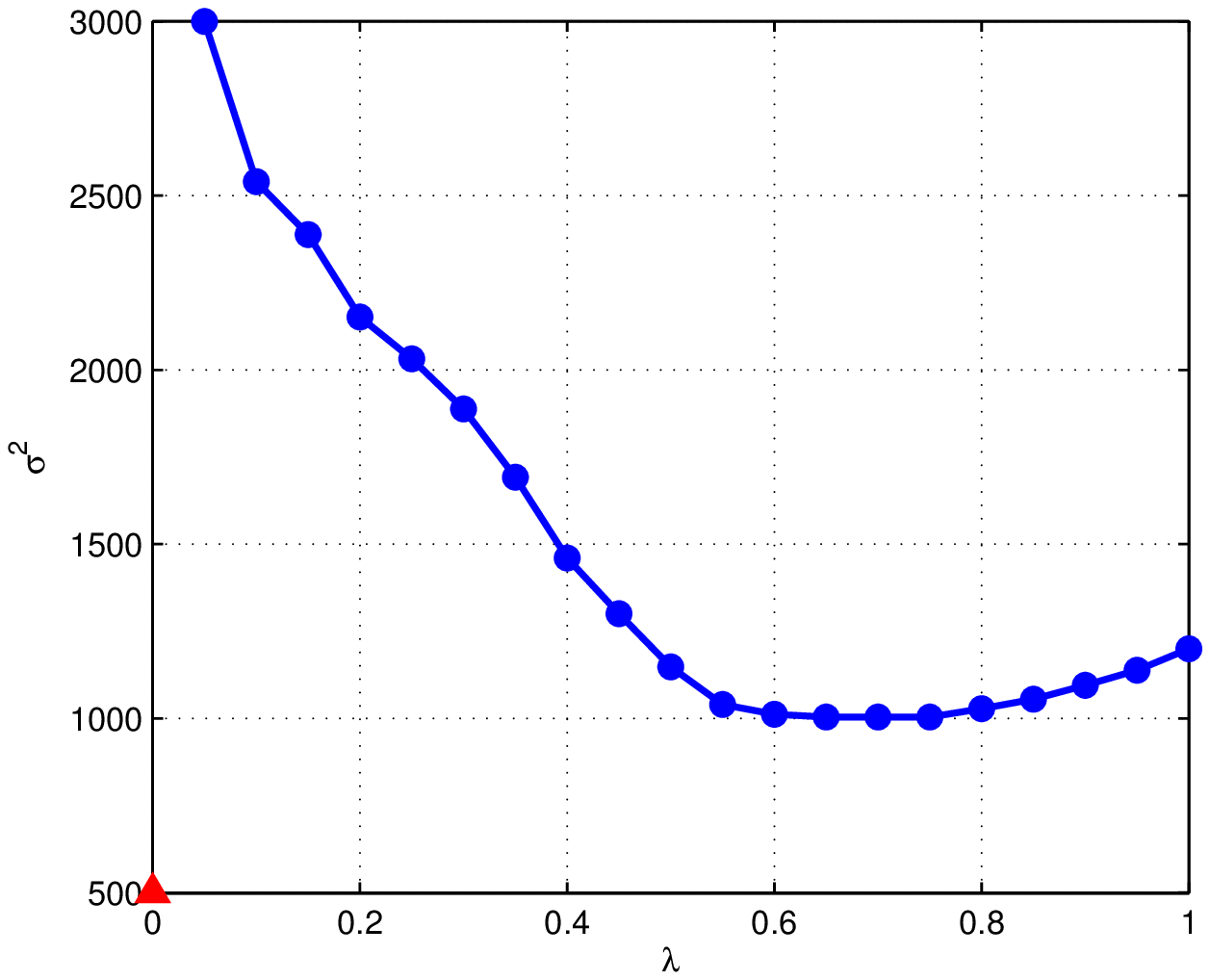} &
\includegraphics[trim=0mm 0mm 0mm 0mm,clip,totalheight=.22\textheight]{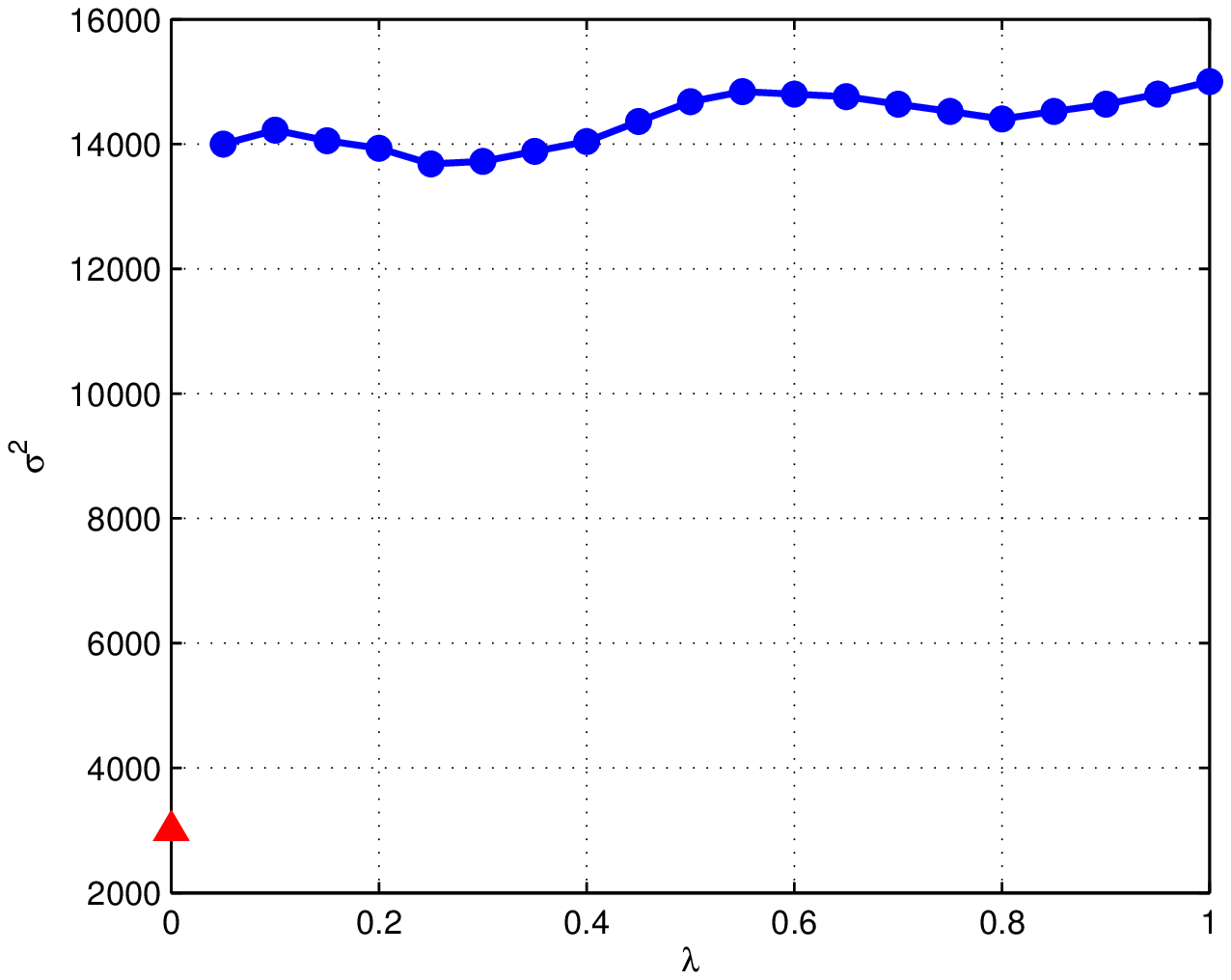}
\end{tabular}
\caption{
	Optimal regularization parameter as a function of
	$(\lambda,\hat{\beta}(\lambda))$ for PL1/FL (left) and PL1/PL2 (right) CRF
	selection policies. In the left figure, PL1/FL, $\lambda$ represents the
	probability of including FL into the objective.  A few FL samples add
	uncertainty to the objective thus a weaker regularizer is preferable.  As
	more FL samples are incorporated, this effect diminishes but still acts to
	regularize since the full likelihood (only) best regularization is
	$\sigma^2=500$ (red triangle).  The right figure, PL1/PL2, exhibits only a
	minor change as $\lambda$ (the probability of incorporating PL2) is
	increased.  It is however, best served by a much weaker regularizer than
	PL2 alone (red triangle).
}\label{fig:chunk_crf_lamsig1}
\end{figure}

As a result, we believe that the lack of win-win for the chunking CRF follows
from two effects.  In the case of the PL1/FL policy the contour plots are
misleading since there is no single $\sigma^2$ that performs well across all
$\lambda\in [0,1]$.  For the PL1/PL2 there is simply little change in
regularization necessary across $\lambda$.

\section{Discussion}
The proposed estimator family facilitates computationally efficient estimation
in complex graphical models. In particular, different $(\beta,\lambda)$ parameterizations of the
stochastic composite likelihood enables the resolution of the complexity-accuracy
tradeoff in a domain and problem specific manner. The framework is generally
suited for Markov random fields, including conditional graphical models and is
theoretically motivated. When the model is prone to overfit, stochastically
mixing lower order components with higher order ones acts as a regularizer and
results in a win-win situation of improving test-set accuracy and reducing
computational complexity at the same time.

{
	\bibliographystyle{plain}
	\bibliography{../../common/groupPapers,../../common/externalPapers}
}

\appendix

\newtheorem{propn}{Proposition}
\newtheorem{lemn}[theorem]{Lemma}
\newtheorem{corn}{Corollary}

\newpage
\section{Proofs} \label{sec:proofs}

The proofs below generalize the classical consistency and asymptotic
efficiency of the mle \citep{Ferguson1996} and the corresponding results for $m$-estimators \citep{Vaart1998}. They follow similar lines as the proofs in \citep{Ferguson1996} and \citep{Vaart1998}, with the necessary modifications due to the stochasticity of the scl function. We assume below that $p_{\theta}(X)>0$ and that $X$ is a discrete and finite RV.

The following lemma generalizes Shannon's inequality \citep{Cover2005} for the KL divergence. We will use it to prove consistency of the SCL estimator.
\begin{lemn} \label{prop:compKL}  
Let $(A_1,B_1),\ldots,(A_k,B_k)$ be a sequence of $m$-pairs that ensures identifiability of $p_{\theta},\theta\in\Theta$ and $\alpha_1,\ldots,\alpha_k$ positive constants. Then
\begin{align}
\sum_{j=1}^k \alpha_k\, D(p_{\theta}(X_{A_j}|X_{B_j})\,||\,p_{\theta'}(X_{A_j}|X_{B_j})) \geq 0\end{align}
where equality holds iff $\theta=\theta'$.
\end{lemn}

\begin{proof}
The inequality follows from applying Jensen's inequality for each conditional KL divergence
\begin{align*}
-D(p_{\theta}(X_{A_j}|X_{B_j})\,||\,p_{\theta'}(X_{A_j}|X_{B_j}))=
\E_{p_{\theta}} \log \frac{p_{\theta'}(X_{A_j}|X_{B_j})}{p_{\theta}(X_{A_j}|X_{B_j})} &\leq \log E_{p_{\theta}}\frac{p_{\theta'}(X_{A_j}|X_{B_j})}{p_{\theta}(X_{A_j}|X_{B_j})}\\
&=\log 1 = 0.
\end{align*}
For equality to hold we need each term to be 0 which follows only if
$p_{\theta}(X_{A_j}|X_{B_j})\equiv p_{\theta'}(X_{A_j}|X_{B_j})$ for all $j$
which, assuming identifiability, holds iff $\theta=\theta'$.
\end{proof}

\begin{propn} 
Let $\Theta\subset \R^r$ be an open set, $p_{\theta}(x)>0$ and  continuous and smooth in $\theta$, and $(A_1,B_1),\ldots,(A_k,B_k)$ be a sequence of $m$-pairs for which $\{(A_j,B_j):\forall j \text{ such that } \lambda_j>0\}$ ensures identifiability. Then the sequence of SCL maximizers is strongly consistent i.e.,
\begin{align}
P\left(\lim_{n\to\infty} \hat\theta_n=\theta_0\right)=1.
\end{align}
\end{propn}

\begin{proof}
The scl function, modified slightly by a linear combination with
a term that is constant in $\theta$ is
\begin{align*}
sc\ell'(\theta) = \frac{1}{n} \sum_{i=1}^n\sum_{j=1}^k \beta_j\left(  Z_{i j}  \log p_{\theta}(X^{(i)}_{A_j}|X^{(i)}_{B_j})
- \lambda_j \log p_{\theta_0}(X^{(i)}_{A_j}|X^{(i)}_{B_j}) \right).
\end{align*}
By the strong law of large numbers, the above expression converges as $n\to\infty$ to its expectation
\[\mu(\theta)=-\sum_{j=1}^k  \beta_j \lambda_j \, D(p_{\theta_0}(X_{A_j}|X_{B_j})\,||\,p_{\theta}(X_{A_j}|X_{B_j})).\]

If we restrict ourselves to the compact set $S=\{\theta: c_1 \leq \|\theta-\theta_0\|\leq c_2\}$ then 
\begin{align}
\sup_{\theta\in S}\sup_Z \Big|\sum_{j=1}^kZ_j\beta_j \log p_{\theta}(X_{A_j}|X_{B_j})-\lambda_j\beta_j\log p_{\theta_0}(X_{A_j}|X_{B_j}) \Big|<K(x)<\infty 
\end{align}
where $K(x)$ is a function satisfying $\E K(X)<\infty$.
As a result, the conditions for the uniform strong law of large numbers \citep{Ferguson1996} hold on $S$ leading to 
\begin{align} \label{eq:ulln}
 P\left\{\lim_{n\to\infty} \sup_{\theta\in S} |scl'(\theta)-\mu(\theta)|=0\right\}=1.
\end{align}

By Proposition \ref{prop:compKL}, $\mu(\theta)$ is non-positive and is zero iff
$\theta=\theta_0$. Since the function $\mu(\theta)$ is continuous it 
attains its negative supremum on the compact $S$: $\sup_{\theta\in S}
\mu(\theta)<0$.
Combining this fact with \eqref{eq:ulln} we have that there exists $N$ such
that for all $n>N$ the scl maximizers on $S$ achieves strictly negative
values of $sc\ell'(\theta)$ with probability 1.
However, since $sc\ell'(\theta)$ can be made to achieve
values arbitrarily close to zero under $\theta=\theta_0$, we have that
$\hat\theta^{\text{msl}}_n\not\in S$ for $n>N$. Since $c_1,c_2$ were chosen
arbitrarily $\hat\theta_n^{\text{msl}}\to \theta_0$ with probability 1.
\end{proof}

\begin{propn}
Making the assumptions of Proposition \ref{prop:consistency} as well as convexity of $\Theta\subset\R^r$ we have the following convergence in distribution
\begin{align}
   \sqrt{n}(\hat\theta_n^{\text{msl}}-\theta_0) \tood N\left(0,\Upsilon \Sigma \Upsilon\right)
\end{align}
where
\begin{align}
\Upsilon^{-1}&=\sum_{j=1}^k \beta_j\lambda_j \Var_{\theta_0} (\nabla S_{\theta_0}(A_j,B_j)) \\
\Sigma&=\Var_{\theta_0}\left(\sum_{j=1}^k\beta_j\lambda_j \nabla S_{\theta_0}(A_j,B_j)\right).
\end{align}
\end{propn}

The notation $\Var_{\theta_0}(Y)$ represents the covariance matrix of the random vector $Y$ under $p_{\theta_0}$ while the notations $\toop,\tood$ in the proof below denote convergences in probability and in distribution \citep{Ferguson1996}.

\begin{proof}
By the mean value theorem and convexity of $\Theta$ there exists $\eta\in(0,1)$
for which $\theta'=\theta_0 + \eta (\hat\theta_n^{\text{msl}}-\theta_0)$ and
\[ \nabla sc\ell_n(\hat\theta_n^{\text{msl}}) = \nabla sc\ell_n(\theta_0) + \nabla^2 sc\ell_n(\theta') (\hat\theta_n^{\text{msl}}-\theta_0)\]
where $\nabla f(\theta)$ and $\nabla^2 f(\theta)$ are the $r\times 1$ gradient
vector and $r\times r$ matrix of second order derivatives of $f(\theta)$. Since
$\hat\theta_n$ maximizes the scl, $\nabla
sc\ell_n(\hat\theta_n^{\text{msl}})=0$ and
\begin{align} \label{eq:step1}
   \sqrt{n}(\hat\theta_n^{\text{msl}}-\theta_0) = -\sqrt{n} (\nabla^2 sc\ell_n(\theta'))^{-1} \nabla sc\ell_n(\theta_0).
\end{align}
By Proposition~\ref{prop:consistency} we have
$\hat\theta_n^{\text{msl}}\toop\theta_0$ which implies that
$\theta'\toop\theta_0$ as well. Furthermore, by the law of large numbers and
the fact that if $W_n\toop W$ then $g(W_n)\toop g(W)$ for continuous $g$,
\begin{align}\label{eq:res1}
   (\nabla^2 sc\ell_n(\theta'))^{-1} &\toop (\nabla^2 sc\ell_n(\theta_0))^{-1}  \\
   &\toop \left(\sum_{j=1}^k \beta_j \lambda_j  \E_{\theta_0}  \nabla^2 S_{\theta_0}(A_j,B_j) \right)^{-1} \nonumber\\
   &= -\left(\sum_{j=1}^k \beta_j\lambda_j \Var_{\theta_0}(\nabla S_{\theta_0}(A_j,B_j))\right)^{-1}.\nonumber
\end{align}
For the remaining term in \eqref{eq:step1} we have
\begin{align*}
   \sqrt{n}\, \nabla sc\ell_n(\theta_0) &= \sum_{j=1}^k \beta_j \sqrt{n}  \,\frac{1}{n}\sum_{i=1}^n  W_{ij}
\end{align*}
where the random vectors $W_{ij}=Z_{i j}\nabla\log
p_{\theta}(X^{(i)}_{A_j}|X^{(i)}_{B_j})$ have expectation 0 and variance
matrix $\Var_{\theta_0}(W_{ij})=\lambda_j \Var_{\theta_0} (\nabla
S_{\theta_0}(A_j,B_j))$. By the central limit theorem
\begin{align*}
   \sqrt{n}\,\frac{1}{n} \sum_{i=1}^n W_{ij} \tood
   N\left(0,\lambda_j \Var_{\theta_0}(\nabla S_{\theta_0}(A_j,B_j))\right).
\end{align*}
The sum $\sqrt{n}\, \nabla sc\ell_n(\theta_0)=\sum_{j=1}^k \beta_j
\sqrt{n}\,\frac{1}{n}\sum_{i=1}^n  W_{ij}$ is asymptotically Gaussian as well
with mean zero since it converges to a sum of Gaussian distributions with mean
zero.  Since in the general case the random variables $\sqrt{n}\,\frac{1}{n}
\sum_{i=1}^n W_{ij}$, $j=1,\ldots,k$ are correlated, the asymptotic variance
matrix of $\sqrt{n}\, \nabla sc\ell_n(\theta_0)$ needs to account for cross
covariance terms leading to
\begin{align}\label{eq:res2}
   \sqrt{n}\, \nabla sc\ell_n(\theta_0)  \tood
   N\left(0,\Var_{\theta_0}\left(\sum_{j=1}^k\beta_j\lambda_j \nabla S_{\theta_0} (A_j,B_j)\right)\right).
\end{align}
We finish the proof by combining \eqref{eq:step1}, \eqref{eq:res1} and
\eqref{eq:res2} using Slutsky's theorem.
\end{proof}

Recall our notation for the case that the true model $P\not\in\{p_{\theta}:\theta\in\Theta\}$.
\begin{align}
\psi_{\theta}(X,Z) &\defeq \nabla m_{\theta}(X,Z)\\
\dot\psi_{\theta}(X,Z) &\defeq \nabla^2 m_{\theta}(X,Z) \quad \text{(matrix of second order derivatives)}\\
\Psi_n(\theta) &\defeq\frac{1}{n}\sum_{i=1}^n \psi_{\theta}(X^{(i)},Z^{(i)}).
\end{align}

\begin{propn}
Assuming the conditions in Proposition~\ref{prop:consistency} as well as $\sup_{\theta:\|\theta-\theta_0\|\geq \epsilon} M(\theta)<M(\theta_0)$ for all $\epsilon>0$ we have 
$\hat\theta_n^{\text{msl}}\to\theta_0$ as $n\to\infty$ with probability 1. 
\end{propn}
\begin{proof}
We assert  
\begin{align} \label{eq:ulln2}
 P\left\{\lim_{n\to\infty} \sup_{\theta\in S} |scl'(\theta)-\mu(\theta)|=0\right\}=1.
\end{align}
on the compact set $S=\{\theta: c_1 \leq \|\theta-\theta_0\|\leq c_2\}$ as in the proof of Proposition~\ref{prop:consistency}. We proceed similarly along the lines of  Proposition~\ref{prop:consistency}, with the necessary modification due to the fact that the true model is outside the parametric family. 

Since the function $\mu(\theta)$ is continuous it 
attains its negative supremum on the compact $S$: $\sup_{\theta\in S}
\mu(\theta)<\mu(\theta_0)\geq 0$. Combining this fact with \eqref{eq:ulln2} we have that there exists $N$ such
that for all $n>N$ the scl maximizers on $S$ achieves strictly negative values of $sc\ell'(\theta)$ with probability 1.

However, since $sc\ell'(\theta)$ can be made to achieve
values arbitrarily close to $\mu(\theta_0)$ as $\hat\theta_n\to\theta_0$, we have that
$\hat\theta^{\text{msl}}_n\not\in S$ for $n>N$. Since $c_1,c_2$ were chosen
arbitrarily $\hat\theta_n^{\text{msl}}\to \theta_0$ with probability 1.
\end{proof}

\begin{propn}
Assuming the conditions of Proposition~\ref{prop:asympVar} as well as 
$\E_{P(X)}\E_{P(Z)} \|\psi_{\theta_0}(X,Z)\|^2 < \infty$, $\E_{P(X)}\E_{P(Z)} \dot\psi_{\theta_0}(X)$ exists and is non-singular, $|\ddot\Psi_{ij}| = |\partial^2\psi_{\theta}(x)/\partial\theta_i\theta_j|<g(x)$ for all $i,j$ and $\theta$ in a neighborhood of $\theta_0$ for some integrable $g$, we have 
\begin{align} 
\sqrt{n}(\hat\theta_n-\theta_0) &= -(\E_{P(X)}\E_{P(Z)} \dot\psi_{\theta_0})^{-1} \frac{1}{\sqrt{n}}\sum_{i=1}^n\psi_{\theta_0}(X^{(i)},Z^{(i)})+o_P(1)\\
&\text{or equivalently} \nonumber \\
\hat\theta_n &= \theta_0  -(\E_{P(X)}\E_{P(Z)} \dot\psi_{\theta_0})^{-1} \frac{1}{n}\sum_{i=1}^n\psi_{\theta_0}(X^{(i)},Z^{(i)}) + o_P\left(\frac{1}{\sqrt{n}}\right).
\end{align}
\end{propn}

\begin{proof}
By Taylor's theorem there exists a random vector $\tilde\theta_n$ on the line segment between $\theta_0$ and $\hat\theta_n$ for which 
\[0=\Psi_n(\hat\theta_n) = \Psi_n(\theta_0) + \dot\Psi_n(\theta_0)
(\hat\theta_n-\theta_0) + \frac{1}{2}(\hat\theta_n-\theta_0)^{\top}\ddot\Psi_n(\tilde\theta_n)(\hat\theta_n-\theta_0).
\]
which we re-arrange as 
\begin{align}\label{eq:robustThmTmp1}
\sqrt{n}\dot\Psi_n(\theta_0) (\hat\theta_n-\theta_0) +            \sqrt{n}\frac{1}{2}(\hat\theta_n-\theta_0)^{\top}\ddot\Psi_n(\tilde\theta_n)(\hat
\theta_n-\theta_0)&=-\sqrt{n}\Psi_n(\hat\theta_n) \\ 
&= -\sqrt{n}\Psi_n(\theta_0) + o_P(1)
\end{align}
where the second equality follows from the fact that $\hat\theta_n\toop \theta_0$ and continuous functions preserves converges in probability. 

Since $\dot\Psi_n(\theta_0)$ converges by the law of large numbers to $\E_{P(X)}\E_{P(Z)}\dot\psi_{\theta}(X,Z)$ and  $\ddot\Psi_n(\tilde\theta_n)$ converges to a matrix of bounded values in the neighborhood of $\theta_0$ (for large $n$), the lhs of \eqref{eq:robustThmTmp1} is  
\begin{multline}  \label{eq:robustThmTmp2}
\sqrt{n}\left(\E_{P(X)}\E_{P(Z)}\dot\psi_{\theta}(X,Z)+o_P(1)+\frac{1}{2}(\hat\theta_n-\theta_0)O_P(1)\right) (\hat\theta_n-\theta_0) \\ =\sqrt{n}(\E_{P(X)}\E_{P(Z)}\dot\psi_{\theta}(X,Z)+o_P(1))(\hat\theta_n-\theta_0)
\end{multline}
since $\hat\theta_n-\theta_0=o_P(1)$ and $o_P(1)O_p(1)=o_P(1)$ (the notation $O_P(1)$ denotes stochastically bounded and it applies to $\ddot \Psi_n(\tilde\theta_n)$ as described above). Putting it together we have
\[\sqrt{n}(\E_{P(X)}\E_{P(Z)}\dot\psi_{\theta}(X,Z)+o_P(1))(\hat\theta_n-\theta_0)=-\sqrt{n} \Psi_n(\theta_0) + o_P(1).\]
Since the matrix $\E_{P(X)}\E_{P(Z)}\dot\psi_{\theta}(X,Z)+o_P(1)$ converges to a non-singular matrix, multiplying the equation above by its inverse finishes the proof. 
\end{proof}

\begin{corn} Assuming the conditions specified in Proposition~\ref{prop:robustVar1} we have
\begin{align}
\sqrt{n}(\hat\theta_n-\theta_0) &\tood N(0,(\E_{P(X)}\E_{P(Z)} \dot\psi_{\theta_0})^{-1} (\E_{P(X)}\E_{P(Z)} \psi_{\theta_0}\psi_{\theta_0}^{\top})(\E_{P(X)}\E_{P(Z)} \dot\psi_{\theta_0})^{-1}).
\end{align}
\end{corn}
\begin{proof}  
Equation~\eqref{eq:normality} follows from \eqref{eq:asymp} by noticing that due to the central limit theorem $\Psi_n(\theta_0)$ (as it is an average of $n$ iid RVs with expectation 0) 
\[ \sqrt{n} \cdot \frac{1}{n} \sum_{i=1}^n \psi_{\theta_0}(X^{(i)},Z^{(i)}) \tood N(0,\E_{P(X)}\E_{P(Z)} \psi_{\theta_0}\psi_{\theta_0}^{\top}).\]
Substituting this in the right hand side of \eqref{eq:asymp} and accounting for the modified variance due to the matrix inverse results in \eqref{eq:normality}.
\end{proof}

\end{document}